\documentclass[a4paper, 12pt,twoside]{article}

\usepackage[english]{babel}
\usepackage[T1]{fontenc}
\usepackage[utf8]{inputenc}

\usepackage{changepage}

\usepackage{amsmath}
\usepackage{amssymb}
\usepackage{amsthm}
\usepackage{nccmath}

\usepackage{graphicx}
\usepackage{float}

\usepackage{enumerate}
\usepackage{enumitem}

\usepackage{url}

\usepackage{subfig}

\usepackage{enumitem}

\usepackage[margin=1.12in]{geometry}

\usepackage{here}

\usepackage{lmodern}

\usepackage{fancyvrb}

\usepackage{imakeidx}
\usepackage[index]{kantlipsum}
\makeindex[options=-s mystyle.ist, columns=2, intoc, title=Index]
\indexsetup{firstpagestyle=fancy}

\usepackage{xcolor}
\usepackage[plainpages=false]{hyperref}

\usepackage{longtable}
\usepackage{multirow}

\usepackage{booktabs}
\usepackage{tabularx}

\usepackage{ytableau}

\usepackage{cancel}

\usepackage{appendix}

\usepackage[ruled]{algorithm2e}

\renewcommand{\thefigure}{\thesubsection.\arabic{equation}}

\usepackage{chngcntr}
\counterwithin{equation}{section}
\numberwithin{equation}{subsection}

\theoremstyle{definition}
\newtheorem{defi}{Definition}[subsection] 
\newtheorem{remark}[defi]{Remark}
\newtheorem{example}[defi]{Example}
\newtheorem{proposition}[defi]{Proposition}
\newtheorem*{proposition*}{Proposition}
\newtheorem{theorem}[defi]{Theorem}
\newtheorem{corollary}[defi]{Corollary}
\newtheorem{claim}[defi]{Claim}
\newtheorem{lemma}[defi]{Lemma}

\usepackage{fancyhdr}
\hypersetup{%
  colorlinks=true,	% hyperlinks no border
  linkcolor={black}, % Hyperref links black
  citecolor={black}, % magenta also good
  urlcolor={black}
}

\DeclareMathOperator*{\argmax}{arg\,max}

\DeclareMathOperator*{\esssup}{ess\,sup}
\begin{document}

%-------------------------------------------------------------------------------------------
% Cover sheet
%-------------------------------------------------------------------------------------------

% Renumbering of the cover page for error-free use of the hyperref package
\pagenumbering{Alph}

% Put title, author and date on the cover page
\title{Master Thesis \\[3pt] 
\begin{Large} % AG Kunoth \\[3pt]
Department of Mathematics and Computer Science \\[3pt]
University of Cologne\\[1.6cm]
\centerline{\includegraphics[width=0.45\textwidth]{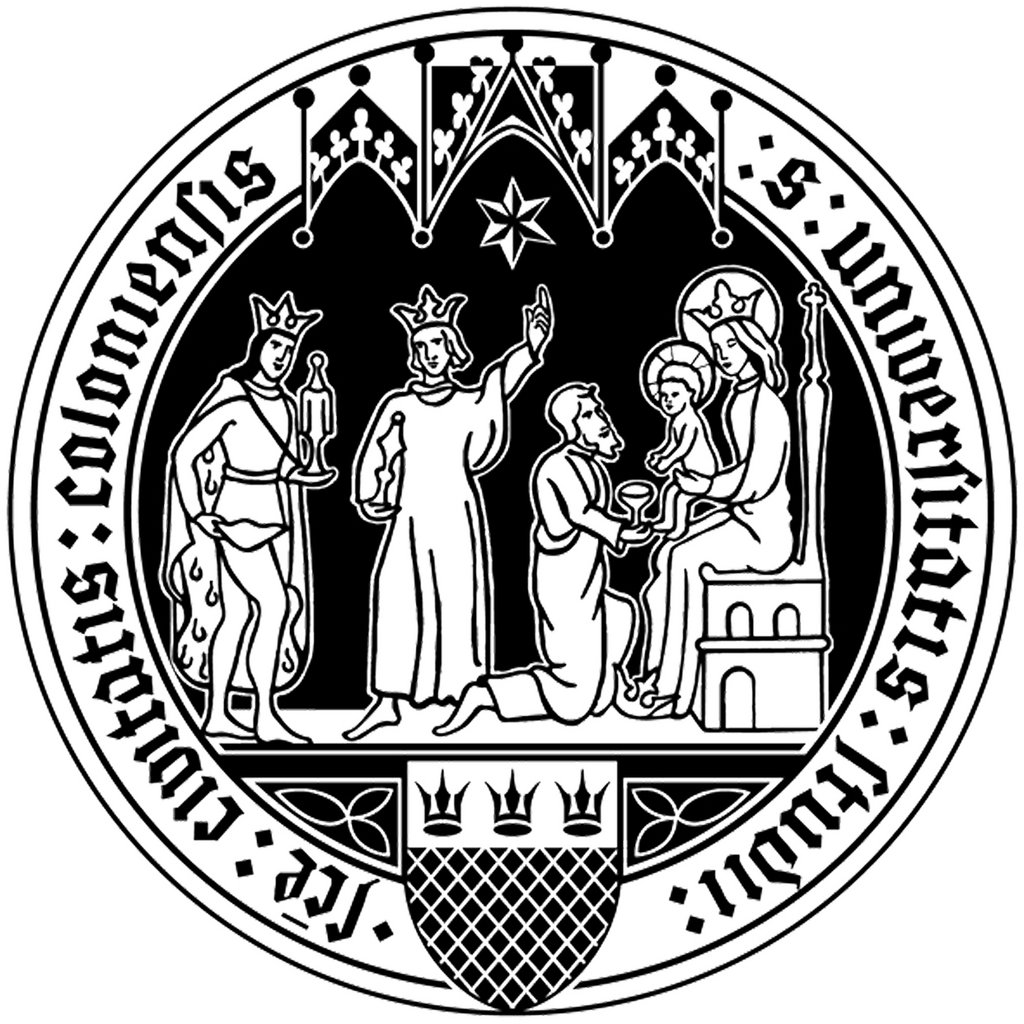}} % Uni-Logo wird geladen
\end{Large}
$ \ $ \\
\textbf{Constructing Gaussian Processes via Samplets\\[2cm]}
\begin{large}
Marcel Neugebauer\\[3pt]
November 8, 2024\\[2cm]

\begin{minipage}[c]{0.4\textwidth}
\begin{center}
	Advisors:\\
	Prof. Dr. Angela Kunoth\\
	Prof. Dr. Michael Multerer\\
\end{center}
\end{minipage}
\date{}
\centering
\end{large}}

% cover page generation
\maketitle

% Suppress page numbering
\thispagestyle{empty}

%% This code should be commented out to insert a blank page
%\newpage
%\thispagestyle{empty}
%\hspace{1cm}

%-------------------------------------------------------------------------------------------
% Acknowledgement
%-------------------------------------------------------------------------------------------

% Changed page numbering to Arabic numerals
\pagenumbering{arabic}

% Remove section/subsection numbers from figure labels 
\renewcommand{\thefigure}{\arabic{figure}}

% !TeX spellcheck = en_GB
\section*{Acknowledgements}
I would like to express my gratitude to Prof. Dr. Angela Kunoth for appointing me as a research associate in her group and for her valuable support throughout this thesis. I also wish to thank Prof. Dr. Michael Multerer for his continuous support and for always being available to assist whenever I had questions.\\

Finally, I would like to extend my heartfelt thanks to my family and friends for their constant encouragement and support throughout my studies.

%-------------------------------------------------------------------------------------------
% Table of contents
%-------------------------------------------------------------------------------------------

% Creation of the table of contents on a new page
\newpage
\tableofcontents

%% This code should be commented out to insert a blank page
%\newpage
%\thispagestyle{empty}
%\hspace{1cm}

%-------------------------------------------------------------------------------------------
% Main part
%-------------------------------------------------------------------------------------------
\newpage
\hypersetup{linkcolor=black}
\fancyhead[RO]{\small{\nouppercase{\leftmark}}}
\fancyhead[LE]{\small{\nouppercase{\rightmark}}}
\fancyhead[LO, RE]{\small{Marcel Neugebauer}}

% !TeX spellcheck = en_GB
\section{Introduction} \label{Chap. 1}
In today's world, there are two common ways to first encounter Gaussian Processes. The first is through the study of stochastic Processes, such as the Wiener Process. In this context, one works with normally distributed function values, which naturally give rise to the concept of a Gaussian Process. The second involves inferring an unknown real-valued function from data using a probabilistic model. Assuming the function values follow a multivariate normal distribution leads directly to a Gaussian Process. The first approach is more traditional, while the latter is more modern and relies on advanced algorithms and computational power. The growing use of Gaussian Processes has driven extensive research, resulting in a robust theoretical framework that spans multiple areas of mathematics.\\

Norbert Wiener worked extensively in the 1920s on the Wiener Process, establishing foundational results such as proving its existence. Since the Wiener Process is a specific type of Gaussian Process, these are considered some of the earliest significant results in Gaussian Process theory. In the 1930s, Kolmogorow used measure theory to develop a rigorous framework for probability theory, laying the groundwork for stochastic Processes. Throughout the 1940s and 1950s, mathematicians like Doob, Itô , Kolmogorow, Levy and Wiener expanded this foundation, producing many key results on Gaussian Processes. 

In the 1950s, Danie Krige applied linear regression to tackle interpolation problems in mining engineering. In the 1960s, Georges Matheron generalized Krige’s work by incorporating the correlation structure of observed data and polynomial models to create best unbiased predictors. He named this method "Kriging" in honor of Krige and developed a framework that leveraged the first two moments of a Gaussian Process. From there, it was a natural step to further extend Kriging by employing the entire probability distribution rather than relying solely on the first two moments. So Gaussian Processes can be viewed as the natural extension of Kriging.

For many years, Gaussian Processes were primarily used in geostatistics to address low-dimensional regression problems. Their broader application began in the 1990s, fueled by advancements in machine learning. Practitioners could rely on the substantial mathematical foundation in stochastic Processes and a non-parametric model that provides uncertainty estimation, interpretability and adaptability. Gaussian Processes quickly became integral to Bayesian Optimization, which aims to optimize black-box functions. In 2006, Rasmussen and Williams published an influential textbook that synthesized the theoretical developments and applied Gaussian Processes to a wide range of regression problems. However, constructing Gaussian Processes for large datasets and selecting the optimal model remain challenging tasks to this day.\\
 
The goal of this master’s thesis is to tackle the challenges of model selection and construction within Gaussian Process frameworks, focusing on applications in up to three dimensions. We examine recent convergence results to identify models with optimal convergence rates and pinpoint essential parameters. Utilizing this model, we propose a Samplet-based approach to efficiently construct and train the Gaussian Processes, reducing the cubic computational complexity to a log-linear scale. This method facilitates optimal regression while maintaining efficient performance.

%------------------------------------------------------------------------------------------------

The primary references used are \cite{RW}, \cite{BGW}, and \cite{HM1}. The first is the influential textbook by Rasmussen and Williams, which provides a comprehensive introduction to Gaussian Processes in machine learning. It begins with foundational concepts and progresses to advanced modeling, approximation techniques, and comparisons with other models. While this textbook is an excellent starting point, it does not delve into convergence results or modern approximation techniques. Reference \cite{BGW} compiles and extends convergence results for Gaussian Process means, offering valuable insights for model selection and establishing reliable convergence rates. Specifically, it facilitates model selection by guiding experimental design, identifying key hyperparameters, recognizing poor model choices, and estimating expected convergence rates. Lastly, the paper \cite{HM1} by Harbrecht and Multerer introduces Samplets. This is a concept originally developed for data compression that also enables kernel matrix compression. Handling Gaussian Processes on large datasets is computationally intensive, as it requires inverting a kernel matrix or solving associated linear systems. To address this, we apply Samplets to compress the kernel matrix into a sparse structure, allowing us to work efficiently with this reduced pattern. This approach brings the cubic computational cost of Gaussian Process construction down to log-linear complexity. However, since the scaling constants increase significantly with dimension, this approach is best suited for low-dimensional approximations.\\

This thesis is organized as follows: Chapter \ref{Chap. 2} covers the foundational concepts of Gaussian Processes that are essential for model construction and understanding their interconnections. Chapter \ref{Chap. 3} examines the convergence properties of Gaussian Processes, assessing their reliability and applications in model selection. Chapter \ref{Chap. 4} introduces the concept of Samplets and demonstrates their use in kernel matrix compression. In Chapter \ref{Chap. 5}, we present algorithms for constructing Gaussian Processes with Samplets and conducting Bayesian Optimization. Chapter \ref{Chap. 6} contains numerical experiments that evaluate the convergence rates discussed in Chapter \ref{Chap. 3} and compares the proposed algorithms to current state-of-the-art methods. All graphics and proofs throughout this thesis have been performed by the author, unless stated otherwise.

\newpage
% !TeX spellcheck = en_GB
\section{Preliminaries} \label{Chap. 2}
In this chapter, Gaussian Processes are motivated and defined. We look at which properties they have, how they are configured and their relation to reproducing kernel hilbert spaces. We then examine a current application of Gaussian Processes.

\subsection{Gaussian Processes} \label{Chap 2.1}
Let $\Omega \subseteq \mathbb{R}^d$ be a non-empty set. Consider an unknown function $f:\Omega \rightarrow \mathbb{R}$ and a problem
$$y = f(\boldsymbol{x}) + \varepsilon,$$
where $\varepsilon \sim \mathcal{N}(0, \sigma^2)$. The aim is to estimate $f$ using a set of observations $\mathcal{D}_N = \{ (\boldsymbol{x}_1, y_1), ..., (\boldsymbol{x}_N, y_N) \}$\index{\textit{$\mathcal{D}_N$,}}. There are two classical approaches. The first assumes $f$ is parametric, meaning it has a specific form, and fits its parameters by minimizing a loss function. The second one is to use a non-parametric method. As example for the latter you can a priori assume that the function values follow a distribution. The most typical used, inter alia due to the multidimensional central limit theorem, is the multivariate normal distribution. That is
$$ \begin{bmatrix}
	f(\boldsymbol{z}_1) \\ \vdots \\ f(\boldsymbol{z}_n) 
\end{bmatrix}
\sim \mathcal{N} \left( 
\begin{bmatrix}
	m(\boldsymbol{z}_1) \\ \vdots \\ m(\boldsymbol{z}_n) 
\end{bmatrix},
\begin{bmatrix}
	k(\boldsymbol{z}_1,\boldsymbol{z}_1) & \cdots & k(\boldsymbol{z}_1,\boldsymbol{z}_n) \\ \vdots 
	& \ddots & \vdots \\ k(\boldsymbol{z}_n,\boldsymbol{z}_1) & \cdots 
	& k(\boldsymbol{z}_n,\boldsymbol{z}_n) 
\end{bmatrix} \right)$$
for every finite amount of points $\boldsymbol{z}_1, ..., \boldsymbol{z}_n \in \Omega$. Here $m:\Omega \rightarrow \mathbb{R}$ is called mean function\index{Mean function,} and the covariance function\index{Covariance function,} $k:\Omega \times \Omega \rightarrow \mathbb{R}$ has to be a symmetric and positive semi-definite kernel to let every multivariate normal distribution be well defined.

\begin{defi} \label{2.1.1}
A function $k:\Omega \times \Omega \rightarrow \mathbb{R}$ is called symmetric and positive semi-definite kernel\index{Kernel,} (in the following simply kernel), if
\begin{enumerate}[label=(\roman*),topsep=5pt]
	\setlength\itemsep{0.1mm}
\item $k(\boldsymbol{z}, \boldsymbol{z}') = k(\boldsymbol{z}', \boldsymbol{z})$ for every $\boldsymbol{z}, \boldsymbol{z'} \in \Omega$,
\item for every $n\in \mathbb{N}, c_1,...,c_n \in \mathbb{R}$ and $\boldsymbol{z}_1, ...,\boldsymbol{z}_n \in \Omega$ is
$$\sum_{i=1}^n \sum_{j=1}^n c_i c_j k(\boldsymbol{z}_i, \boldsymbol{z}_j) \geq 0.$$
\end{enumerate}
\end{defi}

In this scenario, we contemplate an infinite set of distributions, prompting the question of whether such a collection possessing these characteristics exists. The affirmative answer to this query is provided by Theorem 12.1.3 in \cite{D}, where this entity is commonly recognized in literature as a Gaussian Process. It can be seen as infinite-dimensional generalization of the multivariate normal distribution.

\begin{defi} \label{2.1.2}
A random function $f:\Omega \rightarrow \mathbb{R}$ is a Gaussian Process\index{Gaussian Process,} (GP), if any finite set of function values has a multivariate normal distribution. We write $f \sim \mathcal{GP} (m,k)$\index{\textit{$\mathcal{GP} (m,k)$,}} for a Gaussian Process with mean function $m$ and kernel $k$.
\end{defi}

%------------------------------------------------------------------------------------------------

\begin{remark} \label{2.1.3}
As using a different mean function or kernel gives rise to a new Gaussian Process, there is a one-to-one correspondence between Gaussian Processes and pairs $(m,k)$ of mean function $m$ and kernel $k$.
\end{remark}

The probabilistic model we will use further on is $f \sim \mathcal{G}\mathcal{P} (m,k)$. Technically, we assume $f$ is a sample of a Gaussian Process, meaning it is a realization drawn from the distribution defined by a GP. Examples are provided in the following figures.

\begin{example} \label{2.1.4}
Examine the Gaussian Process $\mathcal{G}\mathcal{P} (0,k)$ with
\begin{equation}
k(\boldsymbol{x},\boldsymbol{x}') = \text{exp} \, \left( - \dfrac{1}{2} \| \boldsymbol{x} - \boldsymbol{x}' \|_2^2 \right). \label{eq. 2.1.1}
\end{equation}
The one-dimensional case is visualized in the top left panel of Figure \ref{fig. 1}, where the samples of this GP appear unremarkable. Since this GP does not consider any observations, this characteristic remains unchanged for most a priori Gaussian Processes contemplated in practice.
\end{example}

\begin{figure}[h]
	\centering
	\includegraphics[width=0.85\textwidth]{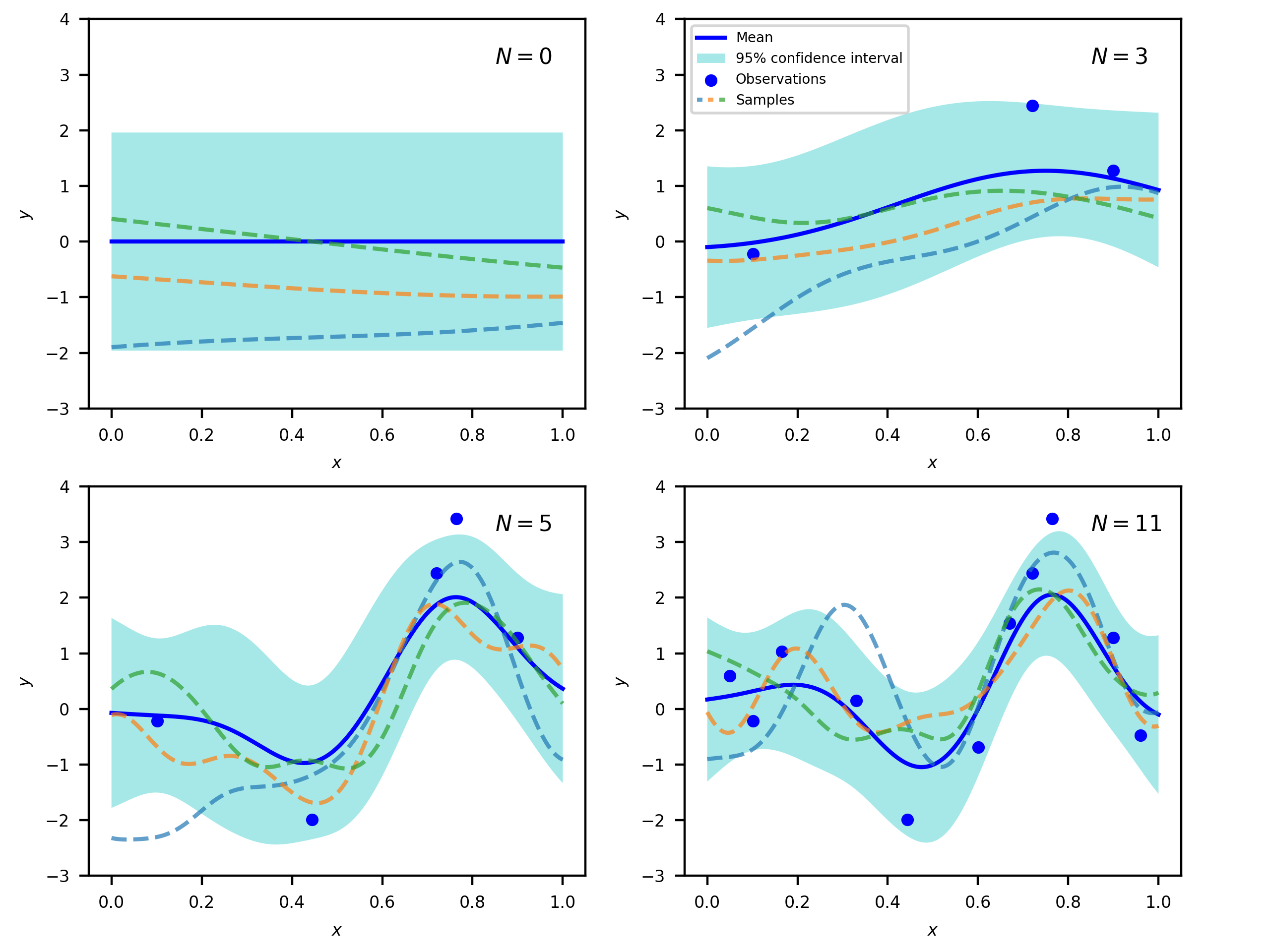}
	\caption{Visualization of different Gaussian Processes by plotting mean function, confidence intervals and samples. The top left panel shows $\mathcal{G}\mathcal{P} (0,k)$ with kernel given by equation (\ref{eq. 2.1.1}). The three other panels consider in addition noisy observations and visualize the associated posterior process $\mathcal{G}\mathcal{P} (m',k')$ for $\sigma^2 = 1$.}\label{fig. 1}
\end{figure}

The following proposition asserts that modeling $f \sim \mathcal{G}\mathcal{P} (m,k)$ enables to access the posterior distribution of the function values given observations. This approach allows for inference on the unknown function values.

\begin{proposition} \label{2.1.5}
Given $f \sim \mathcal{G}\mathcal{P} (m,k)$ and observations $\mathcal{D}_N = \{ (\boldsymbol{x}_1, y_1), ..., (\boldsymbol{x}_N, y_N) \}$. Then for the posterior function holds $f \, | \, \mathcal{D}_N \sim  \mathcal{G}\mathcal{P} (m',k')$\index{\textit{$f \, | \, \mathcal{D}_N$,}}, where \index{\textit{$m'$,}} \index{\textit{$k'$,}}
\begin{equation}
m'(\boldsymbol{x}) = m(\boldsymbol{x}) + \boldsymbol{k}(\boldsymbol{x})
(\boldsymbol{K}+ \sigma^2 \boldsymbol{I}_N )^{-1} (\boldsymbol{y}-\boldsymbol{m}), \label{eq. 2.1.2}
\end{equation}
\begin{equation}
k'(\boldsymbol{x}, \boldsymbol{x}') = k(\boldsymbol{x}, \boldsymbol{x}')- \boldsymbol{k}
(\boldsymbol{x})  (\boldsymbol{K}+ \sigma^2 \boldsymbol{I}_N )^{-1}  \boldsymbol{k}(\boldsymbol{x}') ^T \label{eq. 2.1.3}
\end{equation}
for $\boldsymbol{k}(\boldsymbol{x}) = [k(\boldsymbol{x},\boldsymbol{x}_1),...,k(\boldsymbol{x},\boldsymbol{x}_N)]$\index{\textit{$\boldsymbol{k}(\boldsymbol{x})$,}}, $\boldsymbol{m} = [m(\boldsymbol{x}_1), ..., m(\boldsymbol{x}_N)]^T$\index{\textit{$\boldsymbol{m}$,}}, $\boldsymbol{y} = [y_1, ..., y_N]^T$\index{\textit{$\boldsymbol{y}$,}} and 
$\boldsymbol{K} = [k(\boldsymbol{x}_i, \boldsymbol{x}_j)]_{1 \leq i,j \leq N}$\index{\textit{$\boldsymbol{K}$,}}. 
\end{proposition}

%------------------------------------------------------------------------------------------------

\begin{proof} Evidence is located split up in the Appendices of \cite{G}. The proof relies on properties of the (multivariate) normal distribution and the following assumptions:
\begin{itemize}
\item $\boldsymbol{K} + \sigma^2 \boldsymbol{I}_N$ is symmetric and positive definite (hence invertible),
\item $[\varepsilon_1,...,\varepsilon_N] \sim \mathcal{N} (0, \sigma^2 \boldsymbol{I}_N)$,
\item $[f(\boldsymbol{x}_1),...,f(\boldsymbol{x}_N)]$ and $[\varepsilon_1,...,\varepsilon_N]$ are independent. \qedhere
\end{itemize}
\end{proof}

Most applications assume a constant or even zero mean function in the model $f \sim \mathcal{G}\mathcal{P} (m,k)$. This assumption is common, because the mean function has relatively low impact on the Gaussian Process (see section \ref{Chap. 2.3.1}), a zero mean can be offset by adding an extra term to the kernel, it enhances model interpretability, and it simplifies calculations.

\begin{example} \label{2.1.6}
Consider again $\mathcal{G}\mathcal{P} (0,k)$ where $k$ is given by equation (\ref{eq. 2.1.1}). The other three panels in Figure \ref{fig. 1} show how the initial GP changes as we contemplate observations. The samples seem to be smooth and capture the trend the observations reflect. Notice, that we assume an initial zero mean function, but the mean functions of the posterior Gaussian Processes are not zero. A comparison between the last Gaussian Process and the real function, is pictured in Figure \ref{fig. 2}. The mean function is able to copy the fluctuation and the confidence intervals seem reliable. 
\end{example}

\begin{figure}[h]
	\centering
	\includegraphics[width=0.5\textwidth]{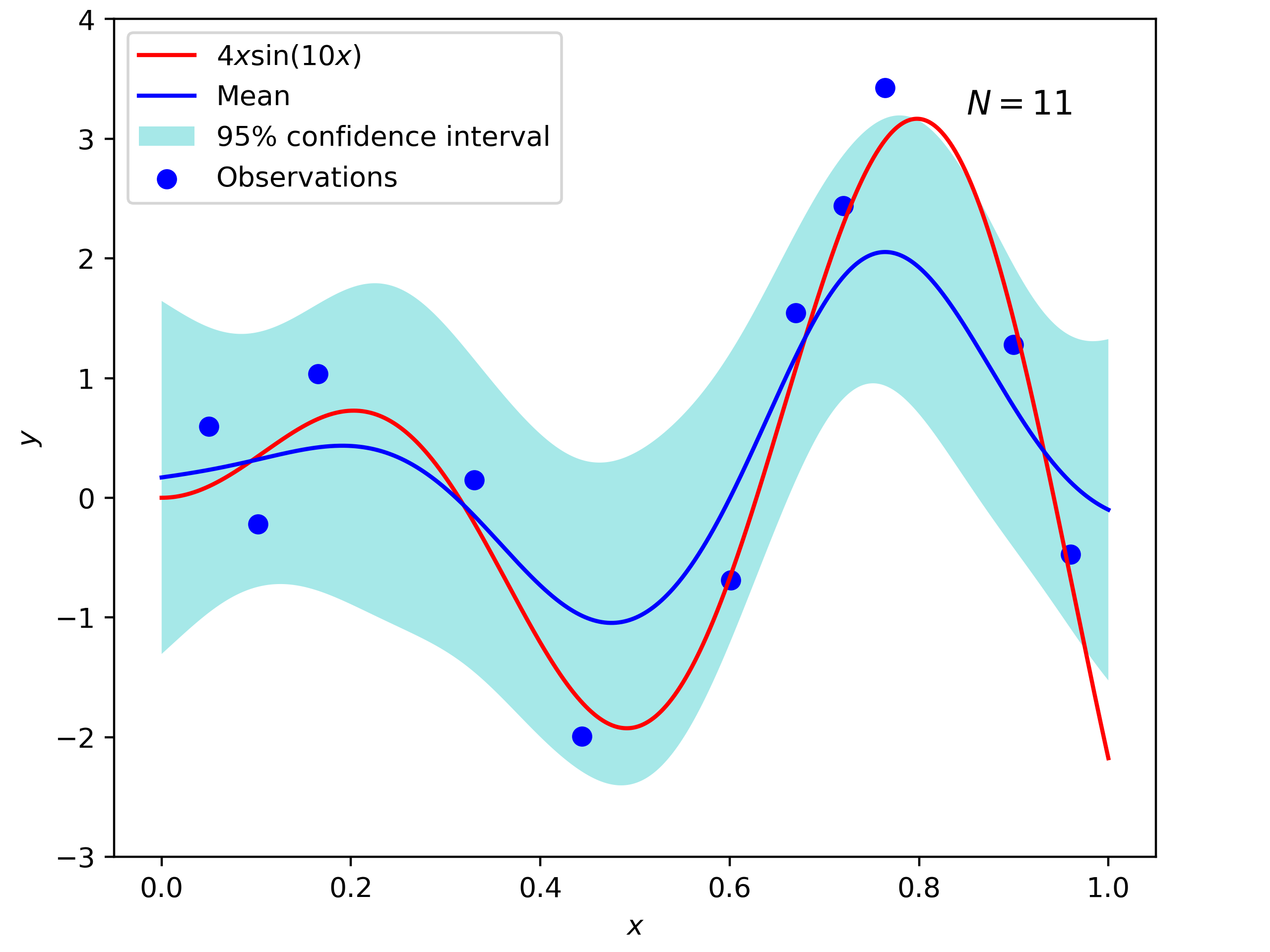}
	\caption{Comparison of the GP in the bottom right panel of Figure \ref{fig. 1} with the function $4x\sin(10x)$, from which the noisy observations were sampled.}
	\label{fig. 2}
\end{figure}

Accordingly, Gaussian Processes can be suitable for carrying out regression. As the posterior Gaussian Process yields
$$
f(\boldsymbol{x}) \, | \, \mathcal{D}_N \sim \mathcal{N} (m'(\boldsymbol{x}),k'(\boldsymbol{x},\boldsymbol{x}))
$$
for every $\boldsymbol{x} \in \Omega$, the prediction of $f(\boldsymbol{x})$ given data $\mathcal{D}_N$ will be the expectation $m'(\boldsymbol{x})$. This implies that our estimate of the unknown function $f$ is the posterior mean function $m'$. Table \ref{tab. 1} lists the pros and cons of this method.\\

%------------------------------------------------------------------------------------------------

\begin{table}
\centering
\begin{tabularx}{0.9\linewidth}{>{\parskip1ex}X@{\kern4\tabcolsep}>{\parskip1ex}X}
\toprule
\hfil\bfseries Pros
&
\hfil\bfseries Cons
\\\cmidrule(r{3\tabcolsep}){1-1}\cmidrule(l{-\tabcolsep}){2-2}

%% PROS, seperated by empty line or \par
Access to quantiles, confidence intervals, standard deviation and correlation coefficients.\vspace{1ex} \par
Prediction is given in closed form.\vspace{1ex} \par
Interpretability of the prediction.\vspace{1ex} \par
Tiny chance of overfitting. Gaussian Processes have almost always a small amount of parameters.\vspace{1ex} \par
Model selection.
&
%% CONS, seperated by empty line or \par
Inadequate model selection can lead to inferior results.\vspace{1ex} \par
Slow computation for large datasets. Computing the matrix inversion in equations (\ref{eq. 2.1.2}) and (\ref{eq. 2.1.3}) takes typically $\mathcal{O}(N^3)$ time.
\\\addlinespace[0.15cm] % Adjust the vertical space here
\bottomrule
\end{tabularx}
\caption{Pros and cons of using Gaussian Processes.}
\label{tab. 1}
\end{table}

The construction of Gaussian Processes inherits both theoretical and numerical challenges. Model selection constitutes a theoretical aspect, and to address this, we examine the convergence rate of Gaussian Process means in Chapter \ref{Chap. 3}. On the other hand, the computation of the matrix inversion presents a numerical challenge. We resolve this by utilizing samplets in Chapter \ref{Chap. 4}, which provide sparse matrices of $\mathcal{O}(N \log N)$ entries. In this context notice, that these two issues usually remain independent of each other. First you select the model, then you start the computation.

%-------------------------------------------------------------------------------------------
% Reproducing Kernel Hilbert Spaces
%-------------------------------------------------------------------------------------------

\subsection{Reproducing Kernel Hilbert Spaces} \label{Chap. 2.2}
The connection between Gaussian Processes and reproducing kernel hilbert spaces plays an important role for establishing a convergence result. In addition, we will gain an understanding of the smoothness property for specific Gaussian Processes, allowing for better interpretation of the posterior mean function. In 1950, Aronszajn laid the groundwork for reproducing kernel Hilbert spaces, as expounded in \cite{A}.

\begin{defi} \label{2.2.1}
Let $\Omega$ be an arbitrary non-empty set and $k:\Omega \times \Omega \rightarrow \mathbb{R}$ a function. A hilbert space $\mathcal{H}$\index{\textit{$\mathcal{H}$,}} of functions $g:\Omega \rightarrow \mathbb{R}$ with inner-product $\langle \cdot, \cdot \rangle_\mathcal{H}$ is called reproducing kernel hilbert space\index{RKHS,} (RKHS), if
 \begin{enumerate}[label=(\roman*),topsep=5pt]
	\setlength\itemsep{0.1mm}
\item for every $\boldsymbol{x} \in \Omega$ is $k(\cdot, \boldsymbol{x}) \in \mathcal{H}$,
\item the function has the reproducing property $\langle g, k(\cdot, \boldsymbol{x}) \rangle_\mathcal{H} = g(\boldsymbol{x})$ for every $g\in \mathcal{H}$ and $\boldsymbol{x} \in \Omega$.
\end{enumerate}
\end{defi}

The following proposition characterizes RKHS. Since it is relatively easy to construct kernels, it offers numerous indirect illustrations of reproducing kernel hilbert spaces.

\begin{proposition} \label{2.2.2}
\textbf{(Moore-Aronszajn Theorem)} For every kernel $k$ exists a unique RKHS and vice versa.
\end{proposition}

\begin{proof}
See \ref{Chap. 8.1} in the Appendix.
\end{proof}

Given a kernel, then the proof of Proposition \ref{2.2.2} shows how to construct the corresponding RKHS, implying the subsequent corollary.

%-------------------------------------------------------------------------------------------

\begin{corollary} \label{2.2.3}
The corresponding RKHS of the kernel $k$ is the completion of
$$
\left\{ g(\cdot) = \sum_{i=1}^{\infty} c_i k(\cdot,\boldsymbol{x}_i) \, : \, \boldsymbol{x}_i \in \Omega, c_i \in \mathbb{R}, \
\| g \|_\mathcal{H}^2 = \sum_{i,j=1}^{\infty} c_i c_j k(\boldsymbol{x}_i, \boldsymbol{x}_j) < \infty \right\}. 
$$
\end{corollary} 

This Corollary suggests that the functions in the RKHS inherit the properties of the kernel. If the kernel is $\nu$ times differentiable, then most functions in the corresponding RKHS are also $\nu$ times differentiable.\\

Consider the same setting as in section \ref{Chap 2.1}, where the goal is to estimate a function $f:\Omega \rightarrow \mathbb{R}$ given observations $\mathcal{D}_N = \{ (\boldsymbol{x}_1, y_1), ..., (\boldsymbol{x}_N, y_N) \}$ by using Gaussian Processes. Proposition \ref{2.1.5} and the assumption $f \sim \mathcal{G}\mathcal{P} (0,k)$ imply that the prediction is given by
\begin{equation}
m'(\boldsymbol{x}) = \boldsymbol{k}(\boldsymbol{x})
(\boldsymbol{K}+ \sigma^2 \boldsymbol{I}_N )^{-1} \boldsymbol{y} = \boldsymbol{k}(\boldsymbol{x}) \boldsymbol{c} = \sum_{i=1}^N c_i k(\boldsymbol{x}, \boldsymbol{x}_i), \label{eq. 2.2.1} 
\end{equation}
where $\boldsymbol{c}=[c_1,...,c_N]^T \in \mathbb{R}^{N \times 1}$ satisfies $(\boldsymbol{K}+ \sigma^2 \boldsymbol{I}_N ) \boldsymbol{c} = \boldsymbol{y}$. Thus, while most samples of $\mathcal{G}\mathcal{P} (0,k)$ do not lie in the corresponding RKHS $\mathcal{H}$, the prediction does. The idea is to use a class of kernels that justify the assumption $f \in \mathcal{H}$ and then to investigate the distance $\| f-m' \|_\mathcal{H}$. To cover as many functions as possible, the class should consider different properties of $f$, such as smoothness, fluctuation and local behaviour.

%-------------------------------------------------------------------------------------------
% Model Selection
%-------------------------------------------------------------------------------------------

\subsection{Model Selection} \label{Chap. 2.3}
Let $\Omega \subseteq \mathbb{R}^d$ be a non-empty set. When estimating $f:\Omega \rightarrow \mathbb{R}$ using observations and Gaussian Processes, the challenge arises on choosing the pair of functions $(m,k)$. This task is challenging, because it requires understanding the impact
%$$m'(\boldsymbol{x}) = m(\boldsymbol{x}) + \boldsymbol{k}(\boldsymbol{x})
%(\boldsymbol{K}+ \sigma^2 I_N )^{-1} (\boldsymbol{y}-\boldsymbol{m}),$$
%$$k'(\boldsymbol{x}, \boldsymbol{x}') = k(\boldsymbol{x}, \boldsymbol{x}')- \boldsymbol{k}
%(\boldsymbol{x})  (\boldsymbol{K}+ \sigma^2 I_N )^{-1}  \boldsymbol{k}(\boldsymbol{x}')$$
of $(m,k)$ on the Gaussian Processes, and there is a vast array of possible mean and kernel functions to consider. This section explores various mean and kernel functions to establish a repertoire of sensible choices.

%-------------------------------------------------------------------------------------------
% Mean function
%-------------------------------------------------------------------------------------------

\subsubsection{Mean function} \label{Chap. 2.3.1}
The mean function $m:\Omega \rightarrow \mathbb{R}$ determines the expected value
$$m(\boldsymbol{x}) = \mathbb{E}[f(\boldsymbol{x})]$$
at any $\boldsymbol{x} \in \Omega$. The expected value is a crucial size, but in the literature numerous authors opt to keep the mean function constant. Initially, this might appear counterintuitive. However, the rationale behind this practice lies in the fact that the kernel predominantly shapes the Gaussian Process. Consequently, many methodologies prioritize the use of a simple mean function to balance against a more intricate kernel. It is also common to normalize the vector $\boldsymbol{y}$ before computing the Gaussian Process, which reinforces this assumption. Setting the prior mean function to zero improves interpretability and allows for an association of Gaussian Processes with the corresponding RKHS. Thus, we can leverage RKHS theory.

%-------------------------------------------------------------------------------------------

As we assume that $f$ is a sample of a Gaussian Process, its properties should match with the properties of the samples. Figure \ref{fig. 3} visualizes samples of different Gaussian Processes. The second plot changes only the mean function compared to the first plot, while the third plot changes only the kernel. It is evident that the change in kernel has a greater impact than the change in mean function. The change in mean merely results in a translation of the space.

\begin{figure}[h]
	\centering
	\includegraphics[width=\textwidth]{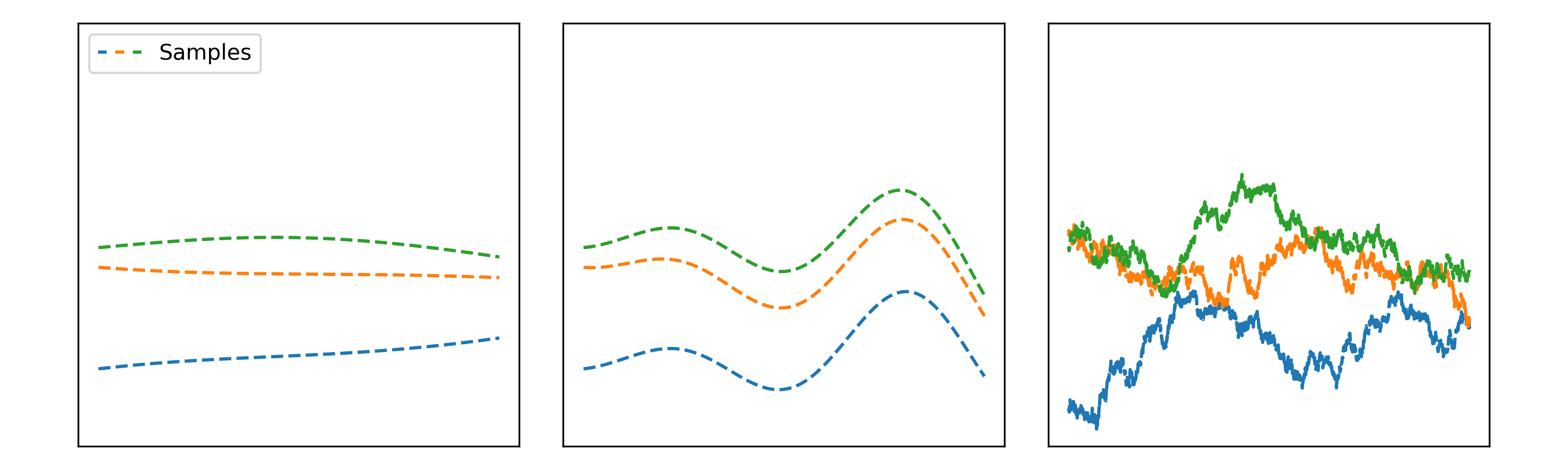}
	\caption{Samples of three Gaussian Processes. The middle plot alters the mean function, while the right plot modifies the kernel, in comparison to the left plot.}\label{fig. 3}
\end{figure}

You can also explore the impact of modifying the prior mean on the posterior Gaussian Process. However, since the posterior kernel function
$$k'(\boldsymbol{x}, \boldsymbol{x}') = k(\boldsymbol{x}, \boldsymbol{x}')- \boldsymbol{k}
(\boldsymbol{x})  (\boldsymbol{K}+ \sigma^2 \boldsymbol{I}_N )^{-1}  \boldsymbol{k}(\boldsymbol{x}')$$
is independent of the prior mean, similar behavior is expected. On the other hand, this function, along with the posterior mean
$$m'(\boldsymbol{x}) = m(\boldsymbol{x}) + \boldsymbol{k}(\boldsymbol{x})
(\boldsymbol{K}+ \sigma^2 \boldsymbol{I}_N )^{-1} (\boldsymbol{y}-\boldsymbol{m}).$$
depend on the prior kernel. This reinforces the idea that the prior kernel fundamentally shapes the Gaussian Process.

\begin{example} \label{2.3.1}
Examine the one-dimensional case of the Gaussian Processes $\mathcal{G}\mathcal{P} (0,k)$ and $\mathcal{G}\mathcal{P} (m,k)$ with $m(\boldsymbol{x}) = (\boldsymbol{x}-5)^2$ and
$$
k(\boldsymbol{x},\boldsymbol{x}') = \text{exp} \, \left( - \dfrac{1}{2} \| \boldsymbol{x} - \boldsymbol{x}' \|_2^2 \right).
$$
The top panels of Figure \ref{fig. 4} display the posterior Gaussian Processes for these two cases. Both Gaussian Processes appear very similar, and the sampled functions (the mean is also a sample) share the same properties. In contrast, the bottom panel visualizes the posterior Gaussian Process of $\mathcal{G}\mathcal{P} (0,k_{1/2,1})$ that uses a different kernel. Here, the sample functions fluctuate more and lack differentiability, highlighting the greater influence of the kernel on the sample properties.
\end{example}

\begin{figure}[h]
	\centering
	\includegraphics[width=0.8\textwidth]{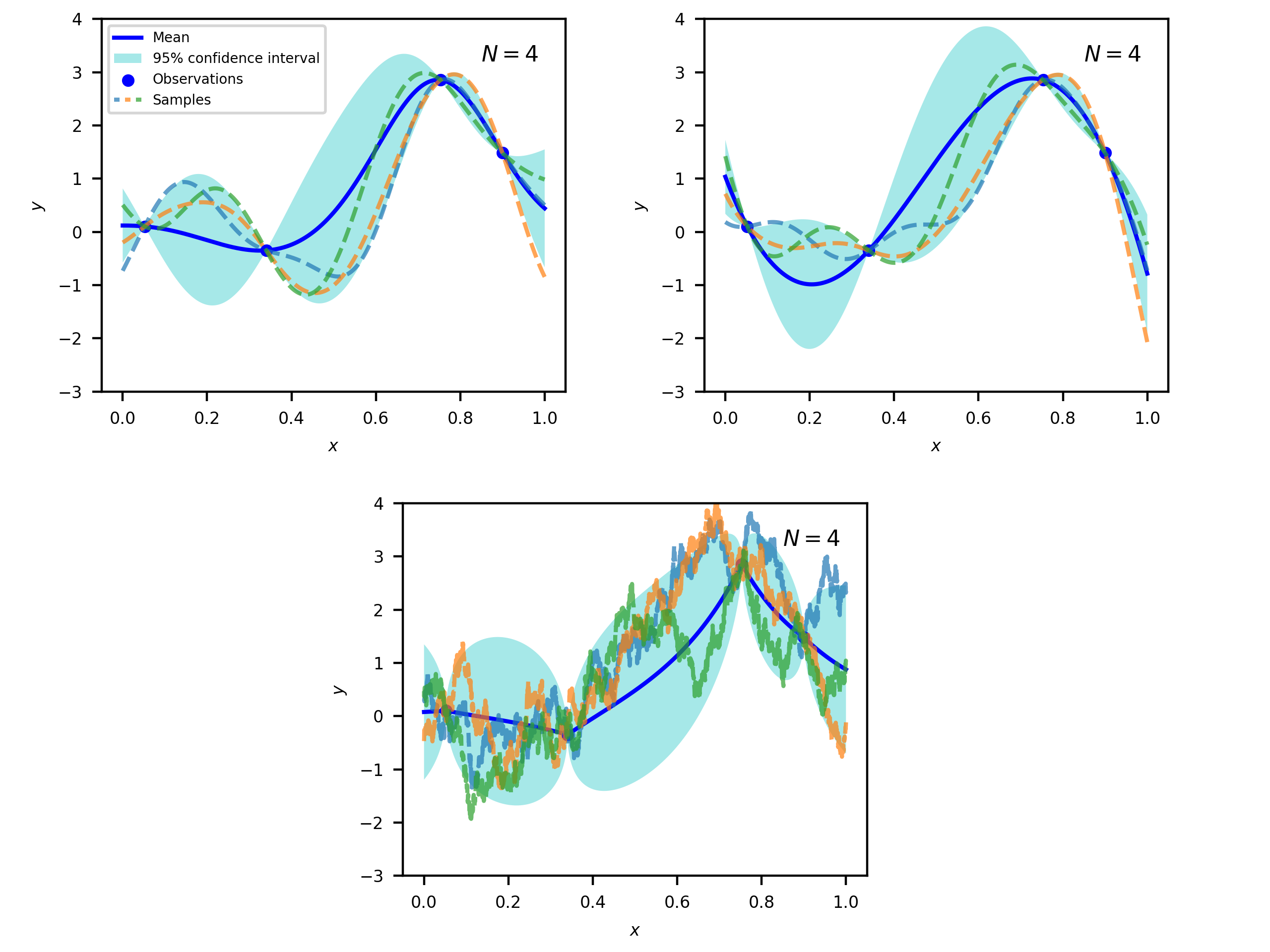}
	\caption{Visualization of posterior Gaussian Processes. The top left panel shows the posterior Process of $\mathcal{G}\mathcal{P} (0,k)$, the top right panel of $\mathcal{G}\mathcal{P} (m,k)$ with parameters from Example \ref{2.3.1} and the bottom panel of $\mathcal{G}\mathcal{P} (0,k_{1/2,1})$ with Matérn $1/2$ kernel. Observations were sampled from the same function as in Figure \ref{fig. 2}, but without noise.}\label{fig. 4}
\end{figure}

As the main part of a Gaussian Process is determined by the kernel, we will use a zero mean function for modelling. The benefits are, that we need to fit less parameters, have a relation to the theory of RKHS and can focus on the impact of kernels on Gaussian Processes.

%-------------------------------------------------------------------------------------------
% Kernel function
%-------------------------------------------------------------------------------------------

\subsubsection{Kernel function}
The kernel function $k:\Omega \times \Omega \rightarrow \mathbb{R}$ determines the covariance
$$k(\boldsymbol{x},\boldsymbol{x}') = \text{Cov} \left(f(\boldsymbol{x}), f(\boldsymbol{x}') \right)$$
between any two function values. This covariance function specifies the similarity between two function values and for this reason how samples of a Gaussian Process behave. We will use mainly stationary kernels\index{Stationary kernel,}, that is the kernel $k$ can be written as
$$k(\boldsymbol{x},\boldsymbol{x}') = k(\boldsymbol{x} - \boldsymbol{x}' , \boldsymbol{0}) \text{ for every } \boldsymbol{x},\boldsymbol{x}' \in \Omega$$
and in this context we use the notation $k(\boldsymbol{\tau}) = k(\boldsymbol{\tau} , \boldsymbol{0})$ for $\boldsymbol{\tau} = \boldsymbol{x} - \boldsymbol{x}'$. A stationary kernel depends only on the difference $\boldsymbol{\tau}$ and not on the inputs $\boldsymbol{x}, \boldsymbol{x}'$ itself. This allows for inference on $k$ based on all pairs with the same difference $\boldsymbol{\tau}$. In 1955 Bochner characterised this kind of kernels and provided a method to construct them. Since the details are not relevant in my thesis, the interested reader can find more about the construction of stationary kernels in chapter four of \cite{RW} or in chapter two of \cite{S}. Table \ref{tab. 2} provides a list of commonly used stationary and non-stationary kernels.\\

% !htbp
\begin{table}[h]
    \centering
    \def\arraystretch{1.7}
    \begin{tabular}{|l|c|c|}
    \hline
    \textbf{Kernel name} & \centering\arraybackslash \textbf{Formula} $k(\boldsymbol{x}, \boldsymbol{x}')$ & \textbf{Parameters}\\[-0.1ex]
    \specialrule{1pt}{\abovetopsep}{\belowbottomsep}
    \text{Squared Exponential\index{Squared Exponential kernel,}} & \centering\arraybackslash $s^2 \exp\left(-\frac{\| \boldsymbol{\tau} \|_2^2}{2\ell^2}\right)$ & $s^2 >0, \ell >0$\\
    \hline
    Exponential\index{Exponential kernel,} & \centering\arraybackslash $s^2 \exp\left(-\frac{\| \boldsymbol{\tau} \|_2}{\ell}\right)$ & $s^2 >0, \ell >0$\\
    \hline
    Polynomial & $(s^2 + \boldsymbol{x}^T \boldsymbol{x}')^m$  & $s^2 \geq 0, m \in \mathbb{N}$ \\
    \hline
    Rational Quadratic & \centering\arraybackslash $\left( 1 + \frac{\| \boldsymbol{\tau} \|_2^2}{2a\ell^2} \right)^{-a}$  & $a>0,\ell>0$\\
    \hline
    Periodic\index{Periodic kernel,} & \centering\arraybackslash $s^2 \exp\left( - \frac{2}{\ell^2} \sum_{i=1}^d \sin^2 \left(\frac{\pi |x_i - x_i'|}{p} \right) \right)$  & $s^2>0 ,p>0, \ell>0$\\
    \hline
    Matérn & \centering\arraybackslash $s^2 \frac{2^{1-\nu}}{\Gamma(\nu)} 
\left( \frac{\sqrt{2} \nu \| \boldsymbol{\tau} \|_2}{\ell} \right)^\nu 
K_\nu \left( \frac{\sqrt{2} \nu \| \boldsymbol{\tau} \|_2}{\ell} \right)$  & $s^2>0, \nu>0, \ell>0$\\
    \hline
    \end{tabular}
    \caption{Examples of kernel functions.}
    \label{tab. 2}
\end{table}

%-------------------------------------------------------------------------------------------

With these fundamental kernels at our disposal, we can construct more complex kernels tailored to specific needs. The following lemma introduces several techniques for creating custom kernels, which can be particularly useful depending on the problem at hand.

\newenvironment{brsm}{% % short for 'bracketed small matrix'
  \bigl[ \begin{smallmatrix} }{%
  \end{smallmatrix} \bigr]}

\begin{lemma} \label{2.3.2}
Let $(k_i)_{i\in \mathbb{N}}$ be a sequence of kernels, then the following functions are also kernels.
\begin{enumerate}[label=(\roman*),topsep=5pt]
	\setlength\itemsep{0.1mm}
	\item $a_1 k_1 + a_2 k_2$ for $a_1,a_2 \geq 0$,  \label{2.3.2 (i)}
	\item $k_1 \cdot k_2$, \label{2.3.2 (ii)}
	\item $\lim_{i \rightarrow \infty} k_i$ if it exists, \label{2.3.2 (iii)}
	\item $k:\Omega^2 \times \Omega^2 \rightarrow \mathbb{R}, k \left(
\begin{brsm} \boldsymbol{x}_1 \\ \boldsymbol{x}_1' \end{brsm},
\begin{brsm} \boldsymbol{x}_2 \\ \boldsymbol{x}_2' \end{brsm} \right)
= k_1(\boldsymbol{x}_1, \boldsymbol{x}_2)+k_2(\boldsymbol{x}_1', \boldsymbol{x}_2')$, \label{2.3.2 (iv)}
	\item $k:\Omega^2 \times \Omega^2 \rightarrow \mathbb{R}, k \left(
\begin{brsm} \boldsymbol{x}_1 \\ \boldsymbol{x}_1' \end{brsm},
\begin{brsm} \boldsymbol{x}_2 \\ \boldsymbol{x}_2' \end{brsm} \right)
= k_1(\boldsymbol{x}_1, \boldsymbol{x}_2) \cdot k_2(\boldsymbol{x}_1', \boldsymbol{x}_2')$. \label{2.3.2 (v)}
\end{enumerate}
\end{lemma}

\begin{proof}
The verification of \ref{2.3.2 (i)}, \ref{2.3.2 (iii)} and \ref{2.3.2 (iv)} is straightforward as well as the symmetry property of the functions in \ref{2.3.2 (ii)} and \ref{2.3.2 (v)}. To show the positive semi-definiteness for the function in \ref{2.3.2 (ii)} let $n\in \mathbb{N}$, $c_1,..,c_n \in \mathbb{R}$, $\boldsymbol{z}_1,...,\boldsymbol{z}_n \in \Omega$ and $\boldsymbol{K} = [k_1(\boldsymbol{z}_i, \boldsymbol{z}_j)]_{1 \leq i,j \leq n}$ be the symmetric and positive semi-definite kernel matrix\index{Kernel matrix,}. Let
$$\boldsymbol{K} = UDU^T$$
be the eigendecomposition of $\boldsymbol{K}$ with orthogonal matrix $U=[u_{i,j}]_{1 \leq i,j \leq n} \in \mathbb{R}^{n \times n}$ and diagonal matrix $D$ with eigenvalues $\lambda_1,...,\lambda_n \geq 0$ on the diagonal. Then,
\begin{align*}
\sum_{i,j=1}^n c_ic_j k_1(\boldsymbol{z}_i, \boldsymbol{z}_j) k_2(\boldsymbol{z}_i, \boldsymbol{z}_j) 
= \sum_{i,j=1}^n &c_ic_j \left( \sum_{\ell=1}^n u_{i,\ell} \lambda_\ell u_{j,\ell} \right) k_2(\boldsymbol{z}_i, \boldsymbol{z}_j) \\
&= \sum_{\ell=1}^n \lambda_\ell \left( \sum_{i,j=1}^n (c_i u_{i,\ell}) (c_j u_{j,\ell}) k_2(\boldsymbol{z}_i, \boldsymbol{z}_j) \right) \geq 0,
\end{align*}
because $k_2$ is positive semi-definite. The proof that the function in \ref{2.3.2 (v)} is positive semi-definite goes analogous.
\end{proof}

%-------------------------------------------------------------------------------------------

\begin{example}
To estimate a function $f:[0,1] \rightarrow \mathbb{R}$, we assume that $f$ is a sample from a Gaussian Process, denoted as $f \sim \mathcal{GP}(0,k)$. Selecting the appropriate kernel is crucial, as the Gaussian Process samples should exhibit properties that closely align with those of $f$. Here are common recommendations: 
\begin{itemize}[topsep=5pt]
	\setlength\itemsep{0.1mm}
\item If the function values are locally similar, the squared exponential kernel may be used.
\item For functions with periodic behavior, the periodic kernel is more suitable. 
\item For functions displaying both periodic and local behavior, a product of the squared exponential and periodic kernels can be preferable. 
\end{itemize}
In Figure \ref{fig. 5}, samples of the considered Gaussian Processes are illustrated. The choice of kernel should aim to generate samples that reflect your understanding of $f$.
\end{example}

\begin{figure}[h]
	\centering
	\includegraphics[width=\textwidth]{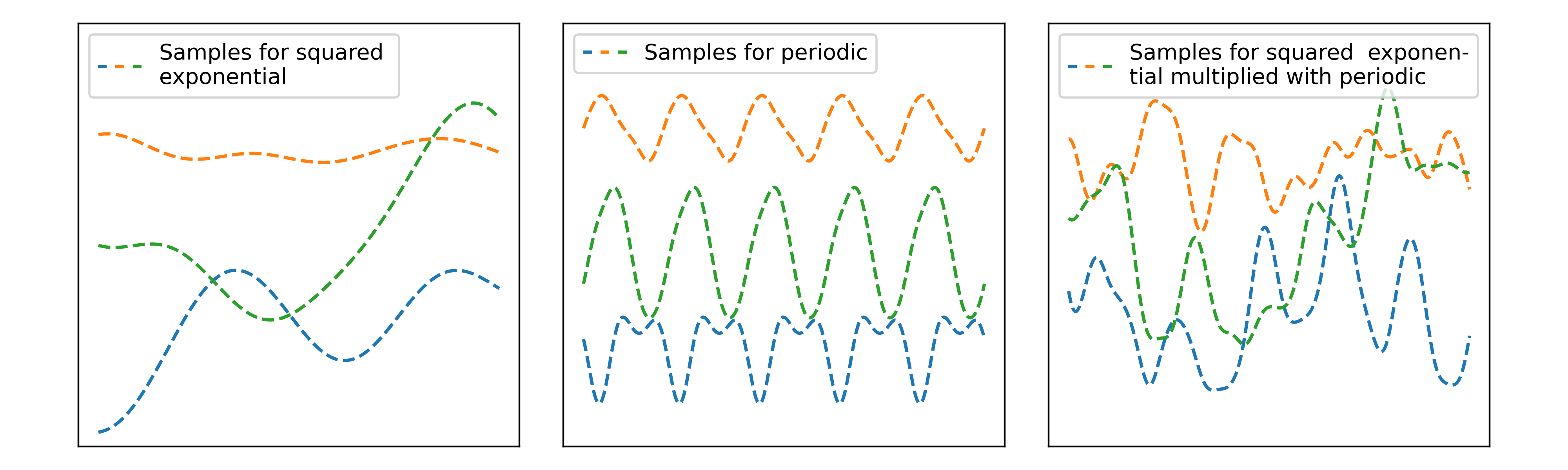}
	\caption{Samples of Gaussian Processes $\mathcal{GP}(0,k)$.}\label{fig. 5}
\end{figure}

A crucial class of kernels that holds significant importance in my thesis are the Matérn kernels, which are defined as follows.

\begin{defi} \label{2.3.4}
Let $\Omega \subseteq \mathbb{R}^d$ be a non-empty set. For $s^2, \nu>0$ and $\ell>0$ the Matérn kernels\index{Matérn kernel,} are
$$s^2 k_{\nu,\ell}(\boldsymbol{x},\boldsymbol{x}') 
= s^2 \dfrac{2^{1-\nu}}{\Gamma(\nu)} 
\left( \dfrac{\sqrt{2 \nu} \| \boldsymbol{x} - \boldsymbol{x}' \|_2}{\ell} \right)^\nu 
K_\nu \left( \dfrac{\sqrt{2 \nu} \| \boldsymbol{x} - \boldsymbol{x}' \|_2}{\ell} \right) \ 
\text{ for } \boldsymbol{x},\boldsymbol{x}' \in \Omega, \index{\textit{$k_{\nu,\ell}$,}}$$
where $\Gamma$ is the gamma function, $K_\nu$ the modified Bessel function of the second kind.
\end{defi}

This class of kernels arises because the Gaussian Process $f \sim \mathcal{GP}(0,k_{\nu,\ell})$ is a solution to a stochastic partial differential equation. The parameter $\nu$ acts as a smoothing parameter, with higher values of $\nu$ indicating a smoother $f$. Specifically:
\begin{itemize}[topsep=5pt]
	\setlength\itemsep{0.1mm}
\item The Gaussian Process $f \sim \mathcal{GP}(0,k_{\nu,\ell})$ is $n$-times mean square differentiable only if $\nu >n$, as detailed on page 85 in reference \cite{RW}.
\item The samples of $\mathcal{GP}(0,k_{\nu,\ell})$ are $n$-times continuously differentiable only if $\nu >n$, as explained in Corollary 2 of reference \cite{C}.
\end{itemize}
The parameter $\ell$, known as the length scale\index{Length scale,}, governs the function's fluctuation. As $\ell$ diminishes, the kernel undergoes more rapid changes, leading to faster variations in the values of $f$. The signal variance\index{Signal variance,} $s^2$, which scales the kernel, also affects the variations. A larger $s^2$ indicates greater similarity between function values.\\
The Matérn kernels are widely embraced due to their versatility. Unlike the periodic, squared exponential, or exponential kernel, which impose rigid assumptions on the underlying function, each Matérn kernel makes distinct assumptions. This suggests that there is likely a Matérn kernel that accurately captures the behavior of $f$. Additionally, the Matérn class enables fine control over the differentiability of $f$, a feature not offered by many other kernels.

%-------------------------------------------------------------------------------------------

\begin{remark} \label{2.3.5}
If $\nu=n+1/2$ for $n \in \mathbb{N}_{\geq 0}$, we can write
$$k_{\nu,\ell}(\boldsymbol{x},\boldsymbol{x}')  
= \text{exp} \, \left( - \dfrac{\sqrt{2\nu} \| \boldsymbol{x} - \boldsymbol{x}' \|_2}{\ell} \right)
\dfrac{\Gamma(n+1)}{\Gamma(2n+1)} \sum_{i=0}^n \dfrac{(n+i)!}{i!(n-i)!}
\left( \dfrac{\sqrt{8\nu} \| \boldsymbol{x} - \boldsymbol{x}' \|_2}{\ell} \right)^{n-i}.$$
That is a product of an exponential kernel and a polynomial in $\|\boldsymbol{\tau}\|_2$ of degree $n$. We have
\begin{align*}
k_{1/2,\ell} (\boldsymbol{x},\boldsymbol{x}') 
&= \exp\left(-\frac{\| \boldsymbol{x} - \boldsymbol{x}' \|_2}{\ell}\right), \\
k_{3/2,\ell} (\boldsymbol{x},\boldsymbol{x}') 
&= \left( 1+ \dfrac{\sqrt{3} \| \boldsymbol{x} - \boldsymbol{x}' \|_2}{\ell} \right)
\exp\left(-\frac{\sqrt{3} \| \boldsymbol{x} - \boldsymbol{x}' \|_2}{\ell}\right), \\
k_{5/2,\ell} (\boldsymbol{x},\boldsymbol{x}')
&= \left( 1+ \dfrac{\sqrt{5} \| \boldsymbol{x} - \boldsymbol{x}' \|_2}{\ell} + \dfrac{5 \| \boldsymbol{x} - \boldsymbol{x}' \|_2^2}{3 \ell^2} \right)
\exp\left(-\frac{\sqrt{5} \| \boldsymbol{x} - \boldsymbol{x}' \|_2}{\ell}\right).
\end{align*}
Thus, the class of Matérn kernels includes for $\nu=1/2$ the exponential kernel and for $\nu \rightarrow \infty$ the squared exponential kernel. To proof the latter, you can compare the spectral density functions of the kernels as indicated in section 4.2 of \cite{RW}.\\
Using the exponential kernel results in non-differentiable samples and using the squared exponential kernel results in smooth samples. While the parameter $\nu$ provides a variety of kernels, it is rare to ascertain whether $f$ is merely twice differentiable or more. As a result, one typically relies on $\nu = 5/2$ or the squared exponential kernel in such scenarios. This situation can be compared to the preference for cubic splines in the realm of splines, predominantly chosen for functions with multiple degrees of differentiability. Additionally, the computation of Matérn kernels becomes more intensive as $\nu$ increases. Figure \ref{fig. 6} illustrates samples of Gaussian Processes utilizing Matérn kernels.
\end{remark}

\begin{figure}[h]
	\centering
	\includegraphics[width=\textwidth]{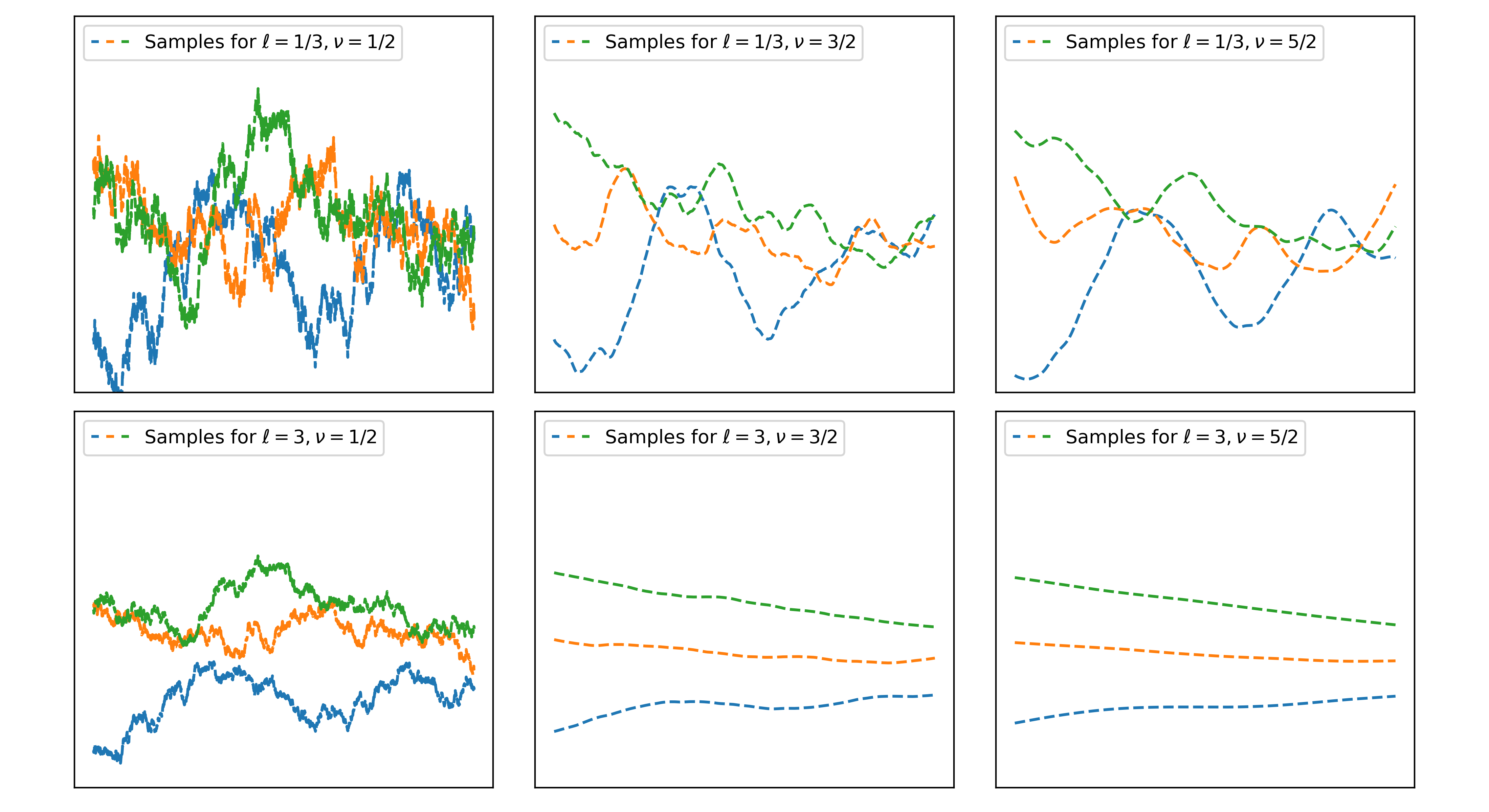}
	\caption{Samples of Gaussian Processes $\mathcal{GP}(0,k_{\nu,\ell})$.}\label{fig. 6}
\end{figure}

%-------------------------------------------------------------------------------------------
% Fitting Hyperparameters
%-------------------------------------------------------------------------------------------

\subsubsection{Fitting Hyperparameters} \label{Chap. 2.3.3}
By choosing a kernel $k$ to model $f \sim \mathcal{GP} (0,k)$, the procedure is not done. As illustrated in Table \ref{tab. 2}, most kernels come with a set of hyperparameters that must be selected or learned to fully specify the model. For instance, opting for the squared exponential kernel
$$k(\boldsymbol{\tau}) = s^2 \exp\left(-\frac{\| \boldsymbol{\tau} \|_2^2}{2\ell^2}\right).$$
necessitates the choice of $s^2>0$ and $\ell>0$. When combining kernels, as discussed in Lemma \ref{2.3.2}, additional parameters may need to be selected. However, it is essential to exercise caution to prevent an excessive number of parameters, which could lead to overfitting. Here are three common methods for selecting hyperparameters:
\begin{itemize}[topsep=5pt]
	\setlength\itemsep{0.1mm}
\item Prior belief,
\item Cross-validation,
\item Maximum likelihood.\\
\end{itemize}

\noindent
\textbf{Prior belief:}\index{Prior belief,} The first method is the simplest. In Example \ref{2.1.4}, I selected $s^2=1$ and $\ell=1$, because I believed it could provide a reasonable estimate for $f$ given the specific problem at hand. While Figure \ref{fig. 2} may indicate that this selection was effective, it is generally not advisable to rely on this approach.\\

\noindent
\textbf{Cross-validation:}\index{Cross-validation,} The technique of cross-validation divides the training data into two disjoint sets, one for actual training, and the other for validation. To assess the model's performance, a loss function is minimized, often the residual sum of squares. Cross-validation is a fundamental concept in machine learning, and further details can be found in section 5.3 of \cite{RW}.\\

\noindent
\textbf{Maximum likelihood:}\index{Maximum likelihood,} It is perhaps the most commonly utilized concept. The objective is to maximize the probability that the model generates the observed data. Essentially maximizing the log marginal likelihood function\index{\textit{$p(\boldsymbol{y} \, \vert \, X)$,}}
\begin{equation}
\log p(\boldsymbol{y} \, | \, X)
= -\dfrac{1}{2} \boldsymbol{y}^T (\boldsymbol{K} + \sigma^2 \boldsymbol{I}_N)^{-1} \boldsymbol{y} - \dfrac{1}{2} \log | \boldsymbol{K} + \sigma^2 \boldsymbol{I}_N | - \dfrac{N}{2} \log 2\pi \label{eq. 2.3.1}
\end{equation}
for points $X = \{ \boldsymbol{x}_1, ..., \boldsymbol{x}_N \}$, $\boldsymbol{y}=[y_1,...,y_N]^T$ and kernel matrix $\boldsymbol{K}=[k(\boldsymbol{x}_i, \boldsymbol{x}_j)]_{1 \leq i,j \leq N}$, where $| \boldsymbol{K} + \sigma^2 \boldsymbol{I}_N | = \det (\boldsymbol{K} + \sigma^2 \boldsymbol{I}_N)$\index{\textit{$\vert \boldsymbol{K} + \sigma^2 \boldsymbol{I}_N \vert$,}}. The form of the likelihood function is justified, as the assumption $f \sim \mathcal{GP} (0,k)$ leads to $\boldsymbol{y} \, | \, X \sim \mathcal{N} (\boldsymbol{0},\boldsymbol{K} + \sigma^2 \boldsymbol{I}_N)$, making $p(\boldsymbol{y} \, | \, X)$ the corresponding probability density function of this multivariate normal distribution. Let $\hat{\boldsymbol{K}} = \boldsymbol{K} + \sigma^2 \boldsymbol{I}$\index{\textit{$\hat{\boldsymbol{K}}$,}}, then for a kernel hyperparameter $\vartheta \in \mathbb{R}$\index{\textit{$\vartheta$,}} is
$$\dfrac{\partial}{\partial \vartheta} \hat{\boldsymbol{K}}^{-1} = - \hat{\boldsymbol{K}}^{-1} \dfrac{\partial \hat{\boldsymbol{K}}}{\partial \vartheta} \hat{\boldsymbol{K}}^{-1} \quad \text{and} \quad 
\dfrac{\partial}{\partial \vartheta} \log | \hat{\boldsymbol{K}} | = \text{tr} \left( \hat{\boldsymbol{K}}^{-1} \dfrac{\partial \hat{\boldsymbol{K}}}{\partial \vartheta} \right),$$
where $\text{tr}(\boldsymbol{A})$\index{\textit{$\text{tr}(\boldsymbol{A})$,}} is the trace of the matrix $\boldsymbol{A}$. This leads to the derivative

%-------------------------------------------------------------------------------------------

\begin{equation}
\dfrac{\partial}{\partial \vartheta} \log p(\boldsymbol{y} \, | \, X) = \dfrac{1}{2} \boldsymbol{c}^T \dfrac{\partial \hat{\boldsymbol{K}}}{\partial \vartheta} \boldsymbol{c} - \dfrac{1}{2} \text{tr} \left( \hat{\boldsymbol{K}}^{-1} \dfrac{\partial \hat{\boldsymbol{K}}}{\partial \vartheta} \right), \label{eq. 2.3.2}
\end{equation}
where $\boldsymbol{c} = [c_1, ..., c_N]^T \in \mathbb{R}^{N \times 1}$ satisfies $(\boldsymbol{K} + \sigma^2 \boldsymbol{I} ) \boldsymbol{c} = \boldsymbol{y}$. \\
Optimizing (\ref{eq. 2.3.1}) is expensive, as computation requires evaluating $\boldsymbol{c}$ and the log determinant of $\hat{\boldsymbol{K}}$. This challenge is similar to that faced when computing the posterior Gaussian Process, where $\boldsymbol{c}$ must also be evaluated. If we compute (\ref{eq. 2.3.2}), the inversion of the matrix $\hat{\boldsymbol{K}} = \boldsymbol{K} + \sigma^2 \boldsymbol{I}$ is required. To overcome these challenges, we will simplify the matrix using samplets and then apply gradient descent methods to optimize (\ref{eq. 2.3.1}). \\

It is important to recognize that the parameters derived from these three methods are sensitive to small changes. Even slight alterations in the data or the fitting technique employed can lead to significant changes. Consequently, the error bounds we aim to establish must account for potential misspecification of hyperparameters.

%-------------------------------------------------------------------------------------------
% Bayesian Optimization
%-------------------------------------------------------------------------------------------

\subsection{Bayesian Optimization}
Gaussian Processes play a pivotal role in Bayesian optimization\index{Bayesian optimization,}, which aims to discover the optimum
$$\boldsymbol{x}^* \in \underset{\boldsymbol{x} \in [0,1]^d}{\argmax} \, f(\boldsymbol{x})$$
of a black box function $f:[0,1]^d \rightarrow \mathbb{R}$. Here, a black-box function means that the underlying mapping is unknown, yet we can observe its values expensively as
$$y = f(\boldsymbol{x}) + \varepsilon,$$
where $\varepsilon \sim \mathcal{N}(0,\sigma^2)$. Notably, lacking access to the gradients of $f$, precludes the use of gradient descent methods. Black-box functions\index{Black-box function,} are pervasive across various fields such as hyperparameter tuning, robotics, asset allocation, advertising and experimental design. You might aim to identify the most effective experimental conditions to expedite drug discovery or enhance material design. Even if it is refining the taste of your coffee or optimizing plant growth in your garden, these scenarios pose black-box optimization challenges. Bayesian optimization now seeks to uncover the optimum by executing the following iterative process:
\begin{enumerate}
\item Start with observations $(\boldsymbol{x}_1,y_1), ..., (\boldsymbol{x}_N,y_N)$.
\item Utilize the observations to construct a surrogate model\index{Surrogate model,} for $f$.
\item Employ an acquisition function\index{Acquisition function,} to determine $\boldsymbol{x}_{N+1}$ and evaluate $y_{N+1}$.
\item Incorporate $(\boldsymbol{x}_{N+1}, y_{N+1})$ into the set of observations and repeat the process.
\end{enumerate}

\begin{figure}[H]
	\centering
	\includegraphics[width=0.88\textwidth]{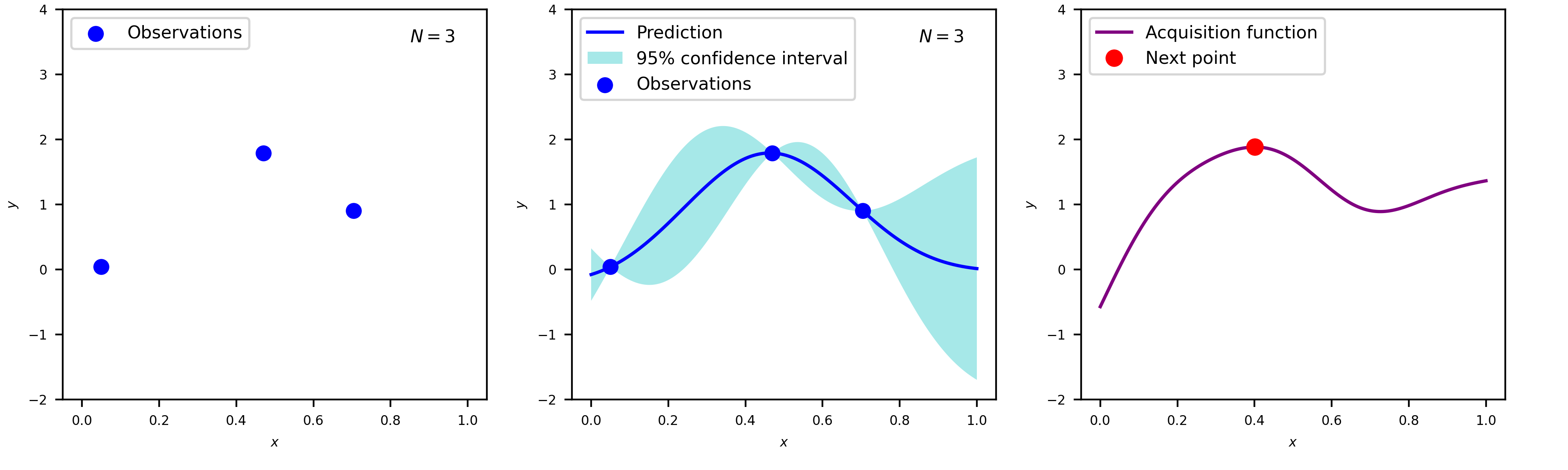}
	\caption{Iterative step of a one-dimensional Bayesian Optimization procedure.}\label{fig. 7}
\end{figure}

%-------------------------------------------------------------------------------------------

The Gaussian Process stands out as the preferred surrogate model due to its widespread adoption. Its key advantage lies in its probabilistic nature, enabling the quantification of prediction uncertainty. Leveraging this uncertainty, we can identify the optimal next point to observe, denoted as $\boldsymbol{x}_{\text{next}}$.

Once the surrogate model is established, we employ an acquisition function, denoted as $a:[0,1]^d \rightarrow \mathbb{R}$, to determine the subsequent observation point $\boldsymbol{x}_{\text{next}}$. More accurate we will set
$$\boldsymbol{x}_{\text{next}} = \underset{\boldsymbol{x} \in [0,1]^d}{\argmax} \, a(\boldsymbol{x}).$$
Various options for acquisition functions exist, with popular choices including expected improvement, probability of improvement and upper confidence bound, detailed in \cite{G}. Here, I will employ an alternative approach known as Thompson sampling\index{Thompson sampling,}. The rationale behind Thompson sampling lies in addressing the explore-exploit tradeoff inherent in all acquisition functions. Figure \ref{fig. 8} shows that two distinct regions are evident for selecting the next point, either within areas of known good solutions or within regions of high uncertainty.

\begin{figure}[h]
	\centering
	\includegraphics[width=0.5\textwidth]{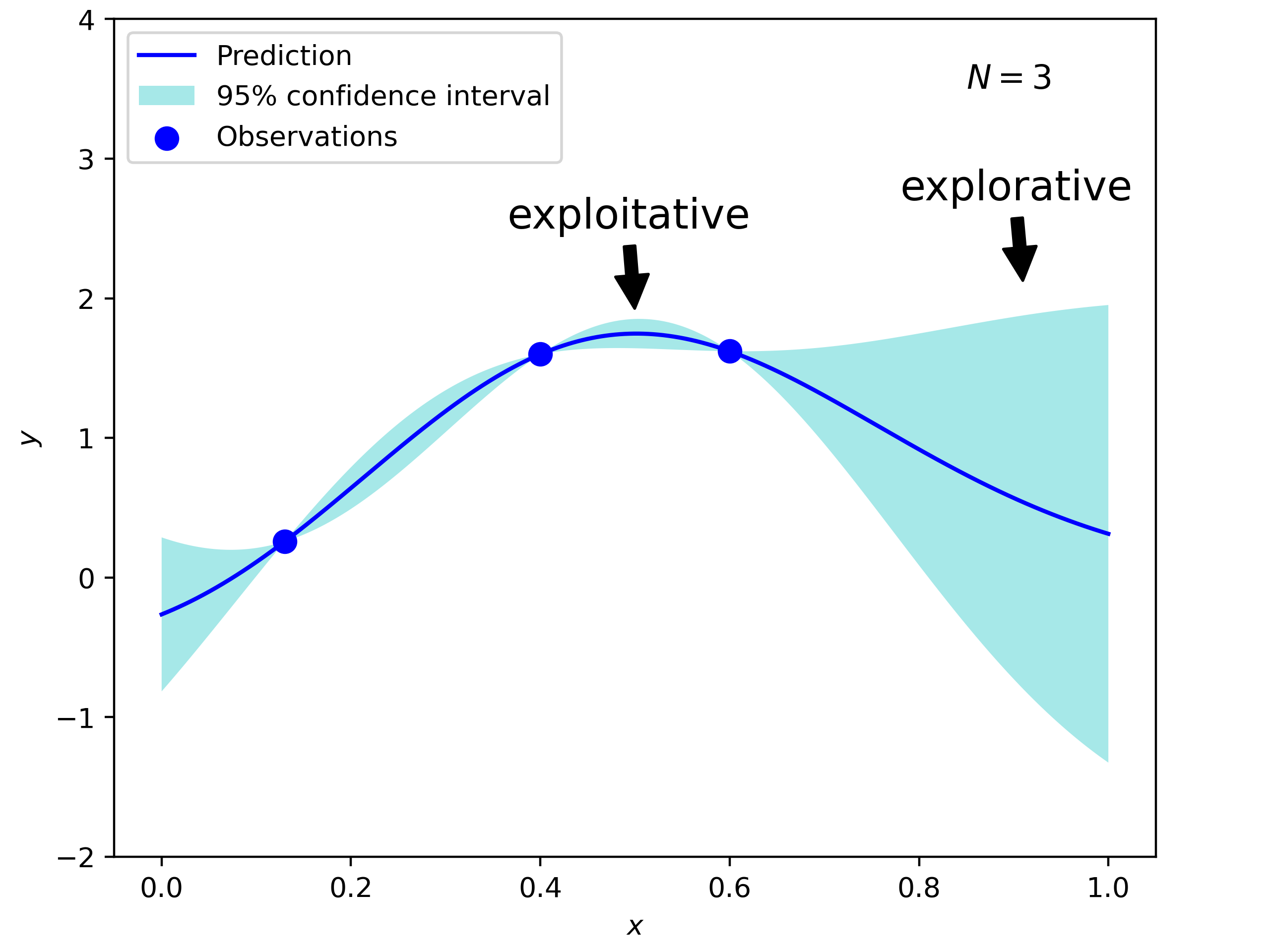}
	\caption{The acquisition function has to deal with an explore-exploit tradeoff.}\label{fig. 8}
\end{figure}

Thompson sampling effectively manages the explore-exploit tradeoff by employing a sample drawn from the Gaussian Process as the acquisition function. In previous sections samples from a Gaussian Process were already mentioned, but the method for their generation has not been outlined. Given that a sample from $\mathcal{GP}(m',k')$ is a function $a:[0,1]^d \rightarrow \mathbb{R}$ with potentially an infinite amount of function values, it is imperative to approximate this function. To do so, we represent the unit cube by using a dense grid, comprising points such as $\boldsymbol{z}_1,...,\boldsymbol{z}_n \in [0,1]^d$ and then sample from
$$\mathcal{N} \left( 
\begin{bmatrix}
	m'(\boldsymbol{z}_1) \\ \vdots \\ m'(\boldsymbol{z}_n) 
\end{bmatrix},
\begin{bmatrix}
	k'(\boldsymbol{z}_1,\boldsymbol{z}_1) & \cdots & k'(\boldsymbol{z}_1,\boldsymbol{z}_n) \\ \vdots 
	& \ddots & \vdots \\ k'(\boldsymbol{z}_n,\boldsymbol{z}_1) & \cdots 
	& k'(\boldsymbol{z}_n,\boldsymbol{z}_n) 
\end{bmatrix} \right).$$
The resultant $n$-dimensional vector represents the sample. It is essential for the value of $n$ to strike a balance, because it should be sufficiently high to approximate a sample from $\mathcal{GP}(m',k')$, yet not excessively so, as an overly high value would make the sampling process prohibitively time-consuming. Algorithm \ref{alg. 1} formulates the procedure.

%-------------------------------------------------------------------------------------------

\begin{algorithm}[h]
\DontPrintSemicolon
\caption{Thompson Sampling}\label{alg. 1}
\KwData{Points $Z = \{ \boldsymbol{z}_1, ..., \boldsymbol{z}_n \}$, mean and kernel functions $(m',k')$}
\KwResult{Function values $\boldsymbol{a}_Z$ of a sample $a \sim \mathcal{GP}(m',k')$}
\SetKwFunction{FMain}{GPsample}
\SetKwProg{Fn}{Function}{:}{}
\Fn{\FMain{$Z$}}{
	$\boldsymbol{m}' = [m'(\boldsymbol{z}_1), ..., m'(\boldsymbol{z}_n)]^T$\;
	$\boldsymbol{K}' = [k'(\boldsymbol{z}_i, \boldsymbol{z}_j)]_{1 \leq i,j \leq n}$\;
	$\boldsymbol{L} = \texttt{Cholesky}(\boldsymbol{K}')$\;
	$\boldsymbol{\varepsilon} = \texttt{Random.standard\_normal} (\text{size}=n)$\;
	$\boldsymbol{a}_Z = \boldsymbol{m}' + \boldsymbol{L} \boldsymbol{\varepsilon}$\;
  	\KwRet $\boldsymbol{a}_Z$\;
}
\textbf{end}
\end{algorithm}

\newpage
$ \ $
\newpage
% !TeX spellcheck = en_GB
\section{Convergence of Gaussian Process Means} \label{Chap. 3}
Let $\Omega \subseteq \mathbb{R}^d$ be a non-empty set. Given a dataset $\mathcal{D}_N$, we employ a Gaussian Process $\mathcal{GP}(0,k)$ to estimate an unknown function $f: \Omega \rightarrow \mathbb{R}$. To address the challenge of selecting an appropriate Gaussian Process model, this chapter explores the convergence rate of Gaussian Process means. In order to provide theoretical justification for our model choices in specific problem contexts, we address the following questions:
\begin{itemize}[topsep=5pt]
	\setlength\itemsep{0.1mm}
\item For which sets $\Omega \subseteq \mathbb{R}^d$ and functions $f: \Omega \rightarrow \mathbb{R}$ can we ensure convergence?
\item Which kernel $k$ should we use, and which hyperparameters are crucial?
\item What factors should be considered when fitting hyperparameters?
\item How does the data set $\mathcal{D}_N$ influence the convergence rate?
\item How does misspecified smoothness or likelihood affect the convergence rate?
\item What are the expected and optimal convergence rates?
\end{itemize}

The first convergence results were established by Narcowich et al. in \cite{N} for the interpolation setting, followed by van der Vaart et al. in \cite{V} for the regression case. Teckentrup further extended the interpolation results in \cite{TE} by investigating the impact of misspecified likelihoods and varying hyperparameters. Briol et al. later unified and expanded these findings in \cite{BGW} by relaxing certain assumptions, incorporating additional norms, and addressing both interpolation and regression settings. Given that the work of Briol et al. is the most recent, this chapter is based on their results.\\

The first section defines Sobolev spaces and summarizes the properties relevant to our study. We then proceed to the experimental setting, outlining key values and conditions for both the data and the Gaussian Process. Following this, we present the convergence guarantees for the interpolation and the regression settings, which are articulated through the theorems in this chapter and also available in reference \cite{BGW}. Last, we provide an overview of the conclusions that can be drawn from these guarantees.

%-------------------------------------------------------------------------------------------
% Sobolev spaces
%-------------------------------------------------------------------------------------------

\subsection{Sobolev spaces}
The convergence results will be formulated under the assumption that the unknown function $f$ belongs to a Sobolev space. Sobolev spaces are sufficiently broad to justify this assumption, which is crucial for deriving reliable convergence results. In this section, we define this spaces and explore relevant properties. Since the following statements are well-established, I rely on existing references rather than providing proofs.

\begin{defi} \label{3.1.1}
Let $\Omega \subseteq \mathbb{R}^d$ be an open and non-empty set and let $1 \leq p$. The $L^p$\index{\textit{$L^p$,}} space is defined as
$$L^p(\Omega) = \bigg\{ g: \Omega \rightarrow \mathbb{R} \, : \, g \text{ is measurable}, \|g\|_{L^p(\Omega)}^p = \int_{\Omega} |g(\boldsymbol{x})|^p \,d\boldsymbol{x} < \infty \bigg\}, $$
where functions that equal almost everywhere on $\Omega$ are identified. Similarly, define
$$L^\infty (\Omega) = \big\{ g: \Omega \rightarrow \mathbb{R} \, : \, g \text{ is measurable}, \|g\|_{L^\infty(\Omega)} < \infty \big\},$$
with the norm 
$$\|g\|_{L^\infty(\Omega)} = \esssup_{\boldsymbol{x} \in \Omega} |g(\boldsymbol{x})| = \inf \{ C >0 : |g(\boldsymbol{x})| \leq C \text{ for almost every } \boldsymbol{x} \in \Omega \}.$$
\end{defi}

%-------------------------------------------------------------------------------------------

The $L^p$ space is essential for constructing Sobolev spaces and examining convergence rates. It can be shown that $L^p(\Omega)$ is a Banach space and for $1 \leq p \leq q \leq \infty$ is $L^q(\Omega) \subseteq L^p(\Omega)$, as detailed in Theorem 2.14 and 2.16 of \cite{AF}.

\begin{defi} \label{3.1.2}
Let $\Omega \subseteq \mathbb{R}^d$ be an open and non-empty set, $1 \leq p < \infty$ and $\tau \in \mathbb{N}_{\geq 0}$. The integer Sobolev space\index{Sobolev space,} is defined as\index{\textit{$W_p^\tau$,}}
$$W_p^\tau (\Omega) = \Bigg\{ g \in L^p(\Omega) \ : \  \|g\|_{W_p^\tau (\Omega)}^p 
= \sum_{0 \leq | \boldsymbol{\beta} | \leq \tau} \|D^{\boldsymbol{\beta}} g\|_{L^p(\Omega)}^p < \infty \Bigg\},$$
where $| \boldsymbol{\beta} |= |\beta_1 | + ... + | \beta_d |$ \index{\textit{$ \vert \boldsymbol{\beta} \vert $,}} for $\boldsymbol{\beta} = (\beta_1,...,\beta_d) \in \mathbb{N}_{\geq 0}^d$ and $D^{\boldsymbol{\beta}}$\index{\textit{$D^{\boldsymbol{\beta}}$,}} is the generalized (or weak) derivative operator\index{Derivative Operator,}, see Definition 1.62 and 3.2 in \cite{AF}. Similarly, we define
$$W_\infty^\tau (\Omega) = \Big\{ g \in L^\infty(\Omega) \ : \  \|g\|_{W_\infty^\tau (\Omega)} 
= \max_{0 \leq | \boldsymbol{\beta} | \leq \tau } \|D^{\boldsymbol{\beta}} g\|_{L^\infty(\Omega)} < \infty \Big\}.$$
\end{defi}

For $\tau \in (0,\infty) \backslash \mathbb{N}_{\geq 1}$, we can also define fractional Sobolev spaces. The standard approach is to use real interpolation, as done in Chapter 7 of \cite{AF}, but to avoid a lengthy digression, I will use the approach in \cite{AL}. Nevertheless, for the  suitable sets $\Omega \subset \mathbb{R}^d$ which we use later, these two approaches are equivalent (see Theorem 14.2.3 in \cite{BS}).

\begin{defi} \label{3.1.3}
For $x \in \mathbb{R}$, define the floor function as $\lfloor x \rfloor = \max \{ m \in \mathbb{Z}: m \leq x \}$\index{\textit{$\lfloor x \rfloor$,}}. Let $\Omega \subseteq \mathbb{R}^d$ be an open and non-empty set, $1 \leq p < \infty$ and $\tau \in (0,\infty) \backslash \mathbb{N}_{\geq 1}$. The fractional Sobolev space can be defined as
$$W_p^\tau (\Omega) = \big\{ g \in W_p^{ \lfloor \tau \rfloor} \ : \ \|g\|^p_{W_p^\tau (\Omega)} 
= \|g\|^p_{W_p^{\lfloor \tau \rfloor} (\Omega)} + |g|^p_{W_p^\tau (\Omega)} < \infty \big\},$$
where we use the semi-norm
$$|g|^p_{W_p^\tau (\Omega)} = \sum_{ |\boldsymbol{\beta}| = \lfloor \tau \rfloor } \int_{ \Omega \times \Omega} \dfrac{| (D^{\boldsymbol{\beta}}g)(\boldsymbol{x}) - (D^{\boldsymbol{\beta}}g)(\boldsymbol{y}) |^p}{\| \boldsymbol{x}-\boldsymbol{y} \|_2^{d+p(\tau - \lfloor \tau \rfloor)}} \, d\boldsymbol{x} \,d\boldsymbol{y} .$$
Similarly, we define
$$W_\infty^\tau (\Omega) = \big\{ g \in W_\infty^{ \lfloor \tau \rfloor} \ : \ \|g\|_{W_\infty^\tau (\Omega)} 
= \max \{ \|g\|_{W_\infty^{\lfloor \tau \rfloor} (\Omega)}, |g|_{W_\infty^\tau (\Omega)} \} < \infty \big\},$$
where we use the semi-norm
$$|g|_{W_\infty^\tau (\Omega)} = \max_{|\boldsymbol{\beta}| = \lfloor \tau \rfloor} \esssup_{\boldsymbol{x}, \boldsymbol{y} \in \Omega, \boldsymbol{x} \neq \boldsymbol{y}} \dfrac{| (D^{\boldsymbol{\beta}}g)(\boldsymbol{x}) - (D^{\boldsymbol{\beta}}g)(\boldsymbol{y}) |}{\| \boldsymbol{x}-\boldsymbol{y} \|_2^{\tau - \lfloor \tau \rfloor}}.$$
\end{defi}

The Sobolev spaces are well known in the theory of partial differential equations, where derivatives are understood in a weak sense to ensure $W_p^\tau (\Omega)$ forms a Banach space. If a function is differentiable in the conventional sense, it is also weakly differentiable, with its weak derivative matching the conventional derivative. This space includes many differentiable functions, as highlighted in the following proposition. For instance, if $\Omega$ is bounded, it contains all polynomials of any degree.

%-------------------------------------------------------------------------------------------

\begin{proposition} \label{3.1.4}
Let $\tau \in \mathbb{N}_{\geq 1}$ and $1 \leq p < \infty$. The Sobolev Space $W_p^\tau (\Omega)$ is the completion of
$$\Bigg\{ g \in C^\tau(\Omega) \ : \ \sum_{0 \leq | \boldsymbol{\beta} | \leq \tau} \|D^{\boldsymbol{\beta}} g\|_{L^p(\Omega)}^p < \infty \Bigg\}.$$
\end{proposition}

\begin{proof}
See Theorem 3.17 in \cite{AF}.
\end{proof}

To effectively investigate the Sobolev spaces, it is essential for $\Omega$ to satisfy the interior cone condition (see \ref{Chap. 8.2} in the Appendix). As indicated by the following proposition, a function that is sufficiently often weakly differentiable is also differentiable in the conventional sense.

\begin{proposition} \label{3.1.5}
Let $\Omega\subset \mathbb{R}^d$ be an open and non-empty set that satisfies the interior cone condition, $1 \leq p < \infty$ and $\lfloor \tau \rfloor >n+d/p$ for $n \in \mathbb{N}_{\geq 0}$. Then, 
% Let $\Omega\subset \mathbb{R}^d$ have a Lipschitz boundary and let $\tau>n+d/2$.
\begin{enumerate}[topsep=5pt]
	\setlength\itemsep{0.1mm}
\item $W_p^\tau (\Omega) \subseteq C_B^{n}(\Omega),$
\item $W_p^\tau (\Omega) \subseteq W_q^{n}(\Omega) \subseteq L^q(\Omega)$ for $p \leq q$,
\end{enumerate}
where $C_B^{n}(\Omega) \subseteq C^{n}(\Omega)$\index{\textit{$C_B^{n}(\Omega)$,}} is the space of functions having bounded and continuous derivatives up to order $n$ on $\Omega$.
%where $C^{n}(\overline{\Omega}) \subseteq C^{n}(\Omega)$ is the closed space of functions having bounded and uniformly continuous derivatives up to order $n$ on $\Omega$.
\end{proposition}

\begin{proof}
Notice that the interior cone condition is equivalent to the cone condition as defined in Definition 4.6 of \cite{AF}. Therefore, the statement follows directly from Theorem 4.12 in \cite{AF} and $W_p^\tau (\Omega) \subseteq W_p^{\lfloor \tau \rfloor} (\Omega)$.
\end{proof}

To provide stronger results and establish a connection between Sobolev spaces and the RKHS associated with Matérn kernels, we require a more restrictive condition for $\Omega$ than the interior cone condition. Specifically, we require $\Omega$ to have a Lipschitz boundary (see \ref{Chap. 8.2} in the Appendix). This condition necessitates that $\Omega$ is open, bounded and has a boundary that can locally be represented as graph of a Lipschitz continuous function. For example, any open, bounded and convex set meets this condition. If we want to emphasize that a domain $\Omega$ with a Lipschitz boundary satisfies the interior cone condition with parameters $R$ and $\delta$, we refer to it as an $\mathcal{L}(R,\delta)$-domain (see \ref{Chap. 8.2} in the Appendix). 

\begin{proposition} \label{3.1.6}
Let $\Omega \subset \mathbb{R}^d$ have a Lipschitz boundary, $0 \leq \tau \leq \tau'$ and $1 \leq p < \infty$. Then is $W_p^{\tau'} (\Omega) \subseteq W_p^\tau (\Omega)$.
\end{proposition}

\begin{proof}
Let $0 < \sigma \leq \sigma' < 1$ and $n \in \mathbb{N}_{\geq 0}$. According to Theorem 14.2.3 in \cite{BS}, the fractional Sobolev space $W_p^{n+\sigma} (\Omega)$ can be equivalently constructed as an interpolation space between $W_p^{n} (\Omega)$ and $W_p^{n+1} (\Omega)$ with norm $\| \cdot \|_{[W_p^n (\Omega), W_p^{n + 1} (\Omega)]_{\sigma, p}}$. This construction yields the inclusions
$$W_p^{n + 1} (\Omega) \subseteq W_p^{n+\sigma'} (\Omega) \subseteq W_p^{ n+\sigma } (\Omega) \subseteq W_p^{ n} (\Omega).$$
By recursively applying this relationship, we obtain $W_p^{\tau'} (\Omega) \subseteq W_p^\tau (\Omega)$.
\end{proof}

The next proposition is a special case of Theorem 13 in \cite{BGW}. From it, one can follow that $g \in W_2^{\tau}(\Omega)$ if and only if there exists $\mathcal{E}g \in W_2^{\tau}(\mathbb{R}^d)$ such that $\left.\mathcal{E}g\right|_{\Omega} = \left.g\right|_{\Omega}$.

%-------------------------------------------------------------------------------------------

\begin{proposition} \label{3.1.7}
Let $\Omega \subset \mathbb{R}^d$ have a Lipschitz boundary and let $\tau \geq 0$. There exists an extension operator $\mathcal{E}: W_2^{\tau}(\Omega) \rightarrow W_2^{\tau}(\mathbb{R}^d)$ such that $\left.\mathcal{E}(f)\right|_{\Omega} \equiv \left.f\right|_{\Omega}$ and $\| \mathcal{E}(f) \|_{W_2^{\tau}(\mathbb{R}^d)} \leq  C \|f \|_{W_2^{\tau}(\Omega)}$ for all $f \in W_2^{\tau}(\Omega)$, where $C=C(\Omega, \tau)>0$ is a constant.
\end{proposition}

\begin{proof}
See note (2.3) in \cite{AL2} and keep in mind that a continuous linear operator is bounded.
\end{proof}

An alternative version of the case $p=2$ and $n=0$ in Proposition \ref{3.1.5} can be found in \cite{AL2} as Proposition 2.1 and 2.2. It states that already $\tau > d/2$ is sufficient to ensure embeddings for the Sobolev spaces. This statement explains why we later assume $\tau > d/2$ and define $\tau^*$.

\begin{proposition} \label{3.1.8}
Let $\Omega \subset \mathbb{R}^d$ have a Lipschitz boundary and $\tau > d/2$. Then,
\begin{enumerate}[topsep=5pt]
	\setlength\itemsep{0.1mm}
\item $W_2^\tau (\Omega) \subseteq C^{0}(\overline{\Omega}),$
\item $W_2^\tau (\Omega) \subseteq W_p^{\ell}(\Omega)$ for $2 \leq p$ and $\ell=0,1,...,\tau^*$ with $\tau^*$ defined as in section \ref{Chap. 3.3},
\end{enumerate}
where $C^{0}(\overline{\Omega})$\index{\textit{$C^{0}(\overline{\Omega})$,}} is the space of uniformly continuous and bounded functions on $\Omega$.
\end{proposition}

\begin{proof}
See Proposition 2.1 and 2.2 in \cite{AL2}.
\end{proof}

The last interesting result about Sobolev spaces is that for suitable domains $\Omega$, the RKHS associated to a Matérn kernel is a Sobolev space.

\begin{proposition} \label{3.1.9}
Let $\Omega \subset \mathbb{R}^d$ have a Lipschitz boundary and let $\tau = \nu+d/2$ for $\nu>0$. Then the RKHS $\mathcal{H}(\Omega)$ of the Matérn kernel $k_{\nu,\ell}$ is norm-equivalent\index{Norm-equivalent,} to the Sobolev space $W_2^\tau (\Omega)$. This means $\mathcal{H}(\Omega) = W_2^\tau (\Omega)$ and there exists constants $C_\ell, C_u >0$ such that
$$C_\ell \|g\|_{\mathcal{H} (\Omega)} \leq  \|g\|_{W_2^\tau (\Omega)} \leq C_u \|g\|_{\mathcal{H} (\Omega)} \ \ \text{for every } g \in \mathcal{H}(\Omega).$$
\end{proposition}

\begin{proof}
According to 7.62 in \cite{AF}, we have \index{\textit{$H^{\tau}(\mathbb{R}^d)$,}}
$$W_2^{\tau}(\mathbb{R}^d) = H^{\tau}(\mathbb{R}^d) = \{ g \in L^2(\mathbb{R}^d) \cap C(\mathbb{R}^d) \, : \, \hat{g}(\cdot) (1+ \| \cdot \|_2^2 )^{\tau/2} \in L^2(\mathbb{R}^d)  \},$$
where $\hat{g}$\index{\textit{$\hat{g}$,}} is the Fourier transform of $g$. By 14.x.6 and 14.x.7 in \cite{BS}, the norm $\| \cdot \|_{W_2^{\tau}(\mathbb{R}^d)}$ is equivalent to the norm
$$ \| g \|^2_{H^{\tau}(\mathbb{R}^d)} = \int_{\mathbb{R}^d} (1+ \|\boldsymbol{x} \|_2^2 )^{\tau} | \hat{g}(\boldsymbol{x}) |^2 \, d\boldsymbol{x}.$$
From equation (4.15) in \cite{RW} the Matérn kernel has the Fourier transform
$$\hat{k}_{\nu, \ell}( \|s \|_2 ) = C \left( 1 + \dfrac{4 \pi^2 \ell^2}{2 \nu} \|s\|_2^2 \right)^{-(\nu + d/2)},$$
where $C=C(\nu, d, \ell)>0$ is a constant. The statement follows now with Corollary 10.48 in \cite{W}. Notice that this Corollary assumes $\tau \in \mathbb{N}_{\geq 1}$, but by utilizing the extension operator from Proposition \ref{3.1.7} and using the equivalent norms, we can drop this assumption.
\end{proof}

%-------------------------------------------------------------------------------------------

Notice, that the norm $\|g\|_\mathcal{H}$ of the Matérn RKHS $\mathcal{H}$ captures the smoothness of $g$, as both norms are equivalent. More accurately, every function in $\mathcal{H}$ is weak differentiable up to order $\tau = \nu+d/2$ and the norm $\| \cdot \|_\mathcal{H}$ captures the weak derivatives up to order $\tau = \nu+d/2$.

\begin{remark} \label{3.1.10}
Up to this point, we have addressed different but equivalent norms for the Sobolev spaces $W_p^\tau (\Omega)$, where $\Omega$ has a Lipschitz boundary. For example:
\begin{enumerate}[label=(\roman*),topsep=5pt]
	\setlength\itemsep{0.1mm}
\item $\| \cdot \|_{W_p^\tau (\Omega)}$ for $p \in [1, \infty]$ and $\tau \geq 0$ from Definition \ref{3.1.2} and \ref{3.1.3}, \label{3.1.10 (i)}
\item $\| \cdot \|_{[W_p^m (\Omega), W_p^{m + 1} (\Omega)]_{\sigma, p}}$ for $p \in [1, \infty)$, $\tau \in (0, \infty) \backslash \mathbb{N}_{\geq 1}$ and $\tau = m + \sigma$, \label{3.1.10 (ii)}
\item $\| \cdot \|_{H^\tau (\Omega)}$ for $p=2$ and $\tau > d/2$. \label{3.1.10 (iii)}
\end{enumerate}
Since these norms are equivalent, the convergence results will be valid for all of them, differing only by constants $C=C(\Omega, \tau, p)>0$. In my thesis I verify the convergence with respect to the norm in \ref{3.1.10 (i)}. This clarification is necessary because the literature can be quite confusing. For example, Briol et al. initially state that fractional Sobolev spaces can be defined through interpolation, suggesting the use of the norm in \ref{3.1.10 (ii)}. They then define the norm in \ref{3.1.10 (iii)} and reference results in \cite{N}, which either use this norm or the one in \ref{3.1.10 (i)}. Additionally, they use results in \cite{AL}, formulated with respect to the norm in \ref{3.1.10 (i)}. Since they consistently use the same notation for all norms, it appears they overlooked necessary constants to transition between these norms. To avoid making additional assumptions, I address this issue by employing the following lemma.
\end{remark}

\begin{lemma} \label{3.1.11}
Let $\Omega \subset \mathbb{R}^d$ have a Lipschitz boundary and let $\tau > d/2$. For fractional Sobolev spaces, we defined a semi-norm in Definition \ref{3.1.3}. For integer Sobolev spaces define the semi-norm 
$$|g|_{W_2^\tau (\Omega)}^2 
= \sum_{| \boldsymbol{\beta} | = \tau} \|D^{\boldsymbol{\beta}} g\|_{L^2(\Omega)}^2.$$
There exists a constant $C(\Omega, \tau) >0$, which is monotone increasing in $\tau$, such that
$$| g |_{W_2^\tau (\Omega)} \leq C(\Omega, \tau) \| g \|_{H^\tau (\Omega)} \quad \text{ for all } g \in W_2^\tau (\Omega).$$
\end{lemma}

\begin{proof}
By solving 14.x.6 and 14.x.7 in \cite{BS} I found the constant
\[ \pushQED{\qed}
C(\Omega, \tau)^2 = \binom{\lfloor \tau \rfloor + d - 1}{d - 1} \max \left\{ \int_{\Omega} \frac{2 - 2 \cos(r_1)}{\|r\|_2^d} \, dr, 1 \right\}. \qedhere \]
\end{proof}

%-------------------------------------------------------------------------------------------
% Experimental setting
%-------------------------------------------------------------------------------------------

\subsection{Experimental setting}
The convergence results in this thesis are valid under specific conditions related to the observations, the kernel, and its hyperparameters. Here, we outline the experimental setting to formulate these results and define key values that will influence the convergence rate.

%-------------------------------------------------------------------------------------------
% Experimental design
%-------------------------------------------------------------------------------------------

\subsubsection{Experimental design}
Let $\Omega \subset \mathbb{R}^d$ be a non-empty set. The convergence rate will depend on the distribution of the points $X = \{ \boldsymbol{x}_1, ..., \boldsymbol{x}_N \} \subset \Omega$. Some references assume the points are sampled from a probability distribution, while others assume they can be freely chosen, as in the case of Bayesian Optimization. In this context, we make no assumptions about the distribution of the points. Instead, the convergence rate depends on the following key values that describe how the points are spaced.

\begin{defi} \label{3.2.1}
Let $\Omega \subset \mathbb{R}^d$ be bounded. For a collection of points $X \subset \Omega$, we define the fill distance\index{Fill distance,} $h_{X, \Omega}$\index{\textit{$h_{X, \Omega}$,}}, separation radius\index{Separation radius,} $q_{X}$\index{\textit{$q_X$,}} and mesh ratio\index{Mesh ratio,} $\rho_{X,\Omega}$\index{\textit{$\rho_{X,\Omega}$,}} as
$$h_{X, \Omega}  = \sup_{\boldsymbol{x} \in \Omega} \inf_{ \boldsymbol{x}_i \in X} \| \boldsymbol{x}-\boldsymbol{x}_i \|_2, \qquad
q_X = \dfrac{1}{2} \min_{i \neq j} \| \boldsymbol{x}_i-\boldsymbol{x}_j \|_2, \qquad
\rho_{X, \Omega} = \dfrac{h_{X, \Omega}}{q_X} \geq 1.$$
We say a sequence of points $( X_N )_{N \in \mathbb{N}} \subset \Omega$ is quasi-uniform, if there exists a constant $C=C(\Omega) \geq 1$ such that
$$\rho_{X_N, \Omega} \leq C$$
for all $N \in \mathbb{N}_{\geq 2}$.
\end{defi}

Quasi-uniform points are used to evenly represent the set $\Omega$ and to achieve optimal convergence rates for various approximation methods. Intuitively, quasi-uniform points form a grid that is a variation of a regular grid, where the step size in each dimension is given by $N^{-1/d}$.

\begin{defi} \label{3.2.2}
We write $C \lesssim D$\index{\textit{$C \lesssim D$,}}, if $C$ can be bounded by a multiple of $D$ and define the equivalence relation $C \sim D$\index{\textit{$C \sim D$,}}, if $C \lesssim D$ and $D \lesssim C$.
\end{defi}

\begin{lemma} \label{3.2.3} 
Let $\text{Vol}_d (\Omega) > 0$. For quasi-uniform points is $q_{X_N} \sim N^{-1/d}$ and $h_{X_N, \Omega} \sim N^{-1/d}$, where the constants hidden in $\sim$ depend only on $\Omega$.
\end{lemma}

\begin{proof}
By the definition of quasi-uniform points, there exists $C \geq 1$ such that
$$h_{X_N, \Omega} \leq C q_{X_N} \qquad \text{and} \qquad q_{X_N} \leq h_{X_N, \Omega}.$$
Therefore, $q_{X_N} \sim h_{X_N, \Omega}$. Denote by $B_{r}(\boldsymbol{x}_i) \subset \mathbb{R}^d$\index{\textit{$B_r(\boldsymbol{x}_i)$,}} the open ball of radius $r>0$ and center $\boldsymbol{x}_i \in X_N$. Then the set 
$$\bigcup_{i=1}^N B_{h_{X_N, \Omega}} (\boldsymbol{x}_i)$$
covers $\Omega$ and thus
$$\text{Vol}_d (\Omega) \leq N \cdot \text{Vol}_d \left(B_{h_{X_N, \Omega}} (\boldsymbol{x}_1)\right) \ \Longleftrightarrow \
(\text{Vol}_d (\Omega) / V_d)^{1/d} N^{-1/d} \leq h_{X_N, \Omega},$$
where $V_d$ denotes the volume of the $d$-dimensional unit ball. Define 
$$\Omega_0 = \{ \boldsymbol{x} + \boldsymbol{y} \ | \ \boldsymbol{x} \in \Omega, \| \boldsymbol{y} \|_2 < \text{diam}(\Omega) \},$$
where $\text{diam}(\Omega) = \sup \{ \| \boldsymbol{x}-\boldsymbol{y} \|_2 : \boldsymbol{x},\boldsymbol{y} \in \Omega \}$\index{\textit{$\text{diam}(\Omega)$,}} is the diameter of $\Omega$. It is
$$\bigcup_{i=1}^N B_{q_{X_N}}(\boldsymbol{x}_i) \subset \Omega_0$$
and thus
$$\text{Vol}_d (\Omega_0) \geq N \cdot \text{Vol}_d \left(B_{q_{X_N}} (\boldsymbol{x}_1)\right) \ \Longleftrightarrow \
(\text{Vol}_d (\Omega_0) / V_d)^{1/d} N^{-1/d} \geq q_{X_N} \geq C^{-1} h_{X_N, \Omega}.$$
This implies $h_{X_N, \Omega} \sim N^{-1/d}$, where the constants hidden in $\sim$ depend only on $\Omega$.
\end{proof}

\newcommand*{\bigtimes}{\mathop{\raisebox{-.4ex}{\hbox{\LARGE{$\times$}}}}}

%-------------------------------------------------------------------------------------------
% Gaussian Process
%-------------------------------------------------------------------------------------------

\subsubsection{Gaussian Process}
To estimate $f:\Omega \rightarrow \mathbb{R}$ given a set of observation $\mathcal{D}_N = \{ (\boldsymbol{x}_1, y_1), ..., (\boldsymbol{x}_N, y_N) \}$, we use a Gaussian Process $\mathcal{GP}(0,k)$ with a kernel that provides fine control over differentiability.

\begin{defi}
A kernel $k:\Omega \times \Omega \rightarrow \mathbb{R}$ is called $\tau$-smooth\index{Kernel is $\tau$-smooth,}, if its RKHS is norm-equivalent to the Sobolev space $W_2^\tau(\Omega) = H^\tau (\Omega)$.
\end{defi}

As established in Section \ref{Chap. 2.2}, the posterior mean function of $\mathcal{GP}(0,k)$ lies in the RKHS of $k$. Consequently, the parameter $\tau$ of an $\tau$-smooth kernel controls the assumed differentiability of $f$.

\begin{example} \label{3.2.5}
According to Proposition \ref{3.1.9}, the Matérn kernel $k_{\nu, \ell}$ is $\tau$-smooth, if $\tau=\nu+d/2$. Additionally, if $\nu = n + 1/2$ for some $n \in \mathbb{N}_{\geq 0}$, the kernel can be computed using the formula in Remark \ref{2.3.5}.
\end{example}

In Section \ref{Chap. 2.3.3}, we examined three different methods for selecting the hyperparameters of a Gaussian Process. As it is common to learn hyperparameters as more data points are observed, we want the convergence results to account for this. We assume that the possible parameters are given by a set $\Theta \subset \mathbb{R}^{d_\Theta}$\index{\textit{$\Theta$,}} and denote the posterior mean function by $m'(\theta_N)$\index{\textit{$m'(\theta_N)$,}} to emphasise its dependence on the parameter values $\theta_N \in \Theta$\index{\textit{$\theta_N$,}}. If the kernel $k(\theta_N)$\index{\textit{$k(\theta_N)$,}} is $\tau (\theta_N)$-smooth, we have constants $C_\ell (\theta_N), C_u(\theta_N)>0$ such that \index{\textit{$C_\ell (\theta_N)$,}} \index{\textit{$C_u(\theta_N)$,}}
$$ C_\ell (\theta_N) \| g \|_{\mathcal{H} (\Omega)} \leq \| g \|_{H^{\tau (\theta_N)} (\Omega)} \leq C_u(\theta_N) \| g \|_{\mathcal{H} (\Omega)} \quad \text{for every } g \in W_2^{\tau (\theta_N)} (\Omega),$$
where $\mathcal{H} (\Omega)$ is the RKHS of the kernel $k(\theta_N)$. To formulate the convergence results, the following extreme values are relevant:
\begin{itemize}[topsep=5pt]
	\setlength\itemsep{0.1mm}
\item $C_n = \sup_{N \geq n} C_u(\theta_N) C_\ell(\theta_N)^{-1}$,
\item $\tau_n^- = \inf_{N \geq n} \tau (\theta_N)$, \index{\textit{$\tau_n^-$,}}
\item $\tau_n^+ = \sup_{N \geq n} \tau (\theta_N)$. \index{\textit{$\tau_n^+$,}}
\end{itemize}
They represent the critical values of the norm equivalence constants and the smoothness after the $n$-th data point is observed. We denote the set of extreme values by\index{\textit{$\Theta_n^*$,}}
$$\Theta_n^* = \{ C_n, \tau_n^-, \tau_n^+ \}.$$
To bound the extreme values, we need to ensure that the parameter selection methods used are appropriately chosen, regardless of the observed data.

%-------------------------------------------------------------------------------------------
% Convergence guarantees
%-------------------------------------------------------------------------------------------

\subsection{Convergence guarantees} \label{Chap. 3.3} 
To estimate $f:\Omega \rightarrow \mathbb{R}$ given a set of observation $\mathcal{D}_N = \{ (\boldsymbol{x}_1, y_1), ..., (\boldsymbol{x}_N, y_N) \}$, we use a Gaussian Process $\mathcal{GP}(0,k)$. 

\begin{defi}
Let $p \in [1, \infty]$ and $s \in [0, \infty)$. Suppose $f,m' \in W_p^s(\Omega)$, where $m'$ is posterior mean function of $\mathcal{GP}(0,k)$. We say that the Gaussian Process mean converges to $f$ with respect to the norm $\| \cdot \|_{W_p^s(\Omega)}$, if
$$\|f - m'(\theta_N)\|_{W_p^s (\Omega)} \rightarrow 0 \quad \text{as } N \rightarrow \infty.$$
Similarly, the convergence of $m'$ to $f$ in expectation is defined.
\end{defi}

%-------------------------------------------------------------------------------------------

While we assume a zero mean function in $\mathcal{GP}(0,k)$, it is important to note that the results in this section also apply to a non-zero mean function. However, it has not been determined whether the convergence rate is slower or faster with a non-zero mean function, as the triangle inequality is used to simplify the analysis to the zero mean case. The following convergence results depend on the smoothness of the approximating function, the smoothness of the true function, and the geometric properties of the design points. They address the combination of corrupted data, misspecified smoothness and misspecified likelihood. We describe by $X_N = \{ \boldsymbol{x}_1, ...,  \boldsymbol{x}_N \} \subset \Omega$ a finite set of $N$ elements and use the following notations:
\index{\textit{$(x)_+$,}} 
\index{\textit{$\tau_0$,}} 
\index{\textit{$\lceil x \rceil$,}}
\index{\textit{$\tau^*$,}}
$$(x)_+ = \max \{ x,0 \}, \quad \tau_0 = \tau - d (1/2 - 1/p)_+, \quad \lceil x \rceil = \min \{ m \in \mathbb{Z}: m \geq x \},$$
$$\tau^* = 
\begin{cases}
  \tau_0,  & \tau \in \mathbb{N}_{\geq 1}, \tau_0 \in \mathbb{N}_{\geq 0} \text{ and } 2 < p < \infty \\
  \tau_0, & \tau \in \mathbb{N}_{\geq 1} \text{ and } p=2 \\
  \lceil \tau_0 \rceil - 1, & \text{else}
\end{cases}.
$$

%-------------------------------------------------------------------------------------------
% Interpolation
%-------------------------------------------------------------------------------------------

\subsubsection{Interpolation}
In this subsection, we assume the observations are noiseless, meaning $y_i = f(\boldsymbol{x}_i) + \varepsilon_i$ with $\varepsilon_i = 0$ for all observations. Consequently, we estimate $f(\boldsymbol{x})$ by the posterior Gaussian Process mean
$$m'(\boldsymbol{x}) = \boldsymbol{k}(\boldsymbol{x}) \boldsymbol{K}^{-1} \boldsymbol{f} \quad \text{for } \boldsymbol{x} \in \Omega,$$
as defined in equation (\ref{eq. 2.1.2}). For a given $n \in \mathbb{N}_{\geq 2}$ we make the following assumptions:
\begin{enumerate}
\item $\Omega$ is an $\mathcal{L}(R, \delta)$-domain. \label{Assumption 1}
\item The kernels $k(\theta_N)$ are $\tau(\theta_N)$-smooth for all $N \geq n$ and the extreme values in $\Theta_n^*$ are finite with $\tau_n^- > d/2$. \label{Assumption 2}
\item The set $\{ \tau (\theta_N) \}_{N \geq n}$ has finitely many values.  \label{Assumption 3}
\item $f \in W_2^{\tau_f}(\Omega)$ for some $\tau_f > d/2$. \index{\textit{$\tau_f$,}} \label{Assumption 4}
\item For all $N \geq n$ is $\tau(\theta_N) \in (d/2, \tau_f] \cup [ \lceil \tau_f \rceil, \infty)$. \label{Assumption 5}
\end{enumerate}

The first assumption is usually satisfied, as hyperrectangles like $\Omega = (0,1)^d$ are commonly considered. The fourth assumption is reasonable given the broad nature of Sobolev spaces. Because we suppose $\tau_f > d/2$, the convergence results will be more reliable in lower dimensions. The second, third and last assumptions depend on the permissible set of parameters and the choice of kernels.

\begin{example}
Let $M \subset \mathbb{N}_{\geq 0}$ be a finite set. If we define
$$\{ \tau (\theta_N) \}_{N \geq n} = \{ m + 1/2 + d/2 \}_{m \in M},$$
then assumptions \ref{Assumption 2} and \ref{Assumption 3} are satisfied for the corresponding Matérn kernels $k_{m+1/2, \ell_m}$. The $\tau(\theta_N)$-smoothness was discussed in Example \ref{3.2.5} and Lemma 3.4 of \cite{TE} shows
$$C_n \leq \max_{m \in M} \max \{ \ell_m, \ell_m^{-1} \}.$$
Here, we can apply the computation formula in Remark \ref{2.3.5}, which makes the Matérn kernels particularly appealing for small values of $m$. If $d\in \mathbb{N}_{\geq 1}$ is odd, then assumptions \ref{Assumption 5} is satisfied. For even $d\in \mathbb{N}_{\geq 1}$, assumption \ref{Assumption 5} may not hold. Nevertheless, due to the computational cost of evaluating the modified Bessel function for $k_{m+1, \ell_m}$, we may prefer the Matérn kernels $k_{m+1/2, \ell_m}$, even if the fifth assumption is not satisfied.
\end{example}

%-------------------------------------------------------------------------------------------

The following theorem describes the convergence rate of the posterior Gaussian Process mean $m'(\theta_N)$ to $f$. It distinguishes between the two cases, when the smoothness is well-specified ($\tau_f \geq \tau_n^+$) and when the smoothness is misspecified ($\tau_f < \tau_n^+$).

\begin{theorem} \label{3.3.2}
Let assumptions \ref{Assumption 1}-\ref{Assumption 4} hold, $p \in [1, \infty]$ and $s \in [0, \min\{ \tau_f, \tau_n^- \}^* ] $. There exist $C, h_0 >0$ such that for all $N \geq n$ and $X_N \subset \Omega$ with $h_{X_N, \Omega} \leq h_0$ is
\begin{itemize}[topsep=5pt]
	\setlength\itemsep{0.1mm}
\item $\tau_f \geq \tau_n^+ \Longrightarrow 
\|f - m'(\theta_N)\|_{W_p^s (\Omega)} \leq C h_{X_N, \Omega}^{\tau_n^- - s - d \left( \frac{1}{2} - \frac{1}{p} \right)_+} \| f \|_{W_2^{\tau_f}(\Omega)},$
\item $\tau_f < \tau_n^+ \Longrightarrow 
\|f - m'(\theta_N)\|_{W_p^s (\Omega)} \leq C h_{X_N, \Omega}^{\min \{\tau_f, \tau_n^- \} - s - d \left( \frac{1}{2} - \frac{1}{p} \right)_+} \rho_{X_N, \Omega}^{\tau_n^+ - \tau_f} \| f \|_{W_2^{\tau_f}(\Omega)},$
\end{itemize}
where $C(\Omega, \tau_f, p, s, \Theta_n^*)$ and $h_0 (R, \delta, d, \tau_f, \Theta_n^*)$ are constants.
\end{theorem}

The results hold only if the points in $X_N$ sufficiently cover $\Omega$, measured by $h_0$. If we set $s=0$, we recover results with respect to the norm $\| \cdot \|_{L^p(\Omega)}$. Additionally, we derive the following:
\begin{itemize}[topsep=5pt]
	\setlength\itemsep{0.1mm}
\item The convergence rate is influenced by the experimental design. The optimal experimental design is achieved if $\rho_{X_N, \Omega}$ is bounded and $h_{X_N, \Omega} = \mathcal{O} (N^{-1/d})$. This occurs, for example, with a sequence of quasi-uniform points.
\item Parameters other than kernel smoothness affect only the constants in the convergence rate. The larger the smoothness $\tau_f$ of the function $f$, the faster the convergence rate can be. If the smoothness is well-specified ($\tau_f \geq \tau_n^+$), the extreme value $\tau_n^-$ limits the convergence rate. If the smoothness is misspecified ($\tau_f < \tau_n^+$), the term $\rho_{X_N, \Omega}^{\tau_n^+ - \tau_f}$ penalizes the overestimation of the smoothness. By using a sequence of quasi-uniform points, we can bound $\rho_{X_N, \Omega}$ by a constant.
\item The optimal convergence rate is $\mathcal{O} (N^{-\tau_f/d})$ with respect to the norm $\| \cdot \|_{L^2(\Omega)} $. This rate is achieved when the smoothness is correctly estimated as $\tau_f = \tau_n^+ = \tau_n^-$ and the points are selected such that $h_{X_N, \Omega} = \mathcal{O} (N^{-1/d})$. The use of quasi-uniform points helps to mitigate the effects of smoothness misspecification.
\end{itemize}

\begin{remark}
Briol et al. also investigate the scenario where the likelihood function is misspecified. This means the assumption $\varepsilon_i=0$ is incorrect, but we still use the interpolation setting. The posterior mean function becomes
$$m'(\boldsymbol{x}) = \boldsymbol{k}(\boldsymbol{x}) \boldsymbol{K}^{-1} (\boldsymbol{f} + \boldsymbol{\varepsilon} ) \quad \text{for } \boldsymbol{x} \in \Omega,$$
where $\boldsymbol{\varepsilon} = [\varepsilon_1, ..., \varepsilon_N]^T$. We allow $\varepsilon_i$ to be deterministic or to follow a probability distribution. It is crucial to address the issue of a misspecified likelihood function because, in most applications, observations are affected by measurement error. Nevertheless, given that this measurement error is usually minimal, we employ an interpolation setting. Assume now that assumptions \ref{Assumption 1}-\ref{Assumption 5} hold. Examining the expected norm $\mathbb{E}[\|f - m'(\theta_N)\|_{W_p^s (\Omega)}]$ in Theorem \ref{3.3.2} enables us to adjust the convergence rate by adding the term
$$C_0 h_{X_N, \Omega}^{d/\max \{ 2,p \} -s} \rho_{X_N, \Omega}^{\tau_n^+ - d/2} \mathbb{E}[\boldsymbol{ \| \varepsilon} \|_2 ]$$
to both upper bounds, where $C_0(\Omega, \tau_f, p, s, \Theta_n^*)$ is a positive constant. Verification of this statement can be found in Theorem 7 of \cite{BGW}. We derive the following conclusion:
\begin{itemize}[topsep=5pt]
	\setlength\itemsep{0.1mm}
\item If the likelihood function is misspecified, the convergence rate is penalized by a term that depends on $\mathbb{E}[\boldsymbol{ \| \varepsilon} \|_2 ]$ and on the experimental setting. Keeping the extreme value $\tau_n^+$ small and using quasi-uniform points helps to mitigate the effects of likelihood misspecification.
\end{itemize}
\end{remark}

%-------------------------------------------------------------------------------------------

We conclude the interpolation setting with an example to assess the accuracy of the upper bound.

\begin{example} \label{3.3.4}
Consider $N$ equispaced points in $\Omega = (-1,1)^3$ with a step size of $N^{-1/3}$ in each dimension, the function
$$f:(-1,1)^3 \rightarrow \mathbb{R}, f(\boldsymbol{x})= |x_1|^{3.5} + x_2^2 - |x_3| $$
and an interpolation setting given by the data $\mathcal{D}_N = \{ (\boldsymbol{x}_1, f(\boldsymbol{x}_1)), ..., (\boldsymbol{x}_N, f(\boldsymbol{x}_N)) \}$. Using the Matérn kernel $k_{1/2,1}$, we approximate $f \in W^2_2 ( \Omega )$ by the posterior mean function $m'$ of the Gaussian Process $\mathcal{GP}(0,k_{1/2,1})$. With the smoothness parameter $\tau_f = 2 = \tau_n^+$ being well-specified, Theorem \ref{3.3.2} allows to bound the relative error by
$$\dfrac{\|f - m'\|_{L^2 (\Omega)}}{\|f\|_{W_2^2 (\Omega)}} = \mathcal{O} (N^{-2/3}) \quad \text{as } N\rightarrow \infty.$$
We aim to evaluate this convergence rate. The weak derivatives of $f$ and their norms can be easily evaluated. To approximate the $L^2$ norm, we use $M=10,000$ equispaced points with a step size of $M^{-1/3}$ and the Trapezoidal Rule. Specifically, we compute
\begin{align*}
\dfrac{1}{M} \sum_{i=1}^{M^{1/3}} \sum_{j=1}^{M^{1/3}} \sum_{\ell=1}^{M^{1/3}} \omega_{ij \ell} (f(z_i, z_j , z_\ell) - m'(z_i, z_j , z_\ell))^2
&\approx \int_{(-1,1)^3} (f(\boldsymbol{z})-m'(\boldsymbol{z}))^2 \, d\boldsymbol{z} \\
&= \|f-m'\|_{L^2 (\Omega)}^2,
\end{align*}
where $\omega_{ij \ell} \in \{ 1, 1/2, 1/4, 1/8 \}$ are weights assigned to interior, face, edge and corner points. Figure \ref{fig. 9} illustrates the calculated relative error compared to the functions $g(N)=N^{-2/3}$ and $h(N)=3 N^{-1}$. The curves suggest that the estimated convergence rate of $\mathcal{O} (N^{-2/3})$ is reliable, with the actual convergence rate appearing slightly faster, particularly at the initial stages.

\begin{figure}[h]
	\centering
	\includegraphics[width=0.5\textwidth]{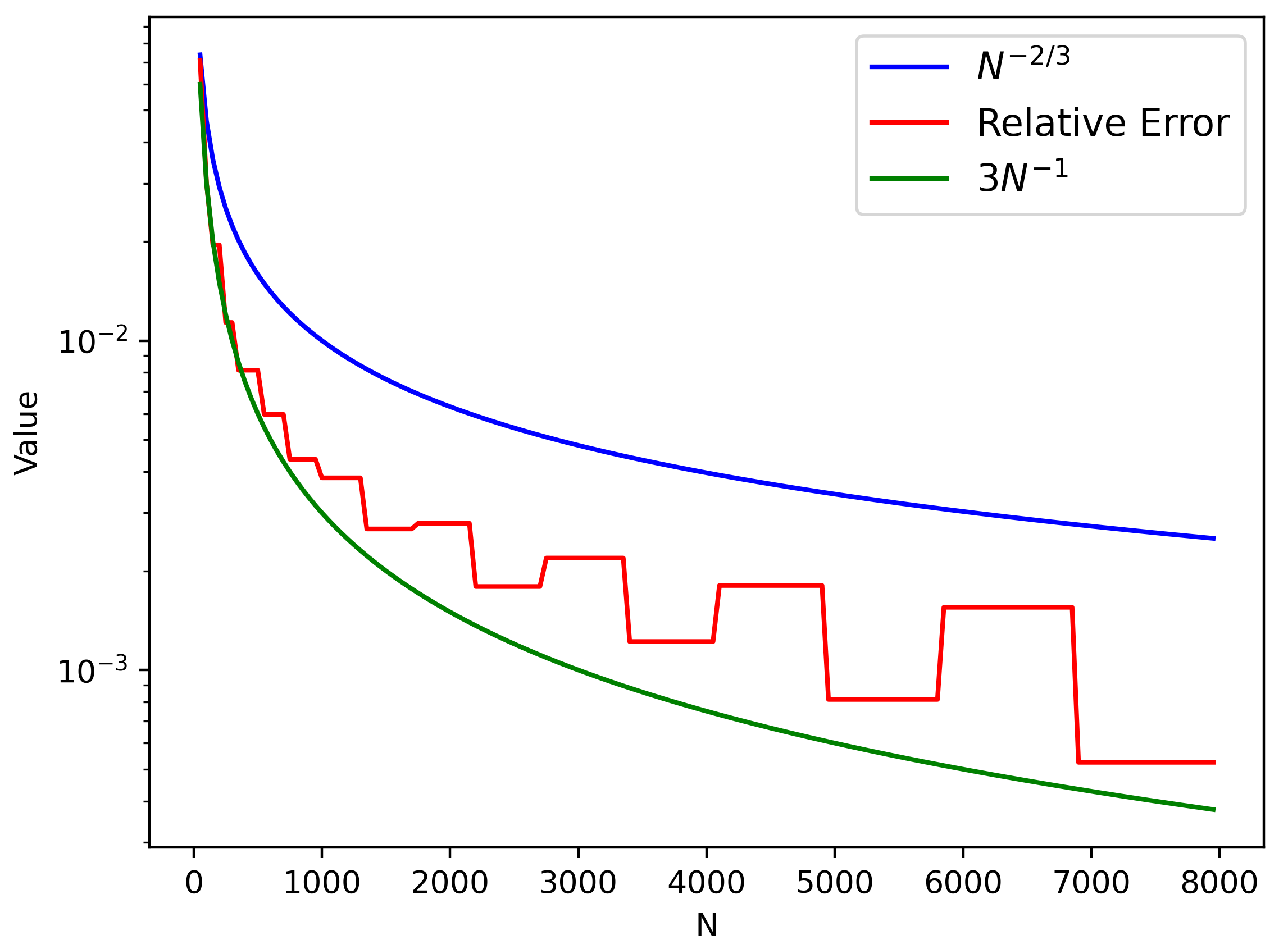}
	\caption{The blue curve represents $g(N)= N^{-2/3}$, the green curve represents $h(N) =3 N^{-1}$ and the red curve describes the relative error of $m'$, computed for $N \in \{ 50,100,150,..., 7950, 8000 \}$.} \label{fig. 9}
\end{figure}
\end{example}

The remainder of this subsection proves Theorem \ref{3.3.2}, as described in \cite{BGW}. This proof builds upon the work in \cite{AL} and \cite{N}, but is too lengthy and digressive to present here in full. While Briol et al. emphasize the theorems used, my focus is on the proof itself. Interested readers can explore the referenced theorems for more details.

%-------------------------------------------------------------------------------------------

\begin{proof}[Proof of Theorem \ref{3.3.2}]
By applying Remark 3.4/Theorem 3.2 in \cite{AL} to the functions $f - m'(\theta_N)$, we obtain $C_1,h_1>0$ such that for all $\tau(\theta_N)$ with $N \geq n$, $s \in [0, \min \{\tau_f, \tau_n^- \}^*]$ and $X_N \subset \Omega$ with $h_{X_N,\Omega} \leq h_1$ is
$$\|f - m'(\theta_N)\|_{W_p^s (\Omega)} 
\leq C_1 h_{X_N, \Omega}^{\min \{\tau_f, \tau(\theta_N) \} - s - d \left( \frac{1}{2} - \frac{1}{p} \right)_+} | f - m'(\theta_N) |_{W_2^{\min \{\tau_f, \tau(\theta_N) \}}(\Omega)},$$
where $C_1(\Omega, \tau_f, p, s, \Theta_n^*)$ and $h_1(R, \delta, d, \tau_f, \Theta_n^*)$ are constants. To find an upper bound for the last term, we distinguish between $\tau_f \geq \tau(\theta_N)$ and $\tau_f < \tau(\theta_N)$. \\

\textbf{Step 1:} Let $\tau_f \geq \tau(\theta_N)$.\\[5mm]
Using Lemma \ref{3.1.11}, we have
$$| f - m'(\theta_N) |_{W_2^{\min \{\tau_f, \tau(\theta_N) \}}(\Omega)}
= | f - m'(\theta_N) |_{W_2^{ \tau(\theta_N) }(\Omega)}
\leq C_2 \| f - m'(\theta_N) \|_{H^{\tau(\theta_N)}(\Omega)},$$
where $C_2(\Omega, \tau_f)>0$ is a constant. Using the norm-equivalence of the RKHS, we get
\begin{align*}
\| f - m'(\theta_N) \|_{H^{\tau(\theta_N)}(\Omega)}
&\leq C_u(\theta_N) \| f - m'(\theta_N) \|_{\mathcal{H}_{ k (\tau(\theta_N)) }(\Omega)}  \\
&\leq C_u(\theta_N) \| f \|_{\mathcal{H}_{ k (\tau(\theta_N)) }(\Omega)} \\
&\leq C_u(\theta_N) C_\ell(\theta_N)^{-1} \| f \|_{H^{ \tau(\theta_N) }(\Omega)} \\
&\leq C_n \| f \|_{H^{ \tau(\theta_N) }(\Omega)}
\leq C_n \| f \|_{H^{ \tau_f }(\Omega)} \leq C_3 \| f \|_{W_2^{ \tau_f }(\Omega)},
\end{align*}
where Corollary 10.25 in \cite{W} is used for the second inequality and $C_3(\Omega, \tau_f, \Theta_n^*)>0$ is a constant. Thus, we obtain
\begin{equation}
\|f - m'(\theta_N)\|_{W_p^s (\Omega)} \leq C_1 C_2 C_3 h_{X_N, \Omega}^{\tau(\theta_N) - s - d \left( \frac{1}{2} - \frac{1}{p} \right)_+} \| f \|_{W_2^{\tau_f}(\Omega)}. \label{eq. 3.3.1}
\end{equation}

%-------------------------------------------------------------------------------------------

\textbf{Step 2:} Let $\tau_f < \tau(\theta_N)$.\\[5mm]
Using Proposition \ref{3.1.7}, we can extend $f$ to a function $\mathcal{E}f \in W_2^{\tau_f}(\mathbb{R}^d)$ satisfying $\left.\mathcal{E}f\right|_\Omega \equiv \left.f\right|_\Omega$ and
$$\| \mathcal{E}f \|_{W_2^{\tau_f}(\mathbb{R}^d)} \leq C_4 \| f \|_{W_2^{\tau_f}(\Omega)},$$
where $C_4(\Omega, \tau_f) >0$ is a constant. According to the definition of band-limited functions on page 178 in \cite{N} and Theorem 3.4 in \cite{N}, there exists $\mathcal{E}f_{\sigma} \in W_2^{\tau(\theta_N)} (\mathbb{R}^d)$ such that 
$$\left.\mathcal{E}f_{\sigma}\right|_{X_N} \equiv \left.\mathcal{E}f\right|_{X_N} \equiv \left.f\right|_{X_N}, \quad 
\|\mathcal{E}f - \mathcal{E}f_{\sigma}\|_{W_2^{\tau_f} (\mathbb{R}^d)} 
\leq C_5  \|\mathcal{E}f\|_{W_2^{\tau_f} (\mathbb{R}^d)} \leq C_5 C_4 \| f \|_{W_2^{\tau_f}(\Omega)}$$
and
$$\| \mathcal{E}f_{\sigma} \|_{W_2^{\tau(\theta_N)} (\mathbb{R}^d)} 
\leq C_6 q_{X_N}^{\tau_f - \tau(\theta_N)} \| \mathcal{E} f_{\sigma} \|_{W_2^{\tau_f} (\mathbb{R}^d)},$$
where $C_5(\Omega, \tau_f)>0$ and $C_6(\Omega, \tau_f, \Theta_n^*)>0$ are constants. Since $\left.\mathcal{E}f_{\sigma}\right|_{X_N} \equiv \left.f\right|_{X_N}$, the prediction $m'(\theta_N)$ for $f$ is also the prediction for $\mathcal{E}f_{\sigma}$ and we can use equation (\ref{eq. 3.3.1}) to obtain
\begin{align*}
\| \mathcal{E}f_{\sigma} - m'(\theta_N) \|_{W_2^{ \tau_f }(\Omega)} 
&\leq C_1 C_2 C_3 h_{X_N, \Omega}^{\tau(\theta_N) - \tau_f} \| \mathcal{E}f_{\sigma} \|_{W_2^{\tau(\theta_N)}(\Omega)} \\
&\leq C_1 C_2 C_3 h_{X_N, \Omega}^{\tau(\theta_N) - \tau_f} \| \mathcal{E}f_{\sigma} \|_{W_2^{\tau(\theta_N)}(\mathbb{R}^d)} \\
&\leq C_1 C_2 C_3 C_6 h_{X_N, \Omega}^{\tau(\theta_N) - \tau_f} q_{X_N}^{\tau_f - \tau(\theta_N)} \| \mathcal{E}f_{\sigma} \|_{W_2^{\tau_f}(\mathbb{R}^d)} \\
&= C_1 C_2 C_3 C_6 \rho_{X_N, \Omega}^{\tau(\theta_N) - \tau_f} \| \mathcal{E}f_{\sigma} - \mathcal{E}f + \mathcal{E}f \|_{W_2^{\tau_f}(\mathbb{R}^d)} \\
&\leq C_1 C_2 C_3 C_6 C_4 (C_5+1) \rho_{X_N, \Omega}^{\tau(\theta_N) - \tau_f} \| f \|_{W_2^{\tau_f}(\Omega)}.
\end{align*}
Using the established inequalities, it follows
\begin{align*}
&\| f - m'(\theta_N) \|_{W_2^{\min \{\tau_f, \tau(\theta_N) \}}(\Omega)} \\
= &\|f - m'(\theta_N)\|_{W_2^{\tau_f} (\Omega)} \\
= &\|f - \mathcal{E}f_{\sigma} + \mathcal{E}f_{\sigma} - m'(\theta_N)\|_{W_2^{\tau_f} (\Omega)} \\
\leq &\|f - \mathcal{E}f_{\sigma} \|_{W_2^{\tau_f} (\Omega)} + \| \mathcal{E}f_{\sigma} - m'(\theta_N)\|_{W_2^{\tau_f} (\Omega)} \\
=  &\| \mathcal{E}f - \mathcal{E}f_{\sigma} \|_{W_2^{\tau_f}(\Omega)} + \| \mathcal{E}f_{\sigma} - m'(\theta_N)\|_{W_2^{\tau_f} (\Omega)} \\
\leq  &C_5 C_4 \|f\|_{W_2^{\tau_f}(\Omega)} + C_1 C_2 C_3 C_6 C_4 (C_5+1) \rho_{X_N, \Omega}^{\tau(\theta_N) - \tau_f} \| f \|_{W_2^{\tau_f}(\Omega)} \\
\leq  &C_7 \rho_{X_N, \Omega}^{\tau(\theta_N) - \tau_f} \|f\|_{W_2^{\tau_f}(\Omega)},
\end{align*}
where $C_7(\Omega, \tau_f, p, s, \Theta_n^*)>0$ is a constant and we used $\rho_{X_N, \Omega} \geq 1$ for the last inequality. Thus, we have
\begin{equation}
\|f - m'(\theta_N)\|_{W_p^s (\Omega)} \leq C_1 C_7 h_{X_N, \Omega}^{\tau_f - s - d \left( \frac{1}{2} - \frac{1}{p} \right)_+} \rho_{X_N, \Omega}^{\tau_n^+ - \tau_f} \| f \|_{W_2^{\tau_f}(\Omega)}. \label{eq. 3.3.2}
\end{equation}

\textbf{Step 3:} Let $\tau_f$ be arbitrary.\\[5mm]
The bounds (\ref{eq. 3.3.1}) and (\ref{eq. 3.3.2}) imply that for $h_{X_N, \Omega} \leq \min \{h_1,  h_1/ |h_1| \}$ is
$$\|f - m'(\theta_N)\|_{W_p^s (\Omega)} \leq C_1 \max \{C_2 C_3, C_7 \} h_{X_N, \Omega}^{\min \{\tau_f, \tau_n^- \} - s - d \left( \frac{1}{2} - \frac{1}{p} \right)_+} \rho_{X_N, \Omega}^{(\tau_n^+ - \tau_f)_+} \| f \|_{W_2^{\tau_f}(\Omega)},$$
where we use $\rho_{X_N, \Omega} \geq 1$ and $h_{X_N, \Omega} \leq 1$. Distinguishing between the cases of smoothness verifies the theorem.
\end{proof}

%-------------------------------------------------------------------------------------------
% Regression
%-------------------------------------------------------------------------------------------

\subsubsection{Regression}
Here we handle the case, where the observations have independently and identically distributed  Gaussian noise, meaning $y_i = f(\boldsymbol{x}_i) + \varepsilon_i$ with $\varepsilon_i \sim \mathcal{N}(0,\sigma^2)$ for all observations. Consequently, we estimate $f(\boldsymbol{x})$ by the posterior Gaussian Process mean
$$m'(\boldsymbol{x}) = \boldsymbol{k}(\boldsymbol{x}) (\boldsymbol{K} + \sigma^2 \boldsymbol{I}_N)^{-1} \boldsymbol{y} \quad \text{for } \boldsymbol{x} \in \Omega,$$
as defined in equation (\ref{eq. 2.1.2}). Because $\varepsilon_i$ is a sample from the normal distribution $\mathcal{N}(0,\sigma^2)$, we have a probability measure $\mathcal{P}$ that allows us to determine the likelihood of $\varepsilon_i$ falling within a specific range. For example, the probability of $|\varepsilon_i| \leq 1.96 \sigma$ is given by
$$\mathcal{P}(|\varepsilon_i| \leq 1.96 \sigma) = \mathcal{P}(-1.96 \sigma \leq \varepsilon_i \leq 1.96 \sigma) \approx 95 \%.$$ 
This is because the interval $[-1.96 \sigma, 1.96 \sigma]$ is known as the $95 \% $ confidence interval for a centered normal distribution. For general intervals $I = [a,b] \subset \mathbb{R}$, we can numerically calculate the probability using the density function:
$$\mathcal{P} ( a \leq \varepsilon_i  \leq b) = \dfrac{1}{\sqrt{2 \pi \sigma^2}} \int_a^b e^{-\frac{x^2}{2 \sigma^2}} \, dx.$$
Similarly, assuming $f \sim \mathcal{GP}(0,k)$ implies that $f$ is a sample from the Gaussian Process. The realizations of this Gaussian Process are not real values, instead they are functions $g:\Omega \rightarrow \mathbb{R}$. We have a probability measure $\Pi_k$\index{\textit{$\Pi_k$,}} that determines the likelihood of $f$ falling within a suitable subset of the sample space. For example, if $\Omega = \{ \boldsymbol{x}_1, ..., \boldsymbol{x}_m \}$ has a finite number of elements, then the function values of $f$ are multivariate normal distributed. Therefore,
\begin{align*}
\Pi_k (\| f \|_\infty \leq c) 
&= \Pi_k (f(\boldsymbol{x}_1),...,f(\boldsymbol{x}_m) \in [-c,c]) \\
&= \dfrac{1}{(2\pi)^{m/2} | \boldsymbol{K} |^{1/2} } \int_{-c}^c \dots \int_{-c}^c e^{-\frac{1}{2} \boldsymbol{z}^T \boldsymbol{K}^{-1} \boldsymbol{z} } \, dz_1 \cdots dz_m
\end{align*}
with $\boldsymbol{K} = [k(\boldsymbol{x}_i, \boldsymbol{x}_j)]_{1 \leq i,j \leq m}$. In general, $\Omega$ has an infinite amount of elements, and $\Pi_k$ is the product measure induced by measures of multivariate normal distributions. For this reason, it is difficult to determine the likelihood of $f$ falling within a specific range when an infinite number of function values is involved. For instance, the probability 
$$\Pi_k(\| f - m' \|_{L^\infty(\Omega)} \leq c)$$ 
is of interest, but not easy to determine. For a given $n \in \mathbb{N}_{\geq 2}$ we formulate the following assumptions:
\begin{enumerate}[topsep=-2pt]
	\setcounter{enumi}{5}
\item Let $\Pi_{k(\theta_N)}$ denote the probability measure associated with $\mathcal{GP}(0,k(\theta_N))$. There exist $c, a_n >0$ such that $\Pi_{k(\theta_N)} (\| f \|_{L^\infty (\overline{\Omega}) } \leq c) \leq e^{-a_n}$ for all $N \geq n$. \label{Assumption 6}
\item Let $\tau_f = \ell + s$ with $n \in \mathbb{N}_{\geq 0}$ and $s \in (0,1]$. Then $f$ has an extension $f^\circ \in C^{\ell, s}(\mathbb{R}^d) \cap W_2^{\tau_f} (\mathbb{R}^d)$, where $C^{\ell, s}(\mathbb{R}^d)$ is the space of Hölder continuous functions. \label{Assumption 7}
\end{enumerate}

%-------------------------------------------------------------------------------------------

Assumption \ref{Assumption 6} describes a small ball property, ensuring that $f$ cannot be uniformly small with exceedingly high probability, a condition sometimes tacitly assumed in the literature. Assumption \ref{Assumption 7} requires that $f$ has sufficient regularity. Together, these assumptions enable us to establish an upper bound for the term
$$\mathbb{E} \big[\| \left.(f - m'(\theta_N))\right|_{X_N} \|_2 \big],$$
which equals zero in the interpolation setting. 

\begin{example} \label{3.3.5}
Consider the Matérn kernel $k_{\nu_N, \ell_N}$. It is common practice to introduce a parameter $s_N^2>0$ as a scaling factor, leading to the Matérn kernel $k(\theta_N) = s_N^2 k_{\nu_N, \ell_N}$ for use in a Gaussian Process. Assume there exists a constant $s_{\text{min}}^2 >0$ such that $s_N^2 \geq s_{\text{min}}^2$ for all $N \geq n$. Given that $f$ is continuous at $\boldsymbol{x}_1 \in \Omega$, we have
\begin{align*}
\Pi_{k(\theta_N)} (\| f \|_{L^\infty (\overline{\Omega}) } \leq 1.96 s_{\text{min}}) 
& \leq \Pi_{k(\theta_N)} (| f(\boldsymbol{x}_1) | \leq 1.96 s_{\text{min}}) \\
& \leq \Pi_{k(\theta_N)} (| f(\boldsymbol{x}_1) | \leq 1.96 s_N) \\
&\approx 0.95 < e^{-0.001}.
\end{align*}
This implies Assumption \ref{Assumption 6} is satisfied if we can ensure that the signal variance $s^2$ is bounded away from zero. Therefore, when using parameter selection methods such as maximum likelihood, it is advisable to enforce a lower bound on $s^2$.
\end{example}

For completeness and a better understanding of Assumption \ref{Assumption 7}, the definition of the Hölder space is provided here.

\begin{defi}
Let $\Omega \subseteq \mathbb{R}^d$ be an open set, $\ell \in \mathbb{N}_{\geq 0}$ and $s \in (0,1]$. Then the Hölder space\index{Hölder space,} $C^{\ell, s}(\Omega)$\index{\textit{$C^{\ell, s}(\Omega)$,}} consists of functions $g \in C^\ell(\Omega)$ that satisfy
$$\| g \|_{C^{\ell, s}(\Omega)} 
= \| D^{\boldsymbol{\beta}} g \|_{C^\ell (\Omega)} + \max_{0 \leq |\boldsymbol{\beta}| \leq \ell} \sup_{\boldsymbol{x}, \boldsymbol{y} \in \Omega, \boldsymbol{x} \neq \boldsymbol{y}} \dfrac{| D^{\boldsymbol{\beta}} g (\boldsymbol{x}) - D^{\boldsymbol{\beta}} g (\boldsymbol{y}) |}{ \|\boldsymbol{x} - \boldsymbol{y} \|_2^s}
< \infty,$$
where
$$\| D^{\boldsymbol{\beta}} g \|_{C^\ell (\Omega)} = \max_{0 \leq |\boldsymbol{\beta}| \leq \ell} \| D^{\boldsymbol{\beta}} g \|_{\infty}.$$
\end{defi}

The following theorem describes the convergence rate of the posterior Gaussian Process mean $m'(\theta_N)$ to $f$. We distinguish between the two cases $\tau_f \geq \tau_n^+$ and $\tau_f < \tau_n^+$.

\begin{theorem} \label{3.3.7}
Let assumptions \ref{Assumption 1}-\ref{Assumption 7} hold, $p \in [1, \infty]$ and $s \in [0, \min\{ \tau_f, \tau_n^- \}^* ] $. There exist $C, h_0 >0$ such that for all $N \geq n$ and $X_N \subset \Omega$ with $h_{X_N, \Omega} \leq h_0$ is
\begin{itemize}[topsep=5pt]
	\setlength\itemsep{0.1mm}
\item $ \begin{aligned}[t] \tau_f \geq \tau_n^+ \Longrightarrow
\mathbb{E}[ \|f - m'(\theta_N)\|_{W_p^s (\Omega)} ] \leq &C h_{X_N, \Omega}^{\frac{d}{\max \{ 2,p\}} - s} \Big( h_{X_N, \Omega}^{\tau_n^- - \frac{d}{2} } \| f \|_{W_2^{\tau_f}(\Omega)} \\
&+ h_{X_N, \Omega}^{\tau_n^- - \frac{d}{2} } N^{\frac{1}{2}} + N^{\frac{d}{4 \tau_n^- }} \Big), \end{aligned}$
\item $ \begin{aligned}[t] \tau_f < \tau_n^+ \Longrightarrow
\mathbb{E} [ \|f - m'(\theta_N)\|_{W_p^s (\Omega)} ] \leq &C h_{X_N, \Omega}^{\frac{d}{\max \{ 2,p\}} - s} \Big( h_{X_N, \Omega}^{\min \{ \tau_f, \tau_n^- \} - \frac{d}{2}} \rho_{X_N, \Omega}^{ \tau_n^+ - \tau_f } \| f \|_{W_2^{\tau_f}(\Omega)} \\
&+ h_{X_N, \Omega}^{\tau_n^- - \frac{d}{2} } N^{\frac{1}{2}} + N^{ \max \big\{ \frac{1}{2} - \frac{\tau_f}{2 \tau_n^+}, \frac{d}{4 \tau_n^-} \big\} } \Big), \end{aligned}$
\end{itemize}
where $C \big(\Omega, \tau_f, p, s, \| f \|_{W_2^{\tau_f}(\Omega)}, \Theta_n^* \big)$ and $h_0 (R, \delta, d, \tau_f, \Theta_n^*)$ are constants.
\end{theorem}

As in Theorem \ref{3.3.2} the results hold only if the points in $X_N$ sufficiently cover $\Omega$, measured by $h_0$. If we set $s=0$, we recover results with respect to the norm $\| \cdot \|_{L^p(\Omega)}$. Additionally, we derive the following:

%-------------------------------------------------------------------------------------------

\begin{itemize}[topsep=5pt]
	\setlength\itemsep{0.1mm}
\item The first two conclusions, made for the interpolation setting, are also valid here. 
\item The optimal convergence rate is $\mathcal{O} \big( N^{-\frac{\tau_f}{2 \tau_f + d}} \big)$ with respect to the norm $\| \cdot \|_{L^p(\Omega)}$, where $p \in [1,2]$. This rate coincides with the minimax rate (see page 2097 in \cite{V}) and is achieved when the smoothness is correctly estimated as $\tau_f + d/2 = \tau_n^+ = \tau_n^-$ and quasi-uniform points are used. Note that this falls under the case $\tau_f < \tau_n^+$, making it inappropriate to refer to this case as misspecified smoothness.
\item Due to the equation
\begin{equation}
h_{X_N, \Omega}^{\frac{d}{\max \{ 2,p\}} - s} h_{X_N, \Omega}^{\tau_n^- - \frac{d}{2} } 
= h_{X_N, \Omega}^{\tau_n^- - s - d \left( \frac{1}{2} - \frac{1}{\max \{ 2,p\}} \right)}
= h_{X_N, \Omega}^{\tau_n^- - s - d \left( \frac{1}{2} - \frac{1}{p} \right)_+}, \label{eq. 3.3.3}
\end{equation}
the first term of the upper bound is identical to the upper bound in Theorem \ref{3.3.2}, which corresponds to the interpolation setting. The second and third terms can therefore be interpreted as penalties for the regression setting, which usually decrease at a slower rate. Selecting points such that $h_{X_N, \Omega} = \mathcal{O} (N^{-1/d})$ helps to mitigate this penalty.
\end{itemize}

\begin{remark}
In the interpolation setting, the smoothness is correctly estimated if $\tau_f = \tau_n^+ = \tau_n^-$, while in the regression setting, the smoothness is for $\tau_f + d/2 = \tau_n^+ = \tau_n^-$ correctly estimated. This difference can be justified by Assumption \ref{Assumption 7}. Using the interpolation setting, we assume $f \in W_2^{\tau_f}(\Omega)$ and match the smoothness by selecting $m' \in W_2^{\tau_f}(\Omega)$. However, in the regression setting we write $\tau_f = \ell + s$ and use the additional assumption $f \in C^{\ell, s}(\Omega)$. Using $m' \in W_2^{\tau_f}(\Omega)$ is not sufficient since it only guarantees $f \in C^0(\Omega)$. Instead, we select $m' \in W_2^{\tau_f + d/2}(\Omega)$ and according to Theorem 4.12 in \cite{AF}, we can ensure $m' \in C^{\ell, s}(\Omega)$.
\end{remark}

\begin{remark} \label{3.3.9}
Briol et al. also investigate the scenario where the likelihood function is misspecified. This means the assumption $\varepsilon_i \sim \mathcal{N} (0,\sigma^2)$ is incorrect, but we still use the regression setting. In the posterior mean function, we allow the variance to vary with $N$, leading to
$$m'(\boldsymbol{x}) = \boldsymbol{k}(\boldsymbol{x}) (\boldsymbol{K} + \sigma_N^2 \boldsymbol{I}_N)^{-1} (\boldsymbol{f} + \boldsymbol{\varepsilon} ) \quad \text{for } \boldsymbol{x} \in \Omega,$$
where $\boldsymbol{\varepsilon} = [\varepsilon_1, ..., \varepsilon_N]^T$. The $\varepsilon_i$ can either be deterministic or follow a probability distribution. Assume now that assumptions \ref{Assumption 1}-\ref{Assumption 5} hold. Due to these fewer assumptions, we achieve a different convergence rate than in Theorem \ref{3.3.7}, which depends on $\sigma_N$ and $\mathbb{E}[\boldsymbol{ \| \varepsilon} \|_2 ]$. More details can be found in Theorem 4 of \cite{BGW}, although the actual convergence rate is expected to be faster. However, we derive the following conclusions:
\begin{itemize}[topsep=5pt]
	\setlength\itemsep{0.1mm}
\item If the likelihood function is misspecified, the convergence rate is penalized by terms depending on $\sigma_N$, $\mathbb{E}[\boldsymbol{ \| \varepsilon} \|_2 ]$ and the experimental setting. Keeping the extreme value $\tau_n^+$ small and using quasi-uniform points helps mitigate the effects of likelihood misspecification. 
\item To improve numerical stability or ensure the kernel matrix is invertible, it is common to select $\sigma_N >0$, with larger values providing greater stability. To optimize the convergence rate when $\tau_f = \tau_n^- = \tau_n^+$, it is recommended to take $\sigma_N \sim h_{X_N, \Omega}^{\tau_f - d/2}$. Using quasi-uniform points, we can guarantee a convergence rate of $\mathcal{O}(N^{-1/2} \mathbb{E}[\boldsymbol{ \| \varepsilon} \|_2 ])$ with respect to the norm $\| \cdot \|_{L^p(\Omega)}$, where $p \in [1,2]$.
\end{itemize}
\end{remark}

%-------------------------------------------------------------------------------------------

Now, consider an example to evaluate the accuracy of the upper bound in the regression setting, where the likelihood function is correctly specified.

\begin{example} \label{3.3.10}
Consider $N$ equispaced points in $\Omega = (-1,1)^3$ with a step size of $N^{-1/3}$ in each dimension, the function
$$f:(-1,1)^3 \rightarrow \mathbb{R}, f(\boldsymbol{x})= |x_1|^{3.5} + x_2^2 - |x_3| $$
and a regression setting given by $\mathcal{D}_N = \{ (\boldsymbol{x}_1, f(\boldsymbol{x}_1) + \varepsilon_1), ..., (\boldsymbol{x}_N, f(\boldsymbol{x}_N) + \varepsilon_N) \}$ with $\varepsilon_i \sim \mathcal{N}(0,1)$. Using the Matérn kernel $k_{3/2,1}$, we approximate $f \in W^{3/2}_2 ( \Omega )$ by the posterior mean function $m'$ of the Gaussian Process $\mathcal{GP}(0,k_{3/2,1})$. Because of $\tau_f + d/2 = 3 = \tau_n^+$, Theorem \ref{3.3.7} allows to bound the relative error by
$$\dfrac{\|f - m'\|_{L^2 (\Omega)}}{\|f\|_{W_2^2 (\Omega)}} = \mathcal{O} (N^{-1/4}) \quad \text{as } N\rightarrow \infty.$$
We aim to evaluate this convergence rate. As in Example \ref{3.3.4}, the weak derivatives of $f$ and their norms can be easily evaluated. We approximate the $L^2$ norm using $M=10,000$ equispaced points with a step size of $M^{-1/3}$ and the Trapezoidal Rule. Figure \ref{fig. 10} illustrates the calculated relative error compared to the functions $g(N)= \frac{1}{5}N^{-1/4}$ and $h(N) = N^{-2/4}$. The curves suggest that the estimated convergence rate of $\mathcal{O} (N^{-1/4})$ is reliable, with the actual convergence rate appearing slightly faster at the initial stages.

\begin{figure}[h]
	\centering
	\includegraphics[width=0.5\textwidth]{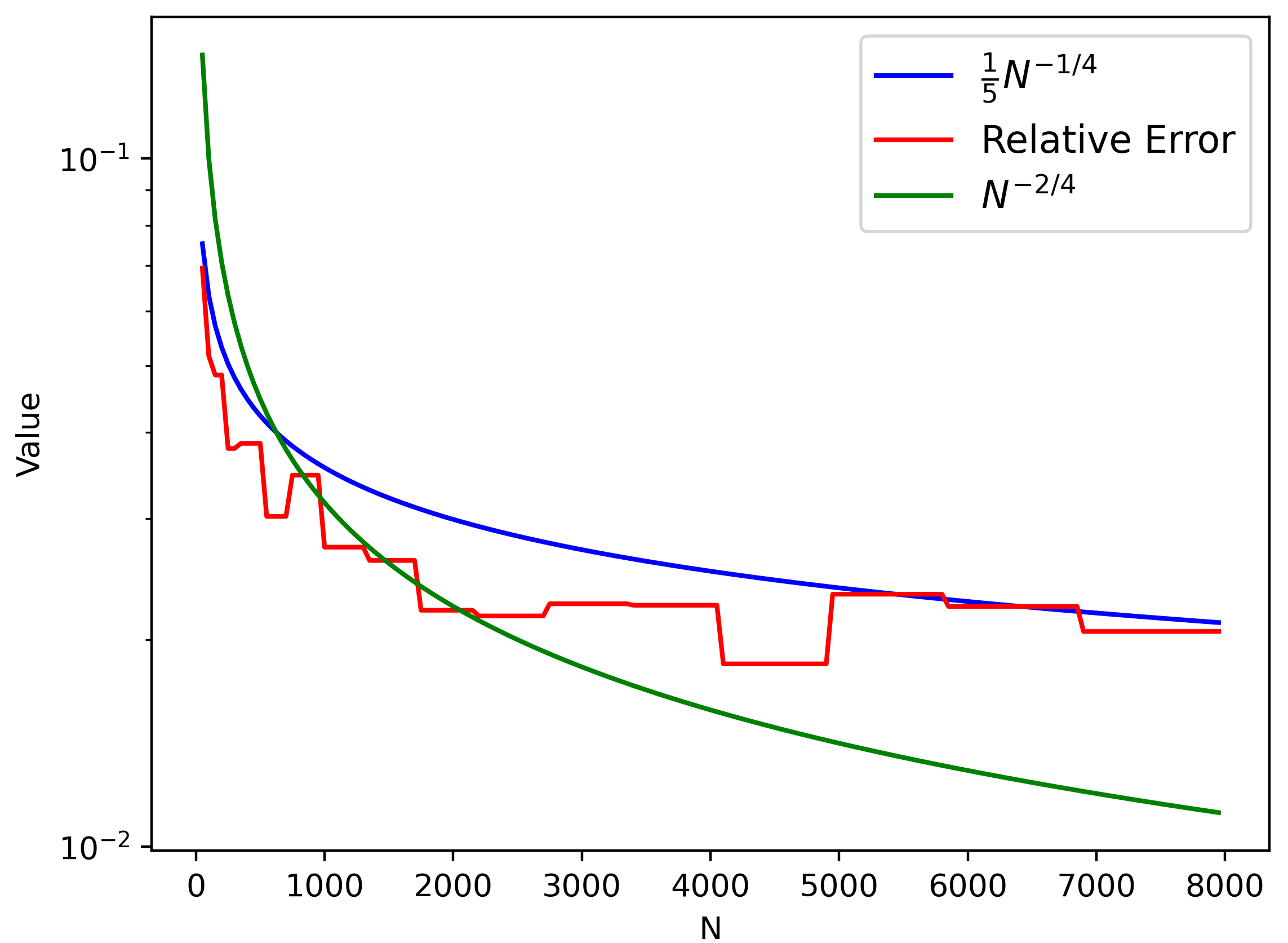}
	\caption{The blue curve represents $g(N)= \frac{1}{5}N^{-1/4}$, the green curve represents $h(N) = N^{-2/4}$ and the red curve describes the relative error of $m'$, computed for $N \in \{ 50,100,150,..., 7950, 8000 \}$.} \label{fig. 10}
\end{figure}
\end{example}

\iffalse
There are several Gaussian Process approximation methods aimed at reducing computation time by transforming the kernel matrix. Instead of using $\boldsymbol{K}$, we use an approximated matrix $\hat{\boldsymbol{K}}$ for Gaussian Process computation. However, this often results in $\hat{\boldsymbol{K}}$ being non-invertible.  To address this and improve numerical stability, we add noise to the matrix and use $\hat{\boldsymbol{K}} + \sigma_N^2 \boldsymbol{I}_N$ for computation. Remark \label{3.3.9} suggests to use $\sigma_N^2 = h_{X_N, \Omega}^{\tau_f - d/2}$.

\begin{example}
Consider the same interpolation setting as in Example 3.3.4. We repeat the experiment, but use $\sigma_N^2 = h_{X_N, \Omega}^{3}$ for increased numerical stability and investigate how it influences the convergence rate. Figure 11 shows the result, which is identical to Figure \ref{fig. 9}.

\begin{figure}[h]
	\centering
	\includegraphics[width=0.5\textwidth]{pictures/Relative_error3}
	\caption{The blue curve represents $g(N)= N^{-2/3}$, the green curve represents $h(N) = 3 N^{-1}$ and the red curve describes the relative error of $m'$, computed for $N \in \{ 50,100,150,..., 7950, 8000 \}$.} \label{fig. 11}
\end{figure}
\end{example}
\fi

The remainder of this subsection proves Theorem \ref{3.3.7}, as outlined in \cite{BGW}. This proof draws on the foundational work presented in \cite{AL}, \cite{BGW}, \cite{N} and \cite{V}, but is too lengthy and digressive to present here in full. Briol et al. highlight the theorems involved, whereas my focus is on the proof itself. Readers seeking more details are encouraged to consult the referenced theorems.

%-------------------------------------------------------------------------------------------

\begin{proof}[Proof of Theorem \ref{3.3.7}]
By applying Remark 3.4/Theorem 3.2 in \cite{AL} to the functions $f - m'(\theta_N)$, we obtain $C_1,h_1>0$ such that for all $\tau(\theta_N)$ with $N \geq n$, $s \in [0, \min \{\tau_f, \tau_n^- \}^*]$ and $X_N \subset \Omega$ with $h_{X_N,\Omega} \leq h_1$ is
\begin{align*}
\|f - m'(\theta_N)\|_{W_p^s (\Omega)} 
\leq &C_1 h_{X_N, \Omega}^{\min \{\tau_f, \tau(\theta_N) \} - s - d \left( \frac{1}{2} - \frac{1}{p} \right)_+} | f - m'(\theta_N) |_{W_2^{\min \{\tau_f, \tau(\theta_N) \}}(\Omega)} \\
&+ C_1 h_{X_N, \Omega}^{\frac{d}{ \max \{ 2,p \}} - s} \| \left.(f - m'(\theta_N))\right|_{X_N} \|_2 ,
\end{align*}
where $C_1(\Omega, \tau_f, p, s, \Theta_n^*)$ and $h_1(R, \delta, d, \tau_f, \Theta_n^*)$ are constants. \\

\textbf{Step 1:} Let $\tau_f \geq \tau(\theta_N)$.\\[5mm]
Using Lemma \ref{3.1.11}, we have
$$| f - m'(\theta_N) |_{W_2^{\min \{\tau_f, \tau(\theta_N) \}}(\Omega)}
= | f - m'(\theta_N) |_{W_2^{\tau(\theta_N)}(\Omega)}
\leq C_2\| f - m'(\theta_N) \|_{H^{\tau(\theta_N)}(\Omega)},$$
where $C_2(\Omega, \tau_f)>0$ is a constant. With the norm-equivalence of the RKHS, we get
\begin{align*}
\| f - m'(\theta_N) \|_{H^{\tau(\theta_N)}(\Omega)}
&\leq \| f \|_{H^{\tau(\theta_N)}(\Omega)} + \| m'(\theta_N) \|_{H^{\tau(\theta_N)}(\Omega)}\\
&\leq \| f \|_{H^{\tau_f}(\Omega)} + C_u (\theta_N) \| m'(\theta_N) \|_{\mathcal{H}_{k(\theta_N)}(\Omega)}.
\end{align*}
As $m'$ solves the optimization problem outlined in section 6.2.2 of \cite{RW} (alternatively section 2.1 in \cite{BGW}), we obtain
\begin{align*}
\| m'(\theta_N) \|_{\mathcal{H}_{k(\theta_N)}(\Omega)}^2
&\leq \| m'(\theta_N) \|_{\mathcal{H}_{k(\theta_N)}(\Omega)}^2 + \sigma^{-2} \sum_{i=1}^N (m'(\theta_N)(\boldsymbol{x}_i) - y_i)^2  \\
&\leq \| f \|_{\mathcal{H}_{k(\theta_N)}(\Omega)}^2 + \sigma^{-2} \| \boldsymbol{\varepsilon} \|_2^2 \\
&\leq C_\ell(\theta_N)^{-2} \| f \|_{H^{\tau(\theta_N)}(\Omega)}^2 + \sigma^{-2} \| \boldsymbol{\varepsilon} \|_2^2 \\
&\leq C_\ell(\theta_N)^{-2} \| f \|_{H^{\tau_f}(\Omega)}^2 + \sigma^{-2} \| \boldsymbol{\varepsilon} \|_2^2.
\end{align*}
From this, we derive
\begin{align*}
| f - m'(\theta_N) |_{W_2^{\min \{\tau_f, \tau(\theta_N) \}}(\Omega)}
&\leq C_2 \left( \| f \|_{H^{\tau_f}(\Omega)} + C_u(\theta_N) (C_\ell(\theta_N)^{-1} \| f \|_{H^{\tau_f}(\Omega)} + \sigma^{-1} \| \boldsymbol{\varepsilon} \|_2) \right) \\
&\leq C_3 \left( \| f \|_{W_2^{\tau_f}(\Omega)} + \sigma^{-1} \| \boldsymbol{\varepsilon} \|_2 \right),
\end{align*}
where $C_3(\Omega, \tau_f, \Theta_n^*)>0$ is a constant. Consequently, we have
\begin{align*}
\mathbb{E} \Big[ | f - m'(\theta_N) |_{W_2^{\min \{\tau_f, \tau(\theta_N) \}}(\Omega)} \Big]
&\leq C_3 \left( \| f \|_{W_2^{\tau_f}(\Omega)} + \sigma^{-1} \mathbb{E} \Bigg[ \sqrt{\sum_{i=1}^N \varepsilon_i^2} \Bigg] \right) \\
&= C_3 \left( \| f \|_{W_2^{\tau_f}(\Omega)} + \mathbb{E} \Bigg[ \sqrt{ \sum_{i=1}^N \Big( \dfrac{\varepsilon_i}{\sigma} \Big)^2 } \Bigg] \right) \\
&\leq C_3 \left( \| f \|_{W_2^{\tau_f}(\Omega)} + N^{1/2} \right),
\end{align*}
where we use that the chi distribution has mean
\begin{align*}
\sqrt{2} \dfrac{\Gamma ((N+1)/2)}{\Gamma(N/2)} 
&\leq \sqrt{2} \dfrac{\Gamma (N/2+1)^{1/2} \Gamma (N/2)^{1/2}}{\Gamma(N/2)} \\
&=  \sqrt{2} \dfrac{\sqrt{N}}{\sqrt{2}}\dfrac{\Gamma (N/2)^{1/2} \Gamma (N/2)^{1/2}}{\Gamma(N/2)}
= N^{1/2}
\end{align*} 
and the inequality follows from the log-convexity of the Gamma function. According to Proposition 21 in \cite{BGW} is

%-------------------------------------------------------------------------------------------

\begin{equation}
\mathbb{E} \big[\| \left.(f - m'(\theta_N))\right|_{X_N} \|_2 \big] 
\leq C_4 N^{ \max \big\{ \frac{1}{2} - \frac{\tau_f}{2\tau_n^+}  , \frac{d}{4 \tau_k^-} \big\} }, \label{eq. 3.3.4}
\end{equation}
where $C_4 \big(\| f \|_{W_2^{\tau_f}(\Omega)}, \Theta_n^* \big)>0$ is a constant. Using the derived inequalities and equation (\ref{eq. 3.3.3}), it follows that
\begin{align}
\begin{split}
\mathbb{E}[ \|f - m'(\theta_N)\|_{W_p^s (\Omega)} ] 
\leq &C_5 h_{X_N, \Omega}^{\frac{d}{\max \{ 2,p\}} - s} \Big( h_{X_N, \Omega}^{\tau_n^- - \frac{d}{2} } \| f \|_{W_2^{\tau_f}(\Omega)} \\
&+ h_{X_N, \Omega}^{\tau_n^- - \frac{d}{2} } N^{\frac{1}{2}} + N^{ \max \big\{ \frac{1}{2} - \frac{\tau_f}{2\tau_n^+} , \frac{d}{4 \tau_k^-} \big\} } \Big), \label{eq. 3.3.5}
\end{split}
\end{align}
where $C_5 \big(\Omega, \tau_f, p, s, \| f \|_{W_2^{\tau_f}(\Omega)}, \Theta_n^* \big)>0$ is a constant.\\

\textbf{Step 2:} Let $\tau_f < \tau(\theta_N)$.\\[5mm]
Define $m'_{\sigma =0}(\theta_N)$ as the posterior mean function in the interpolation setting, while $m'(\theta_N)$ remains the posterior mean for the regression setting. Then,
\begin{align*}
\| f - m'(\theta_N) \|_{W_p^{s}(\Omega)} 
\leq \| f - m'_{\sigma =0}(\theta_N) \|_{W_p^{s}(\Omega)} + \| m'_{\sigma =0}(\theta_N) - m'(\theta_N) \|_{W_p^{s}(\Omega)}.
\end{align*}
The first term can be bounded using the misspecified case from Theorem \ref{3.3.2}. For the second term, note that $\left.m'_{\sigma =0}(\theta_N)\right|_{X_N} = \left.f\right|_{X_N}$, implying $m'(\theta_N)$ is also the prediction for $m'_{\sigma =0}(\theta_N)$. This allows us to use equation (\ref{eq. 3.3.5}) to bound the second term, but we must omit the application of equation (\ref{eq. 3.3.4}), as Assumptions \ref{Assumption 6} and \ref{Assumption 7} may not hold for $m'_{\sigma =0}(\theta_N)$.  Additionally, we replace the first term
$$h_{X_N, \Omega}^{\tau_n^- - \frac{d}{2} } 
\qquad \text{with} \qquad
h_{X_N, \Omega}^{\tau(\theta_N) - \frac{d}{2} },$$
which can be justified by investigating the first step of this proof. Consequently,
\begin{align*}
\mathbb{E}[ \|f - &m'(\theta_N)\|_{W_p^{s}(\Omega)} ]
\leq C_6 h_{X_N, \Omega}^{\min \{\tau_f, \tau_n^- \} - s - d \left( \frac{1}{2} - \frac{1}{p} \right)_+} \rho_{X_N, \Omega}^{\tau_n^+ - \tau_f} \| f \|_{W_2^{\tau_f}(\Omega)}
+ C_6 h_{X_N, \Omega}^{\frac{d}{\max \{ 2,p\}} - s} \\
&\Big( h_{X_N, \Omega}^{\tau(\theta_N) - \frac{d}{2} } \| m'_{\sigma = 0}(\theta_N) \|_{W_2^{\tau(\theta_N)}(\Omega)} + h_{X_N, \Omega}^{\tau_n^- - \frac{d}{2} } N^{\frac{1}{2}} + \mathbb{E} \big[\| \left.(m'_{\sigma = 0}(\theta_N) - m'(\theta_N))\right|_{X_N} \|_2 \big] \Big),
\end{align*}
where $C_6(\Omega, \tau_f, p, s, \Theta_n^*)>0$ is a constant. Similarly to Step 2 within the proof of Theorem \ref{3.3.2}, we can follow that there exists a band-limited function $\mathcal{E}f_{\sigma} \in W_2^{\tau(\theta_N)}(\mathbb{R}^d)$ such that
$$\| \mathcal{E}f_{\sigma} \|_{W_2^{\tau(\theta_N)} (\mathbb{R}^d)} 
\leq C_7 q_{X_N}^{\tau_f - \tau(\theta_N)} \| f \|_{W_2^{\tau_f} (\Omega)}$$
and
$$\| \mathcal{E}f_{\sigma} - m'_{\sigma=0}(\theta_N) \|_{W_2^{\tau(\theta_N)} (\Omega)}
\leq C_8 q_{X_N}^{\tau_f - \tau(\theta_N)} \| f \|_{W_2^{\tau_f} (\Omega)},$$
where $C_7(\Omega, \tau_f, \Theta_n^*)>0$ and $C_8(\Omega, \tau_f, \Theta_n^*)>0$ are constants. Thus,
\begin{align*}
\| m'_{\sigma = 0}(\theta_N) \|_{W_2^{\tau(\theta_N)}(\Omega)} 
&\leq \| \mathcal{E}f_{\sigma} \|_{W_2^{\tau(\theta_N)} (\Omega)} + \| \mathcal{E}f_{\sigma} - m'_{\sigma=0}(\theta_N) \|_{W_2^{\tau(\theta_N)} (\Omega)} \\
&\leq (C_7 + C_8) q_{X_N}^{\tau_f - \tau(\theta_N)} \| f \|_{W_2^{\tau_f} (\Omega)}.
\end{align*}
respectively

%-------------------------------------------------------------------------------------------

\begin{align*}
h_{X_N, \Omega}^{\tau(\theta_N) - \frac{d}{2} } \| m'_{\sigma = 0}(\theta_N) \|_{W_2^{\tau(\theta_N)}(\Omega)} 
&\leq (C_7 + C_8) h_{X_N, \Omega}^{\tau_f - \frac{d}{2} } h_{X_N, \Omega}^{\tau(\theta_N) - \tau_f } q_{X_N}^{\tau_f - \tau(\theta_N)} \| f \|_{W_2^{\tau_f} (\Omega)} \\
&= (C_7 + C_8) h_{X_N, \Omega}^{\tau_f - \frac{d}{2} } \rho_{X_N, \Omega}^{\tau(\theta_N) - \tau_f } \| f \|_{W_2^{\tau_f} (\Omega)} \\
&\leq (C_7 + C_8) h_{X_N, \Omega}^{\min \{\tau_f, \tau_n^- \} - \frac{d}{2} } \rho_{X_N, \Omega}^{\tau_n^+ - \tau_f } \| f \|_{W_2^{\tau_f} (\Omega)},
\end{align*}
where we use $\rho_{X_N, \Omega} \geq 1$ and $h_{X_N, \Omega} \leq 1$. Since $m'_{\sigma=0}(\theta_N)$ interpolates $f$ and we can apply equation (\ref{eq. 3.3.4}), we obtain

\begin{align*}
\mathbb{E} \big[\| \left.(m'_{\sigma = 0}(\theta_N) - m'(\theta_N))\right|_{X_N} \|_2 \big]
= \mathbb{E} \big[\| \left.(f - m'(\theta_N))\right|_{X_N} \|_2 \big] 
\leq C_4 N^{ \max \big\{ \frac{1}{2} - \frac{\tau_f}{2\tau_n^+}  , \frac{d}{4 \tau_k^-} \big\} }.
\end{align*}
Combining the established inequalities and using equation (\ref{eq. 3.3.3}), it follows
\begin{align}
\begin{split}
\mathbb{E} [ \|f - m'(\theta_N)\|_{W_p^s (\Omega)} ] \leq &C_9 h_{X_N, \Omega}^{\frac{d}{\max \{ 2,p\}} - s} \Big( h_{X_N, \Omega}^{\min \{ \tau_f, \tau_n^- \} - \frac{d}{2}} \rho_{X_N, \Omega}^{ \tau_n^+ - \tau_f } \| f \|_{W_2^{\tau_f}(\Omega)} \\
&+ h_{X_N, \Omega}^{\tau_n^- - \frac{d}{2} } N^{\frac{1}{2}} + N^{ \max \big\{ \frac{1}{2} - \frac{\tau_f}{2\tau_n^+}  , \frac{d}{4 \tau_k^-} \big\} } \Big), \label{eq. 3.3.6}
\end{split}
\end{align}
where $C_9 \big(\Omega, \tau_f, p, s, \| f \|_{W_2^{\tau_f}(\Omega)}, \Theta_n^* \big)>0$ is a constant.\\

\textbf{Step 3:} Let $\tau_f$ be arbitrary.\\[5mm]
The bounds (\ref{eq. 3.3.5}) and (\ref{eq. 3.3.6}) imply
\begin{align*}
\mathbb{E} [ \|f - m'(\theta_N)\|_{W_p^s (\Omega)} ] \leq &\max \{ C_5, C_9 \} h_{X_N, \Omega}^{\frac{d}{\max \{ 2,p\}} - s} \Big( h_{X_N, \Omega}^{\min \{ \tau_f, \tau_n^- \} - \frac{d}{2}} \rho_{X_N, \Omega}^{ (\tau_n^+ - \tau_f)_+ } \| f \|_{W_2^{\tau_f}(\Omega)} \\
&+ h_{X_N, \Omega}^{\tau_n^- - \frac{d}{2} } N^{\frac{1}{2}} + N^{ \max \big\{ \frac{1}{2} - \frac{\tau_f}{2 \tau_n^+}, \frac{d}{4 \tau_n^-} \big\} } \Big).
\end{align*}
Distinguishing between the cases of smoothness verifies the theorem.
\end{proof}

%-------------------------------------------------------------------------------------------
% Summary
%-------------------------------------------------------------------------------------------

\subsection{Summary} \label{Chap. 3.4} 
For readers interested in the conclusions rather than the theoretical details, this section distills the key insights from the convergence results. Specifically, we address the questions posed at the beginning of this chapter:
\begin{itemize}[topsep=5pt]
	\setlength\itemsep{0.1mm}
\item Convergence is assured for functions $f \in W_2^{\tau_f}(\Omega)$ where $\tau_f >d/2$ and $\Omega \subset \mathbb{R}^d$ has a Lipschitz boundary.
\item A $\tau$-smooth kernel, such as the Matérn kernel $k_{\tau-d/2, \ell}$, should be selected. The smoothness parameter is the primary hyperparameter that significantly impacts the convergence rate. Other parameters only influence constants.
\item The smoothness parameter $\tau$ should be fixed and not excessively high to ensure well-specified smoothness. Other parameters may vary with the number of observations $N$. If a signal variance parameter $s_N^2>0$ is used as a scaling factor, it should be bounded away from zero.
\item The fill distance $h_{X_N, \Omega}$ of the data points $X_N$ must be sufficiently small to ensure convergence. Using a sequence of points such that $h_{X_N,\Omega} = \mathcal{O}(N^{-1/d})$ improves the convergence rate, while employing a sequence of quasi-uniform points optimizes it.
\item A misspecified smoothness or likelihood results in a slower convergence rate due to a penalty term. To mitigate this penalty when such misspecification is suspected, it is advisable to use quasi-uniform points.
\item In the interpolation setting, the optimal convergence rate is $\mathcal{O}(N^{-\tau_f/d})$ with respect to the norm $\| \cdot \|_{L^2(\Omega)}$. This is achieved when $\tau=\tau_f$ and $h_{X_N, \Omega}=\mathcal{O}(N^{-1/d})$. In the regression setting, the optimal convergence rate is $\mathcal{O}(N^{-\tau_f/(2\tau_f+d)})$ with respect to the norm $\| \cdot \|_{L^2(\Omega)}$. It is achieved when $\tau=\tau_f+d/2$ and quasi-uniform points are used.
\item Utilizing the Matérn kernel $k_{\nu, \ell}$, the optimal convergence rate achievable is $\mathcal{O}(N^{-\nu/d - 1/2})$ in the interpolation setting and $\mathcal{O}(N^{-\nu/(2 \nu +d)})$ in the regression setting with respect to the norm $\| \cdot \|_{L^2(\Omega)}$.
\end{itemize}
\newpage
% !TeX spellcheck = en_GB
\section{Samplets} \label{Chap. 4}
In numerous scenarios, you encounter extensive datasets, prompting the requirement for data compression. As the datasets grow continuously over the years this problem will not vanish. Amidst this abundance of data, the quest is to identify the most relevant information and store it efficiently. Samplets offer a flexible approach for compressing data, while keeping the computational costs low. Specially, it allows for compression of particular kernel matrices to a sparse format with $\mathcal{O}(N \log N)$ entries, which can be useful for Gaussian Process computation.\\

Samplets, as introduced by Harbrecht and Multerer in \cite{HM1}, adapt the wavelet construction from \cite{TW} to discrete data. The foundational concept of wavelets is to emulate the behavior of a function through a linear combination of simpler functions. This concept is age-old, for instance Taylor series represent functions as sum of polynomials, while Fourier series represent functions as sums of frequencies. However, these representations may not always be suitable, requiring numerous building blocks to mimic the function's behavior accurately. To achieve a reliable approximation, we seek only a handful of simple functions capable of capturing the function's essence. When dealing with signals, polynomial representations are inadequate and instead, oscillating and non-stationary functions are preferred that align with the signal's nature. Wavelets address this by offering compactly supported functions with wave-like oscillation, colloquial called "small waves". Remarkably, even a few wavelets can capture a significant amount of information, making them suitable for approximating signals. To control the scale of the wavelets, \cite{TW} proposes a hierarchical decomposition of the space $[0,1]^3$, achieving compression of integral operators.

Subsequent sections will adapt this approach to discrete data, as outlined in reference \cite{HM1}. The objective is to represent discrete data using a small amount of straightforward functions. Thus, compactly supported functions are introduced, resembling samples extracted from the discrete data. Drawing inspiration from the term wavelet, Harbrecht and Multerer dub these functions samplets. By extending the successful concepts from wavelets, it is observed that this adaptation enables compression of discrete data as well as the desired compression of kernel matrices. Figure \ref{fig. 11} visualizes wavelets compressing a signal and samplets compressing discrete data.

\begin{figure}[H]
	\centering
	\includegraphics[width=0.76\textwidth]{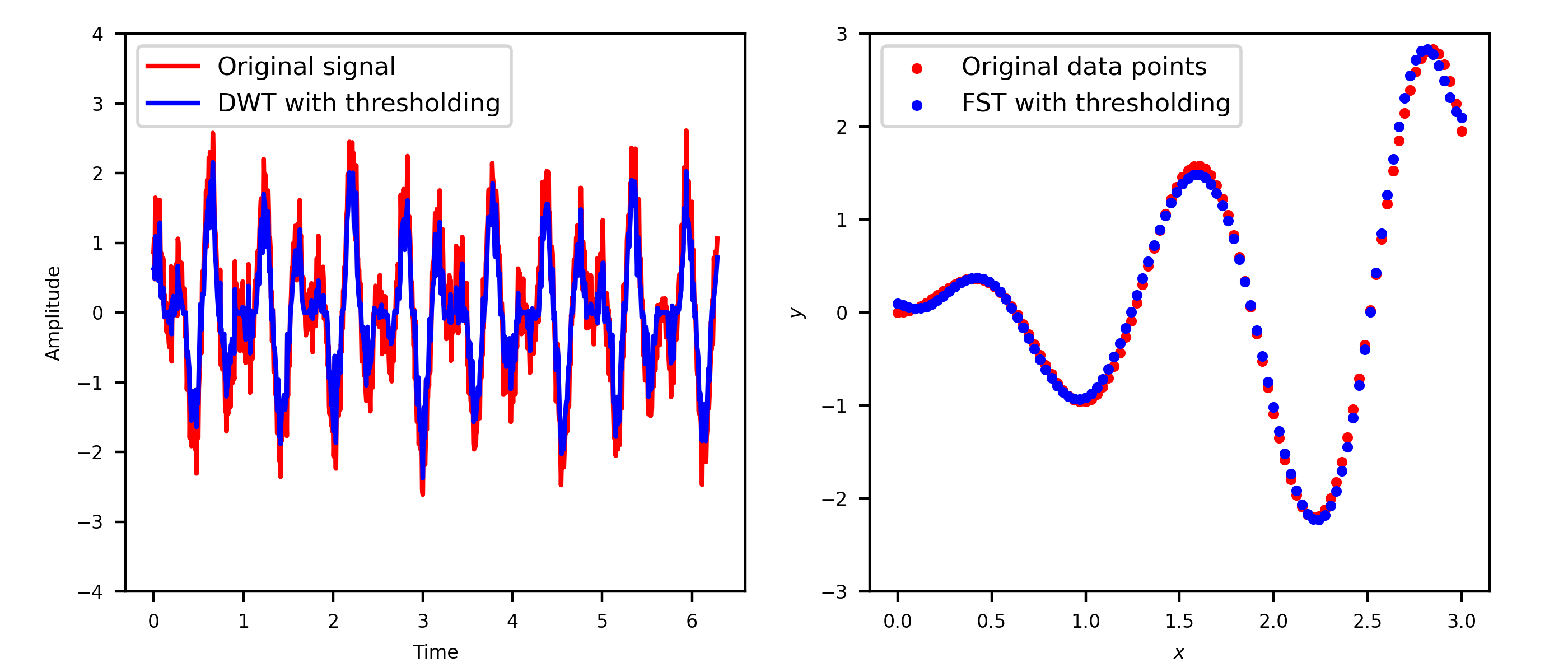}
	\caption{The left panel used DWT with Daubechies wavelets for signal compression via coefficient thresholding. Meanwhile, the right panel employed fast samplet transform (FST) for compressing discrete data points, also via coefficient thresholding.}\label{fig. 11}
\end{figure}

%-------------------------------------------------------------------------------------------
% Definition and Example
%-------------------------------------------------------------------------------------------

\subsection{Definition and Example} \label{Chap 4.1}
Let $X=\{ \boldsymbol{x}_1, ...,  \boldsymbol{x}_N\} \subset \Omega$ be a set of points for $\Omega \subseteq \mathbb{R}^d$ and let
$$\delta_{\boldsymbol{x}_i} (\boldsymbol{x}) = \begin{cases}
  1, & \text{if } \boldsymbol{x}= \boldsymbol{x}_i \\
  0, & \text{otherwise} \index{\textit{$\delta_{\boldsymbol{x}}$,}}
\end{cases}$$
be the dirac measure\index{Dirac measure,} for $\boldsymbol{x}_i \in X$. We extend the notation by denoting $\delta_{\boldsymbol{x}_i}: C(\Omega) \rightarrow \mathbb{R}$ also for the point evaluation\index{Point evaluation,}, defined by
$$(f,\delta_{\boldsymbol{x}_i})_{\Omega} := \delta_{\boldsymbol{x}_i}(f) = f(\boldsymbol{x}_i) \index{\textit{$(f,\delta_{\boldsymbol{x}_i})_{\Omega}$,}}$$
for all continuous functions $f \in C(\Omega)$. Consider the $N$-dimensional real vector space $V = \text{span} \{ \delta_{\boldsymbol{x}_1}, ..., \delta_{\boldsymbol{x}_N} \}$\index{\textit{$V$,}} and define an inner product by
\begin{equation}
\langle u,v \rangle_V = \sum_{i=1}^N u_i v_i, \ \ \ \ 
\text{for } u=\sum_{i=1}^N u_i \delta_{\boldsymbol{x}_i} \text{ and } v=\sum_{i=1}^N v_i \delta_{\boldsymbol{x}_i}.\label{4.1 (1)} \index{\textit{$\langle u,v \rangle_V$,}}
\end{equation}
This space is isometrically isomorph to $\mathbb{R}^N$ equipped with the canonical inner product, as the linear defined function $\Psi:V \rightarrow \mathbb{R}^N, \Psi (\delta_{\boldsymbol{x}_i}) = e_i$ is an isomorphism, where $\{e_1, ..., e_N\}$ is the canonical basis of $\mathbb{R}^N$. Analogous to wavelet constructions, we utilize a sequence $\{ V_j \}_{j=0}^J$\index{\textit{$V_j$,}} forming a multiresolution analysis\index{Multiresolution analysis,}
$$V_0 \subset V_1 \subset ... \subset V_J = V.$$
Samplets are designed to capture information between the levels $j$ and $j+1$, that means their purpose is to explain the detail spaces\index{Detail spaces,} $S_j$\index{\textit{$S_j$,}} defined by
$$V_j \oplus S_j = V_{j+1}. $$

\begin{defi} \label{4.1.1}
Let $\boldsymbol{\Sigma}_j$ be an orthonormal basis of $S_j$ for each $j=0,1,...,J-1$, with respect to the inner product in (\ref{4.1 (1)}). Let $\boldsymbol{\Phi}_0$\index{\textit{$\boldsymbol{\Phi}_0$,}} be an orthonormal basis of $V_0$. Then, the set\index{\textit{$\boldsymbol{\Sigma}_J$,}}
$$\boldsymbol{\Sigma}_J = \boldsymbol{\Phi}_0 \cup \bigcup_{j=0}^{J-1} \boldsymbol{\Sigma}_j$$
is termed samplet basis of $V$, with elements in $\boldsymbol{\Sigma}_0,...,\boldsymbol{\Sigma}_{J-1}$ referred to as samplets\index{Samplets,}.
\end{defi}

When considering data compression, it is advantageous for the samplets to exhibit vanishing moments. To achieve this, recall that $\mathcal{P}_q (\Omega)$\index{\textit{$\mathcal{P}_q (\Omega)$,}} is the space of polynomials on $\Omega$ with total degree of at most $q \in \mathbb{N}_{\geq 0}$. It is spanned by monomials of the form $\boldsymbol{z}^{\boldsymbol{\alpha}}:=z_1^{\alpha_1} z_2^{\alpha_2} \cdot ... \cdot z_d^{\alpha_d}$\index{\textit{$\boldsymbol{z}^{\boldsymbol{\alpha}}$,}}, where $\boldsymbol{\alpha} = (\alpha_1, ..., \alpha_d) \in \mathbb{N}_{\geq 0}^d$ satisfies $0 \leq | \boldsymbol{\alpha} | \leq q$.

\begin{defi} \label{4.1.2}
Samplet $\sigma_{j,k} \in \boldsymbol{\Sigma}_j$ has vanishing moments\index{Vanishing moments,} up to degree $q \in \mathbb{N}_{0}$, if
$$(p, \sigma_{j,k})_{\Omega} = 0 \ \ \ \text{for all } p \in \mathcal{P}_q (\Omega),$$
where $\mathcal{P}_q (\Omega)$ is the space of polynomials on $\Omega$ with total degree at most $q$.
\end{defi}

%-------------------------------------------------------------------------------------------

\begin{example} \label{4.1.3}
Consider $d=1$, $V = \text{span} \{ \delta_{1}, \delta_{2}, \delta_{3}, \delta_{4} \}$ and
\begin{align*}
V_0 = \text{span} \{ \delta_{1} + \delta_{2} + \delta_{3} +\delta_{4} \} \subset 
V_1 = \text{span} \{ \delta_{1} + \delta_{2} + \delta_{3} +\delta_{4}, \delta_{3} +\delta_{4} \} 
\subset V_2 = V.
\end{align*}
Then the sets
\begin{align*}
\boldsymbol{\Phi}_0 &= \{\delta_1 /2 + \delta_{2} /2 + \delta_{3} /2 + \delta_4 /2 \}, \\ 
\boldsymbol{\Sigma}_0 &= \{ \delta_1 /2 + \delta_{2} /2 - \delta_{3} /2 - \delta_4 /2 \}, \\
\boldsymbol{\Sigma}_1 &= \{ \delta_1 / \sqrt{2} - \delta_2 / \sqrt{2}, \delta_3 / \sqrt{2} - \delta_4 / \sqrt{2} \}
\end{align*}
form an orthonormal basis, where samplets have vanishing moments up to degree $0$.
\end{example}

Notice that the samplets in Example \ref{4.1.3} are localized relative to their discretization level. Specifically, a higher index $j$ indicates smaller samplets in $\Sigma_j$. This property is desirable for compression and is formally defined as follows.

\begin{defi} \label{4.1.4}
Recall that we write $C \lesssim D$, if $C$ can be bounded by a multiple of $D$ and $C \sim D$, if $C \lesssim D$ and $D \lesssim C$. A samplet $\sigma_{j,k} \in \boldsymbol{\Sigma}_j$ is localized with respect to his discretization level\index{Localized samplets,}, if 
$$\text{diam}(\text{supp } \sigma_{j,k}) \sim 2^{-j/d}.$$
This means that the diameter of the support of $\sigma_{j,k}$ is of the same magnitude as $2^{-j/d}$.
\end{defi}

%-------------------------------------------------------------------------------------------
% Construction of samplets
%-------------------------------------------------------------------------------------------

\subsection{Construction of Samplets}
Let $X=\{ \boldsymbol{x}_1, ...,  \boldsymbol{x}_N\}$ and $V = \text{span} \{ \delta_{\boldsymbol{x}_1}, ..., \delta_{\boldsymbol{x}_N} \}$ be defined as in section \ref{Chap 4.1}. We want to create a multiresolution analysis
$$V_0 \subset V_1 \subset ... \subset V_J = V$$
and localized samplets that demonstrate vanishing moments. To achieve this, we will employ hierarchical clustering and based on a specific cluster tree, we will construct a samplet basis of $V$ that exhibits the desired properties.

%-------------------------------------------------------------------------------------------
% Cluster trees
%-------------------------------------------------------------------------------------------

\subsubsection{Cluster trees}
\begin{defi}
Let $\mathcal{T}=(P,E)$\index{\textit{$\mathcal{T}$,}} be a tree with nodes $P$ (also called vertices) and edges $E$. The tree is called binary, if every $v \in P$ has at most two sons. The set of leaves is\index{\textit{$\mathcal{L}(\mathcal{T})$,}}\index{Leaf cluster,}
$$\mathcal{L}(\mathcal{T}) = \{ v \in P \ | \ v \text{ has no sons} \}. \index{\textit{$\mathcal{L}(\mathcal{T})$,}}$$
$\mathcal{T}$ is called cluster tree\index{Cluster tree,} for $X$, if $X$ is the root node and all $v \in P \backslash \mathcal{L}(\mathcal{T})$ are disjoint unions of their sons.\\
The level\index{Level,} $j_v$ of $v \in P$ is the number of son relations required traveling from $X$ to $v$. The depth\index{Depth,} of $\mathcal{T}$ is the maximum level of all nodes. The set of clusters\index{Cluster,} on level $j$ is defined as\index{\textit{$T_j$,}}
$$T_j = \{ v \in P \ | \ v \text{ has level } j \}. \index{\textit{$T_j$,}}$$
The bounding box\index{Bounding box,} $B_v$\index{\textit{$B_v$,}} of $v$ is the smallest axis-parallel hyperrectangle that contains all it points.
\end{defi}

%-------------------------------------------------------------------------------------------

We will consider only binary trees, although other choices are possible. To construct local samplets, especially balanced binary trees are of our interest.

\begin{defi} \label{4.2.2}
A cluster tree $\mathcal{T}=(P,E)$ on $X$ with depth $J$ is called balanced binary tree\index{Balanced binary tree,}, if for all $v \in P$ holds:
\begin{enumerate}
\item $v$ has no sons if $j_v = J$ and else it has exactly two sons,
\item $|v| \sim 2^{J-j_v}$.
\end{enumerate}
\end{defi}

A straightforward method for creating a balanced binary tree involves a top-down approach, where a cluster is divided into two sub-clusters of comparable size, until their cardinality falls under a given threshold. To achieve local clusters, we split a clusters bounding box along its longest edge, ensuring that each resulting box contains a nearly equal number of points. 

\begin{example} \label{4.2.3}
Let $d=2$ and $X=\{ (1,1), (1,2), (2,1), (2,2), (3,1) \}$. Figure \ref{fig. 12} shows a balanced binary tree for $X$ and $\text{threshold} = 2$.
\end{example}

\begin{figure}[h]
	\centering
	\includegraphics[width=0.9\textwidth]{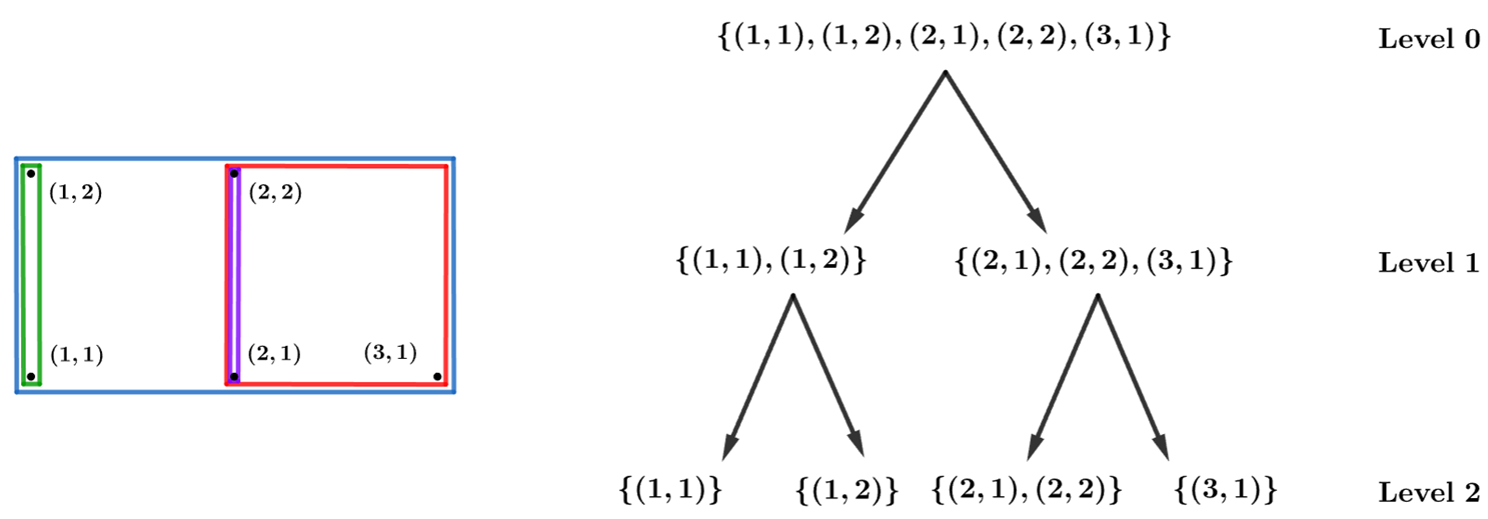}
	\caption{Balanced binary tree for $X=\{ (1,1), (1,2), (2,1), (2,2), (3,1) \}$ with its bounding boxes.}\label{fig. 12}
\end{figure}

\begin{proposition} \label{4.2.4}
This balanced binary tree can be computed in $\mathcal{O}(N \log N)$ time.
\end{proposition}

\begin{proof}
Let the threshold be $\tau$. Since there exists a leaf cluster on level $J$ with less than $\tau$ elements, the parent on level $J-1$ contains fewer than $2 \tau$ elements. By induction, we deduce that the root cluster on level $0$ has less than $2^J \tau$ elements. As the amount of elements in the root cluster is $N$, the depth satisfies
$$ N/2^J  < \tau \ \Longleftrightarrow \ 
J \in \Bigg\{ \left\lceil \dfrac{ \log(N/ \tau) }{\log 2} \right\rceil,  
\left\lceil \dfrac{ \log(N/ \tau) }{\log 2} \right\rceil + 1 \Bigg\},$$ 
resulting in a total of $\mathcal{O}(\log N)$ levels. To subdivide a cluster, we identify the longest edge of its bounding box, which determines the splitting direction. Considering a cluster $v= \{ \boldsymbol{x}_1, ..., \boldsymbol{x}_{|v|} \}$ with elements $\boldsymbol{x}_i = (x_{i, 1}, ..., x_{i, d}) \in \mathbb{R}^d$, the length of the edge along dimension $\ell$ is
$$L_\ell = \max_{i} \, x_{i, \ell} - \min_{j } \, x_{j, \ell}.$$
Consequently, the splitting direction is
$$\ell_{\text{split}} = \underset{\ell \in \{ 1, ..., d\} }{\argmax} \, L_\ell$$
with corresponding value
$$m_{\text{split}} = \text{median of } \{ x_{1, \ell_{\text{split}}}, ..., x_{|v|, \ell_{\text{split}}} \}.$$
Maintaining an equal number of elements, the two children are defined as
$$v_1 = \{ \boldsymbol{x}_i \in v \ | \ x_{i, \ell_{\text{split}}} \geq m_{\text{split}} \} \quad \text{and} \quad v_2 = \{ \boldsymbol{x}_i \in v \ | \ x_{i, \ell_{\text{split}}} < m_{\text{split}} \}.$$
Because this subdividing procedure costs $\mathcal{O}(N)$ for each level, constructing the balanced binary tree entails an overall cost of $\mathcal{O}(N \log N)$.
\end{proof}

%-------------------------------------------------------------------------------------------

The cluster splitting procedure is formulated in Algorithm \ref{alg. 2}.

\begin{algorithm}[h]
\DontPrintSemicolon
\caption{Cluster Splitting}\label{alg. 2}
\KwData{Cluster $v = \{ \boldsymbol{x}_1, ..., \boldsymbol{x}_{|v|} \}$, threshold $1 < \tau < N$}
\KwResult{Children $v_{\text{son} 1}, v_{\text{son} 2}$ of $v$}
\SetKwFunction{FMain}{SplitCluster}
\SetKwProg{Fn}{Function}{:}{}
\Fn{\FMain{$v$}}{
	\eIf{$|v| \geq \tau$}{
    			\For{$\ell = 1$  \KwTo $d$} {
    				$L_\ell=\max \{ x_{1, \ell}, ..., x_{|v|, \ell} \} - \min \{ x_{1, \ell}, ..., x_{|v|, \ell}\}$
    			}
    			$\ell_{\text{split}} = \underset{\ell \in \{ 1, ..., d\} }{\argmax} \, L_\ell$\;
    			$m_{\text{split}} = \text{median} \{ x_{1, \ell_{\text{split}}}, ..., x_{|v|, 	
    			\ell_{\text{split}}} \}$\;
    			$v_{\text{son} 1} = \{ \boldsymbol{x}_i \in v \ | \ x_{i, \ell_{\text{split}}} > 
    			m_{\text{split}} \}$\;
    			\ForEach{$\boldsymbol{x}_i \in v$} {
    				\uIf{$|v_{\text{son} 1}| \leq |v|/2$ \normalfont{\textbf{and}} $\boldsymbol{x}
    				_{i, \ell_{\text{split}}} = 	m_{\text{split}}$}{
    					$v_{\text{son} 1} = v_{\text{son} 1} \cup \{ \boldsymbol{x}_i \}$
    				}\ElseIf{$\boldsymbol{x}_{i, \ell_{\text{split}}} \leq m_{\text{split}}$}{
    					$v_{\text{son} 2} = v_{\text{son} 2} \cup \{ \boldsymbol{x}_i \}$
    				}
    			}
  		}
  			{$v_{\text{son} 1}, v_{\text{son} 2}  = \{ \}$
  		}
  	\KwRet $v_{\text{son} 1}, v_{\text{son} 2}$\;
}
\textbf{end}
\end{algorithm}

The next lemma helps us to investigate the diameter of a cluster and establishes a connection to the diameter of the clusters bounding box.

\begin{lemma} \label{4.2.5}
Let $v$ be a non-singleton cluster of this balanced binary tree. We have
$$\text{diam}(v) \sim \text{diam}(B_v) \sim R_v,$$ 
where $R_v$ denotes the length of $B_v$'s longest edge and the constants in $\sim$ depend only on $\Omega$.
\end{lemma}

%-------------------------------------------------------------------------------------------

\begin{proof}
First, we establish $\text{diam}(v) \sim \text{diam}(B_v)$. Since the elements of $v$ are contained in the bounding box $B_v$, it is
$$\text{diam}(v) \leq \text{diam}(B_v).$$
We have $B_v = \bigtimes_{i=1}^d [\min_{\boldsymbol{x} \in v} x_i, \max_{\boldsymbol{y} \in v} y_i]$ and therefore
$$\text{diam}(B_v)^2 = \sum_{i=1}^d (\max_{\boldsymbol{y} \in v} y_i - \min_{\boldsymbol{x} \in v} x_i)^2.$$
Without loss of generality let $i=1$ be the direction of $B_v$'s longest edge. Then,
\begin{align*}
\text{diam}(v)^2 = \max \bigg\{ \sum_{i=1}^d (&y_i - x_i)^2 : \boldsymbol{x},\boldsymbol{y} \in v \bigg\} 
\geq \left( \max_{\boldsymbol{y} \in v} y_1 - \min_{\boldsymbol{x} \in v} x_1 \right)^2 \\
&\geq \dfrac{1}{d} \sum_{i=1}^d \left( \max_{\boldsymbol{y} \in v} y_i - \min_{\boldsymbol{x} \in v} x_i \right)^2 
= \dfrac{1}{d} \text{diam}(B_v)^2.
\end{align*} 
Next, we concentrate on proving $\text{diam}(B_v) \sim R_v$, where $R_v$ is the length of $B_v$'s longest edge. Clearly, we have
$$\text{diam}(B_v) \geq R_v.$$
Regarding the other direction, note that
\[ \pushQED{\qed} 
\text{diam}(B_v)^2 = \sum_{i=1}^d (\max_{\boldsymbol{y} \in v} y_i - \min_{\boldsymbol{x} \in v} x_i)^2 \leq \sum_{i=1}^d R_v^{2} = d R_v^{2} = (\sqrt{d}R_v)^2. \qedhere \]
\end{proof}

We use the preceding lemma to investigate the diameter of a cluster. In particular, the diameter depends on the fill distance and on the separation radius, introduced in Definition \ref{3.2.1}.

\begin{proposition} \label{4.2.6}
Let $(X_N)_{N \in \mathbb{N}}$ be a sequence of points in $\Omega$ and $v$ be a non-singleton cluster of this balanced binary tree.
\begin{enumerate}[topsep=5pt]
	\setlength\itemsep{0.1mm}
\item It is $\text{diam}(v) \geq C_1 N^{1/d} q_{X_N} 2^{-j_v/d}$, where $C_1(\Omega)>0$ is a constant.  \label{4.2.6 (1)}
\item Let $\Omega$ be bounded. For $C_2 = 2 \sqrt{d} \cdot \text{diam}(\Omega)$ is \label{4.2.6 (2)}
$$\sum_{v \text{ is a cluster on level } j} \text{diam}(v) \leq C_2 2^{j-j/d}.$$\vspace{-5mm}
\item Let $\Omega$ be convex and bounded, and let $q_{X_N} \sim h_{X_N, \Omega} \sim N^{-1/d}$. Then it is $\text{diam}(v) \leq C_3  2^{-j_v/d}$, where $C_3(\Omega)>0$ is a constant.  \label{4.2.6 (3)}
\end{enumerate}
\end{proposition}

%-------------------------------------------------------------------------------------------

\begin{proof} 
$ \ $
\begin{enumerate}[topsep=5pt]
	\setlength\itemsep{0.1mm}
\item For clarity, I will omit the use of the floor and ceiling functions. The open ball $B_{q_{X_N}}(\boldsymbol{0}) \subset \mathbb{R}^d$ covers $[0,q_{X_N} / \sqrt{d})^d$ and therefore, this cube contains at most one point from $X_N$. By merging $|v|-1$ of such cubes in a pattern as
\begin{figure}[H]
	\centering
	\includegraphics[width=0.4\textwidth]{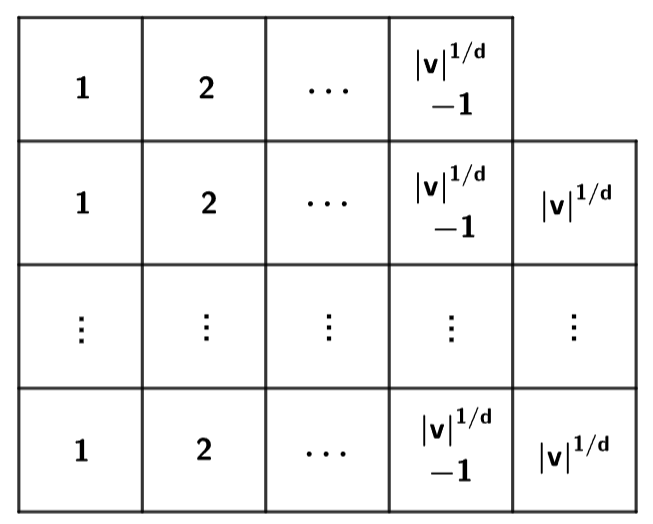}
\end{figure}
\noindent
it follows that
$$\left[0, |v|^{1/d} \dfrac{q_{X_N}}{\sqrt{d}} \right)^{d-1} \times \left[0, (|v|^{1/d} - 1) \dfrac{q_{X_N}}{\sqrt{d}} \right)$$
contains at most $|v|-1$ points from $X_N$. Because the bounding box contains $|v|$ points, we find that
$$R_v \geq (|v|^{1/d}-1) \dfrac{q_{X_N}}{\sqrt{d}} \geq C_1 |v|^{1/d} \dfrac{q_{X_N}}{\sqrt{d}}$$
for $C_1 = 1- 2^{-1/d}>0$. Using $|v| \sim N 2^{-j_v}$ and Lemma \ref{4.2.5}, we obtain
$$\text{diam}(v) \gtrsim R_v \gtrsim N^{1/d} 2^{-j_v/d} q_{X_N}.$$

\item For $j \in \{ 0,1,2, ..., d \}$ we have
$$\sum_{v \text{ is a cluster on level } j} R_v \leq 2^j \text{diam}(\Omega) \leq 2\text{diam}(\Omega) \cdot  2^{j-j/d}.$$
Our splitting procedure divides $B_v$'s longest edge into two edges, ensuring that their combined length does not exceed $R_v$. After cutting the longest edge, at most $d$ additional divisions are required to cut a second time along a direction. From this fact can be followed, that for each $d$-child $v_{d \text{-son} 1}$ of $v$, there exists a unique corresponding $d$-child $v_{d \text{-son} 2}$ of $v$ such that
$$R_{v_{d \text{-son} 1}} + R_{v_{d\text{-son} 2}} \leq R_v.$$
Using that a cluster has $2^d$ $d$-childs and induction, it follows
\begin{align*}
2 \sum_{v \text{ is a cluster on level } j+d} R_v 
&= \sum_{v_{d \text{-son} 1} \text{ is a cluster on level } j+d} R_{v_{d \text{-son} 1}} + R_{v_{d \text{-son} 2}} \\
&\leq 2^d \sum_{v \text{ is a cluster on level } j} R_v \\
&\leq 2^d 2 \text{diam}(\Omega) 2^{j-j/d} = 2 \text{diam}(\Omega) \cdot 2^{j+d-j/d}.
\end{align*}
This implies
$$\sum_{v \text{ is a cluster on level } j+d} R_v \leq 2 \text{diam}(\Omega) \cdot 2^{j+d-(j+d)/d}$$
and using Lemma \ref{4.2.5}, we obtain the result.

%-------------------------------------------------------------------------------------------
%Consider the bounding box $B_v$ of cluster $v$.
\item Denote by $R_v$ the length of $B_v$'s longest edge and by $r_v$ the length of its smallest edge. Since $q_{X_N} \sim h_{X_N, \Omega} \sim N^{-1/d}$ and $\Omega$ is convex, our dividing strategy ensures the existance of a constant $C_4(\Omega)>0$ with $R_v \leq C_4 r_v$. For the same reason, there is a constant $C_5(\Omega)>0$ with
$$C_5 \text{Vol}_d (B_v) \leq |v| \cdot \text{Vol}_d (B_{h_{X_N, \Omega}} (\boldsymbol{x}_1)) 
\Longrightarrow C_5 r_v^d \leq |v| V_d h_{X_N, \Omega}^d,$$
where $V_d$ denotes the volume of the $d$-dimensional unit ball. This implies that
$$r_v \lesssim |v|^{1/d} h_{X_N, \Omega}.$$
Using $|v| \sim N 2^{-j_v}$ and $h_{X_N, \Omega} \sim N^{-1/d}$, we obtain
$$R_v \lesssim r_v \lesssim (N 2^{-j_v})^{1/d} N^{-1/d} = 2^{-j_v / d}.  $$
The statement then follows by applying Lemma \ref{4.2.5}. \qedhere
\end{enumerate}
\end{proof}

Statement \ref{4.2.6 (2)} of Proposition \ref{4.2.6} guarantees that clusters at higher levels, when considered collectively, have shrinking diameters. However, it does not ensure that an individual cluster at a higher level has a smaller diameter than a cluster at a lower level. To achieve the localization of samplets, it is essential for individual clusters to exhibit this individual property, which can be ensured by using quasi-uniform points. Remember that the use of quasi-uniform points implies that $\Omega$ is bounded.

\begin{corollary} \label{4.2.7}
Let $\Omega$ be convex, $\text{Vol}_d(\Omega)>0$ and let $(X_N)_{N \in \mathbb{N}}$ be a sequence of quasi-uniform points in $\Omega$. For each non-singleton cluster $v$ of this balanced binary tree is
$$\text{diam}(v) \sim \text{diam}(B_v) \sim 2^{-j_v/d},$$
where the constants in $\sim$ depend only on $\Omega$.
\end{corollary}

\begin{proof}
After Lemma \ref{3.2.3} is $q_{X_N} \sim h_{X_N, \Omega} \sim N^{-1/d}$, where the constants in $\sim$ depend solely on $\Omega$. The corollary follows with \ref{4.2.6 (1)} and \ref{4.2.6 (3)} of Proposition \ref{4.2.6}.
\end{proof}

\begin{remark} \label{4.2.8}
Since localized samplets are crucial, one might consider whether the assumptions in Corollary \ref{4.2.7} can be relaxed. However, this is not the case. For instance, let $\Omega = [0,1]$ and consider
$$X_N = \Big\{ 0, N^{-2}, \dfrac{1}{N-2}, \dfrac{2}{N-2}, ...,  \dfrac{N-2}{N-2} \Big\}.$$
We have $h_{X_N, \Omega} = \dfrac{1}{2(N-2)} \sim N^{-1}$, but there exists a cluster $v = \{ 0, N^{-2} \}$ with 
$$\text{diam}(v) = N^{-2} \ \cancel{\gtrsim} \ 2N^{-1} = 2^{-j_v}.$$
This shows that the condition $q_{X_N} \sim N^{-1/d}$ cannot be omitted. Similarly, let $\Omega = [-1,1]$ and consider
$$X_N = \Big\{ -1, \dfrac{0}{N-2}, \dfrac{1}{N-2}, \dfrac{2}{N-2}, ...,  \dfrac{N-2}{N-2} \Big\}.$$
Then is $q_{X_N} = \dfrac{1}{2(N-2)} \sim N^{-1}$, but at every level, there exists a cluster $v$ with $\text{diam}(v) \geq 1$. This indicates that the condition $h_{X_N, \Omega} \sim N^{-1/d}$ cannot be dropped. Lastly, let $\Omega$ and $X_{N+N^{1/2}}$ be as illustrated in the accompanying graphic.
\begin{figure}[H]
	\centering
	\includegraphics[width=0.8\textwidth]{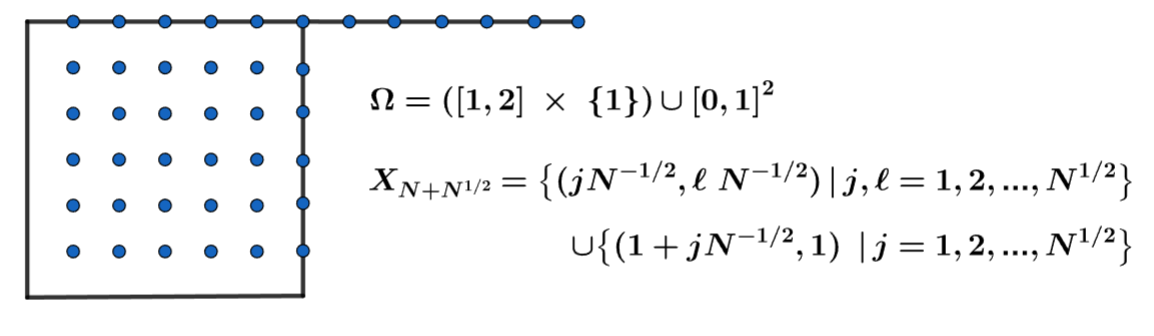}
\end{figure}

\noindent
While $X_{N+N^{1/2}}$ forms a sequence of quasi-uniform points, up to level
$$j = \dfrac{ \log(N^{1/2}+1)}{\log 2}-1 \Longleftrightarrow \dfrac{N+N^{1/2}}{2^j} = 2 N^{1/2},$$
there is always a cluster $v$ with $\text{diam}(v) \geq 1$. This demonstrates that the convexity of $\Omega$ cannot be omitted.
\end{remark}

%-------------------------------------------------------------------------------------------
% Multiscale samplet basis
%-------------------------------------------------------------------------------------------

\subsubsection{Multiscale samplet basis} \label{Chap. 4.2.2}

In the following, we use a cluster tree $\mathcal{T}=(P,E)$ of $X$ to build a samplet basis based on its hierarchical structure. Let $v$ denote a cluster on level $j$. Our approach utilizes scaling functions $\boldsymbol{\Phi}_{j+1}^v = \{ \varphi_{j+1,k}^v \}$ from $v$'s child clusters to create scaling functions $\boldsymbol{\Phi}_{j}^v = \{ \varphi_{j,k}^v \}$\index{\textit{$\boldsymbol{\Phi}_j^v$,}} and samplets $\boldsymbol{\Sigma}_j^v = \{ \sigma_{j,k}^v \}$\index{\textit{$\boldsymbol{\Sigma}_j^v$,}} for cluster $v$. More precisely:
%While Habrecht and Multerer define samplets also for leaf clusters, we will omit this step to avoid creating an additional space in $\{ V_j \}_{j=0}^J$. However, the underlying theory remains consistent and to achieve the construction outlined in \cite{HM1}, you simply need to extend the cluster tree by one more level. Our approach is the following:
\begin{enumerate}[label=(\roman*)]
\item If $v$ is a leaf cluster, set $\boldsymbol{\Phi}_{j}^v = \{ \delta_{\boldsymbol{x}} \ | \ \boldsymbol{x} \in v \}$ and $\boldsymbol{\Sigma}_j^v = \emptyset$,
\item If $v$ has sons, define $[\boldsymbol{\Phi}_{j}^v \, | \, \boldsymbol{\Sigma}_j^v] = \boldsymbol{\Phi}_{j+1}^v \boldsymbol{Q}_j^v$ for a matrix $\boldsymbol{Q}_j^v \in \mathbb{R}^{n \times n}$\index{\textit{$\boldsymbol{Q}_j^v$,}}.
\end{enumerate}
In this context we perceive $\boldsymbol{\Phi}_{j}^v = [\varphi_{j,1}^v, ..., \varphi_{j, | \boldsymbol{\Phi}_{j}^v |}^v]$ and $\boldsymbol{\Sigma}_j^v = [\sigma_{j,1}^v, ..., \sigma_{j,| \boldsymbol{\Sigma}_j^v |}^v]$ as row matrices, and employ the notation
\begin{equation*}
[\boldsymbol{\Phi}_{j}^v \, | \, \boldsymbol{\Sigma}_j^v] := [\varphi_{j,1}^v, ..., \varphi_{j, | \boldsymbol{\Phi}_{j}^v |}^v, \sigma_{j,1}^v, ..., \sigma_{j, | \boldsymbol{\Sigma}_j^v | }^v]. \index{\textit{$[\boldsymbol{\Phi}_{j}^v \, \vert \, \boldsymbol{\Sigma}_j^v]$,}}
\end{equation*}
For simplicity, let $n = | \boldsymbol{\Phi}_{j+1}^v |$. To ensure orthonormality and vanishing moments for the samplets, the transformation $\boldsymbol{Q}_j^v$ must be appropriately constructed. For the vanishing moments
$$(p, \sigma_{j,k})_{\Omega} = 0 \ \ \text{for all } p\in \mathcal{P}_q(\Omega),$$
it is sufficient to ensure that 
$$(\boldsymbol{z}^{\boldsymbol{\alpha}}, \sigma_{j,k})_{\Omega} = 0 \text{ for all } \boldsymbol{\alpha}=(\alpha_1, ..., \alpha_d) \in \mathbb{N}^d_{\geq 0} \text{ with }  | \boldsymbol{\alpha} | \leq q,$$
since $\mathcal{P}_q(\Omega)$ is spanned by monomials of maximum degree $q$. After \ref{Chap. 8.3} in the Appendix, there exist 

%-------------------------------------------------------------------------------------------

\begin{equation*}
m_q = \binom{q+d}{d} = \dfrac{(q+1)(q+2)...(q+d)}{1 \cdot 2 \cdot ... \cdot d} \leq (q+1)^d \index{\textit{$m_q$,}}
\end{equation*}
different $\boldsymbol{\alpha} \in \mathbb{N}^d_{\geq 0}$ with $| \boldsymbol{\alpha} | \leq q$, such that the dimension of $\mathcal{P}_q(\Omega)$ is $m_q$. Lets examine the moment matrix\index{Moment matrix,}
\begin{equation*}
\boldsymbol{M}_{j+1, \boldsymbol{\Phi}}^v 
= \begin{bmatrix}
	(\boldsymbol{z}^{\boldsymbol{0}}, \varphi_{j+1, 1}^v)_{\Omega} & \cdots & (\boldsymbol{z}^{\boldsymbol{0}}, \varphi_{j+1, n}^v)_{\Omega} \\ 
	\vdots & \ddots & \vdots \\ 
	(\boldsymbol{z}^{\boldsymbol{\alpha}}, \varphi_{j+1, 1}^v)_{\Omega} & \cdots & (\boldsymbol{z}^{\boldsymbol{\alpha}}, \varphi_{j+1, n}^v)_{\Omega}
\end{bmatrix}_{|\boldsymbol{\alpha}| \leq q} 
= [(\boldsymbol{z}^{\boldsymbol{\alpha}}, \boldsymbol{\Phi}_{j+1}^v)_{\Omega}]_{|\boldsymbol{\alpha}| \leq q}	\in \mathbb{R}^{m_q \times n},
\index{\textit{$\boldsymbol{M}_{j+1, \boldsymbol{\Phi}}^v$,}}
\index{\textit{$[(\boldsymbol{z}^{\boldsymbol{\alpha}}, \boldsymbol{\Phi}_{j+1}^v)_{\Omega}]_{ \vert\boldsymbol{\alpha} \vert \leq q}$,}}
\end{equation*}
where the rows can be arbitrary ordered. To attain a zero moment matrix $\boldsymbol{M}_{j, \boldsymbol{\Sigma}}^v$ for samplets, we employ the QR decomposition $(\boldsymbol{M}_{j+1, \boldsymbol{\Phi}}^v )^T = \boldsymbol{Q}_j^v \boldsymbol{R}$, where $\boldsymbol{Q}_j^v \in \mathbb{R}^{n \times n}$ is an orthogonal matrix and $\boldsymbol{R} \in \mathbb{R}^{n \times m_q}$ an upper triangular matrix. Using $\boldsymbol{Q}_j^v$ to construct functions, as explained earlier in this section, implies
\begin{align*}
[\boldsymbol{M}_{j, \boldsymbol{\Phi}}^v \, | \, \boldsymbol{M}_{j, \boldsymbol{\Sigma}}^v] 
= [(\boldsymbol{x}^{\boldsymbol{a}}, [\boldsymbol{\Phi}_{j}^v \, | \, \boldsymbol{\Sigma}_{j}^v] &)_{\Omega}]_{|\boldsymbol{a}| \leq q}
= [(\boldsymbol{x}^{\boldsymbol{a}}, \boldsymbol{\Phi}_{j+1}^v  \boldsymbol{Q}_j^v)_{\Omega}]_{|\boldsymbol{a}| \leq q} \\
&= [(\boldsymbol{x}^{\boldsymbol{a}}, \boldsymbol{\Phi}_{j+1}^v )_{\Omega}]_{|\boldsymbol{a}| \leq q} \boldsymbol{Q}_j^v 
= \boldsymbol{M}_{j+1, \boldsymbol{\Phi}}^v \boldsymbol{Q}_j^v = \boldsymbol{R}^T,
\end{align*}
where the third equation can be confirmed through direct calculation. Consequently, this combined moment matrix is a lower triangular matrix. Defining only the first $m_q$ entries of
$$[\boldsymbol{\Phi}_{j}^v \, | \, \boldsymbol{\Sigma}_j^v] = \boldsymbol{\Phi}_{j+1}^v \boldsymbol{Q}_j^v \in \mathbb{R}^{1 \times n}$$
as scaling functions, ensures that the samplets are orthonormal and possess vanishing moments up to degree $q$.

\begin{lemma} \label{4.2.9}
Let $v$ be a cluster with scaling functions $\boldsymbol{\Phi}_{j}^v$ and samplets $\boldsymbol{\Sigma}_{j}^v$ constructed from scaling functions $\boldsymbol{\Phi}_{j+1}^v$. If each leaf cluster has at most $m_q$ elements, we have the following bounds:
\begin{enumerate}[topsep=5pt]
	\setlength\itemsep{0.1mm}
\item $| \boldsymbol{\Phi}_{j+1}^v | \leq m_q \Rightarrow | \boldsymbol{\Sigma}_{j}^v| = 0$,
\item $|\boldsymbol{\Phi}_{j}^v | \leq m_q$, $| \boldsymbol{\Sigma}_{j}^v | \leq m_q$, \label{4.2.9 (ii)}
\item $|\boldsymbol{\Phi}_{j}^v | + | \boldsymbol{\Sigma}_{j}^v| = | \boldsymbol{\Phi}_{j+1}^v | \leq 2 m_q$.
\end{enumerate}
\end{lemma}

\begin{proof}
The initial statement and the first inequality in \ref{4.2.9 (ii)} follow dirctly from the construction. For the third statement notice, that we focus solely on binary trees and that each son cluster cannot provide more than $m_q$ scaling functions. Considering these factors, also the second inequality in \ref{4.2.9 (ii)} emerges.
\end{proof}

The scaling functions of clusters on level $j$ span the spaces
\begin{equation}
V_j = \text{span} \{ \varphi_{j,1}^v, ..., \varphi_{j, | \boldsymbol{\Phi}_{j}^v |}^v \ | \ v \in T_j \} \label{tag 4.2.2}
\end{equation}
and the samplets of clusters on level $j$ generate the detail spaces
\begin{equation}
S_j = \text{span} \{ \sigma_{j,1}^v, ..., \sigma_{j, | \boldsymbol{\Sigma}_{j}^v |}^v \ | \ v \in T_j \} \label{tag 4.2.3}
\end{equation}
satisfying $V_j \oplus S_j = V_{j+1}$.
Combining the scaling functions of the root cluster with the samplets from all clusters gives rise to the samplet basis
\begin{equation}
\boldsymbol{\Sigma}_J = \boldsymbol{\Phi}_0^X \cup \bigcup_{v \in P} \boldsymbol{\Sigma}_{j_v}^v.  \label{eq. 4.2.9}
\end{equation}
In the following, we also write $\boldsymbol{\Sigma}_J = \{ \sigma_j \ | \ 1 \leq j \leq N \}$, with samplets belonging to the same cluster grouped together and those at finer levels assigned higher indices.

%-------------------------------------------------------------------------------------------

\begin{example} \label{4.2.10}
Let us revisit the balanced binary tree that we examined in Example \ref{4.2.3}. For $q=0$ or rather $m_q = 1$, we obtain scaling functions and samplets depicted in Figure \ref{fig. 13}. Notice, that utilizing a different QR decomposition of a moment matrix results in different functions.

\begin{figure}[H]
	\centering
	\includegraphics[width=0.95\textwidth]{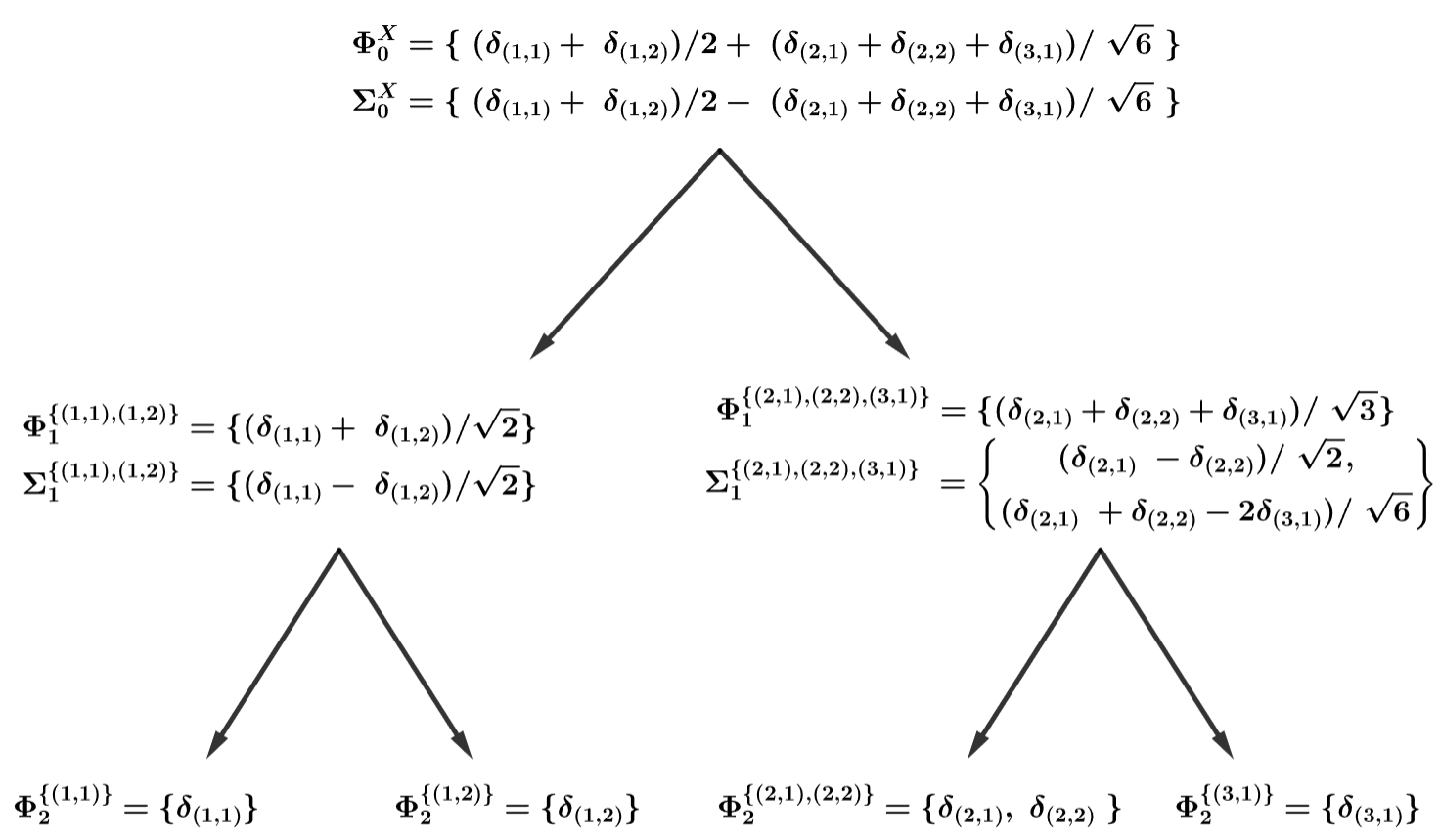}
	\caption{Scaling functions and samplets for the balanced binary tree in Figure \ref{fig. 12}.}\label{fig. 13}
\end{figure}
\end{example}

The procedure to compute the transformations and the samplet basis $\boldsymbol{\Sigma}_J$ is formulated in Algorithm \ref{alg. 3}.

\begin{algorithm}[h]
\DontPrintSemicolon
\caption{Samplet basis}\label{alg. 3}
\KwData{Cluster tree $\mathcal{T}$, cluster $v$, degree $q$}
\KwResult{Transformations $(\boldsymbol{Q}_j^v)_{v \in P \backslash \mathcal{L(T)}}$ and samplet basis $\boldsymbol{\Sigma}_J$}
\SetKwFunction{FMain}{SampletBasis}
\SetKwProg{Fn}{Function}{:}{}
\Fn{\FMain{$v$}}{
	\eIf{children $v_{\text{son} 1}, v_{\text{son} 2}$ of $v$ are leaf clusters}{
	$\boldsymbol{\Phi}^v_{j+1} = \{ \delta_{\boldsymbol{x}} \ | \ \boldsymbol{x} \in v_{\text{son} 1} \cup v_{\text{son} 2} 
	\}$\;} 
	{
	$\boldsymbol{\Phi}^v_{j+1} = \texttt{SampletBasis}(v_{\text{son} 1}) \cup 
	\texttt{SampletBasis}(v_{\text{son} 2})$\;
	}
	$\boldsymbol{M}^v_{j+1, \boldsymbol{\Phi}} = \texttt{MomentMatrix}(\boldsymbol{\Phi}^v_{j+1}, 
	q)$\;
	$\boldsymbol{Q}^v_j = \texttt{QR} \left( (\boldsymbol{M}^v_{j+1, \boldsymbol{\Phi}})^T 
	\right)$\;
	$[\boldsymbol{\Phi}^v_{j} \, | \, \boldsymbol{\Sigma}^v_j] = \boldsymbol{\Phi}^v_{j+1} 
	\boldsymbol{Q}^v_j $\;
	store $(\boldsymbol{Q}^v_j, \boldsymbol{\Sigma}^v_j)$\;
  	\KwRet $\boldsymbol{\Phi}^v_{j}$\;
}
\textbf{end}\;
store $\boldsymbol{\Phi}^X_{0} = \texttt{SampletBasis}(X)$\;
% $\boldsymbol{\Sigma}_J = \boldsymbol{\Phi}^X_{0} \cup \bigcup_{v \in P} \boldsymbol{\Sigma}_j^v$
\end{algorithm}

%-------------------------------------------------------------------------------------------

\begin{remark} \label{4.2.11}
To construct a binary cluster tree, we must choose a threshold $\tau$ and to create a samplet basis, we additionally specify the degree $q$. However, these two parameters might not always be well selected. For instance, with $N=200$, $q=2$, $d=3$ and $\tau=30$, we find $m_q=10$, leading to a disproportionate number of samplets per level. This example is illustrated in Figure \ref{fig. 14}.\\
\noindent
One approach to address this issue is to set $q$ such that $m_q$ and $\tau$ are of similar magnitude. However, fixing $q$ or $\tau$ eliminates flexibility, so we choose a different approach. We aim to increase $q$ only for clusters on level $J-1$ to generate samplets with vanishing moments up to higher degrees. Directly increasing $q$ for these clusters would shift the problem to level $J-2$. Hence, we augment $q$ to $\tilde{q}$ while ensuring $m_{\tilde{q}} \leq 2 \tau$, but still define only the first $m_q$ functions as scaling functions. Consequently, we obtain samplets on level $J-1$ with vanishing moments up to a degree ranging from $q$ to $\tilde{q}$.\\
\noindent
Applying this approach to the scenario in Figure \ref{fig. 14} yields the same amount of scaling functions and samplets, but a cluster on level $2$ now has the following $40$ samplets:
\begin{itemize}[topsep=5pt]
	\setlength\itemsep{0.1mm}
\item $10$ samplets with vanishing moments up to degree $2$,
\item $15$ samplets with vanishing moments up to degree $3$,
\item $15$ samplets with vanishing moments up to degree $4$.
\end{itemize}

\begin{figure}[H]
	\centering
	\includegraphics[width=0.8\textwidth]{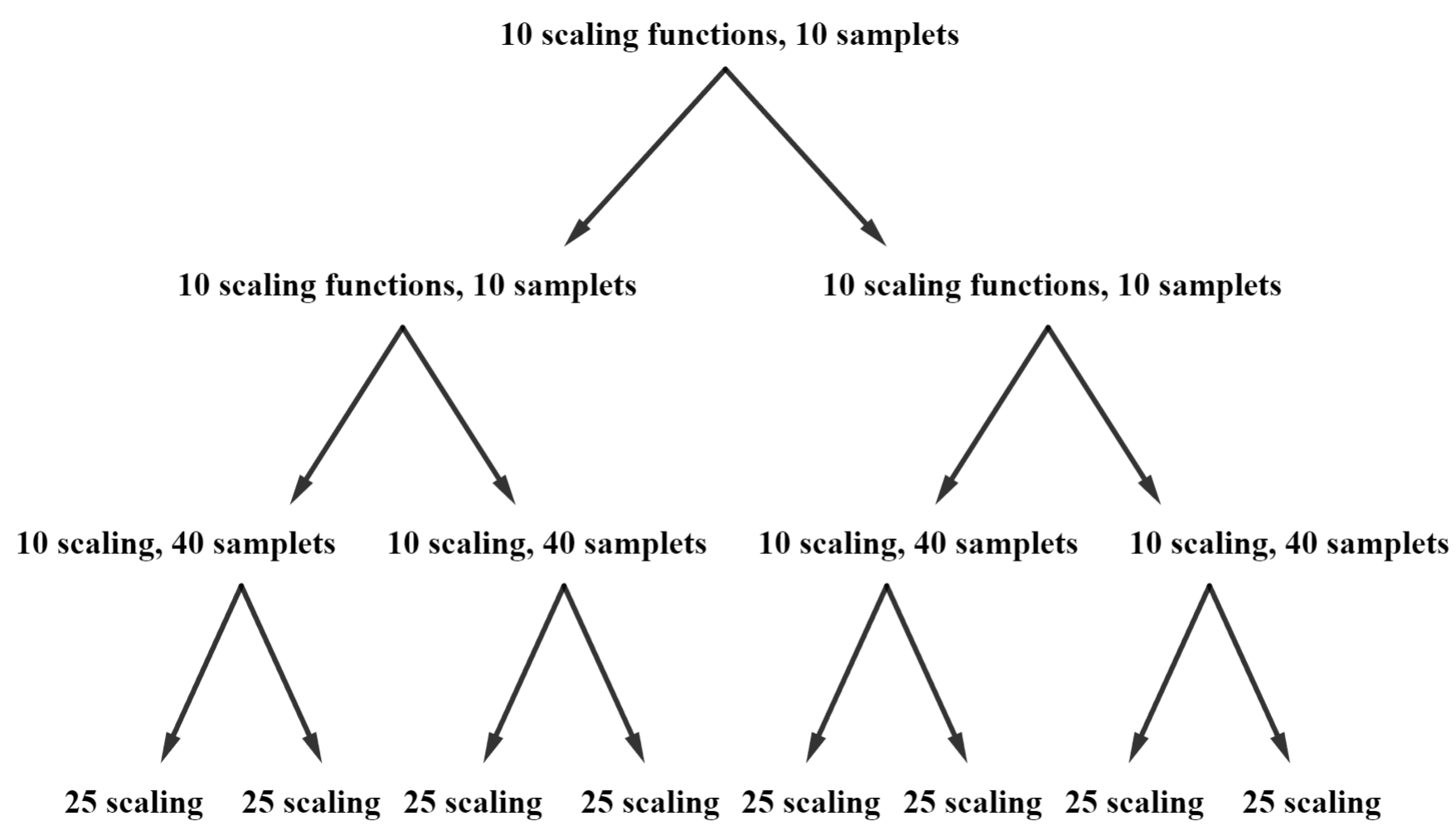}
	\caption{Amount of scaling functions and samplets per cluster, if we consider a binary cluster tree for $N=200$, $m_q=10$ and $\tau=30$.}\label{fig. 14}
\end{figure}
\end{remark}

%-------------------------------------------------------------------------------------------
% Properties of constructed samplets
%-------------------------------------------------------------------------------------------

\subsubsection{Properties of constructed samplets} \label{Chap. 4.2.3}
We have proven that a balanced binary tree can be generated in $\mathcal{O}(N \log N)$ time. The expenses involved in forming the corresponding samplet basis are of comparable magnitude.

\begin{proposition} \label{4.2.12}
The samplet basis of a balanced binary tree of $X$ can be computed in $\mathcal{O}(N)$ time.
\end{proposition}

%-------------------------------------------------------------------------------------------

\begin{proof}
If $v$ is a leaf cluster, then $|v| \sim 1$, indicating that $|v|$ is bounded by a constant. Consequently, computing the moment matrix for the $|v|$ many scaling functions and performing the QR-decomposition, requires a constant level of effort. \\
Let $v$ be a non-leaf cluster on level $j$. The number of scaling functions $|\boldsymbol{\Phi}_{j+1}|$ contributed by the two child clusters, is limited by a constant. Hence, the following tasks incur only a constant time: 
\begin{itemize}[topsep=5pt]
	\setlength\itemsep{0.1mm}
\item Gathering the computed moments in $\boldsymbol{M}_{j+1, \boldsymbol{\Phi}}^v \in \mathbb{R}^{m_q \times |\boldsymbol{\Phi}_{j+1}|}$,
\item Computing the QR-decomposition $(\boldsymbol{M}_{j+1, \boldsymbol{\Phi}}^v )^T = \boldsymbol{Q}_j^v \boldsymbol{R}$,
\item Deriving $[\boldsymbol{\Phi}_{j}^v \, | \, \boldsymbol{\Sigma}_j^v] = \boldsymbol{\Phi}_{j+1}^v \boldsymbol{Q}_j^v \in \mathbb{R}^{1 \times |\boldsymbol{\Phi}_{j+1}|}$,
\item Calculating the moments in $\boldsymbol{M}_{j, \boldsymbol{\Phi}}^v$.
\end{itemize}
Considering the constant effort needed for each cluster and the total number of clusters being
$$\mathcal{O}(1+2+2^2+...+N) = \mathcal{O} \left(\sum_{k=0}^{\log_2 N} 2^k \right) 
= \mathcal{O}  \left( \dfrac{1-2^{\log_2 N + 1}}{1-2} \right)= \mathcal{O}(N),$$ 
we receive an overall runtime of $\mathcal{O}(N)$. Notice that the actual total number of clusters is $2^{J+1}-1$, where $J$ is the depth of the balanced binary tree.
\end{proof}

The following Theorem compiles the results of section \ref{Chap. 4.2.2} for the constructed multiscale analysis and adds further statements.

\begin{theorem} \label{4.2.13}
The spaces $V_j$ defined in (\ref{tag 4.2.2}) define a multiresolution analysis
$$V_0 \subset V_1 \subset ... \subset V_J = V$$
with corresponding detail spaces $S_j$ defined by (\ref{tag 4.2.3}) satisfying $V_j \oplus S_j = V_{j+1}$ for all $j \in \{0,1,...,J \}$. The samplet basis $\boldsymbol{\Sigma}_J$ defined in (\ref{eq. 4.2.9}) forms an orthonormal basis of $V$. In addition, there holds:
\begin{enumerate}[topsep=5pt]
	\setlength\itemsep{0.1mm}
\item The samplets have vanishing moments up to degree $q$. \label{4.2.13 (1)}
\item Each samplet $\sigma_{j,k} \in \boldsymbol{\Sigma}_j$ is supported within a specific cluster $v$ at level $j$. \label{4.2.13 (2)}
\item Let $\Omega$ be convex, $\text{Vol}_d(\Omega)>0$ and $(X_N)_{N \in \mathbb{N}}$ be a sequence of quasi-uniform points. Then the samplets are localized. \label{4.2.13 (3)}
\item Let samplet $\sigma_{j}$ be supported in cluster $v$. Then is $\| \omega_{j} \|_1 \leq \sqrt{|v|}$ for the coefficient vector $\omega_{j} = [\omega_{j,1}, ..., \omega_{j,N}] \in \mathbb{R}^N$ of the samplet. \label{4.2.13 (4)}
\item Let $\Omega \subseteq \mathbb{R}^d$ be open and convex, $f \in C^{q+1}(\Omega)$ and $\sigma_{j}$ be a samplet supported on cluster $v$ with coefficient vector $\omega_{j} \in \mathbb{R}^N$. Then,
$$| (f, \sigma_{j})_{\Omega} | \leq \left( \dfrac{d}{2} \right)^{q+1} \dfrac{ \text{diam}(v)^{q+1} }{ (q+1)! } \| f \|_{ C^{q+1}(\Omega) } \| \omega_{j} \|_1.$$ \label{4.2.13 (5)}
\end{enumerate}
\end{theorem}

%-------------------------------------------------------------------------------------------

\begin{proof}
The statements up number \ref{4.2.13 (2)} are direct consequences of the construction outlined in section \ref{Chap. 4.2.2}.
\begin{enumerate}[topsep=5pt]
	\setlength\itemsep{0.1mm}
\setcounter{enumi}{2}
\item Use statement \ref{4.2.13 (2)} and Corollary \ref{4.2.7}.
\item Without loss of generality write 
$$\sigma_{j} = \sum_{i=1}^{|v|} \omega_{j,i} \delta_{\boldsymbol{x}_i}.$$
Because of the orthonormality is $\| \omega_{j} \|_2 = 1$ and by Cauchy-Schwarz
$$\| \omega_{j} \|_1 
= \langle \begin{pmatrix} | \omega_{j,1} | \\ \vdots \\ | \omega_{j,|v|} | \end{pmatrix},
\begin{pmatrix} 1 \\ \vdots \\ 1 \end{pmatrix} \rangle
\leq \| \omega_{j} \|_2 \cdot \| \overbrace{(1, ..., 1)}^{|v|-\text{times}} \|_2 = \sqrt{|v|}.$$
\item See the proof of Lemma 3.10 in \cite{HM1}. \qedhere
\end{enumerate}
\end{proof}

I omitted the verification of the last statement for two reasons. Firstly, it closely resembles the proof of Proposition \ref{4.3.4}, which we address later on. Secondly, this statement is not crucial for my thesis. Nevertheless, it is fundamental for data compression and singularity detection. 

Let $f:\Omega \rightarrow \mathbb{R}$ be an unknown function and let $\mathcal{D}_N = \{ (\boldsymbol{x}_1, f(\boldsymbol{x}_1) ), ..., (\boldsymbol{x}_N, f(\boldsymbol{x}_N) ) \}$ denote a set of observations. Our understanding of the function can be expressed as
\begin{equation}
f = \sum_{i=1}^N f_i \delta_{\boldsymbol{x}_i}, \label{eq. 4.2.6}
\end{equation}
where $f_i := f(\boldsymbol{x}_i) = (f, \delta_{\boldsymbol{x}_i})_{\Omega}$. Let
\begin{equation}
f = \sum_{i=1}^N f_i^{ \boldsymbol{\Sigma} } \sigma_i \label{eq. 4.2.7}
\end{equation}
be the representation with respect to the samplet basis. The subsequent lemma establishes $f_i^{ \boldsymbol{\Sigma} } = (f, \sigma_i)_{\Omega}$ such that the last statement of Theorem \ref{4.2.13} provides a bound for the coefficients in this representation. If $f$ is smooth, the samplet coefficients tend to be small, allowing them to be set to zero without affecting accuracy. Conversely, if the coefficients are large, it indicates singularity in the region supported by the samplets.

\begin{lemma} \label{4.2.14}
It is $f_i^{\boldsymbol{\Sigma}} = (f, \sigma_i)_{\Omega}$ for all $i \in \{ 1, ..., N \}$.
\end{lemma}

\begin{proof}
Write 
$$\sigma_i = \sum_{n=1}^N \omega_{i,n} \delta_{\boldsymbol{x}_n},$$ 
then
\begin{align*}
(\sigma_i, \sigma_{\ell})_{\Omega} 
= \left(\sum_{n=1}^N \omega_{i,n} \delta_{\boldsymbol{x}_n}, \sum_{j=1}^N \omega_{\ell,j} \delta_{\boldsymbol{x}_j} \right)_{\Omega} 
&= \sum_{n=1}^N \sum_{j=1}^N \omega_{i,n} \omega_{\ell,j} (\delta_{\boldsymbol{x}_n}, \delta_{\boldsymbol{x}_j})_{\Omega}\\
&= \sum_{n=1}^N \omega_{i,n} \omega_{\ell,n}
= \langle \sigma_i, \sigma_{\ell} \rangle_V = \begin{cases}
  1, & \text{if } i = \ell \\
  0, & \text{if } i \neq \ell 
\end{cases}
\end{align*}
and therefore
\[ \pushQED{\qed} 
(f, \sigma_i)_{\Omega} = \left(\sum_{n=1}^N f_n^{ \boldsymbol{\Sigma} } \sigma_n, \sigma_i \right)_{\Omega} 
= \sum_{n=1}^N f_n^{ \boldsymbol{\Sigma} } (\sigma_n, \sigma_i)_{\Omega} 
= f_i^{ \boldsymbol{\Sigma} }. \qedhere \]
\end{proof}

%-------------------------------------------------------------------------------------------
% Fast samplet transform
%-------------------------------------------------------------------------------------------

\subsubsection{Fast samplet transform} \label{Chap. 4.2.4}
We aim to efficiently convert between the representations (\ref{eq. 4.2.6}) and (\ref{eq. 4.2.7}) for an unknown function $f:\Omega \rightarrow \mathbb{R}$, given a set of observations $\mathcal{D}_N = \{ (\boldsymbol{x}_1, f_1 ), ..., (\boldsymbol{x}_N, f_N ) \}$. The fast samplet transform recursively applies the relation
\begin{equation}
[(f, \boldsymbol{\Phi}_{j}^v )_{\Omega} \, | \, (f, \boldsymbol{\Sigma}_{j}^v )_{\Omega}]
= (f, [\boldsymbol{\Phi}_{j}^v \, | \, \boldsymbol{\Sigma}_{j}^v] )_{\Omega}
= (f, \boldsymbol{\Phi}_{j+1}^v )_{\Omega} \boldsymbol{Q}_j^v \label{eq. 4.2.12}
\end{equation}
handled in section \ref{Chap. 4.2.2}, to attain the samplet representation. Here, we employ
$$(f, \boldsymbol{\Phi}_{j}^v )_{\Omega} =  [(f, \varphi_{j,1}^v )_{\Omega}, ..., (f, \varphi_{j, | \boldsymbol{\Phi}_{j}^v| }^v )_{\Omega}] \index{\textit{$(f, \boldsymbol{\Phi}_{j}^v )_{\Omega}$,}}$$
for $\boldsymbol{\Phi}_{j}^v = [\varphi_{j,1}^v, ..., \varphi_{j, | \boldsymbol{\Phi}_{j}^v|}^v]$.\\
The samplet transformation begins at the finest level $J$, because if $v = \{ \boldsymbol{x}_{i_1}, ..., \boldsymbol{x}_{i_{ |v| }} \}$ is a leaf cluster, then the coefficient vector
$$(f, \boldsymbol{\Phi}_{J}^v )_{\Omega} 
= [(f, \delta_{\boldsymbol{x}_{i_1}} )_{\Omega}, ..., (f, \delta_{\boldsymbol{x}_{i_{ |v| }}} )_{\Omega}] 
= [f_{i_1}, ..., f_{i_{ |v| }}]$$
is readily available. Using equation (\ref{eq. 4.2.12}), we can compute the coefficient vectors of its parent cluster. Utilizing this coefficient vectors at level $J-1$, we can determine the coefficient vectors on level $J-2$, and this recursive process continues until we calculate all coefficient vectors up to level $0$. With the definitions
\begin{align*}
\boldsymbol{f}_j^{\boldsymbol{\Sigma}} = \text{Samplet coefficient vector on level } j,\\
\boldsymbol{f}_j^{\boldsymbol{\Phi}} = \text{Scaling coefficient vector on level } j,
\end{align*}
we can visualize the fast samplet transform using a fishbone diagram, as illustrated in Figure \ref{fig. 100}.

\begin{figure}[H]
	\centering
	\includegraphics[width=0.8\textwidth]{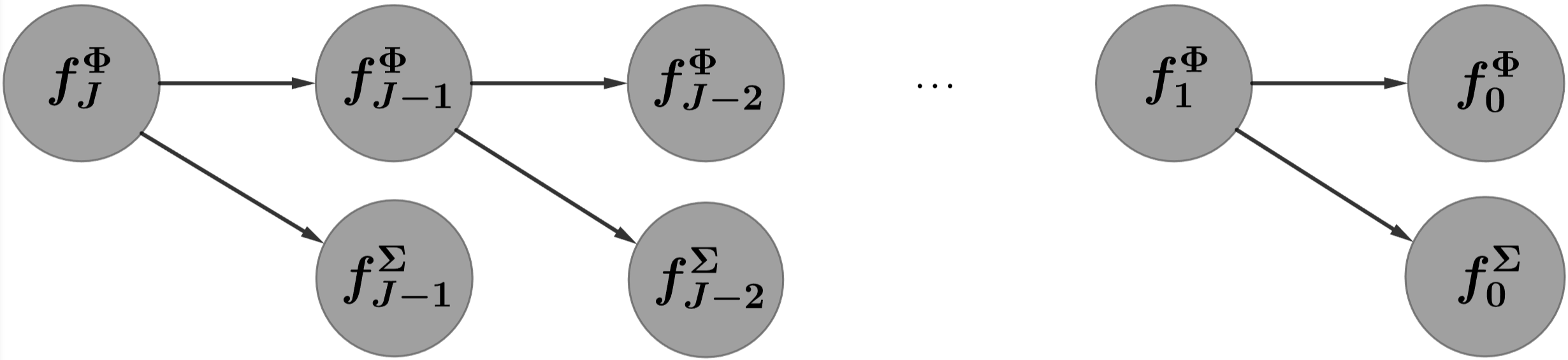}
	\caption{Visualization of the fast samplet transform. This figure is inspired by the fishbone graphic in \cite{HM1}.}\label{fig. 100}
\end{figure}

%-------------------------------------------------------------------------------------------

\begin{algorithm}[h]
\DontPrintSemicolon
\caption{Fast samplet transform}\label{alg. 4}
\KwData{Cluster tree $\mathcal{T}$, cluster $v$, coefficients $[f_1,...,f_N]$, $(\boldsymbol{Q}_j^v)_{v \in P \backslash \mathcal{L(P)}}$}
\KwResult{Samplet coefficients $[f_1^{ \boldsymbol{\Sigma} }, ..., f_N^{ \boldsymbol{\Sigma} }]$}
\SetKwFunction{FMain}{SampletTransform}
\SetKwProg{Fn}{Function}{:}{}
\Fn{\FMain{$v$}}{
	\eIf{children $v_{\text{son} 1}, v_{\text{son} 2}$ of $v$ are leaf clusters }{
	$\boldsymbol{f}_{j+1}^{v} = [ f_{i_1}, ..., f_{i_{| v |}} ], \text{ where } v_{\text{son} 1} 
	\cup v_{\text{son} 2} = \{ \boldsymbol{x}_{i_1}, ..., \boldsymbol{x}_{i_{| v |}} \}$\;} 
	{
	$\boldsymbol{f}_{j+1}^{v_{\text{son} 1}} = \texttt{SampletTransform}(v_{\text{son} 1})$\;
	$\boldsymbol{f}_{j+1}^{v_{\text{son} 2}} = \texttt{SampletTransform}(v_{\text{son} 2})$\;
	$\boldsymbol{f}_{j+1}^{v} = \texttt{append}(\boldsymbol{f}_{j+1}^{v_{\text{son} 1}}, 
	\boldsymbol{f}_{j+1}^{v_{\text{son} 2}})$ \;
	}
	$[(f, \boldsymbol{\Phi}_{j}^v )_{\Omega} \, | \, (f, \boldsymbol{\Sigma}_{j}^v )_{\Omega}] = 
	\boldsymbol{f}_{j+1}^{v} \boldsymbol{Q}_j^v$\;
	store $[f_{j, 1}^{\boldsymbol{\Sigma}}, ..., f_{j ,  | \boldsymbol{\Sigma}_j^v |  }^{\boldsymbol{\Sigma}}] = (f, \boldsymbol{\Sigma}_{j}^v )_{\Omega}$\;
  	\KwRet $(f, \boldsymbol{\Phi}_{j}^v )_{\Omega}$\;
}
\textbf{end}\;
store $(f, \boldsymbol{\Phi}_{0}^X )_{\Omega} = \texttt{SampletTransform}(X)$
\end{algorithm}

The fast samplet transform is formulated in Algorithm \ref{alg. 4}. Analogious we obtain the inverse transformation by reversing the steps and computing
$$(f, \boldsymbol{\Phi}_{j+1}^v )_{\Omega} = [(f, \boldsymbol{\Phi}_{j}^v )_{\Omega} \, | \, (f, \boldsymbol{\Sigma}_{j}^v )_{\Omega}] (\boldsymbol{Q}_j^v)^T.$$
Since the coefficients $(f, \boldsymbol{\Phi}_{0}^X )_{\Omega}$ and $(f, \boldsymbol{\Sigma}_{j}^v )_{\Omega}$ are for every $v \in P$ easily accessible, the inverse transformation initiates at level 0 and recursively determines the coefficients at levels $1,2, ..., J$. This process is outlined in Algorithm \ref{alg. 5}.

The following Porposition demonstrates, that the transformations entail a low computational effort, so that we are able to efficiently convert between the two representations of $f$.\\

\begin{algorithm}[H]
\DontPrintSemicolon
\caption{Inverse transform}\label{alg. 5}
\KwData{Cluster tree $\mathcal{T}$, cluster $v$, samplet coefficients $[f_1^{\boldsymbol{\Sigma}},..., f_N^{\boldsymbol{\Sigma}}]$, $(\boldsymbol{Q}_j^v)_{v \in P \backslash \mathcal{L(P)}}$}
\KwResult{Coefficients $[f_1, ..., f_N]$}
\SetKwFunction{FMain}{InverseTransform}
\SetKwProg{Fn}{Function}{:}{}
\Fn{\FMain{$v$, $(f, \boldsymbol{\Phi}_{j}^v )_{\Omega}$}}{
	$(f, \boldsymbol{\Phi}_{j+1}^v )_{\Omega} = [(f, \boldsymbol{\Phi}_{j}^v )_{\Omega} \, | 
	\, (f, \boldsymbol{\Sigma}_{j}^v )_{\Omega}] (\boldsymbol{Q}_{j}^v)^T$\;
	\eIf{$v = \{ \boldsymbol{x}_{i_1}, ..., \boldsymbol{x}_{i_{ |v| }} \}$ is the parent of a 
	leaf}{
	$ [f_{i_1}, ..., f_{i_{ |v| }}] = (f, \boldsymbol{\Phi}_{j+1}^v )_{\Omega}$\;
	\text{store } $ [f_{i_1}, ..., f_{i_{ |v| }}]$
	} 
	{
	\ForEach{son $v_{\text{son}}$ of $v$} {
    				assign part of $(f, \boldsymbol{\Phi}_{j+1}^v )_{\Omega}$ belonging to 
    				$v_{\text{son}}$ to $(f, \boldsymbol{\Phi}_{j + 1}^{v_{\text{son}}} 
    				)_{\Omega}$\;
    				$\texttt{InverseTransform}(v_{\text{son}}, (f, \boldsymbol{\Phi}
    				_{j + 1}^{v_{\text{son}}} )_{\Omega})$
    			}
	}
}
\textbf{end}\;

$\texttt{InverseTransform}(X, (f, \boldsymbol{\Phi}_0^{X} )_{\Omega})$\;
\end{algorithm}

\begin{proposition} \label{4.2.15}
The computational workload of both the fast samplet transform and its inverse counterpart is $\mathcal{O}(N)$.
\end{proposition}

\begin{proof}
Given that the balanced binary tree has $\mathcal{O}(N)$ clusters and we exert a constant effort per cluster, similar to Proposition \ref{4.2.12}, we deduce that the computational workload for calculating the respective coefficients is $\mathcal{O}(N)$.
\end{proof}

%-------------------------------------------------------------------------------------------
% Compression of the kernel matrix
%-------------------------------------------------------------------------------------------

\subsection{Compression of the kernel matrix}
Let $k: \Omega \times \Omega \rightarrow \mathbb{R}$ denote a kernel on $\Omega \subseteq \mathbb{R}^d$ and for points $X = \{ \boldsymbol{x}_1, ..., \boldsymbol{x}_N \} \subset \Omega$ let $\boldsymbol{K}=[k(\boldsymbol{x}_i, \boldsymbol{x}_j)]_{1 \leq i,j \leq N}$ be the kernel matrix. To ease the computation of Gaussian Processes, we are interested in approximately solving linear equations of the form
$$(\boldsymbol{K} + \sigma^2 \boldsymbol{I}_N) \boldsymbol{c} = \boldsymbol{y} \quad \text{for } \sigma^2 \geq 0.$$
This is necessary, for example, to compute the posterior mean in equation (\ref{eq. 2.2.1}). The idea is to compress the matrix
$$\boldsymbol{K} + \sigma^2 \boldsymbol{I}_N \quad \text{for } \sigma^2 \geq 0.$$ 
We will do this, by using a change of basis for $\boldsymbol{K}$, such that most entries of the transformed matrix
$$\boldsymbol{T} \boldsymbol{K} \boldsymbol{T}^{-1} +  \sigma^2 \boldsymbol{I}_N, \quad \boldsymbol{T} \in \text{GL}_N(\mathbb{R})$$
will be very small and we can set them to zero without affecting accuracy. To perform the change of basis, we will switch from the samplet basis of $V = \text{span} \{ \delta_{\boldsymbol{x}_1}, ..., \delta_{\boldsymbol{x}_N} \}$ to the canonical basis. We start by establishing the theoretical foundation to confirm that this concept is meaningful.

As samplets exhibit vanishing moments up to a certain degree, it is desirable that higher-order derivatives of the kernel are very small. The combination of vanishing moments and small derivatives will imply that compressing the kernel matrix is effective. Not all kernels fulfill this requirement, but the significant class of Matérn kernels, which I explore in my thesis, are capable of doing so.

\begin{defi} \label{4.3.1}
Let $\Omega \subseteq \mathbb{R}^d$ be open. The kernel $k$ is considered $q+1$ asymptotically smooth\index{Asymptotically smooth,} if there exists constants $C(\Omega, k)>0$ and $r > 0$, such that
$$\Big| \dfrac{\partial^{  | \boldsymbol{\alpha}  +  \boldsymbol{\beta} | } }{ \partial \boldsymbol{x}^{\boldsymbol{\alpha}} \boldsymbol{y}^{\boldsymbol{\beta}} } k (\boldsymbol{x}_0, \boldsymbol{y}_0) \Big| \leq 
C \dfrac{| \boldsymbol{\alpha}  +  \boldsymbol{\beta} |!}{(r \cdot \| \boldsymbol{x}_0-\boldsymbol{y}_0 \|_2)^{ | \boldsymbol{\alpha}  +  \boldsymbol{\beta} |  }}
\quad \text{for all } \boldsymbol{x}_0, \boldsymbol{y}_0 \in \Omega \text{ with } \boldsymbol{x}_0 \neq \boldsymbol{y}_0$$
and $\boldsymbol{\alpha} = (\alpha_1,...,\alpha_d), \boldsymbol{\beta} = (\beta_1,...,\beta_d) \in \mathbb{N}^d_{\geq 0}$ with $| \boldsymbol{\alpha} |, | \boldsymbol{\beta} | \leq q+1$.
\end{defi}

For the verification of the following result,  I have based my approach on the proof of Theorem E.8 in \cite{Ha}.

\begin{proposition} \label{4.3.2}
Let $\Omega \subset \mathbb{R}^d$ be open and bounded, $\ell >0$ and $\nu = n+ 1/2$ with $n \in \mathbb{N}_{\geq 0}$. The Matérn kernels $k_{\nu, \ell}$ are infinitely asymptotically smooth.
\end{proposition}

\begin{proof}
Without loss of generality let $\Omega = B_{\rho}(\boldsymbol{0}) \subset \mathbb{R}^d$ be the open ball of radius $\rho >0$ and center $\boldsymbol{0} \in \mathbb{R}^d$. After Remark \ref{2.3.5} a Matérn kernel can be written as
$$k_{\nu,\ell}(s)  
= \text{exp} \, \left( - \dfrac{\sqrt{2\nu} s}{\ell} \right)
\dfrac{\Gamma(n+1)}{\Gamma(2n+1)} \sum_{i=0}^n \dfrac{(n+i)!}{i!(n-i)!}
\left( \dfrac{\sqrt{8\nu} s}{\ell} \right)^{n-i}$$
for $s = \| \boldsymbol{x}_0 - \boldsymbol{y}_0 \|_2$. The further proof will be done in two steps.\\

%-------------------------------------------------------------------------------------------

\textbf{1. Step:} Reduction of the statement to inequality (\ref{eq. 4.3.13}).\\[5mm]
As $k_{\nu,\ell}$ is a stationary kernel, we have
$$\dfrac{\partial^{  | \boldsymbol{\alpha}  +  \boldsymbol{\beta} | } }{ \partial \boldsymbol{x}^{\boldsymbol{\alpha}} \boldsymbol{y}^{\boldsymbol{\beta}} } k (\boldsymbol{x}_0, \boldsymbol{y}_0)
= \dfrac{\partial^{  | \boldsymbol{\alpha}  +  \boldsymbol{\beta} | } }{ \partial \boldsymbol{x}^{\boldsymbol{\alpha}} \boldsymbol{y}^{\boldsymbol{\beta}} } k (\boldsymbol{x}_0 - \boldsymbol{y}_0, \boldsymbol{0})$$
such that it is enough to show the Proposition for $\boldsymbol{x}_0 \in B_{2 \rho}(\boldsymbol{0})-\boldsymbol{0}$ and $\boldsymbol{y}_0 = \boldsymbol{0}$. Using the symmetry of $k_{\nu,\ell}$, we receive
\begin{align*}
\dfrac{\partial}{\partial x_i} k_{\nu,\ell} (\boldsymbol{x}_0, \boldsymbol{0})
= \dfrac{\partial}{\partial t} k_{\nu,\ell} (  \boldsymbol{x}_0 + t\boldsymbol{e}_i, \boldsymbol{0}  ) \Big|_{t=0}
&= \dfrac{\partial}{\partial t} k_{\nu,\ell} ( \boldsymbol{0}, \boldsymbol{x}_0 + t\boldsymbol{e}_i  ) \Big|_{t=0}\\
&= \dfrac{\partial}{\partial y_i} k_{\nu,\ell} (\boldsymbol{0}, \boldsymbol{x}_0) 
= \dfrac{\partial}{\partial y_i} k_{\nu,\ell} (\boldsymbol{x}_0, \boldsymbol{0})
\end{align*}
for every $i \in \{ 1, ..., d \}$. Let $\{\boldsymbol{e}_1, ..., \boldsymbol{e}_d \}$ be the canonical basis of $\mathbb{R}^d$, then for every $i \in \{ 1, ..., d \}$ exists an orthogonal matrix $\boldsymbol{Q}_i \in \mathbb{R}^{d \times d}$ with $\boldsymbol{Q}_i \boldsymbol{e}_i = \boldsymbol{e}_1$. This implies

\begin{align*}
\dfrac{\partial}{\partial y_i} k_{\nu,\ell} (&\boldsymbol{x}_0, \boldsymbol{0}) 
= \dfrac{\partial}{\partial t} k_{\nu,\ell} (  \boldsymbol{x}_0, t \boldsymbol{e}_i ) \Big|_{t=0}
= \dfrac{\partial}{\partial t} k_{\nu,\ell} ( \| \boldsymbol{x}_0 - t \boldsymbol{e}_i \|_2 ) \Big|_{t=0}\\
&= \dfrac{\partial}{\partial t} k_{\nu,\ell} ( \| \boldsymbol{Q}_i \boldsymbol{x}_0 - t \boldsymbol{Q}_i \boldsymbol{e}_i \|_2 ) \Big|_{t=0}
= \dfrac{\partial}{\partial t} k_{\nu,\ell} ( \| \boldsymbol{Q}_i \boldsymbol{x}_0 - t \boldsymbol{e}_1 \|_2 ) \Big|_{t=0}
= \dfrac{\partial}{\partial y_1} k_{\nu,\ell} (\boldsymbol{Q}_i \boldsymbol{x}_0, \boldsymbol{0}).
\end{align*}
It follows

\begin{align*}
\dfrac{\partial^{  | \boldsymbol{\alpha}  +  \boldsymbol{\beta} | } }{ \partial \boldsymbol{x}^{\boldsymbol{\alpha}} \boldsymbol{y}^{\boldsymbol{\beta}} } k (\boldsymbol{x}_0, \boldsymbol{0})
= \dfrac{\partial^{  | \boldsymbol{\alpha}  +  \boldsymbol{\beta} | } }{ \partial \boldsymbol{y}^{\boldsymbol{\alpha} + \boldsymbol{\beta}} } k (\boldsymbol{x}_0, \boldsymbol{0})
= \dfrac{\partial^{  | \boldsymbol{\alpha}  +  \boldsymbol{\beta} | } }{ \partial y_1^{| \boldsymbol{\alpha} + \boldsymbol{\beta} | } } k (\boldsymbol{Q}_2^{\alpha_2+\beta_2}...\boldsymbol{Q}_d^{\alpha_d+\beta_d} \boldsymbol{x}_0, \boldsymbol{0}).
\end{align*}
Therefore, it is sufficient to verify

\begin{equation}
\Big| \dfrac{\partial^{  m } }{ \partial y_1^m } k (\boldsymbol{x}, \boldsymbol{0}) \Big|
\leq C \dfrac{m!}{  r^m \| \boldsymbol{x} \|_2^m } \quad \text{for all }  \boldsymbol{x} \in B_{2 \rho}(\boldsymbol{0})-\boldsymbol{0} \label{eq. 4.3.13}
\end{equation}
and $m \in \mathbb{N}_0$.\\

\textbf{2. Step:} Inequality (\ref{eq. 4.3.13}) is valid.\\[5mm]
The Matérn kernels $k_{\nu, \ell}$ can be defined not only for $s \geq 0$, but also for complex numbers $z \in \mathbb{C}$. Since $k_{\nu, \ell}(z)$ is a product of an exponential function and a polynomial, both of which are holomorphic, $k_{\nu, \ell}$ is also holomorphic. For $R>0$, $\boldsymbol{x} \in B_{2 \rho}(\boldsymbol{0})-\boldsymbol{0}$ and $\delta^2 = x_2^2 + ...+ x_d^2 \geq 0$ define
$$f_{\boldsymbol{x}}: B_R (0) \subset \mathbb{C} \rightarrow \mathbb{C}, f_{\boldsymbol{x}}(z) = k_{\nu, \ell} \left(\sqrt{(x_1-z)^2 + \delta^2} \right),$$
where $\sqrt{z} = e^{\log(z)/2}$ is the square root of a complex number. The function $f_{\boldsymbol{x}}$ is well-defined and holomorphic only if
$$(x_1-z)^2 + \delta^2 \neq 0$$
for all $z \in B_R (0)$. The corresponding equation has the two roots
$$z_{1,2} = x_1 \pm i \delta$$
and therefore we choose $R(\boldsymbol{x}) = |z_{1,2} | /2 = \left(\sqrt{x_1^2 + \delta^2} \right)/2 >0$ to guarantee that $f_{\boldsymbol{x}}$ is holomorphic on $\overline{B_{R(\boldsymbol{x})} (0)}$. Notice that
$$R(\boldsymbol{x}) = \| \boldsymbol{x} \|_2 /2 < \rho.$$
Since $k_{\nu, \ell}$ is continuous, there exists some constant $C(\Omega, k_{\nu, \ell}) >0$ with

%-------------------------------------------------------------------------------------------

$$\sup_{ \xi \in \partial B_{R(\boldsymbol{x})} (0) } | f_{\boldsymbol{x}} (\xi) | 
\leq \sup_{ \boldsymbol{x} \in B_{2 \rho} (\boldsymbol{0}) , \xi \in \overline{B_{\rho} (0)} } \big| k_{\nu, \ell} \left(\sqrt{(x_1-\xi)^2 + \delta^2} \right) \big|
\leq \sup_{ z \in \overline{B_{3 \rho} (0)} } | k_{\nu, \ell}(z) | \leq C.$$
Setting $r = 2^{-1}$ and using the Cauchy estimate yields

$$\Big| \dfrac{\partial^{  m } }{ \partial y_1^m } k (\boldsymbol{x}, \boldsymbol{0}) \Big|
= \Big| \dfrac{\partial^{  m } }{ \partial z^m } f_{\boldsymbol{x}}(0) \Big|
= \big| f_{\boldsymbol{x}}^{(m)}(0) \big| \leq \dfrac{m!}{R(\boldsymbol{x})^m} \sup_{ \xi \in \partial B_{R(\boldsymbol{x})} (0) } | f_{\boldsymbol{x}} (\xi) |  
\leq C \dfrac{m!}{  r^m \| \boldsymbol{x} \|_2^m }. \qedhere$$
\end{proof}

Denote by $\delta_{\boldsymbol{x}_i} \otimes \delta_{\boldsymbol{x}_j}: C(\Omega \times \Omega) \rightarrow \mathbb{R}$\index{\textit{$\delta_{\boldsymbol{x}_i} \otimes \delta_{\boldsymbol{x}_j}$,}} the two point evaluation defined by
\begin{equation}
(\mathcal{K}, \delta_{\boldsymbol{x}_i} \otimes \delta_{\boldsymbol{x}_j})_{\Omega \times \Omega} := ( \delta_{\boldsymbol{x}_i} \otimes \delta_{\boldsymbol{x}_j} ) (\mathcal{K}) = \mathcal{K} (\boldsymbol{x}_i, \boldsymbol{x}_j). \label{eq. 4.3.14} \index{\textit{$(\mathcal{K}, \delta_{\boldsymbol{x}_i} \otimes \delta_{\boldsymbol{x}_j})_{\Omega \times \Omega}$,}}
\end{equation}
The notation is appropriately chosen since for every $\mathcal{K} \in C(\Omega \times \Omega)$, a bilinear map exists, which facilitates a tensor product defined by (\ref{eq. 4.3.14}):

\begin{figure}[H]
	\centering
	\includegraphics[width=0.4\textwidth]{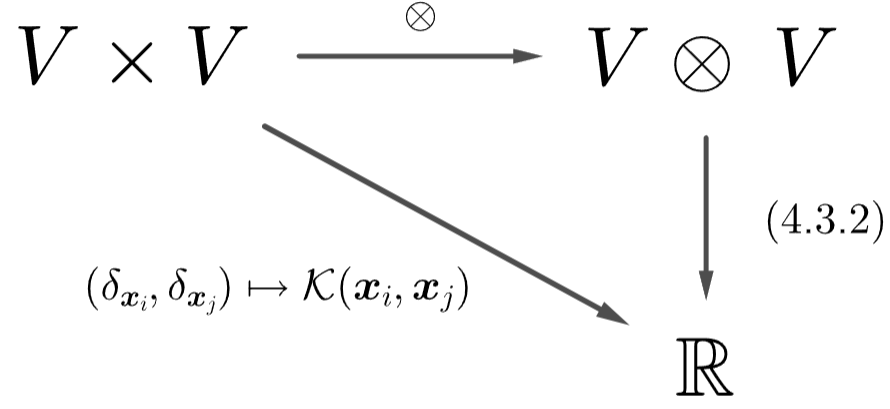}
\end{figure}

\noindent
For a kernel $k: \Omega \times \Omega \rightarrow \mathbb{R}$, this notation enables us to express the kernel matrix as
$$\boldsymbol{K} = [(k, \delta_{\boldsymbol{x}_i} \otimes \delta_{\boldsymbol{x}_j})_{\Omega \times \Omega}]_{1 \leq i,j \leq N}.$$
Let $\sigma_j$ and $\sigma_{j'}$ be samplets. Similar to the previous case where we focused on the coefficients $(f,\sigma_j)_\Omega$ of a function $f:\Omega \rightarrow \mathbb{R}$ for data compression, we are now interested in the coefficients $(k,\sigma_j \otimes \sigma_{j'})_{\Omega \times \Omega}$ to achieve matrix compression. It is desirable for these coefficients to be very small, allowing us to compress the kernel matrix in the samplet representation without affecting accuracy. To investigate these coefficients, we need the calculation rules provided in the following Lemma.

\begin{lemma} \label{4.3.3}
Let $\varphi_j, \varphi_{j'} \in V$ and $\mathcal{K} \in C(\Omega \times \Omega)$.
\begin{enumerate}[label=(\roman*),topsep=5pt]
\item If $\mathcal{K}(\boldsymbol{x},\boldsymbol{y}) = f(\boldsymbol{x})g(\boldsymbol{y})$ for $f,g:C(\Omega) \rightarrow \mathbb{R}$, then \label{4.3.3 (1)}
$$(\mathcal{K}, \varphi_{j} \otimes \varphi_{j'})_{\Omega \times \Omega} = (f, \varphi_j)_{\Omega} \cdot (g, \varphi_{j'})_{\Omega}.$$ \vspace{-0.85cm}
\item It is\label{4.3.3 (2)}
$|(\mathcal{K}, \varphi_{j} \otimes \varphi_{j'})_{\Omega \times \Omega}| 
\leq (| \mathcal{K}|, | \varphi_{j} | \otimes | \varphi_{j'} |)_{\Omega \times \Omega}.$
\item If $\mathcal{K} = \mathcal{K}_1 \mathcal{K}_2$ and $\varphi_j, \varphi_{j'}$ have support in clusters $v,v'$, then  \label{4.3.3 (3)}
$$(| \mathcal{K}|, | \varphi_{j} | \otimes | \varphi_{j'} |)_{\Omega \times \Omega} 
\leq (| \mathcal{K}_1|, | \varphi_{j} | \otimes | \varphi_{j'} |)_{\Omega \times \Omega}
\cdot \sup_{\boldsymbol{x} \in v,\boldsymbol{y} \in v'} |\mathcal{K}_2(\boldsymbol{x}, \boldsymbol{y})|.$$ \vspace{-0.85cm}
\item Let $\boldsymbol{T}$ be the change of basis matrix from samplet basis $\boldsymbol{\Sigma}_J = \{ \sigma_j \ | \ 1 \leq j \leq N \}$ to canonical basis. Then, \label{4.3.3 (4)} $\boldsymbol{T}\boldsymbol{K} \boldsymbol{T}^{-1} = [(k, \sigma_{j} \otimes \sigma_{j'})_{\Omega \times \Omega}]_{1 \leq j,j' \leq N}$.
\end{enumerate}
\end{lemma}

%-------------------------------------------------------------------------------------------

\begin{proof}
Let $\varphi_j, \varphi_{j'} \in V$ have coefficient vectors $\omega_j=[\omega_{j,i}]_{i=1}^N, \omega_{j'}=[\omega_{j',i'}]_{i'=1}^N$ in $\mathbb{R}^N$.
\begin{enumerate}[label=(\roman*),topsep=5pt]
\item \makebox[\linewidth]{\(\begin{aligned}[t]
(\mathcal{K}, \varphi_{j} \otimes \varphi_{j'})_{\Omega \times \Omega} 
&= \sum_{i=1}^N \sum_{i'=1}^N \omega_{j,i} \omega_{j',i'} (fg, \delta_{\boldsymbol{x}_i} \otimes \delta_{\boldsymbol{x}_{i'}})_{\Omega \times \Omega} \\
&= \sum_{i=1}^N \sum_{i'=1}^N \omega_{j,i} \omega_{j',i'} f(\boldsymbol{x}_{i}) g(\boldsymbol{x}_{i'}) 
= \sum_{i=1}^N \omega_{j,i} f(\boldsymbol{x}_{i}) \sum_{i'=1}^N \omega_{j',i'} g(\boldsymbol{x}_{i'}) \\
&= \sum_{i=1}^N \omega_{j,i} (f, \delta_{\boldsymbol{x}_i})_{\Omega} \sum_{i'=1}^N \omega_{j',i'} (g, \delta_{\boldsymbol{x}_{i'}})_{\Omega}
= (f, \varphi_j)_{\Omega} \cdot (g, \varphi_{j'})_{\Omega}.
\end{aligned}\)}

\item Notice, that $| \varphi_{j} | = \sum_{i=1}^N | \omega_{j,i}| \delta_{\boldsymbol{x}_i}.$
Therefore, 
\begin{align*}
|(\mathcal{K}, \varphi_{j} \otimes \varphi_{j'})_{\Omega \times \Omega}| 
&= \big| \sum_{i=1}^N \sum_{i'=1}^N \omega_{j,i} \omega_{j',i'} \mathcal{K}(\boldsymbol{x}_i, \boldsymbol{x}_{i'}) \big| \\
&\leq \sum_{i=1}^N \sum_{i'=1}^N | \omega_{j,i}| | \omega_{j',i'} | | \mathcal{K}(\boldsymbol{x}_i, \boldsymbol{x}_{i'})| \\
&= \sum_{i=1}^N \sum_{i'=1}^N | \omega_{j,i}| | \omega_{j',i'} | (| \mathcal{K} |, \delta_{\boldsymbol{x}_i} \otimes \delta_{\boldsymbol{x}_{i'}})_{\Omega \times \Omega} \hspace{3.5cm} \\
&= (| \mathcal{K}|, | \varphi_{j} | \otimes | \varphi_{j'} |)_{\Omega \times \Omega}.
\end{align*}

\item \makebox[\linewidth]{\(\begin{aligned}[t]
(| \mathcal{K}|, | \varphi_{j} | \otimes | \varphi_{j'} |)_{\Omega \times \Omega} 
&= \sum_{i=1}^N \sum_{i'=1}^N | \omega_{j,i}| | \omega_{j',i'} | | \mathcal{K}_1(\boldsymbol{x}_i, \boldsymbol{x}_{i'})| | \mathcal{K}_2(\boldsymbol{x}_i, \boldsymbol{x}_{i'})| \\
&\leq \sup_{\boldsymbol{x} \in v,\boldsymbol{y} \in v'} |\mathcal{K}_2(\boldsymbol{x}, \boldsymbol{y})| 
\cdot \sum_{i=1}^N \sum_{i'=1}^N | \omega_{j,i}| | \omega_{j',i'} | | \mathcal{K}_1(\boldsymbol{x}_i, \boldsymbol{x}_{i'})| \hspace{1cm} \\
&= \sup_{\boldsymbol{x} \in v,\boldsymbol{y} \in v'} |\mathcal{K}_2(\boldsymbol{x}, \boldsymbol{y})| 
\cdot (| \mathcal{K}_1|, | \varphi_{j} | \otimes | \varphi_{j'} |)_{\Omega \times \Omega}.
\end{aligned}\)}

\item Let $\sigma_j, \sigma_{j'} \in \boldsymbol{\Sigma}_J$ have coefficient vectors $\omega_j=[\omega_{j,i}]_{i=1}^N, \omega_{j'}=[\omega_{j',i'}]_{i'=1}^N$ in $\mathbb{R}^N$. 
\begin{align*}
\begin{bmatrix} \omega_{j,1} & \cdots & \omega_{j,N} \end{bmatrix} 
\boldsymbol{K}
\begin{bmatrix} \omega_{j',1} \\ \vdots \\ \omega_{j',N} \end{bmatrix} 
&= \sum_{i=1}^N \omega_{j,i} \sum_{i'=1}^N (k, \delta_{\boldsymbol{x}_i} \otimes \delta_{\boldsymbol{x}_{i'}})_{\Omega \times \Omega}  \omega_{j',i'} \\
&= \sum_{i=1}^N \sum_{i'=1}^N \omega_{j,i} \omega_{j',i'} (k, \delta_{\boldsymbol{x}_i} \otimes \delta_{\boldsymbol{x}_{i'}})_{\Omega \times \Omega} \\
&= (k, \sigma_{j} \otimes \sigma_{j'})_{\Omega \times \Omega} 
\end{align*}

The Matrix $\boldsymbol{T}$ is the change of basis matrix from $\boldsymbol{\Sigma}_J$ to canonical basis, i. e.
$$\boldsymbol{T} = \begin{bmatrix}
\omega_{1,1} & \cdots & \omega_{1,N}\\
\vdots & \ddots & \vdots \\
\omega_{N,1} & \cdots & \omega_{N,N}
\end{bmatrix}.$$
As the coefficient vectors are orthonormal, we have $\boldsymbol{T}^{-1} = \boldsymbol{T}^T$ and therefore $\boldsymbol{T} \boldsymbol{K} \boldsymbol{T}^{-1} = [(k, \sigma_{j} \otimes \sigma_{j'})_{\Omega \times \Omega}]_{1 \leq j,j' \leq N}$. \qedhere
\end{enumerate}
\end{proof}

%-------------------------------------------------------------------------------------------

The following proposition extends statement \ref{4.2.13 (5)} in Theorem \ref{4.2.13} to the case of a kernel with two input variables. It facilitates the compression of specific kernel matrices, with its proof provided by \cite{HM1}.

\begin{proposition} \label{4.3.4}
Let $\Omega \subseteq \mathbb{R}^d$ be open and convex. Consider two samplets $\sigma_{j}, \sigma_{j'}$ with vanishing moments up to degree $q$, support in clusters $v, v'$ and coefficient vectors $\omega_j, \omega_{j'}$. If $k$ is $q+1$ asymptotically smooth and $\text{dist}(B_v,B_{v'})>0$, then
$$| (k, \sigma_{j} \otimes \sigma_{j'} )_{\Omega} | \leq 
C \dfrac{d^{2(q+1)} \text{diam}(v)^{q+1} \text{diam}(v')^{q+1} }{ (r \cdot \text{dist}(B_v,B_{v'}))^{2(q+1)} } \| \omega_j \|_1 \| \omega_{j'} \|_1,$$
where $C(\Omega, k),r>0$ are the constants from the asymptotically smoothness property.
\end{proposition}

\begin{proof}
The Taylor expansion of $k$ in $\boldsymbol{x}_0 \in \Omega$ is
$$k(\boldsymbol{x}, \boldsymbol{y}) 
= \sum_{ | \boldsymbol{\alpha} | \leq q } \dfrac{(\boldsymbol{x}-\boldsymbol{x}_0)^{\boldsymbol{\alpha}}}{ \boldsymbol{\alpha} ! } \dfrac{\partial^{ | \boldsymbol{\alpha} |} }{ \partial \boldsymbol{x}^{\boldsymbol{\alpha}} } k(\boldsymbol{x}_0, \boldsymbol{y}) 
+  
\overbrace{\sum_{ | \boldsymbol{\alpha} | = q + 1 } 
\dfrac{(\boldsymbol{x}-\boldsymbol{x}_0)^{\boldsymbol{\alpha}}}{ \boldsymbol{\alpha} ! }
\dfrac{\partial^{ | \boldsymbol{\alpha} |} }{ \partial \boldsymbol{x}^{\boldsymbol{\alpha}} } k(\boldsymbol{\xi}_{\boldsymbol{x}_0} (\boldsymbol{x}), \boldsymbol{y})}^{R_{\boldsymbol{x}_0} (\boldsymbol{x}, \boldsymbol{y})},$$
where $\boldsymbol{ \xi}_{\boldsymbol{x}_0} (\boldsymbol{x}) = \boldsymbol{x}_0 + s(\boldsymbol{x} - \boldsymbol{x}_0)$ for some $s \in [0,1]$. The taylor expansion of $k$ in $\boldsymbol{y}_0 \in \Omega$ allows writing
$$R_{\boldsymbol{x}_0} (\boldsymbol{x}, \boldsymbol{y}) = 
\sum_{ | \boldsymbol{\alpha} | = q + 1 }  \sum_{ | \boldsymbol{\beta} | \leq q }
\dfrac{(\boldsymbol{y}-\boldsymbol{y}_0)^{\boldsymbol{\beta}}}{ \boldsymbol{\beta} ! }
\dfrac{(\boldsymbol{x}-\boldsymbol{x}_0)^{\boldsymbol{\alpha}}}{ \boldsymbol{\alpha} ! }
 \dfrac{\partial^{ | \boldsymbol{\alpha} + \boldsymbol{\beta} |} }{ \partial \boldsymbol{x}^{\boldsymbol{\alpha}} \boldsymbol{y}^{\boldsymbol{\beta}} } k(\boldsymbol{\xi}_{\boldsymbol{x}_0}(\boldsymbol{x}), \boldsymbol{y}_0) 
+ R_{\boldsymbol{x}_0, \boldsymbol{y}_0} (\boldsymbol{x}, \boldsymbol{y})$$
with
$$R_{\boldsymbol{x}_0, \boldsymbol{y}_0} (\boldsymbol{x}, \boldsymbol{y}) = 
\sum_{ | \boldsymbol{\alpha} |, \boldsymbol{\beta} | = q + 1 }
\dfrac{(\boldsymbol{y}-\boldsymbol{y}_0)^{\boldsymbol{\beta}}}{ \boldsymbol{\beta} ! }
\dfrac{(\boldsymbol{x}-\boldsymbol{x}_0)^{\boldsymbol{\alpha}}}{ \boldsymbol{\alpha} ! }
\dfrac{\partial^{ 2(q+1)} }{ \partial \boldsymbol{x}^{\boldsymbol{\alpha}} \boldsymbol{y}^{\boldsymbol{\beta}} } k(\boldsymbol{\xi}_{\boldsymbol{x}_0}(\boldsymbol{x}), \boldsymbol{\xi}_{\boldsymbol{y}_0}(\boldsymbol{y})),$$
where $\boldsymbol{ \xi}_{\boldsymbol{y}_0} (\boldsymbol{y}) = \boldsymbol{y}_0 + t(\boldsymbol{y} - \boldsymbol{y}_0)$ for some $t \in [0,1]$. We obtain a decomposition
$$k(\boldsymbol{x}, \boldsymbol{y}) = \sum_{ | \boldsymbol{\alpha} | \leq q } f_{\boldsymbol{\alpha}}(\boldsymbol{x}) b_{\boldsymbol{\alpha}} (\boldsymbol{y}) 
+ \sum_{ | \boldsymbol{\beta} | \leq q } g_{\boldsymbol{\beta}}(\boldsymbol{y}) c_{\boldsymbol{\beta}} (\boldsymbol{x}) 
+ R_{\boldsymbol{x}_0, \boldsymbol{y}_0} (\boldsymbol{x}, \boldsymbol{y}),$$
where 
$$f_{\boldsymbol{\alpha}}(\boldsymbol{x}) = \dfrac{(\boldsymbol{x}-\boldsymbol{x}_0)^{\boldsymbol{\alpha}}}{ \boldsymbol{\alpha} ! }, \qquad
g_{\boldsymbol{\beta}}(\boldsymbol{y}) = \dfrac{(\boldsymbol{y}-\boldsymbol{y}_0)^{\boldsymbol{\beta}}}{ \boldsymbol{\beta} ! }$$
are polynomials of maximum degree $q$. With declaration \ref{4.3.3 (1)} in Lemma \ref{4.3.3} and the property that the samplets possess vanishing moments up to degree $q$, we receive
$$(k, \sigma_{j} \otimes \sigma_{j'} )_{\Omega \times \Omega} = (R_{\boldsymbol{x}_0, \boldsymbol{y}_0}, \sigma_{j} \otimes \sigma_{j'} )_{\Omega \times \Omega}.$$
By using the asymptotically smoothness of the kernel and \ref{4.3.3 (2)} in Lemma \ref{4.3.3}, we get
\begin{align*}
| (R_{\boldsymbol{x}_0, \boldsymbol{y}_0}, \sigma_{j} &\otimes \sigma_{j'} )_{\Omega \times \Omega} | \\ \leq 
&\dfrac{C}{r^{2(q+1)}}
\sum_{ | \boldsymbol{\alpha} |, \boldsymbol{\beta} | = q + 1 } 
\dfrac{|\boldsymbol{\alpha}+ \boldsymbol{\beta}|!}{\boldsymbol{\alpha}! \boldsymbol{\beta}!}
\left( \dfrac{| (\boldsymbol{y}-\boldsymbol{y}_0)^{\boldsymbol{\beta}}(\boldsymbol{x}-\boldsymbol{x}_0)^{\boldsymbol{\alpha}} |}
{ \| \boldsymbol{\xi}_{\boldsymbol{x}_0}(\boldsymbol{x}) - \boldsymbol{\xi}_{\boldsymbol{y}_0}(\boldsymbol{y}) \|_2^{2(q+1)}},| \sigma_{j} | \otimes | \sigma_{j'} | \right)_{\Omega \times \Omega} .
\end{align*}
%-------------------------------------------------------------------------------------------
In the following we examine particular terms to receive the final upper bound. Let $\boldsymbol{x},\boldsymbol{x}_0 \in B_v$ and $\boldsymbol{y},\boldsymbol{y}_0 \in B_{v'}$ then are $\boldsymbol{\xi}_{\boldsymbol{x}_0}(\boldsymbol{x}) \in B_v$ and $\boldsymbol{\xi}_{\boldsymbol{y}_0}(\boldsymbol{y}) \in B_{v'}$. Therefore,
\begin{align*}
\sup_{\boldsymbol{x} \in v,\boldsymbol{y} \in v'} \, \| \boldsymbol{\xi}_{\boldsymbol{x}_0}(\boldsymbol{x}) - \boldsymbol{\xi}_{\boldsymbol{y}_0} (\boldsymbol{y}) \|_2^{ - 2(q+1)  } \leq \text{dist}(B_v,B_{v'})^{-2(q+1)}.
\end{align*}
Next is
\begin{align*}
\sum_{ | \boldsymbol{\alpha} |, \boldsymbol{\beta} | = q + 1 } \dfrac{|\boldsymbol{\alpha}+ \boldsymbol{\beta}|!}{\boldsymbol{\alpha}! \boldsymbol{\beta}!}
&= \sum_{ | \boldsymbol{\alpha} |, \boldsymbol{\beta} | = q + 1 } \binom{| \boldsymbol{\alpha}+ \boldsymbol{\beta}|}{| \boldsymbol{\beta} |} 
\dfrac{|\boldsymbol{\alpha}|! | \boldsymbol{\beta}|!}{\boldsymbol{\alpha}! \boldsymbol{\beta}!}\\
&= \binom{2(q+1)}{q+1} \sum_{ | \boldsymbol{\alpha} |, \boldsymbol{\beta} | = q + 1 } 
\dfrac{(q+1)! (q+1)!}{\boldsymbol{\alpha}! \boldsymbol{\beta}!} \\
&\leq 4^{q+1} \sum_{ | \boldsymbol{\alpha} |, \boldsymbol{\beta} | = q + 1 } 
\dfrac{(q+1)! (q+1)!}{\boldsymbol{\alpha}! \boldsymbol{\beta}!} = 4^{q+1} d^{2(q+1)},
\end{align*}
where we used
$$\sum_{n=0}^{2(q+1)} \binom{2(q+1)}{n} = (1+1)^{2(q+1)} = 4^{q+1}$$
for the first inequality and the formula

$$d^{q+1}=\left(\sum_{i=1}^d 1 \right)^{q+1} = \sum_{| \boldsymbol{\alpha} | = q+1} \dfrac{(q+1)!}{ \boldsymbol{\alpha} !},$$
provided by the multinomial theorem, for the last equality. Furthermore
\begin{align*}
(|(\boldsymbol{x}-\boldsymbol{x}_0 )^{\boldsymbol{\alpha}}|, | \sigma_j | )_{\Omega} 
&= \sum_{i=1}^N | \omega_{j,i} | |(\boldsymbol{x}_i-\boldsymbol{x}_0)^{\boldsymbol{\alpha}}| \\
&\leq \sum_{i=1}^N | \omega_{j,i} | \left(\dfrac{\text{diam}(v)}{2} \right)^{ | \boldsymbol{\alpha} |} 
= \dfrac{\text{diam}(v)^{ | \boldsymbol{\alpha} | }}{2^{ | \boldsymbol{\alpha} | }} \| \omega_{j} \|_1,
\end{align*}
where the inequality follows by choosing $\boldsymbol{x}_0$ as $B_v$'s midpoint, i. e.
$$\boldsymbol{x}_0 = \left(\max_{\boldsymbol{x} \in v} \, x_{1} /2 + \min_{\boldsymbol{y} \in v} \, y_{1} /2, ..., \max_{\boldsymbol{x} \in v} \, x_{d} /2 + \min_{\boldsymbol{y} \in v} \, y_{d} /2 \right)^T.$$
By combining all these inequalities and using \ref{4.3.3 (1)}, \ref{4.3.3 (3)} of Lemma \ref{4.3.3}, we receive
\begin{align*}
|(k, \sigma_{j} &\otimes \sigma_{j'} )_{\Omega \times \Omega}| \\
&\leq \dfrac{C}{r^{2(q+1)}} 4^{q+1}d^{2(q+1)} \text{dist}(B_v,B_{v'})^{-2(q+1)} \dfrac{\text{diam}(v)^{ q+1 } \text{diam}(v')^{ q+1 }}{2^{ q+1 } 2^{ q+1 }} \| \omega_{j} \|_1 \| \omega_{j'} \|_1 \\
&= C \dfrac{d^{2(q+1)} \text{diam}(v)^{q+1} \text{diam}(v')^{q+1} }{ (r \cdot \text{dist}(B_v,B_{v'}))^{2(q+1)} } \| \omega_{j} \|_1 \| \omega_{j'} \|_1 \qedhere
\end{align*}
\end{proof}

Let $\boldsymbol{T}$ be the change of basis matrix from samplet basis $\boldsymbol{\Sigma}_J = \{ \sigma_1, ..., \sigma_N \}$ to canonical basis. Using statement \ref{4.3.3 (4)} of Lemma \ref{4.3.3}, we define the kernel matrix in samplet representation as
$$\boldsymbol{K}^{\boldsymbol{\Sigma}} 
:= \boldsymbol{T} \boldsymbol{K} \boldsymbol{T}^{-1}
= [(k, \sigma_{j} \otimes \sigma_{j'})_{\Omega \times \Omega}]_{1 \leq j,j' \leq N}. \index{\textit{$\boldsymbol{K}^{\boldsymbol{\Sigma}}$,}}
$$
%-------------------------------------------------------------------------------------------
We compress the matrix $\boldsymbol{K}^{\boldsymbol{\Sigma}}$ by setting its small entries to zero through thresholding. Directly calculating all entries only to zero out most of them would be computationally expensive, so we seek for a simple criterion to determine which entries to skip. For $\eta>0$\index{\textit{$\eta$,}} and two clusters $v$ and $v'$, we consider the admissibility condition
\begin{equation}
\text{dist}(B_v,B_{v'}) \geq \eta \max \{ \text{diam}(B_v), \text{diam}(B_{v'}) \}. \label{eq. 4.3.15}
\end{equation}
This inequality is straightforward to check since for $B_v = \bigtimes_{i=1}^d [a_{v,i}, b_{v,i}]$ is
$$\text{diam}(B_v) = \left( \sum_{i=1}^d (b_{v,i} - a_{v,i})^2 \right)^{1/2}$$
and
$$\text{dist}(B_v, B_{v'}) = \left( \sum_{i=1}^d \max \{ 0, a_{v,i} - b_{v',i} \}^2 + \max \{ 0, a_{v',i} - b_{v,i} \}^2 \right)^{1/2}.$$
Storing the bounding boxes for each cluster in Algorithm \ref{alg. 2} ensures that computing these values for a single cluster has a constant cost. However, with $\mathcal{O}(N)$ clusters, the overall computation cost would be $\mathcal{O}(N^2)$. Later, we will demonstrate that the actual computational cost is lower, specifically $\mathcal{O}(N \log N)$.\\

The following Theorem, with its proof provided by \cite{HM1}, demonstrates that the compression strategy is decent, with the parameter $\eta$ determining the extent of compression. For the remainder of this chapter, we adopt the same settings as in Proposition \ref{4.3.4}. Additionally, we require that the cluster tree is sufficiently deep by setting the threshold to less than $m_q +1$, allowing us to apply Lemma \ref{4.2.9} and ensuring that no cluster provides more than $m_q$ samplets.

\begin{theorem} \label{4.3.5}
Let $(X_N)_{N\in \mathbb{N}}$ be a sequence of points with $q_{X_N} \geq C_1 N^{-1/d}$ for some constant $C_1(\Omega)>0$ and set all coefficients of
$$\boldsymbol{K}^{\boldsymbol{\Sigma}} = [(k, \sigma_{j} \otimes \sigma_{j'})_{\Omega \times \Omega}]_{1 \leq j,j' \leq N}$$
to zero that satisfy the admissibility condition (\ref{eq. 4.3.15}). Then the compressed matrix $\boldsymbol{K}^{\boldsymbol{\Sigma}}_{\eta}$\index{\textit{$\boldsymbol{K}^{\boldsymbol{\Sigma}}_{\eta}$,}} satisfies
$$\| \boldsymbol{K}^{\boldsymbol{\Sigma}} - \boldsymbol{K}^{\boldsymbol{\Sigma}}_{\eta} \|_F \leq C \eta^{-2(q+1)} N \sqrt{\log N}$$
for some constant $C( \Omega, k, q) >0$, where $ \| \cdot \|_F$\index{\textit{$\lVert \, \cdot \, \rVert_F$,}} denotes the Frobenius norm.
\end{theorem}

\begin{proof}
Let $\sigma_{j,i}$, $\sigma_{j',i'}$ be two samplets with support in clusters $v$, $v'$, coefficient vectors $\omega_{j,i}$, $\omega_{j',i'}$ and that satisfy condition (\ref{eq. 4.3.15}). After Proposition \ref{4.3.4} is
\begin{align*}
| (k, \sigma_{j,i} \otimes \sigma_{j',i'} )_{\Omega} | 
&\leq C_2 \dfrac{d^{2(q+1)} \text{diam}(v)^{q+1} \text{diam}(v')^{q+1} }{ (r \cdot \text{dist}(B_v,B_{v'}))^{2(q+1)} } \| \omega_{j,i} \|_1 \| \omega_{j',i'} \|_1 \\
&\leq  C_2 \dfrac{d^{2(q+1)} \text{diam}(B_v)^{q+1} \text{diam}(B_{v'})^{q+1} }{ (r \cdot \text{dist}(B_v,B_{v'}))^{2(q+1)} } \| \omega_{j,i} \|_1 \| \omega_{j',i'} \|_1 \\
&\leq C_2 \dfrac{d^{2(q+1)} \min \{ \text{diam}(B_v), \text{diam}(B_{v'}) \}^{q+1} }{ (\eta r)^{2(q+1)} \max \{ \text{diam}(B_v), \text{diam}(B_{v'}) \}^{q+1} } \| \omega_{j,i} \|_1 \| \omega_{j',i'} \|_1 \\
&\leq C_2 \left( \dfrac{d}{\eta r} \right)^{2(q+1)} \theta_{(j,i),(j',i')} \| \omega_{j,i} \|_1 \| \omega_{j',i'} \|_1,
\end{align*}
%-------------------------------------------------------------------------------------------
where $C_2(\Omega, k),r>0$ are constants and we denote
$$\theta_{(j,i),(j',i')} = \dfrac{\min \{ \text{diam}(B_v), \text{diam}(B_{v'}) \}}{\max \{ \text{diam}(B_v), \text{diam}(B_{v'}) \}} \leq 1.$$
For $j \in \{0,...,J-1 \}$ let $\boldsymbol{\Sigma}_j = \{ \sigma_{j,1}, ..., \sigma_{j,|\boldsymbol{\Sigma}_j |} \}$ be the set of samplets that the clusters on level $j$ provide. There are $2^j$ clusters on this level, each of them provides not more than $m_q$ samplets and the amount of elements in a cluster is on average $N/2^j$. Using these facts and statement \ref{4.2.13 (4)} of Theorem \ref{4.2.13} yields
$$\sum_{i=1}^{| \boldsymbol{\Sigma}_j|} \| \omega_{j,i} \|_1^2 \theta_{(j,i), (j',i')}^2  
\leq \sum_{i=1}^{| \boldsymbol{\Sigma}_j|} |v_{j,i}| \theta_{(j,i), (j',i')}
\leq \sum_{i=1}^{2^j m_q} \dfrac{N}{2^j} \theta_{(j,i), (j',i')}
= \dfrac{N}{2^j} \sum_{i=1}^{2^j m_q} \theta_{(j,i), (j',i')},$$
where $\sigma_{j,i}$ has support in cluster $v_{j,i}$ on level $j$. Using statements \ref{4.2.6 (1)} and \ref{4.2.6 (2)} of Proposition \ref{4.2.6}, we receive
\begin{align*}
\sum_{i=1}^{2^j m_q} \sum_{i'=1}^{2^{j'} m_q} \theta_{(j,i), (j',i')}
&= m_q^2 \sum_{\substack{ v \text{ is a cluster on level } j \\ v' \text{ is a cluster on level } j'}} \dfrac{\min \{ \text{diam}(B_v), \text{diam}(B_{v'}) \}}{\max \{ \text{diam}(B_v), \text{diam}(B_{v'}) \}} \\
&\lesssim \sum_{\substack{ v \text{ is a cluster on level } j \\ v' \text{ is a cluster on level } j'}} \dfrac{ \min \{ \text{diam}(B_v), \text{diam}(B_{v'})  \}}{ N^{1/d} q_{X_N} \cdot \max \{ 2^{-j/d}, 2^{-j'/d} \} } \\
&\lesssim \dfrac{ \min \{ 2^{j'} 2^{j-j/d}, 2^{j} 2^{j'-j'/d}  \}}{ \max \{ 2^{-j/d}, 2^{-j'/d} \} }
= 2^{j+j'} \dfrac{ \min \{ 2^{-j/d},  2^{-j'/d}  \}}{ \max \{ 2^{-j/d}, 2^{-j'/d} \} },
\end{align*}
where the constants in $\lesssim$ depend on $\Omega$ and $q$. We further use the estimate
\begin{align*}
\sum_{j'=0}^{J-1} \dfrac{ \min \{ 2^{-j/d},  2^{-j'/d}  \}}{ \max \{ 2^{-j/d}, 2^{-j'/d} \} }
&= \sum_{j'=0}^{j} \dfrac{ 2^{-j/d} }{ 2^{-j'/d} } + \sum_{j'=j+1}^{J-1} \dfrac{ 2^{-j'/d} }{ 2^{-j/d} } \\
&= \sum_{j'=0}^{j} 2^{-j'/d} + \sum_{j'=1}^{J-j-1} 2^{-j'/d} 
\leq 2 \sum_{j=0}^{\infty} 2^{-j'/d} = \dfrac{2}{1-2^{-1/d}}.
\end{align*}
Therefore,
\begin{align*}
\| \boldsymbol{K}^{\boldsymbol{\Sigma}} - \boldsymbol{K}^{\boldsymbol{\Sigma}}_{\eta} \|_F^2 
&= \sum_{(j,i) \text{ and } (j',i') \text{ satisfy } (\ref{eq. 4.3.15})} \sum_{i=1}^{| \boldsymbol{\Sigma}_j|} \sum_{i'=1}^{| \boldsymbol{\Sigma}_{j'}|} | (k, \sigma_{j,i} \otimes \sigma_{j',i'} )_{\Omega} |^2\\
&\leq C_2^2 \left( \dfrac{d}{\eta r} \right)^{4(q+1)} \sum_{j,j'=0}^{J-1} \sum_{i=1}^{| \boldsymbol{\Sigma}_j|} \sum_{i'=1}^{| \boldsymbol{\Sigma}_{j'}|} \theta_{(j,i),(j',i')}^2 \| \omega_{j,i} \|_1^2 \| \omega_{j',i'} \|_1^2\\
&\leq C_2^2 \left( \dfrac{d}{\eta r} \right)^{4(q+1)} \sum_{j,j'=0}^{J-1} \dfrac{N^2}{2^{j+j'}} \sum_{i=1}^{2^j m_q} \sum_{i'=1}^{2^{j'} m_q} \theta_{(j,i),(j',i')} \\
&\lesssim \eta^{-4(q+1)} N^2 \sum_{j=0}^{J-1} \sum_{j'=0}^{J-1} \dfrac{ \min \{ 2^{-j/d},  2^{-j'/d}  \}}{ \max \{ 2^{-j/d}, 2^{-j'/d} \} } 
\lesssim \eta^{-4(q+1)} N^2 J,
\end{align*}
where the constants in $\lesssim$ depend on $\Omega, k$ and $q$. Note that the root cluster does not satisfy (\ref{eq. 4.3.15}) with respect to any other cluster, so the scaling functions within $\boldsymbol{\Sigma}_J$ do not need to be considered. The proof of Proposition \ref{4.2.4} states that the total amount of levels is not more than

%-------------------------------------------------------------------------------------------

$$\bigg\lceil \dfrac{\log (N/2)}{\log 2} \bigg\rceil + 1 
= \bigg\lceil \dfrac{\log N - \log 2}{\log 2} \bigg\rceil + 1
= \bigg\lceil \dfrac{\log N}{\log 2} \bigg\rceil 
< 2 \log N.$$
Taking this into account, we have
$$\| \boldsymbol{K}^{\boldsymbol{\Sigma}} - \boldsymbol{K}^{\boldsymbol{\Sigma}}_{\eta} \|_F \leq C \eta^{-2(q+1)} N \sqrt{\log N}$$
for a constant $C(\Omega, k ,q)>0$.
\end{proof}

A direct consequence is the following corollary, which bounds the relative error of the compressed matrix.

\begin{corollary} \label{4.3.6}
Let $(X_N)_{N\in \mathbb{N}}$ be a sequence of points that satisfies $q_{X_N} \geq C_1 N^{-1/d}$ and let $\| \boldsymbol{K}^{\boldsymbol{\Sigma}} \|_F \geq C_2 N$ for some constants $C_1(\Omega), C_2(\overline{\Omega}, k)>0$. Then
$$\dfrac{\| \boldsymbol{K}^{\boldsymbol{\Sigma}} - \boldsymbol{K}^{\boldsymbol{\Sigma}}_{\eta} \|_F}{\| \boldsymbol{K}^{\boldsymbol{\Sigma}} \|_F} 
\leq C \eta^{-2(q+1)} \sqrt{\log N},$$
where $C(\overline{\Omega}, k, q)>0$ is a constant.
\end{corollary}

As the Matérn kernels are of interest in my thesis, it is important to note that they satisfy the condition to bound the relative error.

\begin{lemma} \label{4.3.7}
For a Matérn kernel exists $C(\Omega, k)>0$ with $\| \boldsymbol{K}^{\boldsymbol{\Sigma}} \|_F \geq C N$.
\end{lemma}

\begin{proof}
The Matérn kernel $k_{\nu,\ell}$ is continuous on $\overline{\Omega} \times \overline{\Omega}$ and positive, implying there exists $C = C(\overline{\Omega}, k)>0$ such that $k_{\nu, \ell} (\boldsymbol{x}, \boldsymbol{x}') \geq C$ for all $\boldsymbol{x}, \boldsymbol{x}' \in \Omega$. Now,
\[ \pushQED{\qed} 
\| \boldsymbol{K}^{\boldsymbol{\Sigma}} \|_F^2 
= \| \boldsymbol{T} \boldsymbol{K} \boldsymbol{T}^{-1} \|_F ^2
= \| \boldsymbol{K} \|_F^2
= \sum_{i,j=1}^N k_{\nu,\ell}(\boldsymbol{x}_i, \boldsymbol{x}_j)^2
\geq \sum_{i,j=1}^N C^2 
= N^2 C^2. \qedhere \]
\end{proof}

For an appropriate experimental design, the compressed matrix is sparse. This ensures that our compression strategy is effective.

\begin{proposition} \label{4.3.8}
Let $(X_N)_{N\in \mathbb{N}}$ be a sequence of quasi-uniform points in $\Omega$. The compressed matrix $\boldsymbol{K}^{\boldsymbol{\Sigma}}_{\eta}$ has $\mathcal{O}(N \log N)$ non-zero entries.
\end{proposition}

\begin{proof}
Let $v_{j,i}$ be the $i$-th cluster at level $j$ and let $L_\ell$ be the length of $B_{v_{j,i}}$'s edge along dimension $\ell$. According to Corollary \ref{4.2.7} and the proof of statement \ref{4.2.6 (1)} in Proposition \ref{4.2.6}, there exists a constant $C_1 (\Omega)>0$ such that
$$L_\ell \geq C_1 2^{-j/d} \quad \text{for every } \ell \in \{ 1,...,d \}.$$
According to Corollary \ref{4.2.7}, there is also a constant $C_2(\Omega)>0$ such that
$$\text{diam}(B_{v_{j,i}}) \leq C_1 C_2 2^{-j/d}.$$
This means condition (\ref{eq. 4.3.15}) becomes at worst

%-------------------------------------------------------------------------------------------

$$\text{dist}(B_{v_{j,i}},B_{v_{j,i'}}) \geq \eta C_1 C_2 2^{-j/d}.$$
This condition fails for at most $3^d \eta^d C_2^d$ clusters at level $j$, as illustrated in the accompanying graphic:
\begin{figure}[H]
	\centering
	\includegraphics[width=0.8\textwidth]{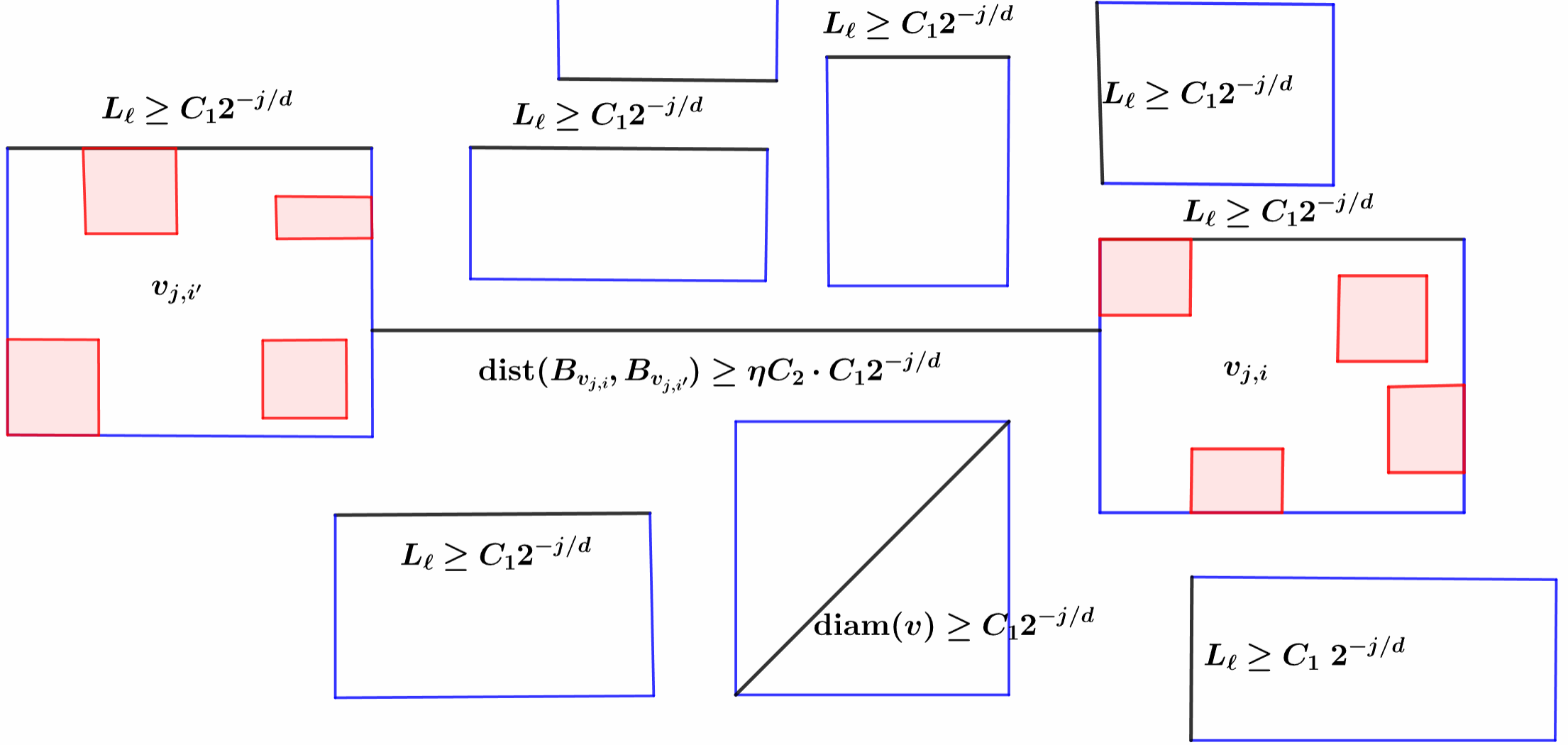}
\end{figure}
\noindent
Taking this into account and noting that $v_{j,i'}$ has $2^{j'-j}$ child clusters at level $j'$, there are $\mathcal{O}(3^d \eta^d C_2^d 2^{j'-j}) = \mathcal{O}(2^{j'-j})$ clusters at level $j'$ that do not satisfy the admissibility condition. Since each cluster contains at most $m_q$ samplets, the number of non-thresholded entries in the symmetric matrix $\boldsymbol{K}^{\boldsymbol{\Sigma}}_{\eta}$ is at most
$$\mathcal{O} \left( \sum_{j=0}^{J} 2^j \sum_{j' \geq j} 2^{j'-j} m_q^2 \right) 
= \mathcal{O} \left( \sum_{j=0}^{J} \sum_{j' \geq j} 2^{j'} \right)$$
In the proof of Proposition \ref{4.2.4}, we have established $J = \mathcal{O}(\log N)$, and clearly we have $2^J = \mathcal{O}(N)$ clusters at the last level. Using the identity
$$\sum_{j=0}^{J} \sum_{j' \geq j} 2^{j'} 
= \begin{array}{rrr} 2^0 + 2^1 + ... + 2^J \\ +2^1 + ... + 2^J \\ \vdots \\ + 2^J \end{array}
= \sum_{j=0}^{J} (j+1) 2^j,$$
we further derive
\begin{align*}
\mathcal{O} \left( \sum_{j=0}^{J} \sum_{j' \geq j} 2^{j'} \right)
&= \mathcal{O} \left( \sum_{j=0}^{J} (j+1) 2^j \right)
= \mathcal{O} \left( (J+1) 2^J \sum_{j=0}^{J} 2^{-j} \right) \\
&= \mathcal{O} \left( J 2^J \sum_{j=0}^{\infty} 2^{-j} \right)
= \mathcal{O} \left( J 2^J \dfrac{1}{1-1/2} \right)
= \mathcal{O} (N \log N).
\end{align*}
This implies that $\mathcal{O} (N \log N)$ entries in $\boldsymbol{K}^{\boldsymbol{\Sigma}}_{\eta}$ are not thresholded.
\end{proof}

It follows an example, where we visualize the sparsity pattern of a compressed kernel matrix.

%-------------------------------------------------------------------------------------------

\begin{example} \label{4.3.9}
Consider $N=10,000$ equispaced points in $(0,1)^2$ with a step size of $0.01$ in each dimension, $q=3$ vanishing moments, $\eta = 0.8$ and the Matérn kernel $k_{3/2,1}$. The kernel matrix $\boldsymbol{K}$ is dense, while the sparse matrix $\boldsymbol{K}^{\boldsymbol{\Sigma}}_\eta$ is visualized in the left panel of Figure \ref{fig. 16} and has a relative compression error of $1.8737 \cdot 10^{-5}$. \\
If we additionally threshold entries smaller than $10^{-6}$, we obtain an even sparser matrix. It is visualized in the right panel of Figure \ref{fig. 16}, with a relative compression error of $1.4865 \cdot 10^{-5}$. Since the difference between the errors is negligible, it is beneficial to further threshold small entries of $\boldsymbol{K}^{\boldsymbol{\Sigma}}_\eta$. For more details on the sparsity pattern, refer to the original work \cite{HM1}.

\begin{figure}[h]
	\centering
	\includegraphics[width=0.96\textwidth]{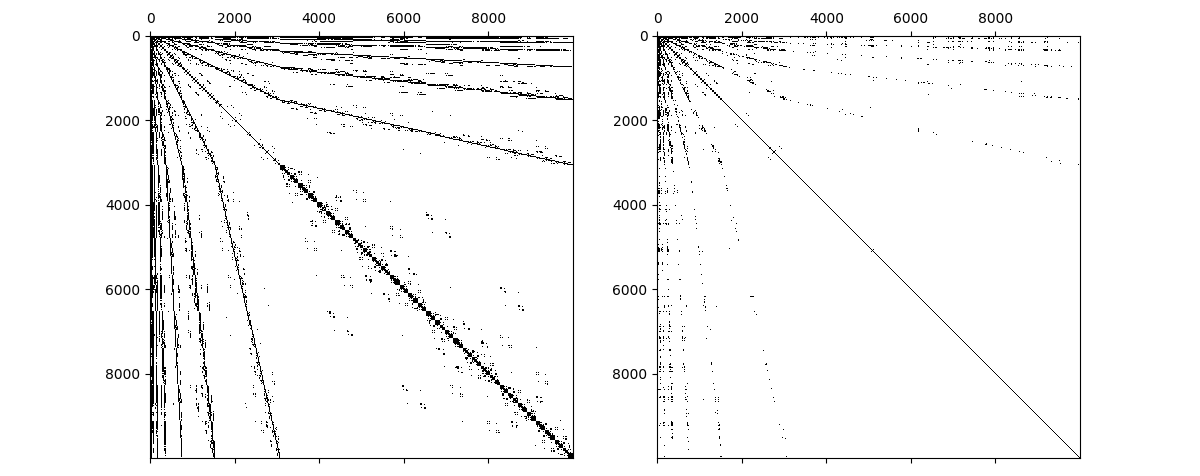}
	\caption{The left panel displays $\boldsymbol{K}^{\boldsymbol{\Sigma}}_\eta$ as described in Example \ref{4.3.9}, while the right panel shows $\boldsymbol{K}^{\boldsymbol{\Sigma}}_\eta$ with entries smaller than $10^{-6}$ thresholded.}\label{fig. 16}
\end{figure}
\end{example}

%-------------------------------------------------------------------------------------------
% Computation of the compressed kernel matrix
%-------------------------------------------------------------------------------------------

\subsection{Computation of the compressed kernel matrix}
A straightforward computation of the compressed kernel matrix $\boldsymbol{K}^{\boldsymbol{\Sigma}}_\eta$ would be too expensive, as it would involve calculating $\boldsymbol{K}^{\boldsymbol{\Sigma}} = \boldsymbol{TK} \boldsymbol{T}^{-1}$ and then checking the admissibility condition (\ref{eq. 4.3.15}) for every pair of clusters $(v,v')$. Instead, our setting enables the direct computation of $\boldsymbol{K}^{\boldsymbol{\Sigma}}_\eta$ using $\mathcal{H}^2$-matrix techniques. Specifically, we recursively compute the blocks of this matrix, where each block corresponds to a pair of cluster $(v,v')$. The next lemma demonstrates that certain blocks can be skipped, and in conjunction with Proposition \ref{4.3.8}, we conclude that for quasi-uniform points only $\mathcal{O}(N \log N)$ blocks need to be computed. Since each block will require a constant amount of computational effort, this provides an efficient strategy for computing $\boldsymbol{K}^{\boldsymbol{\Sigma}}_\eta$.

\begin{lemma} \label{4.4.1}
Let the clusters $v$ and $v'$ satisfy criterion (\ref{eq. 4.3.15}). Then for son clusters $v_{son}$ of $v$ and $v'_{son}$ of $v'$ is:
\begin{enumerate}[label=(\roman*),topsep=5pt]
	\setlength\itemsep{0.1mm}
\item $\text{dist}(B_{v_{son}},B_{v'}) \geq \eta \max \{ \text{diam}(B_{v_{son}}), \text{diam}(B_{v'}) \}$, \label{4.4.1 (i)}
\item $\text{dist}(B_{v_{son}},B_{v'_{son}}) \geq \eta \max \{ \text{diam}(B_{v_{son}}), \text{diam}(B_{v'_{son}}) \}$. \label{4.4.1 (ii)}
\end{enumerate}
\end{lemma}

\begin{proof}
The bounding box $B_{v_{son}}$ lies within $B_v$ and therefore
\begin{align*}
\text{dist}(B_{v_{son}},B_{v'}) 
\geq \text{dist}(B_v,B_{v'}) 
&\geq \eta \max \{ \text{diam}(B_v), \text{diam}(B_{v'}) \} \\
&\geq \eta \max \{ \text{diam}(B_{v_{son}}), \text{diam}(B_{v'}) \}.
\end{align*}
We get \ref{4.4.1 (ii)} by using \ref{4.4.1 (i)} twice.
\end{proof}

%-------------------------------------------------------------------------------------------

\noindent
Let $\{ \boldsymbol{\xi}_1^v, \boldsymbol{\xi}_2^v, ..., \boldsymbol{\xi}_n^v \} \subset \Omega$ be a set of interpolation points for the cluster $v$, where the components satisfy $\xi_{i,j}^v \neq \xi_{i,j'}^v$ for $j \neq j'$. Define the corresponding tensor Lagrange polynomials\index{Lagrange polynomials,} as \index{\textit{$\ell_{\boldsymbol{\alpha}}^v$,}}
$$\ell_{\boldsymbol{\alpha}}^v(\boldsymbol{x}) = \prod_{j = 1}^d \prod_{i \neq \alpha_j}^n \dfrac{x_j-\xi_{i,j}^v}{\xi_{\alpha_j,j}^v - \xi_{i,j}^v} \quad \text{for } \boldsymbol{\alpha} \in \{ 1, 2, ..., n \}^d$$ 
(see section B.3.2 in \cite{Ha}). For an admissible pair of clusters $(v,v')$, we approximate the kernel $k$ via interpolation
$$k(\boldsymbol{x}, \boldsymbol{y}) \approx \sum_{ \boldsymbol{\alpha}, \boldsymbol{\beta} \in \{ 1,2, ..., n\}^d} \ell_{\boldsymbol{\alpha}}^v(\boldsymbol{x}) k(\boldsymbol{\xi}_{\boldsymbol{\alpha}}^v, \boldsymbol{\xi}_{\boldsymbol{\beta}}^{v'}) \ell_{\boldsymbol{\beta}}^{v'}(\boldsymbol{y}) \quad \text{for } \boldsymbol{x} \in B_v, \boldsymbol{y} \in B_{v'},$$
where \index{\textit{$\boldsymbol{\xi}_{\boldsymbol{\alpha}}^v$,}}
$$\boldsymbol{\xi}_{\boldsymbol{\alpha}}^v = (\xi_{\alpha_1, 1}^v, \xi_{\alpha_2, 2}^v, ..., \xi_{\alpha_d, d}^v).$$
This approximation is used only for admissible pairs of clusters, as the smoothness of the kernel is required for a good approximation, which can only be guaranteed when $\boldsymbol{x} \neq \boldsymbol{y}$. Specifically, the asymptotic smoothness property from Definition \ref{4.3.1} suggests that the polynomial approximation improves as $\| \boldsymbol{x} - \boldsymbol{y} \|$ increases. \\
Let $\varphi_{j}^v, \varphi_{j'}^{v'} \in V$ have coefficient vectors $[\omega_{j,i}]_{i=1}^N, [\omega_{j',i'}]_{i'=1}^N$ in $\mathbb{R}^N$ and support in clusters $v, v'$. Then is
\begin{align*}
(k, \varphi_{j}^v \otimes \varphi_{j'}^{v'})_{\Omega \times \Omega}
&= \sum_{i=1}^N \sum_{i'=1}^N \omega_{j,i} \omega_{j',i'} (k, \delta_{\boldsymbol{x}_i} \otimes \delta_{\boldsymbol{x}_{i'}})_{\Omega \times \Omega} \\
&\approx \sum_{s=1}^N \sum_{s'=1}^N \omega_{j,i} \omega_{j',i'} \sum_{ \boldsymbol{\alpha}, \boldsymbol{\beta} \in \{ 1,2, ..., n\}^d} \ell_{\boldsymbol{\alpha}}^v(\boldsymbol{x}_i) k(\boldsymbol{\xi}_{\boldsymbol{\alpha}}^v, \boldsymbol{\xi}_{\boldsymbol{\beta}}^{v'}) \ell_{\boldsymbol{\beta}}^{v'}(\boldsymbol{x}_{i'}) \\
&= \sum_{ \boldsymbol{\alpha}, \boldsymbol{\beta} \in \{ 1,2, ..., n\}^d} \left( \sum_{i=1}^N \omega_{j,i} \ell_{\boldsymbol{\alpha}}^v(\boldsymbol{x}_i) \right) k(\boldsymbol{\xi}_{\boldsymbol{\alpha}}^v, \boldsymbol{\xi}_{\boldsymbol{\beta}}^{v'}) \left( \sum_{i'=1}^N \omega_{j',i'} \ell_{\boldsymbol{\beta}}^{v'}(\boldsymbol{x}_{i'}) \right) \\
&= \sum_{ \boldsymbol{\alpha}, \boldsymbol{\beta} \in \{ 1,2, ..., n\}^d} (\ell_{\boldsymbol{\alpha}}^v, \varphi_{j}^v)_\Omega k(\boldsymbol{\xi}_{\boldsymbol{\alpha}}^v, \boldsymbol{\xi}_{\boldsymbol{\beta}}^{v'}) (\ell_{\boldsymbol{\beta}}^{v'}, \varphi_{j'}^{v'})_\Omega .
\end{align*}
This implies that if $(v,v')$ is admissible, we approximate a matrix block via
\begin{equation*}
\boldsymbol{K}^{\boldsymbol{\Phi}}_{v,v'} 
:= [ (k, \varphi_{j,i}^v \otimes \varphi_{j',i'}^{v'} )_{\Omega \times \Omega} ]_{1 \leq i \leq | \boldsymbol{\Phi}_j^v |, 1 \leq i' \leq | \boldsymbol{\Phi}_{j'}^{v'} |} 
\approx \boldsymbol{V}^{\boldsymbol{\Phi}}_v \boldsymbol{S}_{v,v'} \left( \boldsymbol{V}^{\boldsymbol{\Phi}}_{v'} \right)^T, \index{\textit{$\boldsymbol{K}^{\boldsymbol{\Phi}}_{v,v'}$,}}
\end{equation*}
where the matrices are defined as 
\index{\textit{$\boldsymbol{V}^{\boldsymbol{\Phi}}_v$,}} 
\index{\textit{$\boldsymbol{S}_{v,v'}$,}}
$$\boldsymbol{V}^{\boldsymbol{\Phi}}_v = [ (\ell_{\boldsymbol{\alpha}}^v, \varphi_{j,i}^v )_\Omega ]_{1 \leq i \leq | \boldsymbol{\Phi}_j^v |, \boldsymbol{\alpha} \in \{ 1, ..., n\}^d } \quad \text{and} \quad
\boldsymbol{S}_{v,v'} = [k(\boldsymbol{\xi}_{\boldsymbol{\alpha}}^v, \boldsymbol{\xi}_{\boldsymbol{\beta}}^{v'})]_{\boldsymbol{\alpha}, \boldsymbol{\beta} \in \{ 1,...,n \}^d }.$$
Similarly, we approximate the matrix block
\begin{equation*}
\boldsymbol{K}^{\boldsymbol{\Sigma}}_{v,v'} 
:= [ (k, \sigma_{j,i}^v \otimes \sigma_{j',i'}^{v'} )_{\Omega \times \Omega} ]_{1 \leq i \leq | \boldsymbol{\Sigma}_j^v |, 1 \leq i' \leq | \boldsymbol{\Sigma}_{j'}^{v'} |}
\approx \boldsymbol{V}^{\boldsymbol{\Sigma}}_v \boldsymbol{S}_{v,v'} \left( \boldsymbol{V}^{\boldsymbol{\Sigma}}_{v'} \right)^T, \index{\textit{$\boldsymbol{K}^{\boldsymbol{\Sigma}}_{v,v'}$,}}
\end{equation*}
even though we only require this matrix block for recursive calculations. Using the same notations as in sections \ref{Chap. 4.2.2} and \ref{Chap. 4.2.4}, we can also express \index{\textit{$\boldsymbol{V}^{\boldsymbol{\Sigma}}_v$,}} 
$$\left( \boldsymbol{V}^{\boldsymbol{\Sigma}}_v \right)^T = [(\ell_{\boldsymbol{\alpha}}^v, \boldsymbol{\Sigma}_j^v )_\Omega]_{\boldsymbol{\alpha} \in \{ 1, ..., n\}^d} \quad \text{and} \quad
\left( \boldsymbol{V}^{\boldsymbol{\Phi}}_v \right)^T = [(\ell_{\boldsymbol{\alpha}}^v, \boldsymbol{\Phi}_j^v )_\Omega]_{\boldsymbol{\alpha} \in \{ 1, ..., n\}^d}.$$
The following lemma demonstrates, that this basis matrices can be calculated recursively.

%-------------------------------------------------------------------------------------------

\begin{lemma} \label{4.4.2}
For a child cluster $v_{\text{son}}$ of $v$ define $\boldsymbol{T}_{v, v_{\text{son}}} = [\ell_{\boldsymbol{\alpha}}^v(\boldsymbol{\xi}_{\boldsymbol{\beta}}^{v_{\text{son}}})]_{\boldsymbol{\beta}, \boldsymbol{\alpha} \in \{ 1, ..., n \}^d}$. Then it holds \index{\textit{$\boldsymbol{T}_{v, v_{\text{son}}}$,}}
$$\begin{bmatrix} \boldsymbol{V}^{\boldsymbol{\Phi}}_v \\ \boldsymbol{V}^{\boldsymbol{\Sigma}}_v \end{bmatrix}
= \left( \boldsymbol{Q}_j^v \right)^T \begin{bmatrix}
\boldsymbol{V}^{\boldsymbol{\Phi}}_{v_{\text{son} 1}} \cdot \boldsymbol{T}_{v, v_{\text{son} 1}} \\
\boldsymbol{V}^{\boldsymbol{\Phi}}_{v_{\text{son} 2}} \cdot \boldsymbol{T}_{v, v_{\text{son} 2}}
\end{bmatrix},$$
where $\boldsymbol{Q}_j^v$ is the transformation matrix from section \ref{Chap. 4.2.2}.
\end{lemma}

\begin{proof}
The Lagrange polynomials belong to the subspace
\begin{equation}
\text{span} \big\{ z_1^{\alpha_1} z_2^{\alpha_2} \cdot ... \cdot z_d^{\alpha_d} \ | \ \alpha_1, \alpha_2, ..., \alpha_d \in \{ 0, 1, ..., n-1 \} \big\} \subset \mathcal{P}_{nd} (\Omega) \label{eq. 4.4.1}
\end{equation}
which has dimension $n^d$. Using the identity
\begin{equation}
\ell_{\boldsymbol{\alpha}}^v(\boldsymbol{\xi}_{\boldsymbol{\beta}}^v) 
= \begin{cases}
  1, & \text{if } \boldsymbol{\alpha} = \boldsymbol{\beta} \\
  0, & \text{otherwise}
\end{cases} \label{eq. 4.4.2}
\end{equation}
it is straightforward to demonstrate that the Lagrange polynomials are linearly independent. Given that there are $n^d$ Lagrange polynomials, they form a basis for the subspace defined in  (\ref{eq. 4.4.1}). Therefore, any $\ell_{\boldsymbol{\alpha}}^v$ can be represented using the Lagrange polynomials $\ell_{\boldsymbol{\beta}}^{v_{\text{son}}}$ of one of its child clusters $v_{\text{son}}$. By utilizing equation (\ref{eq. 4.4.2}) for the child cluster $v_{\text{son}}$, the coefficients for this representation can be easily determined, leading to the expression
$$\ell_{\boldsymbol{\alpha}}^v (\boldsymbol{x}) = \sum_{\boldsymbol{\beta} \in \{ 1,...,n \}^d } \ell_{\boldsymbol{\alpha}}^v(\boldsymbol{\xi}_{\boldsymbol{\beta}}^{v_{\text{son}}}) \ell_{\boldsymbol{\beta}}^{v_{\text{son}}} (\boldsymbol{x}) \quad \text{for all } \boldsymbol{x} \in \Omega.$$
Specifically, it holds
\begin{align*}
[ (\ell_{\boldsymbol{\alpha}}^v, \boldsymbol{\Phi}_{j+1}^{v_{\text{son}}} )_\Omega ]_{\boldsymbol{\alpha} \in \{ 1, ..., n\}^d}
&= [\ell_{\boldsymbol{\alpha}}^v(\boldsymbol{\xi}_{\boldsymbol{\beta}}^{v_{\text{son}}})]_{\boldsymbol{\alpha}, \boldsymbol{\beta} \in \{ 1, ..., n \}^d} \cdot [(\ell_{\boldsymbol{\beta}}^{v_{\text{son}}}, \boldsymbol{\Phi}_{j+1}^{v_{\text{son}}} )_\Omega]_{\boldsymbol{\beta} \in \{ 1, ..., n\}^d} \\
&= \left( \boldsymbol{V}^{\boldsymbol{\Phi}}_{v_{\text{son}}} \cdot \boldsymbol{T}_{v, v_{\text{son}}} \right)^T.
\end{align*}
Utilizing this equality and equation (\ref{eq. 4.2.12}), we get
\begin{align*}
\begin{bmatrix} \boldsymbol{V}^{\boldsymbol{\Phi}}_v \\ \boldsymbol{V}^{\boldsymbol{\Sigma}}_v \end{bmatrix}^T
= \left[ \left( \boldsymbol{V}^{\boldsymbol{\Phi}}_v \right)^T \big| \left( \boldsymbol{V}^{\boldsymbol{\Sigma}}_v \right)^T \right] 
&= [ (\ell_{\boldsymbol{\alpha}}^v, \boldsymbol{\Phi}_j^v )_\Omega \, | \, (\ell_{\boldsymbol{\alpha}}^v, \boldsymbol{\Sigma}_j^v )_\Omega]_{\boldsymbol{\alpha} \in \{ 1, ..., n\}^d} \\
&= [(\ell_{\boldsymbol{\alpha}}^v, \boldsymbol{\Phi}_{j+1}^v )_\Omega]_{\boldsymbol{\alpha} \in \{ 1, ..., n\}^d} \boldsymbol{Q}_j^v \\
&= [(\ell_{\boldsymbol{\alpha}}^v, \boldsymbol{\Phi}_{j+1}^{v_{\text{son} 1}} )_\Omega \, | \, (\ell_{\boldsymbol{\alpha}}^v, \boldsymbol{\Phi}_{j+1}^{v_{\text{son} 2}} )_\Omega]_{\boldsymbol{\alpha} \in \{ 1, ..., n\}^d} \boldsymbol{Q}_j^v \\
&= \left[ \left( \boldsymbol{V}^{\boldsymbol{\Phi}}_{v_{\text{son} 1}} \cdot \boldsymbol{T}_{v, v_{\text{son} 1}} \right)^T \big| \left( \boldsymbol{V}^{\boldsymbol{\Phi}}_{v_{\text{son} 2}} \cdot \boldsymbol{T}_{v, v_{\text{son} 2}} \right)^T \right] \boldsymbol{Q}_j^v \\
&= \begin{bmatrix}
\boldsymbol{V}^{\boldsymbol{\Phi}}_{v_{\text{son} 1}} \cdot \boldsymbol{T}_{v, v_{\text{son} 1}} \\
\boldsymbol{V}^{\boldsymbol{\Phi}}_{v_{\text{son} 2}} \cdot \boldsymbol{T}_{v, v_{\text{son} 2}}
\end{bmatrix}^T \boldsymbol{Q}_j^v. \qedhere
\end{align*}
\end{proof}

\noindent
The procedure to compute the basis matrices is formulated in Algorithm \ref{alg. 6}. Define
\begin{equation*}
\boldsymbol{K}^{\boldsymbol{\Phi}, \boldsymbol{\Sigma}}_{v,v'} = (k, \boldsymbol{\Phi}_{j}^v \otimes \boldsymbol{\Sigma}_{j'}^{v'})_{\Omega \times \Omega} = [ (k, \varphi_{j,i}^v \otimes \sigma_{j',i'}^{v'} )_{\Omega \times \Omega} ]_{1 \leq i \leq | \boldsymbol{\Phi}_j^v |, 1 \leq i' \leq | \boldsymbol{\Sigma}_{j'}^{v'} |}, \index{\textit{$(k, \boldsymbol{\Phi}_{j}^v \otimes \boldsymbol{\Sigma}_{j'}^{v'})_{\Omega \times \Omega}$,}}
\end{equation*}
then for an admissible pair $(v,v')$, we are now able to approximate
$$\begin{bmatrix} \boldsymbol{K}^{\boldsymbol{\Phi}, \boldsymbol{\Phi}}_{v,v'} & \boldsymbol{K}^{\boldsymbol{\Phi}, \boldsymbol{\Sigma}}_{v,v'} 
\\ \boldsymbol{K}^{\boldsymbol{\Sigma}, \boldsymbol{\Phi}}_{v,v'} & \boldsymbol{K}^{\boldsymbol{\Sigma}, \boldsymbol{\Sigma}}_{v,v'} \end{bmatrix}
\approx \begin{bmatrix} \boldsymbol{V}^{\boldsymbol{\Phi}}_v \\ \boldsymbol{V}^{\boldsymbol{\Sigma}}_v\end{bmatrix} \boldsymbol{S}_{v,v'} \begin{bmatrix} \left( \boldsymbol{V}^{\boldsymbol{\Phi}}_{v'} \right)^T & \left( \boldsymbol{V}^{\boldsymbol{\Sigma}}_{v'} \right)^T \end{bmatrix}.$$ 
To compute the non-admissible matrix blocks, we use the refinement relations stated in the next lemma.

%-------------------------------------------------------------------------------------------

\begin{algorithm}[h]
\DontPrintSemicolon
\caption{Multiscale cluster basis matrices}\label{alg. 6}
\KwData{Cluster tree $\mathcal{T}$, cluster $v$, $(\boldsymbol{Q}_j^v)_{v \in P \backslash \mathcal{L(T)}}$, $(\boldsymbol{V}_v^{\boldsymbol{\Phi}})_{v \in \mathcal{L(T)}}$, $(\boldsymbol{T}_{v, v_{\text{son}}})_{v \in P \backslash \mathcal{L(T)}}$}
\KwResult{Multiscale cluster basis matrices $(\boldsymbol{V}_v^{\boldsymbol{\Phi}})_{v \in P}$ and $(\boldsymbol{V}_v^{\boldsymbol{\Sigma}})_{v \in P \backslash \mathcal{L(T)}}$}
\SetKwFunction{FMain}{computeMultiscaleClusterBasis}
\SetKwProg{Fn}{Function}{:}{}
\Fn{\FMain{$v$}}{
	\If{children $v_{\text{son} 1}, v_{\text{son} 2}$ of $v$ aren't leaf clusters }{
	$\boldsymbol{V}^{\boldsymbol{\Phi}}_{v_{\text{son} 1}} = 
	\texttt{computeMultiscaleClusterBasis}(v_{\text{son} 1})$\;
	$\boldsymbol{V}^{\boldsymbol{\Phi}}_{v_{\text{son} 2}} = 
	\texttt{computeMultiscaleClusterBasis}(v_{\text{son} 2})$\;
	}
	$\begin{bmatrix} \boldsymbol{V}^{\boldsymbol{\Phi}}_v \\ \boldsymbol{V}^{\boldsymbol{\Sigma}}_v \end{bmatrix}
= \left( \boldsymbol{Q}_j^v \right)^T \begin{bmatrix}
\boldsymbol{V}^{\boldsymbol{\Phi}}_{v_{\text{son} 1}} \cdot \boldsymbol{T}_{v, v_{\text{son} 1}} \\
\boldsymbol{V}^{\boldsymbol{\Phi}}_{v_{\text{son} 2}} \cdot \boldsymbol{T}_{v, v_{\text{son} 2}}
\end{bmatrix}$\;
	store $(\boldsymbol{V}^{\boldsymbol{\Phi}}_v, \boldsymbol{V}^{\boldsymbol{\Sigma}}_v)$\;
  	\KwRet $\boldsymbol{V}^{\boldsymbol{\Phi}}_v$\;
}
\textbf{end}\;
$\texttt{computeMultiscaleClusterBasis}(X)$\;
\end{algorithm}

\begin{lemma} \label{4.4.3}
We have
\begin{align*}
\begin{bmatrix} \boldsymbol{K}^{\boldsymbol{\Phi}, \boldsymbol{\Phi}}_{v,v'_{\text{son} 1}} & \boldsymbol{K}^{\boldsymbol{\Phi}, \boldsymbol{\Phi}}_{v,v'_{\text{son} 2}} \\ 
\boldsymbol{K}^{\boldsymbol{\Sigma}, \boldsymbol{\Phi}}_{v,v'_{\text{son} 1}} & \boldsymbol{K}^{\boldsymbol{\Sigma}, \boldsymbol{\Phi}}_{v,v'_{\text{son} 2}} \end{bmatrix}
\boldsymbol{Q}_{j'}^{v'}
= \begin{bmatrix} \boldsymbol{K}^{\boldsymbol{\Phi}, \boldsymbol{\Phi}}_{v,v'} & \boldsymbol{K}^{\boldsymbol{\Phi}, \boldsymbol{\Sigma}}_{v,v'} 
\\ \boldsymbol{K}^{\boldsymbol{\Sigma}, \boldsymbol{\Phi}}_{v,v'} & \boldsymbol{K}^{\boldsymbol{\Sigma}, \boldsymbol{\Sigma}}_{v,v'} \end{bmatrix}
= \left( \boldsymbol{Q}_j^v \right)^T \begin{bmatrix} 
\boldsymbol{K}^{\boldsymbol{\Phi}, \boldsymbol{\Phi}}_{v_{\text{son} 1},v'} & \boldsymbol{K}^{\boldsymbol{\Phi}, \boldsymbol{\Sigma}}_{v_{\text{son} 1},v'} \\ 
\boldsymbol{K}^{\boldsymbol{\Phi}, \boldsymbol{\Phi}}_{v_{\text{son} 2},v'} & \boldsymbol{K}^{\boldsymbol{\Phi}, \boldsymbol{\Sigma}}_{v_{\text{son} 2},v'} \end{bmatrix}.
\end{align*}
\end{lemma}

\begin{proof}
It is
\begin{align*}
\begin{bmatrix} \boldsymbol{K}^{\boldsymbol{\Phi}, \boldsymbol{\Phi}}_{v,v'} & \boldsymbol{K}^{\boldsymbol{\Phi}, \boldsymbol{\Sigma}}_{v,v'} \end{bmatrix}
&= \begin{bmatrix} (k, \boldsymbol{\Phi}_{j}^v \otimes \boldsymbol{\Phi}_{j'}^{v'})_{\Omega \times \Omega} & (k, \boldsymbol{\Phi}_{j}^v \otimes \boldsymbol{\Sigma}_{j'}^{v'})_{\Omega \times \Omega} \end{bmatrix} \\
&= \left(k, \boldsymbol{\Phi}_{j}^v \otimes [ \boldsymbol{\Phi}_{j'}^{v'} \, | \, \boldsymbol{\Sigma}_{j'}^{v'} ] \right)_{\Omega \times \Omega}
= \left(k, \boldsymbol{\Phi}_{j}^v \otimes ( \boldsymbol{\Phi}_{j'+1}^{v'} \boldsymbol{Q}_{j'}^{v'} ) \right)_{\Omega \times \Omega} \\
&= \left(k, \boldsymbol{\Phi}_{j}^v \otimes  \boldsymbol{\Phi}_{j'+1}^{v'} \right)_{\Omega \times \Omega} \boldsymbol{Q}_{j'}^{v'}
= \left(k, \boldsymbol{\Phi}_{j}^v \otimes \left[ \boldsymbol{\Phi}_{j'+1}^{v'_{\text{son} 1}} \, | \, \boldsymbol{\Phi}_{j'+1}^{v'_{\text{son} 2}} \right] \right)_{\Omega \times \Omega} \boldsymbol{Q}_{j'}^{v'} \\
&= \begin{bmatrix} (k, \boldsymbol{\Phi}_{j}^v \otimes \boldsymbol{\Phi}_{j'+1}^{v'_{\text{son} 1}})_{\Omega \times \Omega} & (k, \boldsymbol{\Phi}_{j}^v \otimes \boldsymbol{\Phi}_{j'+1}^{v'_{\text{son} 2}})_{\Omega \times \Omega} \end{bmatrix} \boldsymbol{Q}_{j'}^{v'} \\
&= \begin{bmatrix} \boldsymbol{K}^{\boldsymbol{\Phi}, \boldsymbol{\Phi}}_{v,v'_{\text{son} 1}} & \boldsymbol{K}^{\boldsymbol{\Phi}, \boldsymbol{\Phi}}_{v,v'_{\text{son} 2}} \end{bmatrix} \boldsymbol{Q}_{j'}^{v'},
\end{align*}
where the fourth equality can be verified through direct calculation. Similarly, we get
$$\begin{bmatrix} \boldsymbol{K}^{\boldsymbol{\Sigma}, \boldsymbol{\Phi}}_{v,v'} & \boldsymbol{K}^{\boldsymbol{\Sigma}, \boldsymbol{\Sigma}}_{v,v'} \end{bmatrix}
= \begin{bmatrix} \boldsymbol{K}^{\boldsymbol{\Sigma}, \boldsymbol{\Phi}}_{v,v'_{\text{son} 1}} & \boldsymbol{K}^{\boldsymbol{\Sigma}, \boldsymbol{\Phi}}_{v,v'_{\text{son} 2}} \end{bmatrix} \boldsymbol{Q}_{j'}^{v'}.$$
Also is
\begin{align*}
\begin{bmatrix} \boldsymbol{K}^{\boldsymbol{\Phi}, \boldsymbol{\Phi}}_{v,v'} \\ \boldsymbol{K}^{\boldsymbol{\Sigma}, \boldsymbol{\Phi}}_{v,v'} \end{bmatrix}
&= \begin{bmatrix} (k, \boldsymbol{\Phi}_{j}^v \otimes \boldsymbol{\Phi}_{j'}^{v'})_{\Omega \times \Omega} \\ (k, \boldsymbol{\Sigma}_{j}^v \otimes \boldsymbol{\Phi}_{j'}^{v'})_{\Omega \times \Omega} \end{bmatrix}
= \left(k, [ \boldsymbol{\Phi}_{j}^{v} \, | \, \boldsymbol{\Sigma}_{j}^{v} ] \otimes \boldsymbol{\Phi}_{j'}^{v'} \right)_{\Omega \times \Omega}
= \left(k, ( \boldsymbol{\Phi}_{j+1}^{v} \boldsymbol{Q}_j^v ) \otimes  \boldsymbol{\Phi}_{j'+1}^{v'} \right)_{\Omega \times \Omega} \\
&= \left( \boldsymbol{Q}_j^v \right)^T \left(k, \boldsymbol{\Phi}_{j+1}^{v} \otimes  \boldsymbol{\Phi}_{j'+1}^{v'} \right)_{\Omega \times \Omega}
= \left( \boldsymbol{Q}_j^v \right)^T \left(k, [ \boldsymbol{\Phi}_{j+1}^{v_{\text{son} 1}} \, | \, \boldsymbol{\Phi}_{j+1}^{v_{\text{son} 2}} ] \otimes  \boldsymbol{\Phi}_{j'+1}^{v'} \right)_{\Omega \times \Omega} \\
&= \left( \boldsymbol{Q}_j^v \right)^T \begin{bmatrix} (k, \boldsymbol{\Phi}_{j+1}^{v_{\text{son} 1}} \otimes \boldsymbol{\Phi}_{j'}^{v'})_{\Omega \times \Omega} \\ (k, \boldsymbol{\Phi}_{j+1}^{v_{\text{son} 2}} \otimes \boldsymbol{\Phi}_{j'}^{v'})_{\Omega \times \Omega} \end{bmatrix}
= \left( \boldsymbol{Q}_j^v \right)^T \begin{bmatrix} \boldsymbol{K}^{\boldsymbol{\Phi}, \boldsymbol{\Phi}}_{v_{\text{son} 1} ,v'} \\ \boldsymbol{K}^{\boldsymbol{\Phi}, \boldsymbol{\Phi}}_{v_{\text{son} 2} ,v'} \end{bmatrix},
\end{align*}
where the fourth equality can be verified through direct calculation. Similarly, we get
\[ \pushQED{\qed} 
\begin{bmatrix} \boldsymbol{K}^{\boldsymbol{\Phi}, \boldsymbol{\Sigma}}_{v,v'} \\ \boldsymbol{K}^{\boldsymbol{\Sigma}, \boldsymbol{\Sigma}}_{v,v'} \end{bmatrix}
= \left( \boldsymbol{Q}_j^v \right)^T \begin{bmatrix} \boldsymbol{K}^{\boldsymbol{\Phi}, \boldsymbol{\Sigma}}_{v_{\text{son} 1} ,v'} \\ \boldsymbol{K}^{\boldsymbol{\Phi}, \boldsymbol{\Sigma}}_{v_{\text{son} 2} ,v'} \end{bmatrix}. \qedhere \]
\end{proof}

%-------------------------------------------------------------------------------------------

Algorithm \Ref{alg. 7} outlines the recursive process for calculating a matrix block. It can be applied to compute $\boldsymbol{K}^{\Sigma}_{\eta}$, as described in \cite{HM1}. However, to avoid nested functions, I formulate the computation of $\boldsymbol{K}^{\Sigma}_{\eta}$ without using $\texttt{recursivelyDetermineBlock}(v,v')$.

\begin{algorithm}[H]
\DontPrintSemicolon
\caption{Compute matrix block}\label{alg. 7}
\KwData{Cluster tree $\mathcal{T}$, $(\boldsymbol{Q}_j^v)_{v \in P \backslash \mathcal{L(T)}}$, $(\boldsymbol{V}_v^{\boldsymbol{\Phi}})_{v \in P}$, $(\boldsymbol{V}_v^{\boldsymbol{\Sigma}})_{v \in P \backslash \mathcal{L(T)}}$, $\{ \boldsymbol{\xi}_1^v, ... , \boldsymbol{\xi}_n^v \}_{v \in P}$}
\KwResult{Matrix block $\begin{bmatrix} \boldsymbol{K}^{\boldsymbol{\Phi}, \boldsymbol{\Phi}}_{v,v'} & \boldsymbol{K}^{\boldsymbol{\Phi}, \boldsymbol{\Sigma}}_{v,v'} 
\\ \boldsymbol{K}^{\boldsymbol{\Sigma}, \boldsymbol{\Phi}}_{v,v'} & \boldsymbol{K}^{\boldsymbol{\Sigma}, \boldsymbol{\Sigma}}_{v,v'} \end{bmatrix}$ for a pair of clusters $(v,v')$}
\SetKwFunction{FMain}{recursivelyDetermineBlock}
\SetKwProg{Fn}{Function}{:}{}
\Fn{\FMain{$v,v'$}}{
	\uIf{$(v,v')$ is admissible }{
	$\boldsymbol{S}_{v,v'} = [k(\boldsymbol{\xi}_{\boldsymbol{\alpha}}^v, \boldsymbol{\xi}
_{\boldsymbol{\beta}}^{v'})]_{\boldsymbol{\alpha}, \boldsymbol{\beta} \in \{ 1,...,n \}^d }$\;
	\KwRet $\begin{bmatrix} \boldsymbol{V}^{\boldsymbol{\Phi}}_v \\ \boldsymbol{V}^{\boldsymbol{\Sigma}}_v\end{bmatrix} \boldsymbol{S}_{v,v'} \begin{bmatrix} \left( \boldsymbol{V}^{\boldsymbol{\Phi}}_{v'} \right)^T & \left( \boldsymbol{V}^{\boldsymbol{\Sigma}}_{v'} \right)^T \end{bmatrix}$\;
	} \uElseIf{$v,v'$ are leaf clusters}{
	\KwRet $\boldsymbol{K}_{v,v'}^{\boldsymbol{\Phi}, \boldsymbol{\Phi}} = [ (k, \delta_{\boldsymbol{x}} \otimes \delta_{\boldsymbol{x}'} )_{\Omega \times \Omega} ]_{\boldsymbol{x} \in v, \boldsymbol{x}' \in v'}$\;
	}
	\uElseIf{$v$ is a leaf cluster}{
	$\begin{bmatrix} \boldsymbol{K}^{\boldsymbol{\Phi}, \boldsymbol{\Phi}}_{v,v'_{\text{son} 1}} & \boldsymbol{K}^{\boldsymbol{\Phi}, \boldsymbol{\Sigma}}_{v,v'_{\text{son} 1}} \\ \boldsymbol{K}^{\boldsymbol{\Sigma}, \boldsymbol{\Phi}}_{v,v'_{\text{son} 1}} & \boldsymbol{K}^{\boldsymbol{\Sigma}, \boldsymbol{\Sigma}}_{v,v'_{\text{son} 1}} \end{bmatrix} = \texttt{recursivelyDetermineBlock}(v,v'_{\text{son} 1})$\;
	$\begin{bmatrix} \boldsymbol{K}^{\boldsymbol{\Phi}, \boldsymbol{\Phi}}_{v,v'_{\text{son} 2}} & \boldsymbol{K}^{\boldsymbol{\Phi}, \boldsymbol{\Sigma}}_{v,v'_{\text{son} 2}} \\ \boldsymbol{K}^{\boldsymbol{\Sigma}, \boldsymbol{\Phi}}_{v,v'_{\text{son} 2}} & \boldsymbol{K}^{\boldsymbol{\Sigma}, \boldsymbol{\Sigma}}_{v,v'_{\text{son} 2}} \end{bmatrix} = \texttt{recursivelyDetermineBlock}(v,v'_{\text{son} 2})$\;
	\KwRet $\begin{bmatrix} \boldsymbol{K}^{\boldsymbol{\Phi}, \boldsymbol{\Phi}}_{v,v'_{\text{son} 1}} & \boldsymbol{K}^{\boldsymbol{\Phi}, \boldsymbol{\Phi}}_{v,v'_{\text{son} 2}} \\ \boldsymbol{K}^{\boldsymbol{\Sigma}, \boldsymbol{\Phi}}_{v,v'_{\text{son} 1}} & \boldsymbol{K}^{\boldsymbol{\Sigma}, \boldsymbol{\Phi}}_{v,v'_{\text{son} 2}} \end{bmatrix}
\boldsymbol{Q}_{j'}^{v'}$\;
	}
	\uElse{
	$\begin{bmatrix} \boldsymbol{K}^{\boldsymbol{\Phi}, \boldsymbol{\Phi}}_{v_{\text{son} 1}, v'} & \boldsymbol{K}^{\boldsymbol{\Phi}, \boldsymbol{\Sigma}}_{v_{\text{son} 1}, v'} \\ \boldsymbol{K}^{\boldsymbol{\Sigma}, \boldsymbol{\Phi}}_{v_{\text{son} 1}, v'} & \boldsymbol{K}^{\boldsymbol{\Sigma}, \boldsymbol{\Sigma}}_{v_{\text{son} 1}, v'} \end{bmatrix} = \texttt{recursivelyDetermineBlock}(v_{\text{son} 1}, v')$\;
	$\begin{bmatrix} \boldsymbol{K}^{\boldsymbol{\Phi}, \boldsymbol{\Phi}}_{v_{\text{son} 2}, v'} & \boldsymbol{K}^{\boldsymbol{\Phi}, \boldsymbol{\Sigma}}_{v_{\text{son} 2}, v'} \\ \boldsymbol{K}^{\boldsymbol{\Sigma}, \boldsymbol{\Phi}}_{v_{\text{son} 2}, v'} & \boldsymbol{K}^{\boldsymbol{\Sigma}, \boldsymbol{\Sigma}}_{v_{\text{son} 2}, v'} \end{bmatrix} = \texttt{recursivelyDetermineBlock}(v_{\text{son} 2}, v')$\;
	\KwRet $\left( \boldsymbol{Q}_j^v \right)^T \begin{bmatrix} 
\boldsymbol{K}^{\boldsymbol{\Phi}, \boldsymbol{\Phi}}_{v_{\text{son} 1},v'} & \boldsymbol{K}^{\boldsymbol{\Phi}, \boldsymbol{\Sigma}}_{v_{\text{son} 1},v'} \\ 
\boldsymbol{K}^{\boldsymbol{\Phi}, \boldsymbol{\Phi}}_{v_{\text{son} 2},v'} & \boldsymbol{K}^{\boldsymbol{\Phi}, \boldsymbol{\Sigma}}_{v_{\text{son} 2},v'} \end{bmatrix}$\;
	}
	\textbf{end}\;
}
\textbf{end}\;
\end{algorithm}

Algorithm \ref{alg. 8} finalizes the computation of $\boldsymbol{K}^{\Sigma}_{\eta}$ by sequentially calculating each column, starting with the last column and moving from bottom to top, then continuing in the same manner for the second-to-last column, and so forth, until the first column is finished. During the computation of a column, admissible blocks are skipped, and previously computed blocks to the right and below are reused. The admissible matrix blocks are calculated only if necessary. The coupling matrices\index{Coupling matrices,} $\boldsymbol{S}_{v,v'}$ are calculated on the fly as in Algorithm \ref{alg. 7}, but for the sake of brevity I have omitted this detail in the notation of Algorithm \ref{alg. 8}. This algorithm enables efficient computation of $\boldsymbol{K}^{\Sigma}_{\eta}$ while minimizing storage requirements.

\begin{proposition} \label{4.4.4}
Let $(X_N)_{N\in \mathbb{N}}$ be a sequence of quasi-uniform points in $\Omega$. The computational workload of computing $\boldsymbol{K}^{\Sigma}_{\eta}$ using Algorithms \ref{alg. 6} and \Ref{alg. 8} is $\mathcal{O}(N \log N)$, with the storage requirements also being $\mathcal{O}(N \log N)$.
\end{proposition}

%-------------------------------------------------------------------------------------------

\begin{proof}
By Proposition \ref{4.2.4}, we can construct a balanced binary tree $\mathcal{T} = (P,E)$ in $\mathcal{O}(N \log N)$ time and by Proposition \ref{4.2.12}, we can compute $(\boldsymbol{Q}_j^v)_{v \in P \backslash \mathcal{L(T)}}$ in $\mathcal{O}(N)$ time. For non-leaf clusters, the matrix $\boldsymbol{T}_{v, v_{\text{son}}}$ requires only constant time to compute and for leaf clusters, the matrix $\boldsymbol{V}^{\boldsymbol{\Phi}}_v$ also requires constant time. Therefore, gathering the inputs for Algorithm \ref{alg. 6} takes $\mathcal{O}(N \log N)$ time. For a non-leaf cluster, the calculation 
$$\begin{bmatrix} \boldsymbol{V}^{\boldsymbol{\Phi}}_v \\ \boldsymbol{V}^{\boldsymbol{\Sigma}}_v \end{bmatrix}
= \left( \boldsymbol{Q}_j^v \right)^T \begin{bmatrix}
\boldsymbol{V}^{\boldsymbol{\Phi}}_{v_{\text{son} 1}} \cdot \boldsymbol{T}_{v, v_{\text{son} 1}} \\
\boldsymbol{V}^{\boldsymbol{\Phi}}_{v_{\text{son} 2}} \cdot \boldsymbol{T}_{v, v_{\text{son} 2}}
\end{bmatrix}$$
also requires constant time. Given that there are $\mathcal{O}(N)$ clusters, the execution of  Algorithm \ref{alg. 6} takes $\mathcal{O}(N)$ time. \\
The computation of all coupling matrices $(\boldsymbol{S}_{v,v'})_{v,v' \in P}$ would take $\mathcal{O}(N^2)$ time, but since we compute only the necessary ones on the fly while calculating non-admissible matrix blocks, and by Proposition \ref{4.3.8} there are $\mathcal{O}(N \log N)$ non-admissible matrix blocks, the computation of these coupling matrices takes $\mathcal{O}(N \log N)$ time.\\
When executing Algorithm \ref{alg. 8}, we only calculate $\mathcal{O}(N \log N)$ matrix blocks, as admissible blocks are excluded. Determining whether a pair of clusters is admissible takes constant time, and using Lemma \ref{4.4.1} we limit the checks to $\mathcal{O}(N \log N)$ clusters. Since the involved multiplications for each matrix block take constant time, the execution of Algorithm 8 also takes $\mathcal{O}(N \log N)$ time. In summary, we can compute $\boldsymbol{K}^{\Sigma}_{\eta}$ in $\mathcal{O}(N \log N)$ time. Similarly, we derive that the storage requirements are $\mathcal{O}(N \log N)$.
\end{proof}

%-------------------------------------------------------------------------------------------

\begin{algorithm}[p]
\DontPrintSemicolon
\caption{Computation of the compressed kernel matrix}\label{alg. 8}
\KwData{Cluster tree $\mathcal{T}$, $(\boldsymbol{Q}_j^v)_{v \in P \backslash \mathcal{L(T)}}$, $(\boldsymbol{V}_v^{\boldsymbol{\Phi}})_{v \in P}$, $(\boldsymbol{V}_v^{\boldsymbol{\Sigma}})_{v \in P \backslash \mathcal{L(T)}}$, $\{ \boldsymbol{\xi}_1^v, ... , \boldsymbol{\xi}_n^v \}_{v \in P}$}
\KwResult{Sparse matrix $\boldsymbol{K}^{\boldsymbol{\Sigma}}_{\eta}$}
$ \ $ \;

\SetKwFunction{FMain}{setupColumn}
\SetKwProg{Fn}{Function}{:}{}
\Fn{\FMain{$v'$}}{
	\ForEach{son $v'_{\text{son}}$ of $v'$} {
    				$\texttt{setupColumn}(v'_{\text{son}})$\;
    	}
    	$\texttt{setupRow}(X, v')$\;
}
\textbf{end}\;

\SetKwFunction{FMain}{setupRow}
\SetKwProg{Fn}{Function}{:}{}
\Fn{\FMain{$v,v'$}}{
	%\uIf{$(v,v')$ is admissible }{
	%	$\boldsymbol{K}^{\boldsymbol{\Sigma}, \boldsymbol{\Sigma}}_{v, v'} = \boldsymbol{0}$
	%}
	\uIf{$v$ is not a leaf }{
		\ForEach{son $v_{\text{son}}$ of $v$} {
			\eIf{$(v_{\text{son}},v')$ is non-admissible }{
			$\begin{bmatrix} \boldsymbol{K}^{\boldsymbol{\Phi}, \boldsymbol{\Phi}}_{v_{\text{son}}, v'} & \boldsymbol{K}^{\boldsymbol{\Phi}, \boldsymbol{\Sigma}}_{v_{\text{son}}, v'} \\ \boldsymbol{K}^{\boldsymbol{\Sigma}, \boldsymbol{\Phi}}_{v_{\text{son}}, v'} & \boldsymbol{K}^{\boldsymbol{\Sigma}, \boldsymbol{\Sigma}}_{v_{\text{son}}, v'} \end{bmatrix} = \texttt{setupRow}(v_{\text{son}}, v')$\;
			}
			{
			$\begin{bmatrix} \boldsymbol{K}^{\boldsymbol{\Phi}, \boldsymbol{\Phi}}_{v_{\text{son}}, v'} & \boldsymbol{K}^{\boldsymbol{\Phi}, \boldsymbol{\Sigma}}_{v_{\text{son}}, v'} \\ \boldsymbol{K}^{\boldsymbol{\Sigma}, \boldsymbol{\Phi}}_{v_{\text{son}}, v'} & \boldsymbol{K}^{\boldsymbol{\Sigma}, \boldsymbol{\Sigma}}_{v_{\text{son}}, v'} \end{bmatrix} = \begin{bmatrix} \boldsymbol{V}^{\boldsymbol{\Phi}}_{v_{\text{son}}} \\ \boldsymbol{V}^{\boldsymbol{\Sigma}}_{v_{\text{son}}} \end{bmatrix} \boldsymbol{S}_{v_{\text{son}},v'} \begin{bmatrix} \left( \boldsymbol{V}^{\boldsymbol{\Phi}}_{v'} \right)^T & \left( \boldsymbol{V}^{\boldsymbol{\Sigma}}_{v'} \right)^T \end{bmatrix}$\;
			}
		}
		$\begin{bmatrix} \boldsymbol{K}^{\boldsymbol{\Phi}, \boldsymbol{\Phi}}_{v,v'} & \boldsymbol{K}^{\boldsymbol{\Phi}, \boldsymbol{\Sigma}}_{v,v'} 
\\ \boldsymbol{K}^{\boldsymbol{\Sigma}, \boldsymbol{\Phi}}_{v,v'} & \boldsymbol{K}^{\boldsymbol{\Sigma}, \boldsymbol{\Sigma}}_{v,v'} \end{bmatrix} 
= \left( \boldsymbol{Q}_j^v \right)^T \begin{bmatrix} \boldsymbol{K}^{\boldsymbol{\Phi}, \boldsymbol{\Phi}}_{v_{\text{son} 1},v'} & \boldsymbol{K}^{\boldsymbol{\Phi}, \boldsymbol{\Sigma}}_{v_{\text{son} 1},v'} \\ 
\boldsymbol{K}^{\boldsymbol{\Phi}, \boldsymbol{\Phi}}_{v_{\text{son} 2},v'} & \boldsymbol{K}^{\boldsymbol{\Phi}, \boldsymbol{\Sigma}}_{v_{\text{son} 2},v'} \end{bmatrix}$\;
	} 
	\uElseIf{$v'$ is a leaf cluster}{
		$\begin{bmatrix} \boldsymbol{K}^{\boldsymbol{\Phi}, \boldsymbol{\Phi}}_{v, v'} & \boldsymbol{K}^{\boldsymbol{\Phi}, \boldsymbol{\Sigma}}_{v, v'} \\ \boldsymbol{K}^{\boldsymbol{\Sigma}, \boldsymbol{\Phi}}_{v, v'} & \boldsymbol{K}^{\boldsymbol{\Sigma}, \boldsymbol{\Sigma}}_{v, v'} \end{bmatrix} = [ (k, \delta_{\boldsymbol{x}} \otimes \delta_{\boldsymbol{x}'} )_{\Omega \times \Omega} ]_{\boldsymbol{x} \in v, \boldsymbol{x}' \in v'}$\;
	}
	\uElse{
		\ForEach{son $v'_{\text{son}}$ of $v'$} {
			\eIf{$(v,v'_{\text{son}})$ is non-admissible }{
			load already computed block $\begin{bmatrix} \boldsymbol{K}^{\boldsymbol{\Phi}, \boldsymbol{\Phi}}_{v,v'_{\text{son}}} & \boldsymbol{K}^{\boldsymbol{\Phi}, \boldsymbol{\Sigma}}_{v,v'_{\text{son}}} \\ \boldsymbol{K}^{\boldsymbol{\Sigma}, \boldsymbol{\Phi}}_{v,v'_{\text{son}}} & \boldsymbol{K}^{\boldsymbol{\Sigma}, \boldsymbol{\Sigma}}_{v,v'_{\text{son}}} \end{bmatrix}$\;
			}
			{
			$\begin{bmatrix} \boldsymbol{K}^{\boldsymbol{\Phi}, \boldsymbol{\Phi}}_{v,v'_{\text{son}}} & \boldsymbol{K}^{\boldsymbol{\Phi}, \boldsymbol{\Sigma}}_{v,v'_{\text{son}}} \\ \boldsymbol{K}^{\boldsymbol{\Sigma}, \boldsymbol{\Phi}}_{v,v'_{\text{son}}} & \boldsymbol{K}^{\boldsymbol{\Sigma}, \boldsymbol{\Sigma}}_{v,v'_{\text{son}}} \end{bmatrix} = \begin{bmatrix} \boldsymbol{V}^{\boldsymbol{\Phi}}_{v} \\ \boldsymbol{V}^{\boldsymbol{\Sigma}}_{v} \end{bmatrix} \boldsymbol{S}_{v,v'_{\text{son}}} \begin{bmatrix} \left( \boldsymbol{V}^{\boldsymbol{\Phi}}_{v'_{\text{son}}} \right)^T & \left( \boldsymbol{V}^{\boldsymbol{\Sigma}}_{v'_{\text{son}}} \right)^T \end{bmatrix}$\;
			}
		}
		$\begin{bmatrix} \boldsymbol{K}^{\boldsymbol{\Phi}, \boldsymbol{\Phi}}_{v,v'} & \boldsymbol{K}^{\boldsymbol{\Phi}, \boldsymbol{\Sigma}}_{v,v'} 
\\ \boldsymbol{K}^{\boldsymbol{\Sigma}, \boldsymbol{\Phi}}_{v,v'} & \boldsymbol{K}^{\boldsymbol{\Sigma}, \boldsymbol{\Sigma}}_{v,v'} \end{bmatrix} 
= \begin{bmatrix} \boldsymbol{K}^{\boldsymbol{\Phi}, \boldsymbol{\Phi}}_{v,v'_{\text{son} 1}} & \boldsymbol{K}^{\boldsymbol{\Phi}, \boldsymbol{\Phi}}_{v,v'_{\text{son} 2}} \\ \boldsymbol{K}^{\boldsymbol{\Sigma}, \boldsymbol{\Phi}}_{v,v'_{\text{son} 1}} & \boldsymbol{K}^{\boldsymbol{\Sigma}, \boldsymbol{\Phi}}_{v,v'_{\text{son} 2}} \end{bmatrix}
\boldsymbol{Q}_{j'}^{v'}$\;
	}
	\textbf{end}\;
	store $\boldsymbol{K}^{\boldsymbol{\Sigma}, \boldsymbol{\Sigma}}_{v,v'}$ in $\boldsymbol{K}^{\boldsymbol{\Sigma}}_{\eta}$ and \KwRet $\begin{bmatrix} \boldsymbol{K}^{\boldsymbol{\Phi}, \boldsymbol{\Phi}}_{v, v'} & \boldsymbol{K}^{\boldsymbol{\Phi}, \boldsymbol{\Sigma}}_{v, v'} \\ \boldsymbol{K}^{\boldsymbol{\Sigma}, \boldsymbol{\Phi}}_{v, v'} & \boldsymbol{K}^{\boldsymbol{\Sigma}, \boldsymbol{\Sigma}}_{v, v'} \end{bmatrix}$\;
}
\textbf{end}\;

Initialize $\boldsymbol{K}^{\boldsymbol{\Sigma}}_{\eta}$ as $N \times N$ zero matrix\;
$\texttt{setupColumn}(X)$\;
\end{algorithm}

\newpage
% !TeX spellcheck = en_GB
\section{Approximation Methods} \label{Chap. 5}
To compute a Gaussian Process, it is imperative to either invert an $N \times N$ kernel matrix or to solve the corresponding linear systems. However, employing the Cholesky decomposition for this task usually demands $\mathcal{O}(N^3)$ time, making it impractical for large datasets ($N > 10,000$). Consequently, utilizing approximation methods and leveraging the structure of the kernel are essential for efficiently computing the Gaussian Process.\\

In this chapter, we examine various approximation methods and introduce a new approach using Samplets. The sparsity pattern of the compressed kernel matrix suggests leveraging a sparse Cholesky decomposition, which will be a key component for achieving a fast and reliable approximation.

%-------------------------------------------------------------------------------------------
% Overview
%-------------------------------------------------------------------------------------------

\subsection{Overview}
Several approximation and kernel learning approaches are available, and the following provides an overview of popular methods, with many of them detailed in \cite{RW}:

\begin{itemize}[topsep=5pt]
	\setlength\itemsep{0.1mm}
\item Compactly supported kernels yield a kernel that corresponds to the zero function within a subset of the domain, resulting in sparse kernel matrices. The computational reduction depends on the sparsity pattern of the matrix.

\item Low-rank approximation employs a rank $n < N$ matrix to approximate the kernel matrix. Storage and inverse computation are typically reduced to $\mathcal{O}(nN)$ and $\mathcal{O}(n^2 N)$ respectively. The Nyström approximation, outlined in section 8.1 of \cite{RW}, stands out as the most commonly used low-rank approximation.

\item Subset of data approximation entails selecting a subset of $n < N$ points from the dataset based on a specific criterion, establishing this amount as a new dataset. Storage and inverse computation are reduced to $\mathcal{O}(n^2)$ and $\mathcal{O}(n^3)$ respectively.

\item Subset of Regressors exploits the fact that the prediction can be expressed as
$\sum_{i=1}^N c_i k(\boldsymbol{x}, \boldsymbol{x}_i)$, as utilized in equation (\ref{eq. 2.2.1}). It seeks to approximate the Gaussian process by selecting a subset of $n < N$ functions from the $N$ regressors $k(\boldsymbol{x}, \boldsymbol{x}_1),...,k(\boldsymbol{x}, \boldsymbol{x}_N)$ based on a criterion. This reduces storage and inverse computation to $\mathcal{O}(nN)$ and $\mathcal{O}(n^2 N)$ respectively.

\item Bayesian committee machine predicts the function values $\boldsymbol{f}_{*}$ of $n$ test points by splitting the dataset into parts $\mathcal{D}^1, ..., \mathcal{D}^p$. It approximates the multivariate normal distribution of $\boldsymbol{f}_{*} \, | \, \mathcal{D}_N$, derived in Proposition \ref{2.1.5}, by using the distributions of $\boldsymbol{f}_{*} \, | \, \mathcal{D}^1, ..., \boldsymbol{f}_{*} \, | \, \mathcal{D}^p$. The computational effort reduces to $\mathcal{O}(n^2 N)$.

\item Conjugate gradient solves the linear system $(\boldsymbol{K} + \sigma^2 \boldsymbol{I}_N)\boldsymbol{c}=\boldsymbol{y}$ iteratively to compute the Gaussian process according to equation (\ref{eq. 2.2.1}). If the method is terminated at iteration $n$, its complexity amounts to $\mathcal{O}(n N^2)$.

\item Variational learning approximates the posterior distribution by selecting a subset of data, consisting of $n < N$ inducing points, and minimizing the Kullbeck-Leibler divergence of the approximation from the true posterior distribution. Computation is reduced to $\mathcal{O}(n^2 N)$ and details can be found in \cite{TI}. Stochastic variational regression (SvGP) enhances this process by employing stochastic optimization techniques.

%-------------------------------------------------------------------------------------------

\item Expectation propagation approximates the posterior distribution by iteratively refining the approximation. It uses $n < N$ points for the calculation, resulting in a cost of $\mathcal{O}(n^2 N)$. Further details are provided in \cite{TU}.

\item Blackbox Matrix-Matrix multiplication utilizes a Lanczos algorithm to optimize the likelihood function, uses pivoted Cholesky decomposition to precondition the kernel matrix and applies a modified conjugate gradient method to compute the Gaussian Process. The key innovation lies in the parallel calculations, which reduces computational effort to $\mathcal{O}(N^2)$. For a comprehensive understanding, refer to \cite{GPW}.

\item Kernel interpolation for scalable structured Gaussian Processes (KissGP) approximates the kernel function using interpolation based on $n < N$ inducing points. This approach reduces computation and storage to $\mathcal{O}(N)$.
\end{itemize}

Various methods are categorized as sparse approximations because they involve removing data. Additionally, they are often termed greedy approximations, as they select points in a greedy manner rather than identifying the optimal subset of training points. This approach is influenced by combinatorial effects. Nonetheless, nearly all approximation methods follow at least one of the three ideas depicted in Figure \ref{fig. 20}.

\begin{figure}[h]
	\centering
	\includegraphics[width=0.50\textwidth]{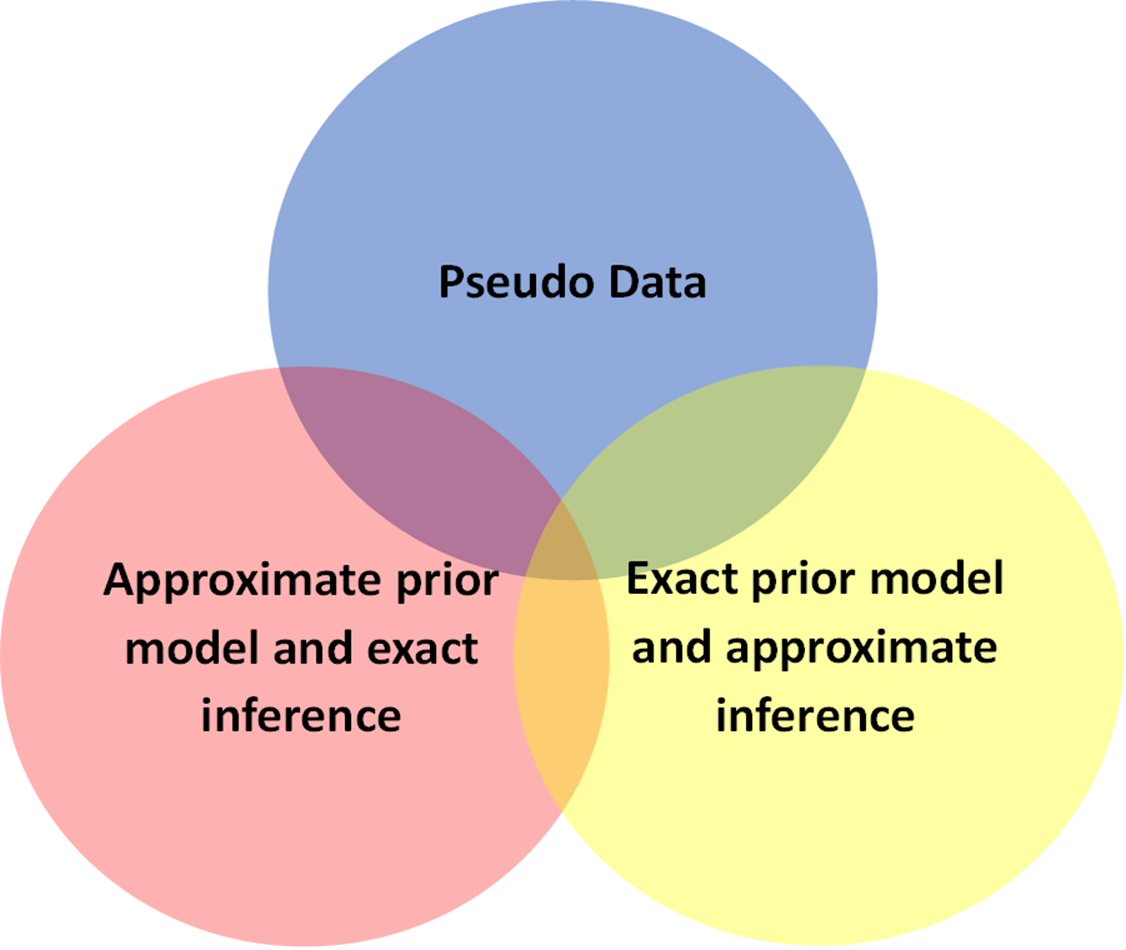}
	\caption{Categories of approximation methods. Inspired by a similar Figure in \cite{TU}.}\label{fig. 20}
\end{figure}

The choice of method may vary depending on the specific problem. For instance, if only prediction is of interest, the conjugate gradient method might be suitable. If the focus is on predicting for test points, the Bayesian committee machine could be more appropriate. If the entire Gaussian Process is of interest, the Nyström approximation may be employed. By using samplets, we aim to compute the entire Gaussian Process with an exact prior model while approximating the posterior functions. 

Current state-of-the-art approaches for this employ preconditioned conjugate gradient methods, often combined with stochastic Lanczos quadrature. The Python library GPyTorch is committed on implementing cutting-edge algorithms and accelerating computations, including GPU support. Therefore, we will compare our samplet-based algorithms with those implemented in GPyTorch.

%-------------------------------------------------------------------------------------------
% Gaussian Processes via Samplets
%-------------------------------------------------------------------------------------------

\subsection{Gaussian Processes via Samplets}
To compute the posterior functions in Proposition \ref{2.1.5} or the marginal likelihood function in equation (\ref{eq. 2.3.1}), we need to solve linear systems of the form
\begin{equation}
(\boldsymbol{K} + \sigma^2 \boldsymbol{I}_N) \boldsymbol{c} = \boldsymbol{y}, \label{eq. 5.2.1}
\end{equation}
where $\boldsymbol{c}=[c_1, ..., c_N]^T$ is a real-valued vector. For simplicity, we use the vector $\boldsymbol{y}=[y_1,...,y_N]^T$ on the right-hand side of (\ref{eq. 5.2.1}), although in practice, we also need to compute solutions for different vectors. Let $\boldsymbol{T} \in \mathbb{R}^{N \times N}$ be the orthogonal change of basis matrix from the samplet basis to the canonical basis. Then, we have the equivalent linear systems
$$(\boldsymbol{K} + \sigma^2 \boldsymbol{I}_N) \boldsymbol{c} = \boldsymbol{y} \Longleftrightarrow \boldsymbol{T} (\boldsymbol{K} + \sigma^2 \boldsymbol{I}_N) \boldsymbol{T}^{-1} \boldsymbol{T} \boldsymbol{c} = \boldsymbol{T} \boldsymbol{y} 
\Longleftrightarrow (\boldsymbol{K}^{\boldsymbol{\Sigma}} + \sigma^2 \boldsymbol{I}_N) \boldsymbol{T} \boldsymbol{c} = \boldsymbol{T} \boldsymbol{y}.$$
Instead of solving the linear system for the dense matrix $\boldsymbol{K}^{\boldsymbol{\Sigma}}$, which consists of many small entries, we use its compressed version $\boldsymbol{K}^{\boldsymbol{\Sigma}}_\eta \approx \boldsymbol{K}^{\boldsymbol{\Sigma}}$ to approximately compute the solution $\boldsymbol{c}$. This results in solving the perturbed linear system
\begin{equation}
(\boldsymbol{K}^{\boldsymbol{\Sigma}}_\eta + \sigma^2 \boldsymbol{I}_N) \boldsymbol{T} \tilde{\boldsymbol{c}} = \boldsymbol{T} \boldsymbol{y}, \label{eq. 5.2.2}
\end{equation}
where $\tilde{\boldsymbol{c}} \approx \boldsymbol{c}$ is the approximate solution. The matrix $\boldsymbol{K}^{\boldsymbol{\Sigma}}_\eta + \sigma^2 \boldsymbol{I}_N$ is symmetric, and since it is a compressed version of the positive definite matrix $\boldsymbol{K}^{\boldsymbol{\Sigma}} + \sigma^2 \boldsymbol{I}_N$, we assume it remains positive definite. If this assumption fails, it is likely that the parameters $q, \eta$, or the additional threshold for the matrix entries were poorly selected. It is also common to use $\tilde{\sigma}^2 > \sigma^2$ to enforce positive definiteness. As suggested in Remark \ref{3.3.9}, it is recommended to use $\tilde{\sigma}^2 \sim h_{X_N, \Omega}^{\tau - d/2}$ when working with a $\tau$-smooth kernel.\\

The perturbed linear system is, under suitable conditions, faster to solve than the original system in equation (\ref{eq. 5.2.1}). Let $\Omega \subset \mathbb{R}^d$ be an open and convex set, $k_{\nu, \ell}$ a Matérn kernel and $(X_N)_{N \in \mathbb{N}}$ be a sequence of quasi-uniform points in $\Omega$. In chapter \ref{Chap. 4}, we established the following results:
\begin{itemize}[topsep=5pt]
	\setlength\itemsep{0.1mm}
\item $\boldsymbol{K}^{\boldsymbol{\Sigma}}_\eta + \sigma^2 \boldsymbol{I}_N$ is a sparse matrix with $\mathcal{O}(N \log N)$ non-zero entries.
\item $\boldsymbol{K}^{\boldsymbol{\Sigma}}_\eta + \sigma^2 \boldsymbol{I}_N$ can be computed in $\mathcal{O}(N \log N)$ time, requiring $\mathcal{O}(N \log N)$ storage.
\item $\boldsymbol{T} \boldsymbol{y}$ can be computed via the fast samplet transform in $\mathcal{O}(N)$ operations. 
\item Given $\boldsymbol{T} \tilde{\boldsymbol{c}}$, then $\tilde{\boldsymbol{c}} = \boldsymbol{T}^{-1} (\boldsymbol{T} \tilde{\boldsymbol{c}})$ can be computed via the inverse fast samplet transform in $\mathcal{O}(N)$ operations.
\end{itemize}
This implies that for large $N$, the primary computational effort is in solving equation (\ref{eq. 5.2.2}) for $\boldsymbol{T} \tilde{\boldsymbol{c}}$. The method employed should take advantage of the sparsity pattern of the matrix $\boldsymbol{K}^{\boldsymbol{\Sigma}}_\eta + \sigma^2 \boldsymbol{I}_N$. Typically, a sparse Cholesky decomposition or a modified conjugate gradient method would be the initial approaches to consider. If such a method requires $\mathcal{O}(N^2)$ operations, then the overall cost of approximating the Gaussian Process would also be $\mathcal{O}(N^2)$. However, accurately determining the computational effort of these methods is challenging due to fill-in during the process. Given the Cholesky decomposition $\boldsymbol{L} \boldsymbol{L}^T = \boldsymbol{K}^{\boldsymbol{\Sigma}}_\eta + \sigma^2 \boldsymbol{I}_N$, we formulate Algorithm \ref{alg. 9} to compute $\tilde{\boldsymbol{c}} \approx \boldsymbol{c}$.\\

%-------------------------------------------------------------------------------------------

\begin{algorithm}[h]
\DontPrintSemicolon
\caption{Perturbed linear system}\label{alg. 9}
\KwData{Cluster tree $\mathcal{T}$, Cholesky $\boldsymbol{L} \boldsymbol{L}^T = \boldsymbol{K}^{\boldsymbol{\Sigma}}_\eta + \sigma^2 \boldsymbol{I}_N$, vector $\boldsymbol{y}$, $(\boldsymbol{Q}_j^v)_{v \in P \backslash \mathcal{L(P)}}$}
\KwResult{Approximate solution $\tilde{\boldsymbol{c}} \approx \boldsymbol{c}$ of the system $(\boldsymbol{K} + \sigma^2 \boldsymbol{I}_N) \boldsymbol{c} = \boldsymbol{y}$}
\SetKwFunction{FMain}{SolvePerturbedSystem}
\SetKwProg{Fn}{Function}{:}{}
\Fn{\FMain{$\boldsymbol{L}, \boldsymbol{y}$}}{
	$\boldsymbol{Ty} = \texttt{SampletTransform}(\boldsymbol{y})$\;
	$\boldsymbol{T} \tilde{\boldsymbol{c}} = \boldsymbol{L}^T \backslash ( \boldsymbol{L} 
	\backslash \boldsymbol{Ty} )$\;
	$\tilde{\boldsymbol{c}} = \texttt{InverseTransform}(\boldsymbol{T} \tilde{\boldsymbol{c}})$
	\;
  	\KwRet $\tilde{\boldsymbol{c}}$\;
}
\textbf{end}
\end{algorithm}

Using Algorithm \ref{alg. 8} to compute $\boldsymbol{K}^{\boldsymbol{\Sigma}}_\eta$, a sparse Cholesky decomposition to calculate $\boldsymbol{L} \boldsymbol{L}^T = \boldsymbol{K}^{\boldsymbol{\Sigma}}_\eta + \sigma^2 \boldsymbol{I}_N$ and Algorithm \ref{alg. 9} to solve the corresponding linear systems, the computation of the posterior functions in Proposition \ref{2.1.5} becomes straightforward. Additionally, by utilizing the approximation
\begin{align*}
\text{det}(\boldsymbol{K} + \sigma^2 \boldsymbol{I}_N) 
&= \text{det}(\boldsymbol{T} (\boldsymbol{K} + \sigma^2 \boldsymbol{I}_N) \boldsymbol{T}^{-1})\\
&= \text{det} (\boldsymbol{K}^{\boldsymbol{\Sigma}} + \sigma^2 \boldsymbol{I}_N)
\approx \text{det} (\boldsymbol{L} \boldsymbol{L}^T)
= \text{det}^2 (\boldsymbol{L}) = \left( \prod_{i=1}^N L_{i,i} \right)^2,
\end{align*}
we can also compute the log marginal likelihood in equation (\ref{eq. 2.3.1}) as
\begin{equation}
\log p(\boldsymbol{y} \, | \, X)
\approx -\dfrac{1}{2} \boldsymbol{y}^T \tilde{\boldsymbol{c}} - \sum_{i=1}^N \log L_{i,i} - \dfrac{N}{2} \log 2\pi \label{eq. 5.2.3}
\end{equation}
in $\mathcal{O}(N)$ operations. To learn the hyperparameters of the Gaussian Process, it is common to optimize this likelihood function through gradient descent methods, which requires evaluation of equation (\ref{eq. 2.3.2}). Let $\vartheta \in \mathbb{R}$ be a hyperparameter of the Gaussian Process, then we approximate
\begin{align*}
\boldsymbol{c}^T \dfrac{\partial \hat{\boldsymbol{K}}}{\partial \vartheta} \boldsymbol{c} 
= \boldsymbol{c}^T \boldsymbol{T}^{-1} \boldsymbol{T} \dfrac{\partial \hat{\boldsymbol{K}}}{\partial \vartheta} \boldsymbol{T}^{-1} \boldsymbol{T} \boldsymbol{c} 
= (\boldsymbol{T} \boldsymbol{c})^T \left( \dfrac{\partial \hat{\boldsymbol{K}}}{\partial \vartheta} \right)^{\boldsymbol{\Sigma}} \boldsymbol{T} \boldsymbol{c}
\approx (\boldsymbol{T} \tilde{\boldsymbol{c}})^T \left( \dfrac{\partial \hat{\boldsymbol{K}}}{\partial \vartheta} \right)^{\boldsymbol{\Sigma}}_{\eta} \boldsymbol{T} \tilde{\boldsymbol{c}}
\end{align*}
and estimate
\begin{align*}
\text{tr} \left( \hat{\boldsymbol{K}}^{-1} \dfrac{\partial \hat{\boldsymbol{K}}}{\partial \vartheta} \right)
&= \text{tr} \left( \boldsymbol{T}^{-1} \boldsymbol{T} \hat{\boldsymbol{K}}^{-1} \boldsymbol{T}^{-1} \boldsymbol{T} \dfrac{\partial \hat{\boldsymbol{K}}}{\partial \vartheta} \right)
= \text{tr} \left( \boldsymbol{T} \hat{\boldsymbol{K}}^{-1}  \boldsymbol{T}^{-1} \boldsymbol{T} \dfrac{\partial \hat{\boldsymbol{K}}}{\partial \vartheta} \boldsymbol{T}^{-1} \right)\\
&= \text{tr} \left(  \left( \hat{\boldsymbol{K}}^{\boldsymbol{\Sigma}} \right)^{-1} \left( \dfrac{\partial \hat{\boldsymbol{K}}}{\partial \vartheta} \right)^{\boldsymbol{\Sigma}} \right)
\approx \text{tr} \left(  \left( \hat{\boldsymbol{K}}^{\boldsymbol{\Sigma}}_{\eta} \right)^{-1} \left( \dfrac{\partial \hat{\boldsymbol{K}}}{\partial \vartheta} \right)^{\boldsymbol{\Sigma}}_{\eta} \right),
\end{align*}
where we have $\hat{\boldsymbol{K}}^{\boldsymbol{\Sigma}}_{\eta} = \boldsymbol{K}^{\boldsymbol{\Sigma}}_{\eta} + \sigma^2 \boldsymbol{I}_N$. In \cite{HM2}, the authors demonstrated that under certain assumptions, the involved inverse matrix has only $\mathcal{O}(N \log N)$ non-zero entries and for large $N$, we have 

%-------------------------------------------------------------------------------------------

$$\left( \hat{\boldsymbol{K}}^{\boldsymbol{\Sigma}}_{\eta} \right)^{-1} \approx \left( \hat{\boldsymbol{K}}^{\boldsymbol{\Sigma}} \right)^{-1} 
\qquad \text{in the same sense as} \qquad
\hat{\boldsymbol{K}}^{\boldsymbol{\Sigma}}_{\eta} \approx \hat{\boldsymbol{K}}^{\boldsymbol{\Sigma}}.$$
They also propose an algorithm for computing this inverse using the Cholesky decomposition, but a direct computation is generally inefficient. Instead, let $\boldsymbol{z}_1, ..., \boldsymbol{z}_t$ be random variables with $\boldsymbol{z}_i \sim \mathcal{N}(\boldsymbol{0}, \boldsymbol{I}_N)$, and apply the stochastic trace estimation
\begin{align*}
\text{tr} \left(  \left( \hat{\boldsymbol{K}}^{\boldsymbol{\Sigma}}_{\eta} \right)^{-1} \left( \dfrac{\partial \hat{\boldsymbol{K}}}{\partial \vartheta} \right)^{\boldsymbol{\Sigma}}_{\eta} \right)
&= \mathbb{E} \left[ \boldsymbol{z}_i^T \left( \hat{\boldsymbol{K}}^{\boldsymbol{\Sigma}}_{\eta} \right)^{-1} \left( \dfrac{\partial \hat{\boldsymbol{K}}}{\partial \vartheta} \right)^{\boldsymbol{\Sigma}}_{\eta} \boldsymbol{z}_i \right] \\
&\approx \dfrac{1}{t} \sum_{i=1}^t \left( \left( \hat{\boldsymbol{K}}^{\boldsymbol{\Sigma}}_{\eta} \right)^{-1} \boldsymbol{z}_i \right)^T \left( \dfrac{\partial \hat{\boldsymbol{K}}}{\partial \vartheta} \right)^{\boldsymbol{\Sigma}}_{\eta} \boldsymbol{z}_i.
\end{align*}
This estimate is unbiased and widely used (see e.g. \cite{GPW}), as it requires only a small number $t$ for a reliable approximation. By leveraging the precomputed Cholesky decomposition, we can efficiently solve the linear systems. Consequently, we compute the derivative of the log marginal likelihood function from equation (\ref{eq. 2.3.2}) as
\begin{align*}
\dfrac{\partial}{\partial \vartheta} \log p(\boldsymbol{y} \, | \, X) 
\approx \dfrac{1}{2} (\boldsymbol{T} \tilde{\boldsymbol{c}})^T \left( \dfrac{\partial \hat{\boldsymbol{K}}}{\partial \vartheta} \right)^{\boldsymbol{\Sigma}}_{\eta} \boldsymbol{T} \tilde{\boldsymbol{c}} - \dfrac{1}{2t} \sum_{i=1}^t \left( \left( \hat{\boldsymbol{K}}^{\boldsymbol{\Sigma}}_{\eta} \right)^{-1} \boldsymbol{z}_i \right)^T \left( \dfrac{\partial \hat{\boldsymbol{K}}}{\partial \vartheta} \right)^{\boldsymbol{\Sigma}}_{\eta} \boldsymbol{z}_i. \label{eq. 5.2.4}
\end{align*}
It is not obvious why this approximation should be computationally faster than calculating the gradient using equation (\ref{eq. 2.3.2}). However, with an understanding of the involved sparsity pattern and the structure of the derivatives, the time savings become clear. Since the Cholesky decomposition $\boldsymbol{L} \boldsymbol{L}^T = \hat{\boldsymbol{K}}^{\boldsymbol{\Sigma}}_{\eta}$ typically reflects the sparsity pattern of $\hat{\boldsymbol{K}}^{\boldsymbol{\Sigma}}_{\eta}$, we expect the computational cost of solving the linear systems to be significantly less than the usual $\mathcal{O}(N^2)$ operations. Specifically, because the compressed matrix contains only $\mathcal{O}(N \log N)$ non-zero entries, we anticipate the cost also to be $\mathcal{O}(N \log N)$.\\
The derivative of $\hat{\boldsymbol{K}}$ with respect to the hyperparameter $\vartheta \in \mathbb{R}$ often inherits the kernel matrix $\boldsymbol{K}$ or a variation of it. For example, consider the Matérn kernel $s^2 k_{\nu, \ell}$ with $\ell, s^2 >0$ and $\nu = n + 1/2$, where $n \in \mathbb{N}_{\geq 1}$. In this case, the hyperparameters are $\{ \ell, s^2, \sigma^2 \}$, and the derivatives are given by
\begin{align*}
\dfrac{\partial \hat{\boldsymbol{K}}}{\partial \ell} = \dfrac{s^2 \sqrt{2 \nu}}{\ell^2} \boldsymbol{\mathfrak{K}} - \dfrac{1}{\ell} \boldsymbol{K} \qquad
\dfrac{\partial \hat{\boldsymbol{K}}}{\partial s^2} = \dfrac{\boldsymbol{K}}{s^2} \qquad
\text{and} \qquad
\dfrac{\partial \hat{\boldsymbol{K}}}{\partial \sigma^2} = \boldsymbol{I}_N,
\end{align*}
where \index{\textit{$\boldsymbol{\mathfrak{K}}$,}}
$$\boldsymbol{\mathfrak{K}} := \big[\| \boldsymbol{x}_i - \boldsymbol{x}_j \|_2 k_{\nu, \ell}(\boldsymbol{x}_i, \boldsymbol{x}_j) \big]_{1 \leq i,j \leq N}.$$
As before, the matrix $\boldsymbol{K}$ can be compressed using samplets. The function
\begin{equation}
(\boldsymbol{x}, \boldsymbol{y}) \mapsto \| \boldsymbol{x} - \boldsymbol{y} \|_2 k_{\nu, \ell}(\boldsymbol{x}, \boldsymbol{y}) \label{eq. 5.2.5}
\end{equation}
is symmetric, but generally not positive definite, meaning $\boldsymbol{\mathfrak{K}}$ is not a kernel matrix. However, since the theory of samplets is not restricted to kernel functions, we can still compress $\boldsymbol{\mathfrak{K}}$. Furthermore, since Proposition \ref{4.3.2} holds more generally for symmetric and holomorphic functions, the function in (\ref{eq. 5.2.5}) is sufficiently asymptotically smooth to ensure that Theorem \ref{4.3.5} applies and we have

%-------------------------------------------------------------------------------------------

$$\boldsymbol{\mathfrak{K}}^{\boldsymbol{\Sigma}}_{\eta} \approx \boldsymbol{\mathfrak{K}}^{\boldsymbol{\Sigma}} 
\qquad \text{in the same sense as} \qquad
\boldsymbol{K}^{\boldsymbol{\Sigma}}_{\eta} \approx \boldsymbol{K}^{\boldsymbol{\Sigma}}.$$
Aside from that, $\boldsymbol{\mathfrak{K}}^{\boldsymbol{\Sigma}}_{\eta}$ consists of $\mathcal{O}(N \log N)$ non-zero entries. Using Algorithm \ref{alg. 8}, we require $\mathcal{O}(N \log N)$ operations to compute the derivatives
\begin{align*}
\left( \dfrac{\partial \hat{\boldsymbol{K}}}{\partial \ell} \right)^{\boldsymbol{\Sigma}}_\eta = \dfrac{s^2 \sqrt{2 \nu}}{\ell^2} \boldsymbol{\mathfrak{K}}^{\boldsymbol{\Sigma}}_\eta - \dfrac{1}{\ell} \boldsymbol{K}^{\boldsymbol{\Sigma}}_\eta \qquad
\left( \dfrac{\partial \hat{\boldsymbol{K}}}{\partial s^2} \right)^{\boldsymbol{\Sigma}}_\eta = \dfrac{\boldsymbol{K}^{\boldsymbol{\Sigma}}_\eta}{s^2} \qquad
\text{and} \qquad
\left( \dfrac{\partial \hat{\boldsymbol{K}}}{\partial \sigma^2} \right)^{\boldsymbol{\Sigma}}_\eta = \boldsymbol{I}_N,
\end{align*}
where $\boldsymbol{K}^{\boldsymbol{\Sigma}}_\eta$ has already been computed when solving the linear system in equation (\ref{eq. 5.2.1}). With efficient coding, it should be possible to compute $\boldsymbol{K}^{\boldsymbol{\Sigma}}_\eta$ and $\boldsymbol{\mathfrak{K}}^{\boldsymbol{\Sigma}}_\eta$ in a single process.\\
The main computational effort involved in approximating the derivative of the log marginal likelihood lies in solving the linear systems. If this takes $\mathcal{O}(N \log N)$ operations, the effort of evaluating the gradient is also $\mathcal{O}(N \log N)$. Algorithm \ref{alg. 10} outlines the procedure for approximating a Gaussian Process via Samplets. Based on the computational cost of the Cholesky decomposition, and the sparsity pattern of $\boldsymbol{L}$, the algorithm requires $\mathcal{O} (N \log N)$ operations at best.

%-------------------------------------------------------------------------------------------

%\newpage
\begin{algorithm}[p]
\DontPrintSemicolon
\caption{SampletsGP: Gaussian Process via Samplets}\label{alg. 10}
\KwData{Data $\mathcal{D}_N = \{ X_{train}, \boldsymbol{y}_{train} \}$, kernel $k$, initial hyperparameters $\theta_{N,1}$, vanishing moments $q$, compression parameter $\eta$, compression threshold $\tau_{comp}$, gradient based optimizer, optimization steps $n_{\text{steps}}$}
\KwResult{Posterior Process $\mathcal{GP}(m',k')$ ready for prediction}
$ \ $

\SetKwFunction{FMain}{TrainGP}
\SetKwProg{Fn}{Function}{:}{}
\Fn{\FMain{$X_{\text{train}}, \boldsymbol{y}_{\text{train}}, n_{\text{steps}}$}}{
	store ST = $\texttt{SampletTree} (X_{\text{train}}, q)$\;
	\For{$i = 1$ \KwTo $n_{steps}$} {
		$\boldsymbol{K}_\eta^{\boldsymbol{\Sigma}} = \texttt{SampletKernelCompressor} 
		(\text{ST}, k(\theta_{N,i}), X_{train}, \eta, \tau_{comp})$\;
		$\boldsymbol{L} = \texttt{SparseCholesky} (\boldsymbol{K}_\eta^{\boldsymbol{\Sigma}} + 
		\sigma_{N,i}^2 \boldsymbol{I}_N)$\;
		$\boldsymbol{Ty}_{train} = \texttt{SampletTransform}(\boldsymbol{y}_{train})$\;
		$\boldsymbol{T} \tilde{\boldsymbol{c}} = \boldsymbol{L}^T \backslash ( \boldsymbol{L} 
		\backslash \boldsymbol{Ty}_{train} )$\;
		$[\boldsymbol{z}_1, ..., \boldsymbol{z}_{50}] = \texttt{Random.standard\_normal}(\text{size} = N, \text{dimension} = 50)$\;
		$[\tilde{\boldsymbol{u}}_1, ..., \tilde{\boldsymbol{u}}_{50}] = \boldsymbol{L}^T \backslash ( \boldsymbol{L} 
	\backslash [\boldsymbol{z}_1, ..., \boldsymbol{z}_{50}] )$\;
		$\nabla = \{ \}$\;
		\ForEach{$\text{hyperparameter } \vartheta$} {
			$D = \dfrac{1}{2} (\boldsymbol{T} \tilde{\boldsymbol{c}})^T \left( \dfrac{\partial \hat{\boldsymbol{K}}}{\partial \vartheta} \right)^{\boldsymbol{\Sigma}}_{\eta} \boldsymbol{T} \tilde{\boldsymbol{c}} - \dfrac{1}{100} \displaystyle \sum_{i=1}^{50} \tilde{\boldsymbol{u}}_i^T \left( \dfrac{\partial \hat{\boldsymbol{K}}}{\partial \vartheta} \right)^{\boldsymbol{\Sigma}}_{\eta} \boldsymbol{z}_i$\;
			$\nabla = \nabla \cup \{ D \}$\;
		}
		store $\theta_{N,i+1} = \texttt{Update}(optimizer, \nabla , \theta_{N,i})
		$\;
    	}
    	$\boldsymbol{K}_\eta^{\boldsymbol{\Sigma}} = \texttt{SampletKernelCompressor} 
	(\text{ST}, k(\theta_{N,n_{steps}}), X_{train}, \eta, \tau_{comp})$\;
	store $\boldsymbol{L} = \texttt{SparseCholesky} (\boldsymbol{K}_\eta^{\boldsymbol{\Sigma}} + 
	\sigma_{N,n_{steps}}^2 \boldsymbol{I}_N)$\;
	store $\tilde{\boldsymbol{c}} = \texttt{SolvePerturbedSystem}(\boldsymbol{L}, \boldsymbol{y}_{\text{train}})$\;
	store the models log-likelihood $-\dfrac{1}{2} \boldsymbol{y}_{train}^T 
	\tilde{\boldsymbol{c}} - \sum_{i=1}^N \log L_{i,i} - \dfrac{N}{2} \log 2\pi$\;
}
\textbf{end}\\
$ \ $

\SetKwFunction{FMain}{PredictGP}
\Fn{\FMain{$X_{\text{pred}}$}}{
	$cov = k(\theta_{N, n_{steps}})$\;
	$\boldsymbol{K}_1 = cov (X_{pred}, X_{train})$\;
	$\boldsymbol{K}_2 = cov (X_{pred}, X_{pred})$\;
	$\boldsymbol{y}_{pred} = \boldsymbol{K}_1 \tilde{\boldsymbol{c}}$\;
	$\tilde{\boldsymbol{A}} = \boldsymbol{L} \backslash (\boldsymbol{T} \boldsymbol{K}_1^T)$\;
	$\boldsymbol{K}_{pred} = \boldsymbol{K}_2 - \tilde{\boldsymbol{A}}^T \tilde{\boldsymbol{A}}$\;
	\KwRet $\boldsymbol{y}_{pred}, \boldsymbol{K}_{pred}$\;
}
\textbf{end}
\end{algorithm}

%-------------------------------------------------------------------------------------------
% Bayesian Optimization via Samplets
%-------------------------------------------------------------------------------------------

\newpage
\subsection{Bayesian Optimization via Samplets} \label{Chap. 5.3}
To optimize a black-box function $f:[0,1]^d \rightarrow \mathbb{R}$ through Bayesian optimization, we employ Gaussian Processes as surrogate models and use Thompson samples as acquisition functions.\\
For large $N$, we approximate the Gaussian Process using samplets, following the procedure in Algorithm \ref{alg. 10}. Each GP requires constructing a samplet tree based on the $N$ observations. However, when only a few new observations are added, it is inefficient to build a completely new samplet tree. Instead, we update the existing tree by incorporating new observations into a single cluster at the highest level of the tree. This strategy ensures that only $\mathcal{O}(\log N)$ samplets need to be updated. Specifically, updating the samplet tree takes $\mathcal{O}(\log N)$ operations, instead of the $\mathcal{O}(N \log N)$ operations required for building a new one. Since only a small number of samplets are modified, updating the corresponding entries in the compressed kernel matrix is also more efficient than computing the entire matrix. Nevertheless, the computational gain is mostly in constants, as the update still requires $\mathcal{O}(N \log N)$ operations. To balance efficiency with accuracy, it is beneficial to rebuild the samplet tree periodically, for example, after every $1,000$ new observations, to fully leverage the potential of samplets.\\
One advantage of using Thompson sampling is the ability to select multiple new observations before constructing the next surrogate model. For large $N$, it becomes inefficient and not worthwhile to compute the surrogate model after selecting each individual point. Instead, we determine $100$ new observations at once before computing the next GP. To maintain a quasi-uniform distribution of points, we apply the $\gamma$-stabilized algorithm framework as described in \cite{BGW}. Let $k$ denote a stationary and $\tau$-smooth kernel with $\tau > d/2 + 1$. For a Thompson sample $a:[0,1]^d \rightarrow \mathbb{R}$ and $\gamma \in (0,1]$, the next point is chosen as
$$\boldsymbol{x}_{\text{next}} = \underset{\boldsymbol{x} \in [0,1]^d_{N, \gamma}}{\argmax} \, a(\boldsymbol{x}),$$
where $[0,1]^d_{N, \gamma} := \{ \boldsymbol{x} \in [0,1]^d \ | \ k'(\boldsymbol{x}, \boldsymbol{x}) \geq \gamma \| k' \|_{L^\infty([0,1]^d)} \}$ and $k'$ is the posterior kernel function after observing $N$ points. We use Algorithm \ref{alg. 1} to generate Thompson samples, with each sample based on $100d$ candidate points. Once $N_0-1$ points have been selected, the final point is chosen to maximize the posterior mean function $m'$, leading to
$$| f(\boldsymbol{x}^*) - f(\boldsymbol{x}_{N_0}) | 
\leq | f(\boldsymbol{x}^*) - m'(\boldsymbol{x}^*) - f(\boldsymbol{x}_{N_0}) - m'(\boldsymbol{x}_{N_0})|
\leq 2 \|  f-m' \|_{L^\infty([0,1]^d)}.$$
Since we maintain a quasi-uniform distribution of points, Theorems \ref{3.3.2} and \ref{3.3.4} guarantee convergence as $N_0$ increases. If the smoothness of the function is correctly estimated by the kernel, we achieve a convergence rate of $\mathcal{O}(N_0^{-\tau_f /d + 1/2})$ in the interpolation setting and $\mathcal{O} (N_0^{- \tau_f /(2\tau_f + d) + d/2})$ in the regression setting, with respect to the norm $\| \cdot \|_{L^\infty([0,1]^d)}$. Algorithm \ref{alg. 11} details the procedure for performing Bayesian optimization using samplets and incorporating Thompson sampling in conjunction with the $\gamma$-stabilized algorithm framework. 

%-------------------------------------------------------------------------------------------

\begin{algorithm}[h]
\DontPrintSemicolon
\caption{Bayesian Optimization via Samplets}\label{alg. 11}
\KwData{Black-box function $f$, optimization steps $N_0$, $\gamma \in (0,1]$, Data from Algorithm \ref{alg. 10}}
\KwResult{Estimate of the optimum $\left( \boldsymbol{x}^*, f(\boldsymbol{x}^*) \right)$}
\SetKwFunction{FMain}{SampletsBO}
\SetKwProg{Fn}{Function}{:}{}
\Fn{\FMain{$N_0, \gamma$}}{
	$X_{train}, \boldsymbol{y}_{train} = \{ \}, \{ \}$\;
	\For{$i = 1$ \KwTo $N_0 - 1$} {
		\If{$i$ is a multiple of $100$}{
			$\texttt{TrainGP}(X_{train}, \boldsymbol{y}_{train}, n_{steps})$\;
		}
		$X = \texttt{Random.rand}(\text{size} = 100d, \text{dimension} = d)$\;
		$\boldsymbol{m}', \boldsymbol{K}' = \texttt{PredictGP}(X)$\;
		$\boldsymbol{L} = \texttt{Cholesky}(\boldsymbol{K}')$\;
		$\boldsymbol{\varepsilon} = \texttt{Random.standard\_normal}(\text{size} = 100d)$\;
		$[ a(\boldsymbol{x}_1), ..., a(\boldsymbol{x}_{100d}) ] = \boldsymbol{m}' + \boldsymbol{L} \boldsymbol{\varepsilon}$\;
		$\boldsymbol{x}_{\text{next}} = \underset{\boldsymbol{x} \in X_{\lfloor i/100 \rfloor \cdot 100, \gamma}}{\argmax} \, a(\boldsymbol{x})$\;
		$X_{train} = X_{train} \cup \{ \boldsymbol{x}_{\text{next}} \}$\;
		$y_{train} = y_{train} \cup \{ y(\boldsymbol{x}_{\text{next}}) \}$\;
	}
	$\texttt{TrainGP}(X_{train}, \boldsymbol{y}_{train}, n_{steps})$\;
	$X = \texttt{Random.rand}(\text{size} = 100d, \text{dimension} = d)$\;
	$[ m'(\boldsymbol{x}_1), ..., m'(\boldsymbol{x}_{100d}) ], \boldsymbol{K}' = \texttt{PredictGP}(X)$\;
	$\boldsymbol{x}_{\text{next}} = \underset{\boldsymbol{x} \in X}{\argmax} \, m'(\boldsymbol{x})$\;
	$X_{train} = X_{train} \cup \{ \boldsymbol{x}_{\text{next}} \}$\;
	$y_{train} = y_{train} \cup \{ y(\boldsymbol{x}_{\text{next}}) \}$\;
	$\boldsymbol{x}^*, f(\boldsymbol{x}^*) = \underset{\boldsymbol{x} \in X_{ train }}{\argmax} \, y(\boldsymbol{x}), \underset{ y \in \boldsymbol{y}_{train} }{\max} y$\;
  	\KwRet $\boldsymbol{x}^*, f(\boldsymbol{x}^*)$\;
}
\textbf{end}
\end{algorithm}

\newpage
$ \ $
\newpage
% !TeX spellcheck = en_GB
\section{Numerical Experiments} \label{Chap. 6}
In this chapter, we compare the computational advantages of using Samplets against multiple GP algorithms implemented in GPyTorch across three examples. We also assess the theoretical results developed in Chapter \ref{Chap. 3} and examine how well they hold in practice. All computations were conducted on a system equipped with an Intel i5-12400F CPU (6 cores) and 48 GB of 3600 MHz RAM.

%-------------------------------------------------------------------------------------------
% Rastrigin function
%-------------------------------------------------------------------------------------------

\subsection{Rastrigin function}
For $d=1,2,3$, we aim to estimate the Rastrigin function
$$f: (5.12, 5.12)^d \rightarrow \mathbb{R}, f(\boldsymbol{x}) = 10d + \sum_{i=1}^d \left( x_i^2 - 10 \cos (2 \pi x_i) \right)$$
using noisy data with $\varepsilon \sim \mathcal{N}(0,1)$. This function is smooth and typically, the smoother the chosen kernel, the better the approximation. Nonetheless, we employ the Matérn $5/2$ kernel, which offers only moderate smoothness. The data points are randomly generated according to the uniform distribution on $\Omega = (5.12, 5.12)^d$ and they are generally not quasi-uniform. To compare the performance of the algorithms, we optimize the initial hyperparameters $\ell, s^2, \sigma^2 = 1$ only once. Since the samplet tree only needs to be constructed once for an entire optimization procedure, its setup time is excluded from the comparison. We then use Monte Carlo integration, to compute 
\begin{equation}
\dfrac{\| f-m' \|_{L^2 (\Omega)}}{\| f \|_{L^2 (\Omega)}} 
\approx \left( \dfrac{ \sum_{i=1}^m (f(\boldsymbol{z}_i)-m'(\boldsymbol{z}_i))^2 }{\sum_{i=1}^m f(\boldsymbol{z}_i)^2} \right)^{0.5} \label{eq. 6.1.1}
\end{equation}
based on $m=1,000$ randomly sampled points from $\Omega$. This allows us to validate the theoretical results presented in Chapter \ref{Chap. 3}, even when the points are not quasi-uniform. Specifically, following section \ref{Chap. 3.4}, the optimal convergence rate achievable is
$$\dfrac{\| f-m' \|_{L^2 (\Omega)}}{\| f \|_{L^2 (\Omega)}}
= \mathcal{O} (N^{-\frac{2.5}{ 5 + d }})
 \quad \text{as } N \rightarrow \infty.$$
Details regarding the GPyTorch algorithms and the implementation of the SampletsGP algorithm can be found in Section \ref{Chap. 6.4}. In this experiment, we consider two SampletsGP models with the following hyperparameters:

\begin{table}[h!]
\centering
\begin{tabular}{|c|c|c|c|}
\hline
$ \ $ & $q$ & $\eta$ & $\tau_{comp}$ \\ \hline
Model 1 & 7 & 1.0 & $10^{-4}$ \\ \hline
Model 2 & 3 & 1.0 & $10^{-2}$ \\ \hline
\end{tabular}
\end{table}
For $d=1$, we exclude Model 2 from our analysis because its low number of vanishing moments and large threshold are unnecessary for this context. Furthermore, Model 2 tends to produce unstable numerical results, meaning the computation of $\boldsymbol{K}_{\eta}^{\boldsymbol{\Sigma}}$ does not produce a positive definite matrix. Consequently, when utilizing Model 2, we treat $\sigma^2$ as a constant rather than a hyperparameter. It is set to the smallest integer that ensures $\boldsymbol{K}_{\eta}^{\boldsymbol{\Sigma}}$ remains positive definite. Figure 18 presents the outcomes of this experiment.

%-----------------------------------------------------------------------------------------------

\begin{figure}[H]
	\centering
	\includegraphics[width=0.81\textwidth]{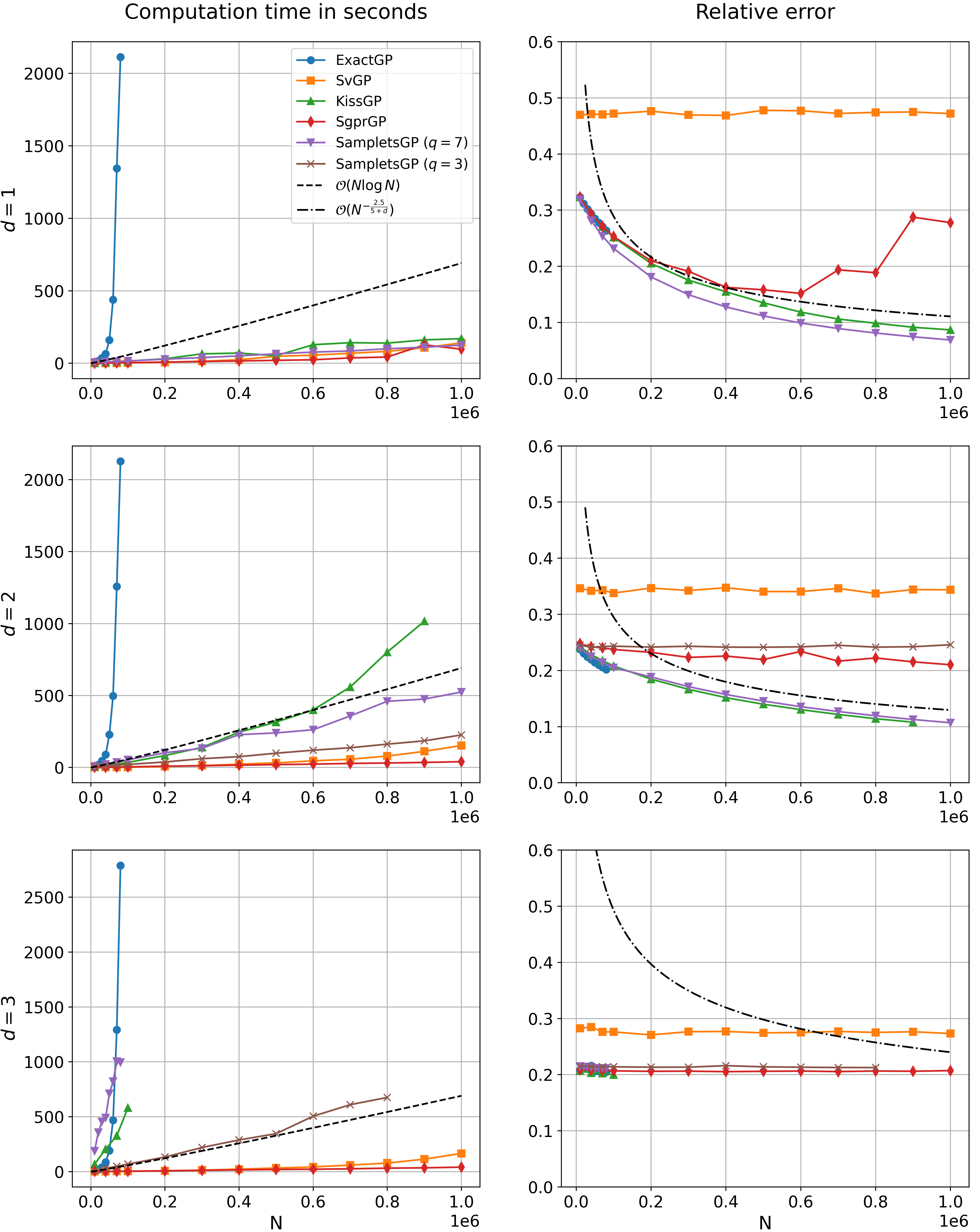}
	\caption{Computation times for a single step of hyperparameter optimization with the time taken to perform predictions on $m=1,000$ randomly sampled points. The figure also includes a comparison of the relative prediction errors.}\label{fig. 18}
\end{figure}

For a single dimension, Model 1 efficiently trains the GP and produces reliable predictions. Compared to the other models, it achieves a lower prediction error, likely due to the different implementations of the Adam optimizer in the Flux and Torch libraries. After the optimization step, Flux results in a signal noise of $s^2 = 1.1$, while Torch sets it to $s^2 = 1.05$. Model 1 also performs well for $d=2$, whereas most other GP models maintain a constant relative error. This is largely due to the reliance of other models on sparse methods, like inducing points, which limits their ability to leverage the convergence results from Chapter \ref{Chap. 3}. Additionally, these models use conjugate gradients which may terminate after $1,000$ iterations. For $d=3$, the first model becomes impractical due to high computational and memory requirements. The second model exhibits a constant prediction error and does not allow for optimization of $\sigma^2$. The constant error stems from a growing approximation error.

%-------------------------------------------------------------------------------------------
% Medicare Inpatient Hospitals - by Provider and Service
%-------------------------------------------------------------------------------------------

\subsection{Medicare Inpatient Hospitals - by Provider and Service} \label{Chap. 6.2}
Medicare is a federal health insurance program in the USA that covers people aged 65 or older and younger individuals with certain disabilities. The Medicare Inpatient Hospitals by Provider and Service \href{https://data.cms.gov/provider-summary-by-type-of-service/medicare-inpatient-hospitals/medicare-inpatient-hospitals-by-provider-and-service/data}{dataset} provides information about IPPS hospital stays for Medicare Part A beneficiaries (the part of Medicare that covers hospital care). Under the IPPS, each hospital stay is categorized into a Medicare Severity Diagnosis Related Group (MS-DRG) based on the patient’s diagnosis, the procedures performed, complicating conditions, age, sex, and discharge status. Hospitals set the prices they charge for the services they provide, and these charges reflect what the hospital bills the patient.\\
For our analysis, we use the \href{https://data.cms.gov/provider-summary-by-type-of-service/medicare-inpatient-hospitals/medicare-inpatient-hospitals-by-provider-and-service/data}{data} from 2022, which contains $145,742$ different MS-DRGs. There are $d=3$ key predictors:
\begin{enumerate}[topsep=5pt]
	\setlength\itemsep{0.1mm}
\item Total discharges: The number of patients treated and discharged.
\item Average submitted covered charges: The average amount hospitals billed for these patients.
\item Average Medicare Payment: The average amount Medicare actually paid.
\end{enumerate}
Our goal is to predict the average total payment amount using these predictors. The total payment amount includes multiple components such as the DRG amount, per diem payments, payments by other insurers, patient Part A coinsurance, deductible amounts, and any outlier payments related to the diagnosis. This amount is typically a bit higher than the Medicare payment, but drastically lower than the submitted charges by the hospital.\\

\noindent
Define $\Omega \subset \mathbb{R}^3$ to encompass the range of the data, specifically
$$\Omega = (10, 2751) \times (3433, 9709953) \times (144, 660249).$$ 
We assume that the unknown function  $f:\Omega \rightarrow \mathbb{R}$ is at least twice differentiable, such that $f \in W^2_2(\Omega)$. Since we are working within an interpolation framework, we employ the exponential kernel $k_{1/2,\ell}$ as a $2$-smooth kernel. Consider $N=145,000$ observations to train the Gaussian Processes and the remaining $m = 742$ observations to compute the relative error through Monte Carlo integration (\ref{eq. 6.1.1}). The points are not quasi-uniform, as it is unrealistic for Medicare payments to exceed the charges initially submitted by the hospital. Nevertheless, we test the theoretical results in Chapter \ref{Chap. 3} and the computational benefits in using samplets. Following section \ref{Chap. 3.4}, the optimal convergence rate achievable is
$$\dfrac{\| f-m' \|_{L^2 (\Omega)}}{\| f \|_{L^2 (\Omega)}}
= \mathcal{O} (N^{-\frac{1+d}{ 2d }})
 \quad \text{as } N \rightarrow \infty.$$
To compare algorithm performance, the initial hyperparameters $\ell, s^2, \sigma^2 = 1$ are being optimized over five trials. In this experiment, we evaluate two SampletsGP models with the following hyperparameters:

\begin{table}[H]
\centering
\begin{tabular}{|c|c|c|c|}
\hline
$ \ $ & $q$ & $\eta$ & $\tau_{comp}$ \\ \hline
Model 1 & 5 & 1.0 & $10^{-3}$ \\ \hline
\multirow{2}{*}{Model 2} & 3 & 1.0 & $10^{-2}$ \\ \cline{2-4} 
                         & 5 & 1.0 & $10^{-3}$ \\ \hline
\end{tabular}
\end{table}

%-----------------------------------------------------------------------------------------------

Model 1 and KissGP are expected to have high computational demands in three dimensions, so we limit their comparison to two dimensions by excluding the "Total discharges" predictor. Model 2 uses different $q$ and $\tau_{comp}$ during kernel learning and prediction. We first optimize the kernel parameters with a more efficient configuration ($q=3$), then apply a more refined setup ($q=5$) to minimize approximation error for reliable predictions. In this comparison, the samplets setup is included, though they only need to be constructed once. Figure \ref{fig. 19} presents the results.

\begin{figure}[h]
	\centering
	\includegraphics[width=0.85\textwidth]{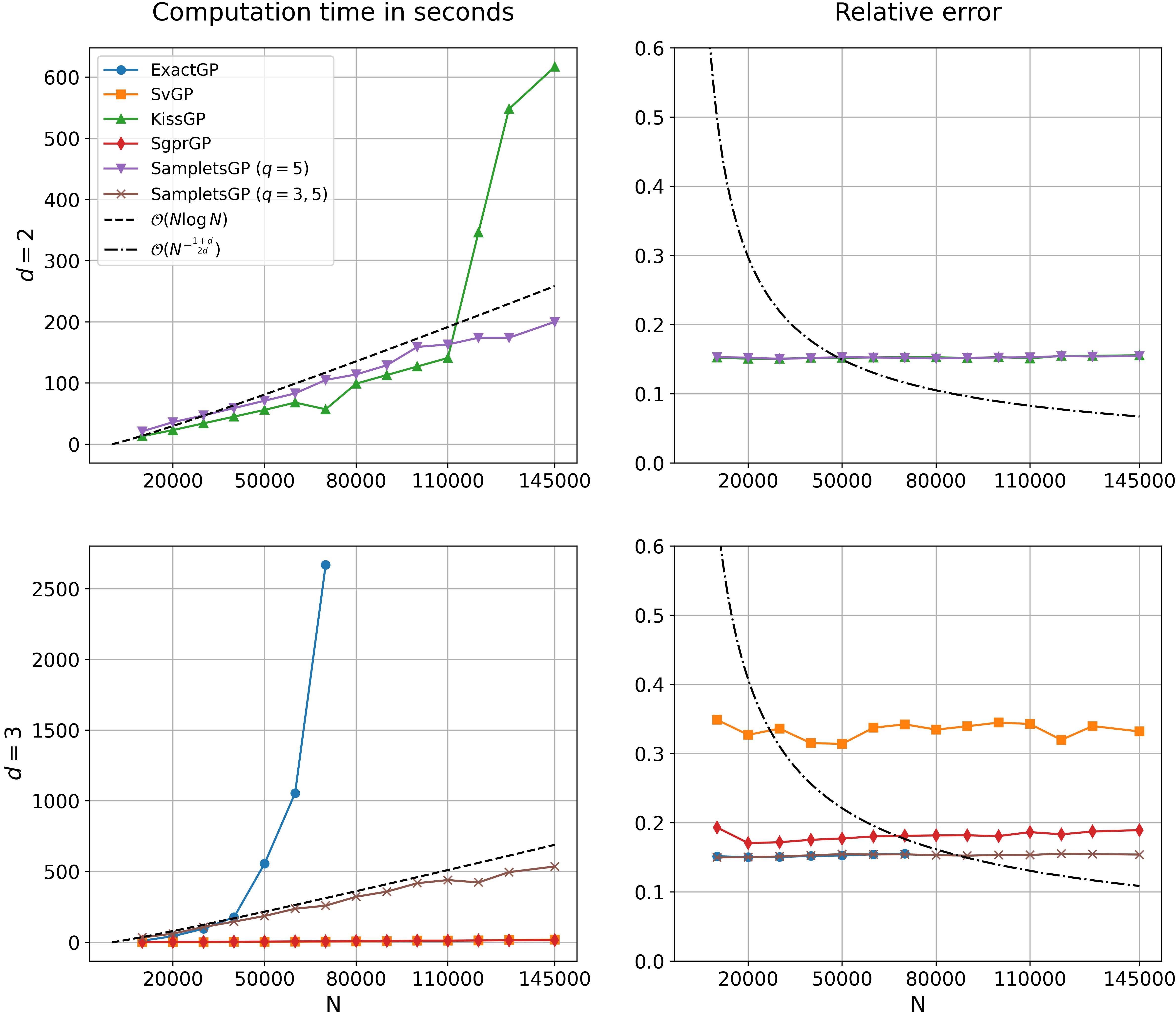}
	\caption{Computation times for five steps of hyperparameter optimization with the time taken to perform predictions on $m=742$ points. The figure also includes a comparison of the relative prediction errors.}\label{fig. 19}
\end{figure}

In both cases, the SampletsGP models perform well, achieving log-linear complexity with a stable prediction error. Even for the exact GP algorithm, the prediction error remains constant, indicating that the expected convergence results do not hold. Possible reasons for this could be that the assumption $f \in W^2_2(\Omega)$ is incorrect, the points are not quasi-uniform, or key predictors are missing. The good performance of the samplets approach indicates it can be effective even with non quasi-uniform points. This experiment highlights the importance of carefully selecting the hyperparameters $q$ and $\tau_{comp}$, with distinct considerations for kernel learning versus prediction. It also suggests that the "Total Discharges" predictor may be less important than initially expected.

%-------------------------------------------------------------------------------------------
% Ackley function
%-------------------------------------------------------------------------------------------

\newpage
\subsection{Ackley function} \label{Chap. 6.3}
This example aims to compare Algorithm \ref{alg. 11} with the conventional approach of performing Bayesian Optimization using exact Gaussian Processes. In Bayesian Optimization, the number of observations is typically limited, as function evaluations are costly. For low-dimensional settings, this often restricts us to a few hundred observations, whereas in higher-dimensional scenarios, the limit may extend to several thousand. However, in the rare instance where a low-dimensional problem includes up to 20,000 observations, we are interested in exploring whether Gaussian Processes with Samplets provide any advantages. To evaluate this, we aim to maximize the 2D negative Ackley function
$$f(x, y) = 20 \exp \left(-0.2 \cdot \sqrt{ 0.5 (x^2 + y^2)} \right) + \exp\left( 0.5 \left(\cos 2\pi x + \cos 2\pi y \right)\right) - 20 - e$$
for $x,y \in [-5,5]$ within an interpolation framework. This function is non-convex and exhibits multiple local maxima, with a global maximum at $(0,0)$. To capture its behaviour, we employ the Matérn 5/2 kernel and initialize the hyperparameters $\ell, s^2, \sigma^2 = 1$ for the first GP model. Every $100$ observations, a new GP model is constructed, with the previous hyperparameters serving as starting values. For each model, these hyperparameters undergo two rounds of optimization. Following section \ref{Chap. 5.3}, the optimal convergence rate achievable is
$$| f(\boldsymbol{x}^*) - f(\boldsymbol{x}_N) |
= \mathcal{O} (N^{-0.75})
\quad \text{as } N \rightarrow \infty.$$
Since exactGP is faster than SampletsGP for a moderate number of observations, we use exactGP in Algorithm \ref{alg. 11} for up to $10,000$ observations. Beyond this threshold, we switch to the SampletsGP model with the following hyperparameters:

\begin{table}[h!]
\centering
\begin{tabular}{|c|c|c|c|}
\hline
$ \ $ & $q$ & $\eta$ & $\tau_{comp}$ \\ \hline
Model 1 & 5 & 1.0 & $10^{-3}$ \\ \hline
\end{tabular}
\end{table}

For simplicity, we use a single SampletsGP model without distinguishing between kernel learning and prediction. For comparison, the ExactBO algorithm is simply Algorithm \ref{alg. 11}, excluding the $\gamma$-stabilized algorithm framework and relying solely on exactGP. Figure \ref{fig. 20} presents the results.

\begin{figure}[H]
	\centering
	\includegraphics[width=0.87\textwidth]{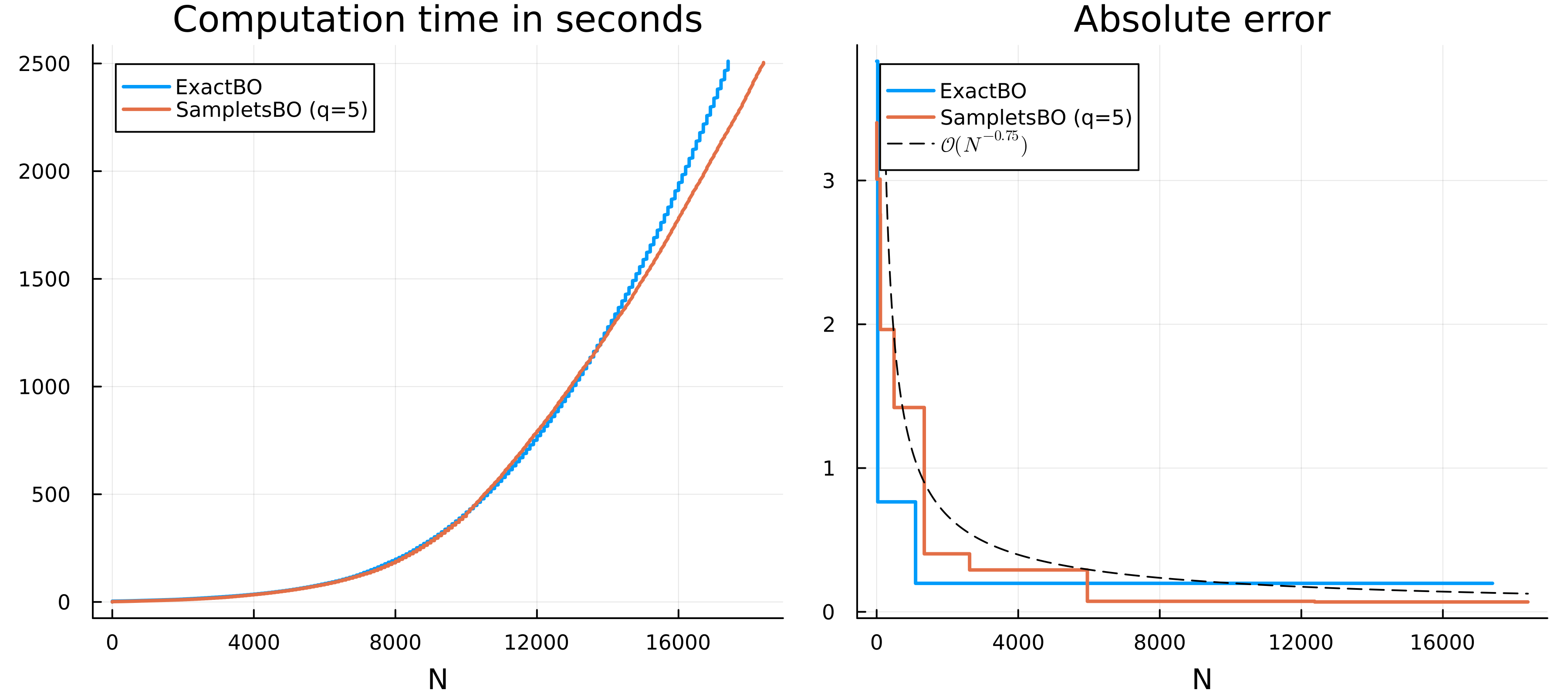}
	\caption{Computation times and absolute error $| f(\boldsymbol{x}^*) - f(\boldsymbol{x}_N) |$ for two Bayesian Optimization algorithms.}\label{fig. 20}
\end{figure}

%-----------------------------------------------------------------------------------------------

This experiment demonstrates that $20,000$ observations are inadequate to make SampletsBO more advantageous than the ExactBO algorithm. Although the SampletsBO algorithm has a log-linear complexity, the constants involved are too large. This suggests that utilizing Samplets in Bayesian Optimization is not advantageous, as a large number of observations is seldom available to fully leverage their capabilities. Large datasets are typically available in higher-dimensional spaces, while the Samplets approach is most effective in low-dimensional settings.\\
The convergence rates of both algorithms are similar. However, with a large number of observations, the ExactBO algorithm tends to get stuck in exploitation. In contrast, the $\gamma$-stabilized framework for SampletsBO promotes continued exploration. Further numerical experiments are needed to comprehensively study the convergence rates.

%-------------------------------------------------------------------------------------------
% Details on the Experiments
%-------------------------------------------------------------------------------------------

\subsection{Details on the Experiments} \label{Chap. 6.4}
To compute a Gaussian Process via Samplets, I implemented Algorithm \ref{alg. 10} in Julia. For the computation of the compressed kernel matrix, I utilized the \href{https://github.com/muchip/fmca}{FMCA library}, which implements the algorithms discussed in Chapter \ref{Chap. 4} in \texttt{C++}. For efficient sparse Cholesky decomposition, I employed the \href{https://github.com/DrTimothyAldenDavis/SuiteSparse/tree/dev/CHOLMOD}{CHOLMOD library} from \href{https://github.com/DrTimothyAldenDavis/SuiteSparse}{SuiteSparse}. It employs an approximate minimum degree algorithm by default to minimize fill-in. For larger or denser matrices, it uses graph partitioning methods for nested dissection. Until now, many programming languages, including Python, Julia and MATLAB, use this library by default for performing sparse matrix Cholesky decompositions. The hyperparameters used are the lengthscale $\ell$, signal variance $s^2$ and noise variance $\sigma^2$. To optimize these, I use the Adam optimizer from the \href{https://github.com/FluxML/Flux.jl}{Flux library}, written in Julia, with the following constraints: 
$$\ell \in [0.005, 2.0], \quad s^2 \in [0.05, 20.0] \quad \text{and} \quad \sigma^2 \in [0.1, 2.0].$$
Before computing a Gaussian Process, I normalize the training data to reinforce the assumption of a zero-mean function and to improve numerical stability.\\
The algorithms used for comparison are implemented in GPyTorch, with implementation details available on their \href{https://gpytorch.ai/}{homepage}. For consistency in the comparisons, I also normalize the data, employ the same hyperparameters, apply the same constraints, and utilize the Adam optimizer.\\

The KissGP implementation requires specifying a grid size to interpolate the kernel function. For all experiments, I automatically determined this value using the helper function $\texttt{gpytorch.utils.grid.choose\_grid\_size}$, as recommended by GPyTorch. Manual grid size settings were also tested but yielded similar results. The SvGP and SgprGP implementations require selecting a set of inducing points. For SvGP, approximately $\sqrt{N}$ points were randomly selected from the total $N$. For SgprGP, the number of inducing points varied between $500$ and $1000$, chosen based on the best performance.\\

The experiment in Section \ref{Chap. 6.2} used the Medicare Inpatient Hospitals by Provider and Service dataset, accessible at \href{https://data.cms.gov/provider-summary-by-type-of-service/medicare-inpatient-hospitals/medicare-inpatient-hospitals-by-provider-and-service/data}{https://data.cms.gov/provider-summary-by-type-of-service/medicare-inpatient-hospitals/medicare-inpatient-hospitals-by-provider-and-service/data}.

\newpage
% !TeX spellcheck = en_GB
\section{Conclusion and Outlook} \label{Chap. 7}
The convergence results in Chapter \ref{Chap. 3} address the model selection problem for Gaussian Processes in regression tasks. They advocate the use of $\tau$-smooth kernels and outline the configurations needed to attain optimal convergence rates in Sobolev norms. These results depend on the smoothness of both the function and the kernel, while remaining unaffected by other kernel parameters. This independence enables the use of maximum likelihood estimation for parameter learning without compromising the convergence rates. Importantly, these findings are particularly applicable to the significant class of Matérn kernels. The experiments indicate that minor deviations from the assumptions can invalidate the convergence rates. However, it is still prudent to apply these findings during the model selection process to identify the model with the highest potential for optimal rates.

Samplets tackle the construction challenges of Gaussian Processes when dealing with a large number of observations in low-dimensional settings. Specifically, we use Samplets to compress the kernel matrix into a sparse format, enabling efficient computation. Algorithm \ref{alg. 10} details this construction process and because all observations are retained, we can utilize the convergence rates discussed in Chapter \ref{Chap. 3}. In the best case this approach reduces the cubic computational complexity to log-linear without loosing accuracy. Numerical experiments show that even with slight deviations from quasi-uniform point distributions, the log-linear complexity and prediction accuracy remain robust. Compared to state-of-the-art methods, SampletsGP provides enhanced flexibility with tunable parameters and includes all observations instead of relying on a subset. In three and four dimensions, careful selection of these tunable parameters is essential to balance approximation error, numerical stability, and computational cost. When using the SampletsGP algorithm, it is recommended to initially optimize the kernel parameters with a more efficient configuration, followed by a refined setup for prediction to minimize approximation error. In efficient configurations, it is often beneficial to treat the noise variance $\sigma^2$ as a constant rather than optimizing it, due to potential numerical instability. Since the Samplets only need to be constructed once for the kernel parameter optimization, SampletsGP is an efficient method for parameter learning.\\

A significant advancement in the Samplet theory would be extending their application to higher-dimensional spaces, though this currently appears unfeasible. For Gaussian Processes, future research could explore integrating Samplets with existing GP approximation methods or replacing Cholesky decomposition with conjugate gradient methods to improve efficiency.

The implemented algorithms also exhibit potential for further optimization. Currently, the SampletsGP algorithm computes the two matrices $\boldsymbol{K}_{\eta}^{\Sigma}$ and $\boldsymbol{\mathfrak{K}}^{\boldsymbol{\Sigma}}_\eta$ separately, which accounts for approximately $40 \%$ of the total computation time. Merging these calculations into a single process could lead to a $20 \%$ reduction in time. Additionally, converting the sparse matrix format from Python to Julia’s SparseMatrixCSC format consumes roughly $10 \%$ of the runtime, suggesting that a direct implementation in Python could save another $10 \%$. The Bayesian Optimization algorithm currently lacks an implementation for updating the Samplet tree.
\newpage
% !TeX spellcheck = en_GB
\appendix
\section{Appendix} \label{Chap. 8}
\subsection{Reproducing Kernel Hilbert Spaces} \label{Chap. 8.1}
The Moore-Aronszajn Theorem (Proposition \ref{2.2.2}) is well-known in the literature. However, despite Aronszajn's initial work, there are not many comprehensive references that verify it. Thus, a proof will be provided in this section, based on the proof outline provided on page 11 of \cite{HK}. The following two fundamental results from functional analysis are necessary:

\begin{proposition}
Every inner product space can be completed to a Hilbert space.
\end{proposition}

\begin{proposition}
Let $M \subseteq \mathcal{H}$ be a closed subspace of a Hilbert space $\mathcal{H}$. Then is
$$\mathcal{H} = M \oplus M^\perp.$$
\end{proposition}

These statements and their verifications can be found in \cite{M} as Example 10.7 and Theorem 10.12. With these foundations established, we now proceed to the proof of the Moore-Aronszajn Theorem.

\begin{claim}
For every kernel $k$ exists a unique RKHS.
\end{claim}

%\textbf{1. Step:} For every kernel $k$ exists a unique RKHS.\\[5mm]
\begin{proof}
Consider the space of functions
\begin{equation}
\mathcal{H}_0 = \left\{ g(\cdot) = \sum_{i=1}^n c_i k(\cdot,\boldsymbol{x}_i) \, : \, n\in \mathbb{N}, \boldsymbol{x}_i \in \Omega, c_i \in \mathbb{R} \right\}. \label{A.1.1}
\end{equation}
Let $g'(\cdot) = \sum_{j=1}^{n'} c_j' k(\cdot,\boldsymbol{x}_j ')$ and set
$$\langle g, g' \rangle_{\mathcal{H}_0 }
= \sum_{i=1}^n \sum_{j=1}^{n'} c_i c_j' k(\boldsymbol{x}_i, \boldsymbol{x}_j').$$
We want $\langle \cdot, \cdot \rangle_{\mathcal{H}_0 }$ to define an inner product on $\mathcal{H}_0$. The linearity and symmetry can be easily verified. By the definition of a kernel, we have
$$\langle g, g \rangle_{\mathcal{H}_0 } = \sum_{i,j=1}^n c_i c_j k(\boldsymbol{x}_i, \boldsymbol{x}_j) \geq 0.$$
It is $k(\cdot, \boldsymbol{x}) \in \mathcal{H}_0$ for every $\boldsymbol{x} \in \Omega$ and
$$\langle g, k(\cdot, \boldsymbol{x}) \rangle_{\mathcal{H}_0}
= \sum_{i=1}^n c_i k(\boldsymbol{x}_i, \boldsymbol{x})
= \sum_{i=1}^n c_i k(\boldsymbol{x}, \boldsymbol{x}_i)
= g(\boldsymbol{x}).$$
Let $\langle g, g \rangle_{\mathcal{H}_0 } = 0$ and $t \in \mathbb{R}$. For $\boldsymbol{x} \in \Omega$, we have
\begin{align*}
0 \leq \langle g + t k (\cdot, \boldsymbol{x}), g + t k (\cdot, \boldsymbol{x} \rangle_{\mathcal{H}_0} = 2t \langle g, k (\cdot, \boldsymbol{x}) \rangle_{\mathcal{H}_0} + t^2 k(\boldsymbol{x}, \boldsymbol{x}) = 2 t g(\boldsymbol{x}) + t^2 k(\boldsymbol{x}, \boldsymbol{x}).
\end{align*}
If $k(\boldsymbol{x}, \boldsymbol{x})=0$, setting $t = - g(\boldsymbol{x})$ gives $g(\boldsymbol{x})=0$. If $k(\boldsymbol{x}, \boldsymbol{x}) >0$, setting $t = - g(\boldsymbol{x})/k(\boldsymbol{x}, \boldsymbol{x})$ yields
$$2\dfrac{g(\boldsymbol{x})}{k(\boldsymbol{x}, \boldsymbol{x})} g(\boldsymbol{x}) 
\leq \dfrac{g(\boldsymbol{x})^2 }{k(\boldsymbol{x}, \boldsymbol{x})^2 } k(\boldsymbol{x}, \boldsymbol{x}) \Longrightarrow 2 g(\boldsymbol{x})^2 \leq g(\boldsymbol{x})^2 \Longrightarrow g(\boldsymbol{x}) = 0.$$
%-------------------------------------------------------------------------------------------
This implies $g \equiv 0$, showing that $\langle \cdot, \cdot \rangle_{\mathcal{H}_0 }$ is an inner product. Because the inner product is continuous, the completion $\mathcal{H}:=\overline{\mathcal{H}_0}$ is a RKHS for $k$. \\
Let $G$ be another RKHS for $k$ and $g \in G$. As $k(\cdot, \boldsymbol{x}) \in G$ we have $\mathcal{H} \subseteq G$ and can write $g = g_{\mathcal{H}} + g_{\mathcal{H}^\perp}$ with $g_{\mathcal{H}} \in \mathcal{H}$ and $g_{\mathcal{H}^\perp} \in \mathcal{H}^\perp$. Because of
$$\langle k(\boldsymbol{x}, \cdot)), k(\cdot, \boldsymbol{x}') \rangle_\mathcal{H}
= k(\boldsymbol{x}, \boldsymbol{x}')
= \langle k(\boldsymbol{x}, \cdot)), k(\cdot, \boldsymbol{x}') \rangle_G$$
and linearity of the inner products, it is $\langle \cdot , \cdot \rangle_\mathcal{H} = \langle \cdot , \cdot \rangle_G$ for functions in $\mathcal{H}$. It follows
\begin{align*}
g(\boldsymbol{x}) = \langle g, k(\cdot, \boldsymbol{x}) \rangle_G 
= \langle g_{\mathcal{H}} , &k(\cdot, \boldsymbol{x}) \rangle_G 
		+ \langle g_{\mathcal{H}^\perp} , k(\cdot, \boldsymbol{x}) \rangle_G\\
&= \langle g_{\mathcal{H}} , k(\cdot, \boldsymbol{x}) \rangle_G 
= \langle g_{\mathcal{H}} , k(\cdot, \boldsymbol{x}) \rangle_\mathcal{H} 
= g_\mathcal{H} (\boldsymbol{x})
\end{align*}
for every $\boldsymbol{x} \in \Omega$, respectively $g \equiv g_\mathcal{H} \in \mathcal{H}$.
\end{proof}

To verify the reverse direction of the claim, additional results from the functional analysis literature are not needed.

\begin{claim}
For every RKHS, there exists a unique kernel.
\end{claim}

\begin{proof}
First show that the corresponding function $k$ is a kernel. Since
$$k(\boldsymbol{x}, \boldsymbol{x}') 
= \langle k(\cdot, \boldsymbol{x}'), k(\cdot, \boldsymbol{x}) \rangle_\mathcal{H}
= \langle k(\cdot, \boldsymbol{x}), k(\cdot, \boldsymbol{x}') \rangle_\mathcal{H}
= k(\boldsymbol{x}', \boldsymbol{x})$$
for every $\boldsymbol{x},\boldsymbol{x}'\in \Omega$, the function is symmetric. To show that $k$ is positive semi-definite, let $n \in \mathbb{N}$, $c_1,...,c_n \in \mathbb{R}$ and $\boldsymbol{x}_1,...,\boldsymbol{x}_n \in \Omega$. Then,
\begin{align*}
\sum_{i=1}^n \sum_{j=1}^n c_i c_j k(\boldsymbol{x}_i, \boldsymbol{x}_j) 
&= \sum_{i=1}^n \sum_{j=1}^n c_i c_j \langle k(\cdot, \boldsymbol{x}_j), k(\cdot, \boldsymbol{x}_i) \rangle_\mathcal{H} \\
&= \langle \sum_{j=1}^n c_j k(\cdot, \boldsymbol{x}_j), \sum_{i=1}^n c_i k(\cdot, \boldsymbol{x}_i) \rangle_\mathcal{H}
= \big\| \sum_{i=1}^n c_i k(\cdot, \boldsymbol{x}_i) \big\|_\mathcal{H}^2 \geq 0.
\end{align*}
To verify the uniqueness, assume two kernels $k,k'$ both satisfy definition \ref{2.2.1} for the same RKHS. Then,
\begin{align*}
k(\boldsymbol{x}, \boldsymbol{x}') 
= \langle k(\boldsymbol{x}, \cdot), k'(\cdot, \boldsymbol{x}') \rangle_\mathcal{H}
&= \langle k'(\cdot, \boldsymbol{x}'), k(\boldsymbol{x}, \cdot) \rangle_\mathcal{H}\\
&= \langle k'(\cdot, \boldsymbol{x}'), k(\cdot, \boldsymbol{x}) \rangle_\mathcal{H}
= k'(\boldsymbol{x}, \boldsymbol{x}')
\end{align*}
for every $\boldsymbol{x},\boldsymbol{x}' \in \Omega$. It follows $k \equiv k'$.
\end{proof}

%-------------------------------------------------------------------------------------------
% Interior cone condition and Lipschitz boundary
%-------------------------------------------------------------------------------------------

\subsection{Interior cone condition and Lipschitz boundary} \label{Chap. 8.2}

To effectively investigate Sobolev spaces, it is essential for $\Omega \subset \mathbb{R}^d$ to satisfy the interior cone condition. This condition allows us to utilize embedding Theorems.

\begin{defi}
An open and non-empty set $\Omega\subset \mathbb{R}^d$ satisfies the interior cone condition\index{Interior cone condition}, if
there exists an angle $\delta \in (0,\pi /2)$ and a radius $R>0$ such that for every $\boldsymbol{x} \in \Omega$ a unit vector $\boldsymbol{\xi}(\boldsymbol{x})$ exists such that the cone \index{\textit{$C (\boldsymbol{x}, \boldsymbol{\xi}(\boldsymbol{x}), \delta, R)$,}}
$$C (\boldsymbol{x}, \boldsymbol{\xi}(\boldsymbol{x}), \delta, R) = \big\{ \boldsymbol{x}+\lambda \boldsymbol{y} \ | \ \boldsymbol{y} \in \mathbb{R}^d, \|\boldsymbol{y}\|_2 = 1, \boldsymbol{y}^T \boldsymbol{\xi}(\boldsymbol{x}) \geq \cos (\delta), \lambda \in [0,R] \big\}$$
is contained in $\Omega$.
\end{defi}

%-------------------------------------------------------------------------------------------

A second condition often required for $\Omega \subset \mathbb{R}^d$ is that it has a Lipschitz boundary. This condition creates a link between Sobolev spaces and the RKHS associated with Matérn kernels. The embedding Theorems can also be formulated more restrictive.

\begin{defi} \label{A.2.2}
A domain is an open, connected an non-empty set in $\mathbb{R}^d$.\index{Domain} A domain $\Omega_i \subseteq \mathbb{R}^d$ is called special Lipschitz domain\index{Special Lipschitz domain}, if there exists a rotation $\tilde{\Omega}_i$ of $\Omega_i$ and a function $\psi: \mathbb{R}^{d-1} \rightarrow \mathbb{R}$ which satisfy:
\begin{enumerate}[topsep=5pt]
	\setlength\itemsep{0.1mm}
\item $\tilde{\Omega}_i = \{ (\boldsymbol{x},y) \in \mathbb{R}^d \, | \, y > \psi(\boldsymbol{x}) \}$,
\item $\psi$ is Lipschitz continuous, meaning
$| \psi(\boldsymbol{x}) - \psi(\boldsymbol{x}') | \leq L \| \boldsymbol{x}-\boldsymbol{x}' \|_2$ for all $\boldsymbol{x},\boldsymbol{x}' \in \mathbb{R}^{d-1}$ and a constant $L>0$.
\end{enumerate}
An open, non-empty and bounded set $\Omega\subset \mathbb{R}^d$ has Lipschitz boundary\index{Lipschitz boundary}, if there exists $\varepsilon >0, n \in \mathbb{N}_{\geq 1}, M>0$ and open sets $U_1,U_2,... \subset \mathbb{R}^d$ such that:
\begin{enumerate}[label=(\roman*), topsep=5pt]
	\setlength\itemsep{0.1mm}
\item For any $\boldsymbol{x} \in \partial \Omega$, there exists an index $i$ such that $B_{\varepsilon}(\boldsymbol{x}) \subset U_i$, \label{A.2.2 (i)}
\item $U_{i_1} \cap ... \cap U_{i_{n+1}} = \emptyset$ for any distinct indices $\{ i_1, ..., i_{n+1} \}$,
\item For each index $i$, there exists a special Lipschitz domain $\Omega_i \subset \mathbb{R}^d$ with $U_i \cap \Omega = U_i \cap \Omega_i$ and Lipschitz constant $L \leq M$. \label{A.2.2 (iii)}
\end{enumerate}
\end{defi}

As described on page 84 of \cite{AF}, it is known that the Lipschitz boundary condition is more restrictive than the interior cone condition. Since I was unable to find a proof of this in the literature, I have decided to provide a proof myself.

\begin{lemma}
If $\Omega\subset \mathbb{R}^d$ has Lipschitz boundary, it satisfies the interior cone
condition.
\end{lemma}

\begin{proof}
First, let $\boldsymbol{x}_0 \in \Omega$ with $\text{dist}(\boldsymbol{x}_0, \partial \Omega) < \varepsilon /2$ and assume without restriction $\boldsymbol{x}_0 = \boldsymbol{0}$. According to condition \ref{A.2.2 (i)}, there exists an index $i$ with $\boldsymbol{0} \in U_i \cap \Omega$. Based on condition \ref{A.2.2 (iii)}, there is a Lipschitz continuous function $\psi$ with Lipschitz constant $L \leq M$ such that, up to a rotation, 
$$\boldsymbol{0} \in \tilde{\Omega}_i = \{ (\boldsymbol{x},y) \in \mathbb{R}^d  \, | \, \psi(\boldsymbol{x}) < y \}.$$
Consider the cone $\tilde{C}(\boldsymbol{0}) = \{ (\boldsymbol{x},y) \in \mathbb{R}^d  \, | \, \| \boldsymbol{x} \|_2 \leq M^{-1}y \}$:
\begin{figure}[H]
	\centering
	\includegraphics[width=0.6\textwidth]{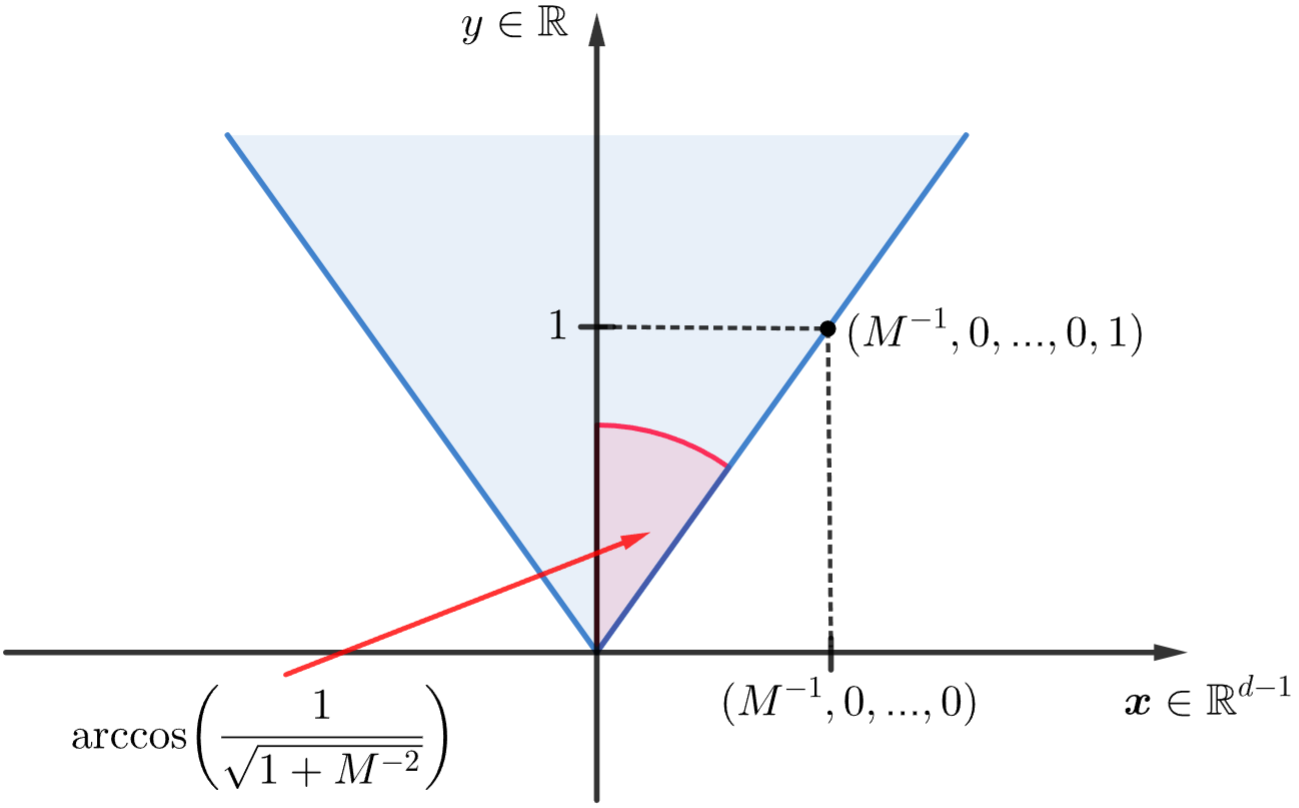}
\end{figure}
Due to the Lipschitz continuity of $\psi$, we have
$$ M \| \boldsymbol{x} \|_2 \leq y 
\Longrightarrow  \psi (\boldsymbol{x}) 
\leq \psi (\boldsymbol{0}) + L \| \boldsymbol{x}-\boldsymbol{0} \|_2
<  L \| \boldsymbol{x} \|_2 \leq M \| \boldsymbol{x} \|_2 \leq y
\text{ for all } \boldsymbol{x} \in \mathbb{R}^{d-1}.$$
%-------------------------------------------------------------------------------------------
Therefore,
$$\tilde{C}(\boldsymbol{0}) \subseteq \{ (\boldsymbol{x},y) \in \mathbb{R}^d  \, | \, \psi(\boldsymbol{x}) < y \} = \tilde{\Omega}_i.$$
By restricting the height of the cone to $\varepsilon/2$ and setting $\boldsymbol{\xi}(\boldsymbol{x})= (0,...,0,1)^T $, we express
\begin{align*}
\tilde{C}(\boldsymbol{0}) &= \tilde{C} \left(\boldsymbol{x}, \boldsymbol{\xi} (\boldsymbol{x}), \arccos \left( \dfrac{1}{\sqrt{1+M^{-2}}} \right), \dfrac{\varepsilon}{2} \right) \\
&= \Big\{ \lambda \boldsymbol{y} \ \big| \  \boldsymbol{y} \in \mathbb{R}^d, \| \boldsymbol{y} \|_2 = 1, \boldsymbol{y}^T \boldsymbol{\xi}(\boldsymbol{x}) \geq \dfrac{1}{\sqrt{1+M^{-2}}}, \lambda \in [0,\varepsilon/2] \Big\}.
\end{align*}
Thus, we have
$$\tilde{C}(\boldsymbol{0}) \subseteq \overline{B_{\varepsilon /2}(\boldsymbol{0})} \cap \tilde{\Omega}_i.$$
By suitably rotating this cone to $C(\boldsymbol{0})$, we obtain
$$C(\boldsymbol{0}) \subseteq  \overline{B_{\varepsilon /2}(\boldsymbol{0})} \cap \Omega_i \subseteq U_i \cap \Omega_i = U_i \cap \Omega \subseteq \Omega,$$
ensuring that this cone is contained in $\Omega$. \\
Now, let $\boldsymbol{x}_0 \in \Omega$ with $\text{dist}(\boldsymbol{x}_0, \partial \Omega) \geq \varepsilon /2$. Since $C(\boldsymbol{0})$ has a height of $\varepsilon /2$, translating $C(\boldsymbol{0})$ from its vertex at $\boldsymbol{0}$ to $\boldsymbol{x}_0$ and shrinking the height to $\varepsilon /4$, ensures that it is also contained within $\Omega$.
\end{proof}

To highlight that a set with a Lipschitz boundary satisfies the interior cone condition with a specific angle and radius, we introduce a new term.

\begin{defi}
A domain $\Omega \subset \mathbb{R}^d$ is called an $\mathcal{L}(R, \delta)$-domain\index{\textit{$\mathcal{L}(R, \delta)$,}} if it has a Lipschitz boundary and satisfies the interior cone condition with radius $R>0$ and angle $\delta \in (0,\pi/2)$.
\end{defi}

%-------------------------------------------------------------------------------------------
% Dimension of polynomial space
%-------------------------------------------------------------------------------------------

\subsection{Dimension of the polynomial space} \label{Chap. 8.3}
Let $q \in \mathbb{N}_0$. The number of solutions to
\begin{equation}
\alpha_1 + \alpha_2 + ... + \alpha_d = q  \label{binomial_0}
\end{equation}
for $\alpha_1, ..., \alpha_d \in \mathbb{N}_0$ is
\begin{equation}
\binom{q+d-1}{d-1} = \dfrac{(q+d-1)!}{q!(d-1)!}. \label{binomial}
\end{equation}
To verify this consider $d+q-1$ boxes
\begin{figure}[H]
	\centering
	\begin{ytableau}
       ~   &    &   & \none[\dots]  &	&	& \\
	\end{ytableau}
\end{figure}\noindent
and notice that the problem can be seen as using $q-1$ gray boxes as walls and the number of empty boxes between two walls represent the corresponding integer. As example for $d=3$ and $q=3$ is
\newcommand{\labelrow}[2]{\makebox[0pt][l]{\hspace{#1}#2}}
\begin{figure}[H]
	\hspace{1.5cm}
	\begin{ytableau}
       {}   &  *(gray)  & {}  &	*(gray) & {} \labelrow{1em}{$ = 1+1+1,$}
    \end{ytableau}
    \hspace{3cm}
    \begin{ytableau}
       *(gray)   &  {}  & {}  &	*(gray) & {} \labelrow{1em}{$ = 0+2+1,$}
    \end{ytableau}
\end{figure}
\vspace{-0.5cm}
\begin{figure}[H]
	\hspace{1.5cm}
	\begin{ytableau}
       {}   &  {}  & {}  &	*(gray) & *(gray) \labelrow{1em}{$ = 3+0+0,$}
    \end{ytableau}
    \hspace{3cm}
    \begin{ytableau}
       {}   &  {}  & *(gray)  &	*(gray) & {} \labelrow{1em}{$ = 2+0+1.$}
    \end{ytableau}
\end{figure}
%-------------------------------------------------------------------------------------------
\noindent
Therefore, the number of solutions are the possibilities to draw $q-1$ gray boxes out of the initial $d+q-1$ white boxes, which is the binomial coefficient in (\ref{binomial}). Switching $"="$ in formula (\ref{binomial_0}) to $"\leq"$, we get
\begin{equation}
\sum_{\ell = 0}^q \binom{\ell+d-1}{d-1} = \binom{q+d}{d} \label{eq. A.3.3}
\end{equation}
solutions, where the equality can be proven by induction over $q$. This means the dimension of the polynomial space $\mathcal{P}_q(\Omega)$ for $\Omega \subseteq \mathbb{R}^d$ is given by equation (\ref{eq. A.3.3}).

%-------------------------------------------------------------------------------------------
% Bibliography
%-------------------------------------------------------------------------------------------

% Creating a bibliography on the same page
\newpage
\fancyhead[RO, RE]{}
\fancyhead[LE, LO]{}

%-------------------------------------------------------------------------------------------
% Index
%-------------------------------------------------------------------------------------------
\printindex

\end{document}